\documentclass[12pt]{article}
\usepackage{amsmath,amsfonts,amssymb,mathrsfs,amsthm,bm,bbm,mathtools,graphicx,color,tcolorbox,algorithm,algpseudocode}
\usepackage{caption,fancyhdr,titlesec,indentfirst,booktabs,verbatim,hyperref,cases,titletoc,multirow,wasysym,empheq,setspace,natbib,chngcntr,subcaption,caption,overpic,enumitem}

\newcommand{\KwStep}[1]{\State \textbf{#1}}
\usepackage[font=small,labelfont=bf]{caption}

\newcommand{\rom}[1]{\uppercase\expandafter{\romannumeral #1\relax}}
\usepackage[page,header]{appendix}
\usepackage[english]{babel}
\usepackage[utf8]{inputenc}

\newtheorem{theorem}[]{Theorem}
\newtheorem{lemma}[]{Lemma}
\newtheorem{corollary}[]{Corollary}
\newtheorem{proposition}[]{Proposition}

\newtheorem{assumption}[]{Assumption}

\newcommand{\EE}{\mathbb{E}}
\newcommand{\PP}{\mathbb{P}}
\newcommand{\RR}{\mathbb{R}}

\newcommand{\cC}{\mathcal{C}}

\newcommand{\cH}{\mathcal{H}}
\newcommand{\cI}{\mathcal{I}}

\newcommand{\cP}{\mathcal{P}}

\newcommand{\cX}{\mathcal{X}}
\newcommand{\cY}{\mathcal{Y}}

\newcommand{\sC}{\mathscr{C}}
\newcommand{\sE}{\mathscr{E}}
\newcommand{\sG}{\mathscr{G}}
\newcommand{\sN}{\mathscr{N}}

\newcommand{\Var}{\mathrm{Var}}
\newcommand{\Cov}{\mathrm{Cov}}
\newcommand{\op}{\mathrm{op}}
\newcommand{\Tr}{\mathrm{Tr}}
\newcommand{\wt}{\tilde}
\newcommand{\wh}{\hat}
\newcommand{\KL}{\mathrm{KL}}
\newcommand{\TV}{\mathrm{TV}}
\providecommand{\Description}[1]{\caption*{\textit{Alt text:} #1}}

\begin{document}

\title{\Large{\textbf{Prediction-Powered Conditional Inference}}}
\author{
\bigskip
{\sc Yang Sui, Jin Zhou, Hua Zhou, and Xiaowu Dai} \\
{\it {\normalsize University of California, Los Angeles}}
}
\date{}
\maketitle

\begin{footnotetext}[1]
{\textit{Address for correspondence:} Xiaowu Dai, Department of Statistics and Data Science, University of California, Los Angeles, 28125 Math Sciences Bldg \#951554, CA 90095, USA. E-mail: daix@ucla.edu. }
\end{footnotetext}

\begin{abstract}
\noindent
We study prediction-powered conditional inference in the setting where labeled data are scarce, unlabeled covariates are abundant, and a black-box machine-learning predictor is available. The goal is to perform statistical inference on conditional functionals evaluated at a fixed target point, such as conditional means, without imposing a parametric model for the conditional relationship.
Our approach combines localization with prediction-based variance reduction. First, we introduce an RKHS localization method that learns a data-adaptive weight from covariates and reformulates the target conditional moment at the target point as a weighted unconditional moment.
Second, we incorporate machine-learning predictions through a correction-based decomposition of this localized moment, yielding a prediction-powered estimator and confidence interval that reduce variance when the predictor is informative while preserving validity regardless of predictor accuracy.
We establish nonasymptotic error bounds and, in the abundant-unlabeled regime, minimax-optimal convergence rates for the resulting estimator, prove pointwise asymptotic normality with consistent variance estimation, and provide an explicit variance decomposition that characterizes how machine-learning predictions and unlabeled covariates improve statistical efficiency. Numerical experiments on simulated and real datasets demonstrate valid conditional coverage and substantially sharper confidence intervals than alternative methods.

\end{abstract}
\bigskip


\newpage
\baselineskip=22pt

%%%%%%%%%%%%%%%%%%%%%%%%%%%%%%%%%%%%%%%%%%%%%%%%%%%%%%%%%%%%%%%%%%%%%%%%%%%%%%%%%%%%%%%%%%%%%%%%%

\section{Introduction}
\noindent
In many modern scientific and engineering applications, gold-standard labeled data are costly to obtain and therefore limited in number, whereas unlabeled covariates can often be collected at scale. At the same time, black-box machine-learning (ML) models can generate large numbers of inexpensive but imperfect predictions from covariates alone. This combination of scarce labels, abundant covariates, and readily available predictions arises in diverse domains, including genomics, medical imaging, and proteomics \citep{esteva2017dermatologist, jumper2021highly, vaishnav2022evolution}.
Motivated by this information structure, we study statistical inference for conditional functionals evaluated at a fixed covariate value. Let $\{(X_i,Y_i)\}_{i=1}^n$ be labeled data drawn i.i.d.\ from the joint distribution $\rho$ of $(X,Y)$ on $\cX\times\cY$, and let $\{\wt X_u\}_{u=1}^N$ be unlabeled covariates drawn i.i.d.\ from the marginal distribution $\rho_X$ on $\cX\subset\RR^d$. In addition, suppose a black-box ML predictor $f:\cX\to\cY$ is available, producing predictions $f(x)$ for any $x\in\cX$.

For a fixed covariate $x_0\in\cX$, our goal is to use the three information sources $\{(X_i,Y_i)\}_{i=1}^n\cup\{\wt X_u\}_{u=1}^N\cup\{f\}$ to conduct valid inference for a conditional functional $\theta_0(x_0)$.
To illustrate, consider conditional mean inference, where the target is $\theta_0(x_0):=\EE[Y|X=x_0]$.
A convenient characterization of $\theta_0(\cdot)$ is through a conditional moment restriction. Let $\Theta\subset\RR$ be a parameter space and let $\ell(Y;\theta)$ be an estimating function. For the conditional mean, one may take $\ell(Y;\theta)=Y-\theta$, so that $\EE[\ell(Y;\theta_0(x_0))| X=x_0]=0$. More generally, with an appropriate choice of $\ell$, this formulation covers a broad class of conditional functionals.
Unlike inference for global parameters defined through unconditional moments such as $\EE_\rho[\ell(Y;\theta_0)]=0$, conditional inference provides uncertainty quantification tailored to the fixed target point $x_0$. Conditional functionals are often the scientifically relevant objects in modern data-driven decision systems \citep{chakraborti2025personalized}. For example, in clinical risk assessment, a model that is well calibrated at the population level may still exhibit highly variable uncertainty across patients~\citep{begoli2019need}. Similarly, in economic and demographic analyses such as census income studies \citep{angelopoulos2023prediction}, global summaries can obscure substantial variation across subpopulations. Conditional inference at specific covariate values provides uncertainty quantification aligned with this heterogeneity, rather than a single population-level summary.

\begin{figure}[t]
\begin{center}
        \begin{overpic}[width=1\linewidth,height=0.37\textheight]{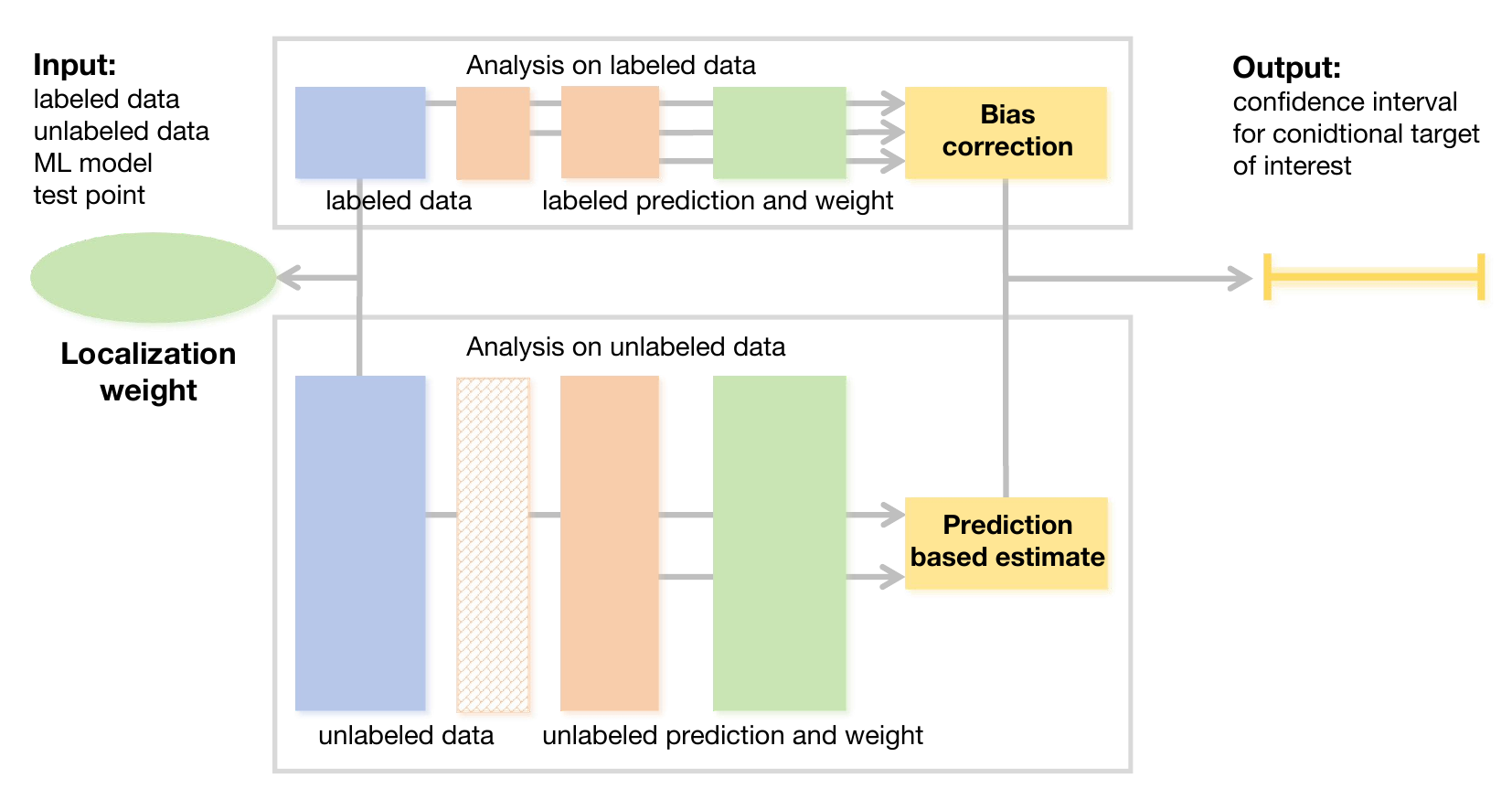}
        \put(10.4,37.6){\footnotesize$f(\cdot)$}
        \put(10.1,35.8){\footnotesize${x_0}$}
        \put(7,31){\footnotesize$w_{x_0}(\cdot)$}
        \put(22.8,39.4){\footnotesize$X$}
        \put(31.8,39.4){\footnotesize$Y$}
        \put(37.8,39.4){\footnotesize$f(X)$}
        \put(47.8,39.4){\footnotesize$w_{x_0}(X)$}

        \put(22.8,15){\footnotesize$\wt X$}
        \put(31.8,15){\footnotesize$\wt Y$}
        \put(37.8,15){\footnotesize$f(\wt X)$}
        \put(47.8,15){\footnotesize$w_{x_0}(\wt X)$}

        \put(90,37.5){\footnotesize$\theta_0(x_0)$}
        \put(88.5,29){\footnotesize$\cC(x_0)$}

    \end{overpic}
    \caption{Protocol for prediction-powered conditional inference. The procedure takes labeled data, unlabeled covariates, an ML predictor, and a target point $x_0$ as inputs. A localization step uses the covariate distribution to learn weights that capture the local structure around $x_0$. The upper block estimates a bias correction from labeled data, while the lower block computes a plug-in term using predictions from the unlabeled data. These components are combined to produce a valid confidence interval $\mathcal{C}(x_0)$ for the conditional target $\theta_0(x_0)$.}
    \label{fig:PPCI_diagram}
    \Description{Flow diagram showing a predictor and target point producing localization weights; labeled residual correction and unlabeled prediction terms are combined into a point estimate and confidence interval.}
\end{center}
\end{figure}
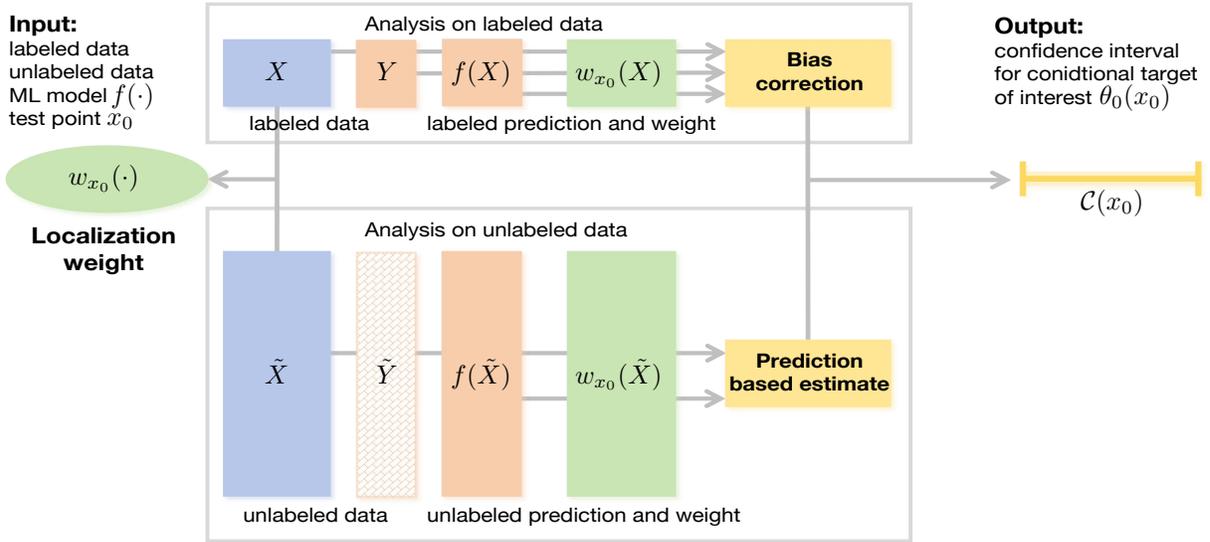

In this paper, we develop a prediction-powered conditional inference (PPCI) framework that combines nonparametric localization with ML predictions. The approach converts the conditional moment defining the target at $x_0$ into a localized unconditional moment using an RKHS localization weight learned from the covariate distribution. This enables pointwise inference without specifying a global model for the conditional functional and leverages abundant unlabeled covariates $\{\wt X_u\}_{u=1}^N$ to estimate the local structure around $x_0$.
Localization reduces the effective sample size, so naive localized estimators are typically variance-dominated. To address this, we incorporate ML predictions through a correction-based decomposition of the localized moment. The bias correction term, estimated from labeled data, depends on prediction residuals, while a plug-in term depending only on predictions is estimated from the large unlabeled sample. This decomposition reduces variance and yields a prediction-powered localized estimating equation whose solution $\wh\theta(x_0)$ estimates $\theta_0(x_0)$. The PPCI procedure is illustrated in Figure~\ref{fig:PPCI_diagram}.

We establish a nonasymptotic error bound for the PPCI estimator $\wh\theta(x_0)$, treating it as the root of a localized estimating equation. The error decomposes into three components: a prediction-powered stochastic term governed by the labeled sample size $n$, a weight-learning error determined by the unlabeled covariate sample size $N$ in the cross-fitted construction, with an $n+N$ analogue for the non-split variant, and an approximation bias controlled by the regularization parameter $\lambda$. In the abundant-unlabeled regime, the resulting convergence rate matches the minimax rate for pointwise estimation.
{We further prove pointwise asymptotic normality of the PPCI estimator $\wh\theta(x_0)$. Writing $J_\lambda(x_0)$ for the localized Jacobian, its asymptotic variance is $J_\lambda(x_0)^{-2}$ times
\begin{equation*}
\frac{1}{n}\Var\left(w_{x_0,\lambda}(X)\{\ell(Y;\theta)-\ell(f(X);\theta)\}\right)
+
\frac{1}{N}\Var\left(w_{x_0,\lambda}(X)\ell(f(X);\theta)\right),
\end{equation*}
which separates the contributions of labeled and unlabeled samples. In contrast, the corresponding labeled-only estimator has asymptotic variance $J_\lambda(x_0)^{-2}(1/n)\Var(w_{x_0,\lambda}(X)\ell(Y;\theta))$. When $N \gg n$, the second term is negligible and the variance is driven by prediction residuals, yielding efficiency gains when the predictor $f$ is informative.}
We use these results to construct confidence intervals for $\theta_0(x_0)$ with asymptotically correct coverage. Under asymmetric data acquisition costs, we also derive a budget-optimal sampling strategy that minimizes interval width.

Our approach is closely related to prediction-powered inference (PPI), which provides general mechanisms for incorporating black-box ML predictions into statistical inference~\citep{angelopoulos2023prediction,angelopoulos2023ppi++,zrnic2024cross,gu2024local}. Existing PPI methods primarily focus on global, population-level parameters and rely on full-sample averaging to improve statistical efficiency. Our work complements this line of research by extending the prediction-powered paradigm to pointwise conditional inference, where estimation is governed by the much smaller local effective sample size near $x_0$. To accommodate this local structure, we introduce a localization scheme that converts the conditional target into a weighted unconditional moment suitable for prediction-powered estimation. Unlike PPI++ and cross-PPI, PPCI combines prediction-powered correction with RKHS localization for a pointwise target. It differs from KRR: the RKHS is used to construct localization weights rather than to fit the conditional target directly.

We validate the proposed PPCI through extensive simulations and two real-world applications: Census income estimation and popularity prediction with the BlogFeedback dataset. Because standard PPI targets population-level parameters rather than pointwise conditional targets, we include it as an out-of-scope global reference. Classical localized estimators based only on labeled data achieve nominal coverage but produce wide intervals due to limited effective local sample sizes. Our PPCI method improves this trade-off, producing substantially sharper intervals while preserving coverage across a range of conditional targets.

The rest of the paper is organized as follows. Section~\ref{sec:bg} introduces the background, including RKHS, conditional targets, and the prediction-powered framework.
Section~\ref{sec:PPCI} presents the proposed PPCI method. Section~\ref{sec:theory} establishes its estimation and inference results. Section~\ref{sec:practical} discusses practical considerations.
Section~\ref{sec:experiments} reports numerical experiments.
Section~\ref{sec:conclusion} concludes the paper with a discussion on future directions.

\section{Background}\label{sec:bg}
\noindent
In this section, we introduce the basic notation and RKHS background, define the conditional inference target of interest, and present the prediction-powered framework.
\subsection{Notations}
\noindent
We use the following notations throughout the paper.
Let the input domain $\cX \subset \RR^d$ be convex, compact, and with nonempty interior. Let $\rho_X$ be a probability measure on $\cX$ with density bounded above and below:
$0 < \rho_0 \le \frac{d\rho_X}{dx}(x) \le \rho_1 < \infty$.
Denote by $L^2(\rho_X)$ the Hilbert space of square-integrable functions on $\cX$ with inner product
$\langle h,g\rangle_{L^2(\rho_X)}=\int_\cX h(x)g(x)\,d\rho_X(x)$, and let $\|h\|_\infty=\sup_{x\in\cX}|h(x)|$ denote the supremum norm.
We write $\EE_\rho$ for expectation under the joint law of $(X,Y)$ and $\EE_{\rho_X}$ for expectation under the covariate marginal. An unsubscripted $\EE$ uses the distribution implied by its argument and the sampling statement.

For vectors $v\in\RR^d$, $\|v\|_2$ denotes the Euclidean norm. For matrices $A\in\RR^{d\times d}$, $\Tr(A)$ denotes the matrix trace. We write $a\wedge b=\min\{a,b\}$ and $a\vee b=\max\{a,b\}$.
For nonnegative sequences $\{a_k\}$ and $\{b_k\}$, we write $a_k\lesssim b_k$ or $a_k=O(b_k)$ if $a_k\le c\,b_k$ for some universal constant $c>0$, and $a_k\gtrsim b_k$ if $a_k\ge c\,b_k$. We write $a_k\asymp b_k$ if both hold, $a_k=O(1)$ if $a_k$ is bounded, and $a_k=o(b_k)$ if $a_k/b_k\to0$.
For random sequences $\{A_k\}$ and $\{B_k\}$ with $B_k>0$ almost surely, we write $A_k=O_p(B_k)$ if, for every $\varepsilon>0$, there exists a finite $M_\varepsilon$ such that $\sup_k\PP(|A_k/B_k|>M_\varepsilon)\le\varepsilon$, and $A_k=o_p(B_k)$ if $A_k/B_k\to0$ in probability.

\subsection{Reproducing Kernel Hilbert Space (RKHS)}
\noindent
Let $K:\cX\times\cX\to\RR$ be a bounded, measurable, positive definite kernel, which satisfies
$\sup_{x\in\cX} K(x,x) = \kappa^2 < \infty$.
By the Moore--Aronszajn theorem, $K$ determines a unique RKHS $\cH$ obtained as the completion of finite linear combinations of kernel sections, with inner product satisfying $\langle K_x,K_{x'}\rangle_{\cH}=K(x,x')$. We assume this RKHS is separable~\citep{wahba1990spline}.
For each $x\in\cX$, define the feature map $K_x := K(x,\cdot)\in\cH$. The reproducing property holds:
$h(x) = \langle h, K_x\rangle_{\cH}, \forall\, h\in\cH$.
Define the integral operator $T_K:\cH\to\cH$ by
$(T_K h)(x) = \int_{\cX} K(x,t)\,h(t)\,d\rho_X(t), h\in\cH$.
For an integer $m>d/2$, the Sobolev space
$H^m(\cX)=\{h\in L^2(\cX):D^\alpha h\in L^2(\cX)\text{ for every weak derivative }D^\alpha\text{ with }|\alpha|\le m\}$
embeds continuously into continuous functions. Mat\'ern kernels with suitable smoothness generate RKHSs norm-equivalent to such Sobolev spaces on regular bounded domains.

\subsection{Conditional Inference Target}\label{sec:target}
\noindent
Let $x_0$ be an interior point of $\cX$, and let $\Theta\subset\RR$ be a parameter space. Our goal is to conduct inference for a conditional target $\theta_0(x_0)\in\Theta$, defined as the unique solution to a conditional moment restriction at $x_0$. We focus on scalar targets. Extending the theory to vector-valued targets requires a matrix-valued Jacobian and joint covariance analysis and is left for future work. Let $\ell:\cY\times\Theta\to\RR$ be an estimating function and define the conditional moment by
\begin{equation}\label{eq:cond-risk}
\eta(x_0;\theta) := \EE[\ell(Y;\theta)| X=x_0].
\end{equation}
The target $\theta_0(x_0)$ is characterized as a root of $\eta(x_0;\theta_0(x_0))=0$,
and our objective is to construct valid inference for $\theta_0(x_0)$.
Many conditional functionals admit this representation. We consider the following examples:
\begin{itemize}
\item \textbf{Conditional mean:}
Choosing $\ell(Y;\theta)=Y-\theta$ yields $\theta_0(x_0)=\EE[Y| X=x_0]$.
\item \textbf{Conditional log-odds in binary classification:}
For $Y\in\{0,1\}$, taking $\ell(Y;\theta)=Y-(1+e^{-\theta})^{-1}$ identifies $\theta_0(x_0) = \log(\PP(Y=1| X=x_0)/\PP(Y=0| X=x_0))$.

\item \textbf{Conditional expected shortfall:}
Let $\mathrm{Quant}_\tau(X)$ denote the conditional $\tau$-quantile for $\tau\in(0,1)$ and define
$Y^\ast = \mathrm{Quant}_\tau(X) + \tau^{-1}(Y-\mathrm{Quant}_\tau(X))\mathbbm{1}\{Y\le \mathrm{Quant}_\tau(X)\}$.
Choosing $\ell(Y^\ast;\theta)=Y^\ast-\theta$ yields
$\theta_0(x_0)=\EE[Y^\ast| X=x_0]$ \citep{he2023robust,yu2025estimation}.
\end{itemize}
Throughout, we assume that for each $\theta\in\Theta_0$, the function $x\mapsto\eta(x;\theta)$ belongs to $\cH$; Appendix~C.3 relaxes this condition to a spectral source condition.
For example, this assumption requires
$x\mapsto\EE[Y| X=x]\in\cH$ in the conditional mean setting,
$x\mapsto\PP(Y=1| X=x)\in\cH$ in binary classification,
and $x\mapsto\EE[Y^\ast| X=x]\in\cH$ for conditional expected shortfall.

\subsection{Prediction-Powered Framework}
\label{subsec:ppi_framework}
\noindent
We consider a setting with three sources: labeled data
$\{(X_i,Y_i)\}_{i=1}^n$, unlabeled covariates
$\{\wt X_u\}_{u=1}^N$, and a black-box ML predictor
$f:\cX\to\cY$ under the following sampling scheme.
\begin{assumption}\label{ass:sampling-scheme}
The labeled observations $\{(X_i,Y_i)\}_{i=1}^n$ are
i.i.d.\ from $\rho$, the unlabeled covariates
$\{\wt X_u\}_{u=1}^N$ are i.i.d.\ from $\rho_X$, and the two samples
are independent. The predictor used for inference is either fixed
relative to these samples, as when it is trained on an independent
sample, or cross-fitted so that each evaluated prediction is obtained
without using the corresponding outcome~\citep{zrnic2024cross}.
\end{assumption}

The prediction-powered framework leverages $f$ to improve statistical efficiency. For a global parameter $\theta_0$ defined by the moment condition $\EE_\rho[\ell(Y;\theta_0)]=0$, prediction-powered methods decompose the target moment as ~\citep{angelopoulos2023prediction}
\begin{equation*}
\EE_\rho[\ell(Y;\theta)] = \EE_\rho[\ell(Y;\theta)-\ell(f(X);\theta)] + \EE_{\rho_X}[\ell(f(X);\theta)].
\end{equation*}
The first term is estimated using labeled data, typically with reduced variance when $f$ is informative. The second term is estimated using the large unlabeled sample, with variance diminishing as the unlabeled sample size grows.

In this paper, we extend the prediction-powered framework from global parameters to pointwise conditional inference for $\theta_0(x_0)$. This extension raises two challenges. First, conditioning on $\{X=x_0\}$ prevents expressing the conditional moment $\eta(x_0;\theta)$ in~\eqref{eq:cond-risk} as an unconditional average over labeled data $\{(X_i,Y_i)\}_{i=1}^n$.
Second, localization near $x_0$ yields a small effective sample size and therefore high variance. To address these challenges, we introduce in the next section a localization scheme that converts~\eqref{eq:cond-risk} into an unconditional localized moment and combine it with prediction-powered estimation that leverages the predictor $f$ and abundant unlabeled covariates to improve efficiency.

\section{Prediction-Powered Conditional Inference}
\label{sec:PPCI}
\noindent
In this section, we detail our PPCI procedure for constructing valid confidence intervals for the conditional target $\theta_0(x_0)$.

\subsection{Algorithm}
\noindent
Our prediction-powered conditional inference (PPCI) procedure has three steps.
Step~1 selects tuning parameters within each unlabeled training fold and constructs an RKHS localization weight capturing the covariate structure around the target point $x_0$.
Step~2 converts the defining conditional moment into an unconditional localized moment and uses a prediction-powered decomposition with the predictor $f$ and unlabeled covariates $\{\wt X_u\}_{u=1}^N$ for variance reduction.
Step~3 constructs a confidence interval for the conditional estimate.
Algorithm~\ref{alg:ppci_twofold} summarizes the procedure, followed by details of each step. For clarity, the following subsections present the construction for a generic fixed kernel $K$ and regularization level $\lambda$. In the numerical implementation, the same formulas are applied foldwise using the tuning values selected in Step~1 of Algorithm~\ref{alg:ppci_twofold}.

{
\begin{algorithm}[t]
\caption{Prediction-Powered Conditional Inference}
\label{alg:ppci_twofold}
\begin{algorithmic}[1]
\Require
    \Statex Labeled data $\{(X_i,Y_i)\}_{i=1}^n$, unlabeled covariates $\{\wt X_u\}_{u=1}^{N}$,
    predictor $f$, target point $x_0$, confidence level $(1-\alpha)$,
    kernel family $\{K_h:h>0\}$, candidate grids $\mathcal G_h$ and $\mathcal G_\lambda$,
    and tuning constants $c_{\rm op}$, $c_{\rm loc}$, and $c_{\rm bias}$.

\KwStep{Step 1: Tuning and RKHS localization weights}
\State Randomly partition $\{1,\dots,N\}$ into two equal-sized disjoint folds $\cI_1$ and $\cI_2$.
\For{$m\in\{1,2\}$}
    \State Apply the tuning procedure in Section~\ref{sec:tuning_twofold} using $\{\tilde{X}_u\}_{u\in\mathcal I_m}$ to obtain $(\wh h_m,\wh\lambda_m)$.
    \State Compute $\Sigma_t^{(m)}\in\RR^{|\cI_m|\times|\cI_m|}$ with $(\Sigma_t^{(m)})_{uu'}=K_{\wh h_m}(\wt X_u,\wt X_{u'})$, and $k_{0,t}^{(m)}\in\RR^{|\cI_m|}$ with $(k_{0,t}^{(m)})_u=K_{\wh h_m}(\wt X_u,x_0)$, for all $u,u'\in\cI_m$.
    \State Solve $\xi^{(m)}=(\Sigma_t^{(m)}+|\cI_m|\wh\lambda_m I_{|\cI_m|})^{-1}k_{0,t}^{(m)}$.
    \State Construct the fold-specific empirical RKHS localization weight $\displaystyle \wh w^{(m)}_{x_0}(x)=\frac{K_{\wh h_m}(x,x_0)-K_{\wh h_m}(x,\wt X_{\cI_m})^\top\xi^{(m)}}{\wh\lambda_m}$.
\EndFor
\State Define the averaged weight for labeled data: $\overline w_{x_0}(\cdot)=\frac{1}{2}\{\wh w^{(1)}_{x_0}(\cdot)+\wh w^{(2)}_{x_0}(\cdot)\}$.

\KwStep{Step 2: Prediction-powered conditional estimation}
\State Construct the cross-fitted estimating equation using the fold-specific weights as in \eqref{eq:eta-cf_twofold}.
\State Compute the final estimator $\wh\theta(x_0)$ as in \eqref{eq:thetahat-x0_twofold}.

\KwStep{Step 3: Confidence interval}
\State Compute the Jacobian and variance using the fold-specific weights as in \eqref{eq:H-hat-new_twofold} and \eqref{eq:V-hat-new_twofold}.
\State Construct the confidence interval $\cC(x_0)$ as in \eqref{eq:CI-new_twofold}.

\Ensure
    \Statex Confidence interval $\cC(x_0)$ for $\theta_0(x_0)$.
\end{algorithmic}
\end{algorithm}
}

\subsection{RKHS Localization}
\label{sec:localization_twofold}
\noindent
We introduce an RKHS localization scheme that converts the conditional moment \eqref{eq:cond-risk} at $x_0$ into an unconditional weighted moment. Fix $\lambda>0$ and define the RKHS localization weight
\begin{equation}
\label{eq:rkhs-localization-pop-x0_twofold}
w_{x_0,\lambda}
:=
(T_K+\lambda I)^{-1}K(x_0,\cdot)
\in\cH,
\end{equation}
where $I$ is the identity on $\cH$ and the inverse is defined via functional calculus~\citep{smale2007learning}. Using $w_{x_0,\lambda}$, we define the localized moment
\begin{equation}
\label{eq:M-lambda-def-x0_twofold}
\eta_\lambda(x_0;\theta)
:=
\EE\big[w_{x_0,\lambda}(X)\ell(Y;\theta)\big],
\end{equation}
which replaces conditioning on $\{X=x_0\}$ with a tractable weighted expectation. This reformulation has a natural interpretation: $\eta_\lambda(\cdot;\theta)$ exactly corresponds to the population Tikhonov-regularized approximation of the true conditional moment $\eta(\cdot;\theta)$ in $\cH$, given by
\begin{equation}
\label{eqn:Tikhonov-regu_twofold}
\eta_\lambda(\cdot;\theta)
=
\underset{g\in\cH}{\arg\min}
\left\{
\EE\big[(g(X)-\eta(X;\theta))^2\big]
+
\lambda\|g\|_\cH^2
\right\}.
\end{equation}
We use the term RKHS localization weight for \eqref{eq:rkhs-localization-pop-x0_twofold} and its empirical analogues, abbreviated to localization weight when no alternative weighting scheme is under discussion.
The regularization bias is therefore $\eta_\lambda(x_0;\theta)-\eta(x_0;\theta)$. Hence, statistical inference for $\eta(x_0;\theta)$ can be based on $\eta_\lambda(x_0;\theta)$, provided $\lambda$ is chosen so that this bias is negligible.

In summary, the localization weight \eqref{eq:rkhs-localization-pop-x0_twofold} converts the conditional moment \eqref{eq:cond-risk} into the unconditional localized moment \eqref{eq:M-lambda-def-x0_twofold}, which is the foundation for the prediction-powered conditional estimation developed next.

\subsection{Prediction-Powered Conditional Estimation}
\label{sec:ppi-framework_twofold}
\noindent
We incorporate prediction-powered ideas into the localized moment \eqref{eq:M-lambda-def-x0_twofold} to improve efficiency in conditional estimation.
Using the predictor $f$, we decompose
\begin{equation}
\label{eqn:Mc_twofold}
\eta_\lambda(x_0;\theta)
=
\EE\big[w_{x_0,\lambda}(X)\{\ell(Y;\theta)-\ell(f(X);\theta)\}\big]
+
\EE\big[w_{x_0,\lambda}(X)\ell(f(X);\theta)\big].
\end{equation}
To estimate the localization weights $w_{x_0,\lambda}(X)$, we employ two-fold cross-fitting on the unlabeled data; the extension to general multi-fold cross-fitting is straightforward.
We randomly split $\{\wt X_u\}_{u=1}^N$ into two disjoint folds $\cI_1$ and $\cI_2$ of equal size.
For each fold $m\in\{1,2\}$, we construct the empirical weight $\wh w^{(m)}_{x_0,\lambda}$ using only $\{\wt X_u\}_{u\in\cI_m}$.
For unlabeled observations in fold $\cI_m$, we evaluate them using the out-of-fold weight $\wh w^{(3-m)}_{x_0,\lambda}$ trained on the other fold.
For labeled data $\{(X_i,Y_i)\}_{i=1}^n$, the weights are constructed independently of the responses $Y_i$, and we evaluate labeled observations using the averaged weight $\overline w_{x_0,\lambda}(\cdot)=\frac{1}{2}\{\wh w^{(1)}_{x_0,\lambda}(\cdot)+\wh w^{(2)}_{x_0,\lambda}(\cdot)\}$.
Replacing expectations in \eqref{eqn:Mc_twofold} by cross-fitted sample averages yields the empirical moment
\begin{equation}
\label{eq:eta-cf_twofold}
\wh\eta_\lambda(x_0;\theta)
=
\frac{1}{n}
\sum_{i=1}^n
\overline w_{x_0,\lambda}(X_i)
\big\{
\ell(Y_i;\theta)-\ell(f(X_i);\theta)
\big\}
+
\frac{1}{N}
\sum_{m=1}^2
\sum_{u\in\cI_m}
\wh w^{(3-m)}_{x_0,\lambda}(\wt X_u)
\ell(f(\wt X_u);\theta).
\end{equation}
The prediction-powered conditional estimator $\wh\theta(x_0)$ is defined as the solution to
\begin{equation}
\label{eq:thetahat-x0_twofold}
\wh\eta_\lambda(x_0;\wh\theta(x_0))
=
0.
\end{equation}
We take the locally unique solution in $\Theta_0$; Lemma~3 in Appendix~E.3 shows that this root exists with probability tending to one under the conditions used below.

\subsection{Prediction-Powered Conditional Inference}
\label{sec:ppci-inference_twofold}
\noindent
Based on the prediction-powered estimator $\wh\theta(x_0)$ defined in \eqref{eq:thetahat-x0_twofold}, we construct a confidence interval for the conditional target $\theta_0(x_0)$ at the target point $x_0$.
We first compute the empirical Jacobian of the cross-fitted localized moment \eqref{eq:eta-cf_twofold} evaluated at $\wh\theta(x_0)$:
\begin{equation}
\label{eq:H-hat-new_twofold}
\wh J_\lambda(x_0)
:=
\left.
\partial_\theta
\wh\eta_\lambda(x_0;\theta)
\right|_{\theta=\wh\theta(x_0)}.
\end{equation}
Next, for each labeled observation, define $\wh\zeta_i(x_0):=\overline w_{x_0,\lambda}(X_i)\{\ell(Y_i;\wh\theta(x_0))-\ell(f(X_i);\wh\theta(x_0))\}$, $i=1,\dots,n$. For each unlabeled observation $u\in\cI_m$, where $m\in\{1,2\}$, define $\wh\zeta_u(x_0):=\wh w^{(3-m)}_{x_0,\lambda}(\wt X_u)\ell(f(\wt X_u);\wh\theta(x_0))$. Let $\wh\sigma^2_{Y-f}(x_0)$ and $\wh\sigma^2_f(x_0)$ denote the sample variances of $\{\wh\zeta_i(x_0)\}_{i=1}^n$ and $\{\wh\zeta_u(x_0)\}_{u=1}^N$, respectively.
We estimate the variance of the localized moment~\eqref{eq:eta-cf_twofold} by
\begin{equation}
\label{eq:V-hat-new_twofold}
\wh V(x_0)
=
\frac{1}{n}\wh\sigma^2_{Y-f}(x_0)
+
\frac{1}{N}\wh\sigma^2_f(x_0).
\end{equation}
Finally, the $(1-\alpha)$ PPCI confidence interval for $\theta_0(x_0)$ is
\begin{equation}
\label{eq:CI-new_twofold}
\cC(x_0)
:=
\left(
\wh\theta(x_0)
\pm
z_{1-\alpha/2}
\sqrt{\wh V(x_0)}/|\wh J_\lambda(x_0)|
\right),
\end{equation}
where $z_{1-\alpha/2}$ is the $(1-\alpha/2)$ quantile of the standard normal distribution.
%%%%%%%%%%%%%%%%%%%%%%%%%%%%%%%%%%%%%%%%%%%%%%5
\section{Theoretical Guarantees}\label{sec:theory}
\noindent We provide theoretical guarantees for the proposed PPCI estimator $\wh\theta(x_0)$ and its associated inference procedure in this section.
\subsection{Convergence Rates of Estimation}
\label{sec:estimation-rates}
\noindent
We show that the PPCI estimator $\wh\theta(x_0)$ defined in~\eqref{eq:thetahat-x0_twofold} achieves the pointwise minimax rate, up to logarithmic factors, in the abundant-unlabeled regime specified below. We begin by introducing mild regularity conditions.

\begin{assumption}\label{ass:B_bound_predictor}
There exists a compact interval $\Theta_0\subset\Theta$ containing $\theta_0(x_0)$ and a constant $\delta_\theta>0$ with
$[\theta_0(x_0)-2\delta_\theta,\theta_0(x_0)+2\delta_\theta]\subset\Theta_0$.
For each $\theta\in\Theta_0$, let $B_{Y-f}(\theta)\!=\!\sup_{(x,y)}|\ell(y;\theta)-\ell(f(x);\theta)|$ and $B_{f}(\theta)=\sup_x |\ell(f(x);\theta)|$.
Then, $\sup_{\theta\in\Theta_0}B_{Y-f}(\theta)<\infty$ and $\sup_{\theta\in\Theta_0}B_{f}(\theta)<\infty$.
\end{assumption}

\begin{assumption}\label{ass:deriv-bdd-jac}
On $\Theta_0$, the maps
$\theta\mapsto \{\ell(y;\theta)-\ell(f(x);\theta)\}$
and $\theta\mapsto \ell(f(x);\theta)$
are continuously differentiable, and their derivatives are uniformly bounded. That is, by letting $G_{Y-f}(\theta)=\sup_{(x,y)}|\partial_\theta (\ell(y; \theta) - \ell(f(x);\theta))|$ and $G_{f}(\theta)=\sup_x|\partial_\theta\ell(f(x);\theta)|$, we have $\sup_{\theta\in\Theta_0}G_{Y-f}(\theta)<\infty$ and $\sup_{\theta\in\Theta_0}G_{f}(\theta)<\infty$.
\end{assumption}

\begin{assumption}\label{ass:J_lambda_bound}
Let $J_\lambda(x_0;\theta):=\partial_\theta \eta_\lambda(x_0;\theta) =\EE[w_{x_0,\lambda}(X)\,\partial_\theta\ell(Y;\theta)]$.
There exist constants $0<c_J\le C_J<\infty$ such that $c_J \le |J_\lambda(x_0;\theta)| \le C_J$ for all $\theta\in\Theta_0$.
These bounds hold uniformly for all sufficiently small $\lambda$ in the admissible regularization range.
\end{assumption}

\begin{assumption}\label{ass:eigenfunctions_bound}
The conditional moments satisfy
$\sup_{\theta\in\Theta_0}\|\eta(\cdot;\theta)\|_{\cH}\le B_\eta<\infty$.
Separately, $\cH$ is norm-equivalent to $H^m(\cX)$ for some $m>d/2$, and the kernel in \eqref{eq:rkhs-localization-pop-x0_twofold} admits the expansion
$K(x,x')=\sum_{j\ge1}\mu_j\phi_j(x)\phi_j(x')$, where $\{\phi_j\}_{j\ge1}$ is orthonormal in $L^2(\rho_X)$, $B_\phi:=\sup_j\|\phi_j\|_\infty<\infty$, and $\mu_j\asymp j^{-2m/d}$.
\end{assumption}

\begin{assumption} \label{ass:homo_res}
There exists a constant $\underline{\sigma}^2 > 0$ such that
$\mathrm{Var}\big(\ell(Y;\theta) | X\big) \ge \underline{\sigma}^2$ almost surely for all $\theta \in \Theta_0$.
\end{assumption}

\noindent
Assumption~\ref{ass:B_bound_predictor} bounds the estimating function and places the target in the interior of the parameter interval, while Assumption~\ref{ass:deriv-bdd-jac} controls its derivative with respect to $\theta$. Neither condition requires the predictor $f$ to be smooth as a function of $x$: prediction error is corrected in the labeled term, so only the displayed bounds on $\ell$ are used. Assumption~\ref{ass:J_lambda_bound} provides local identification uniformly over the regularization range. Assumption~\ref{ass:eigenfunctions_bound} separates the uniform RKHS smoothness of the conditional moment from the kernel--design spectral conditions used to quantify pointwise localization. Finally, Assumption~\ref{ass:homo_res} keeps the localized moment variance nondegenerate. Appendix~C explains which conditions are verified for the running examples and which remain structural assumptions.

A key quantity in our analysis is the \emph{pointwise leverage} at the target point $x_0$, $
D(x_0;\lambda) := \big\langle K_{x_0},(T_K+\lambda I)^{-1}K_{x_0}\big\rangle_{\cH}$.
{
We also use the quadratic size of the localization weight, $
Q(x_0;\lambda):=\|w_{x_0,\lambda}\|_{L^2(\rho_X)}^2
=\mathbb E\{w_{x_0,\lambda}^2(X)\}$.
}
In classical RKHS regression problems, complexity is typically measured by the \emph{global effective dimension} $D(\lambda) := \Tr((T_K+\lambda I)^{-1}T_K)$ \citep{caponnetto2007optimal,wainwright2019hd,fischer2020sobolev}, which characterizes the estimation of the entire regression function on $\cX$. In contrast, our objective is a pointwise estimate at $x_0$. {The Tikhonov identity $D(x_0;\lambda)=Q(x_0;\lambda)+\lambda\|w_{x_0,\lambda}\|_{\cH}^2$
shows that $D(x_0;\lambda)$ decomposes into the $L^2(\rho_X)$ variance proxy $Q(x_0;\lambda)$ and the RKHS regularization component $\lambda\|w_{x_0,\lambda}\|_{\cH}^2$.} Furthermore, the eigenvalues of $(T_K+\lambda I)^{-1}$ scale as $(\mu_j+\lambda)^{-1}$, which strictly increase as $\lambda \to 0$. Intuitively, as the regularization level $\lambda$ vanishes, the RKHS localization weight $w_{x_0,\lambda}$ sharpens to approximate a Dirac delta function at $x_0$, leading its $L^2(\rho_X)$-norm to diverge. Consequently, $D(x_0;\lambda)$ grows and the pointwise estimation problem becomes highly variable, whereas a larger $\lambda$ reduces this leverage and stabilizes the estimation.

\begin{proposition}\label{prop:Dx0-order-H}
Let $x_0$ be an interior point of $\cX$. As $\lambda\to0$,
\[
{
D(x_0;\lambda)\asymp Q(x_0;\lambda)\asymp D(x_0;\lambda)-Q(x_0;\lambda)\asymp D(\lambda)\asymp \lambda^{-d/(2m)}.
}
\]
\end{proposition}
\noindent
{Proposition~\ref{prop:Dx0-order-H} shows that for RKHS $\cH=H^m(\cX)$ and interior $x_0$, the pointwise leverage $D(x_0;\lambda)$ and the quadratic weight size $Q(x_0;\lambda)$ have the same asymptotic order as the global effective dimension $D(\lambda)$.}

\smallskip
\smallskip
\noindent\textbf{Nonasymptotic Upper Bound for Estimation.}
We first establish a nonasymptotic upper bound for the PPCI estimator
$\wh\theta(x_0)$ in \eqref{eq:thetahat-x0_twofold}. The bound separates moment estimation error,
weight estimation error, and regularization bias, and shows how labeled data,
unlabeled data, and the predictor $f$ jointly determine the estimation rate.

\begin{theorem}
\label{thm:thetahat-theta0-bound-twofold}
Let $x_0$ be an interior point of $\cX$, and let $\theta_\lambda(x_0)$ denote the locally unique solution in $\Theta_0$ to $\eta_\lambda(x_0;\theta)=0$. Suppose Assumptions~\ref{ass:sampling-scheme}--\ref{ass:eigenfunctions_bound} hold. Assume $\lambda=\lambda_{n,N}\to0$ and that the regularization parameter satisfies
\begin{equation}
\label{eq:cond-upper-clean-twofold}
\lambda^{-1}=o\left(\left(\frac{N}{\log N}\right)^{\frac{m}{d}}\wedge\left(\frac{n\wedge N}{\log N}\right)^{\frac{2m}{d}}\right).
\end{equation}
Then for all sufficiently large $n,N$, with probability at least $1-56/N-4/n^2-8/N^2$,
we have
\begin{equation}
\label{eq:thetahat-theta0-bound-clean-twofold}
\begin{aligned}
|\wh\theta(x_0)-\theta_0(x_0)|
\le
&
\underbrace{\frac{8\sqrt{2}}{c_J}\sqrt{Q(x_0;\lambda)\log(n+N)}
\left(\frac{B_{Y-f}(\theta_\lambda(x_0))}{\sqrt n}+\frac{B_f(\theta_\lambda(x_0))}{\sqrt N}\right)}_{\text{moment estimation error}}
\\
&+
\underbrace{\frac{32\sqrt{2}}{c_J}\kappa\|\eta(\cdot;\theta_\lambda(x_0))\|_{\cH}\sqrt{\frac{D(x_0;\lambda)\log N}{N}}}_{\text{weight estimation error}}
\\
&+
\underbrace{\frac{1}{c_J}\|\eta(\cdot;\theta_0(x_0))\|_{\cH}\sqrt{\lambda\{D(x_0;\lambda)-Q(x_0;\lambda)\}}}_{\text{regularization bias}}.
\end{aligned}
\end{equation}
\end{theorem}
\noindent
To state the resulting Sobolev rate, define
\[
\sN
:=
\sqrt{\log(n+N)}
\left(
\frac{8\sqrt{2}B_{Y-f}(\theta_\lambda(x_0))}{\sqrt n}
+
\frac{8\sqrt{2}B_f(\theta_\lambda(x_0))}{\sqrt N}
\right)
+
\frac{32\sqrt{2}\kappa
\|\eta(\cdot;\theta_\lambda(x_0))\|_{\cH}\sqrt{\log N}}
{\sqrt N}.
\]
Suppose that $N=n^{r_1}$ for some fixed $r_1\ge1$. Under the
standing uniform upper bounds, assume additionally that
$B_{Y-f}(\theta_\lambda(x_0))$ and
$\|\eta(\cdot;\theta_0(x_0))\|_{\cH}$ are bounded away from zero. Then
$\sN\asymp\sqrt{\log n/n}$. If, in addition, $r_1>d/m$, the formal
oracle balance
$\lambda\asymp
\{\sN/\|\eta(\cdot;\theta_0(x_0))\|_{\cH}\}^{2}$
satisfies~\eqref{eq:cond-upper-clean-twofold} in
Theorem~\ref{thm:thetahat-theta0-bound-twofold}. Consequently, there
exists a constant $C_0>0$ such that, with the probability stated in
that theorem,
\begin{equation}
\label{eq:upper_bound_sobolev_rate_twofold}
|\wh\theta(x_0)-\theta_0(x_0)|
\le
\frac{2\sqrt{C_0}}{c_J}
\|\eta(\cdot;\theta_0(x_0))\|_{\cH}^{\frac{d}{2m}}
\sN^{1-\frac{d}{2m}}.
\end{equation}

We make three remarks on Theorem \ref{thm:thetahat-theta0-bound-twofold}. First, the bound in \eqref{eq:thetahat-theta0-bound-clean-twofold} has a clear statistical interpretation.
The first term of \eqref{eq:thetahat-theta0-bound-clean-twofold} represents the error in estimating the localized moment \eqref{eq:eta-cf_twofold}, which depends on the residual size $B_{Y-f}$ and the prediction magnitude $B_f$. When the unlabeled sample size $N$ is large and the predictor $f$ is informative, the residual term $B_{Y-f}$ is small, resulting in substantial error reduction.
The second term of \eqref{eq:thetahat-theta0-bound-clean-twofold} is the error from estimating the population localization weight \eqref{eq:rkhs-localization-pop-x0_twofold} by its empirical version.
The third term of \eqref{eq:thetahat-theta0-bound-clean-twofold} is the regularization bias induced by the Tikhonov approximation in \eqref{eqn:Tikhonov-regu_twofold}.
The factor $\|\eta(\cdot;\theta_0(x_0))\|_{\cH}$ is the smoothness-complexity radius of the conditional moment: a larger value permits sharper variation around $x_0$ and therefore enlarges the worst-case regularization bias.

Second, in contrast to classical kernel ridge regression (KRR) that focuses on estimation of the regression function itself \citep{wahba1990spline,dai2023orthogonalized,ma2023optimally, dai2025nonparametric}, we target a moment-defined functional $\theta_0(x_0)$ in~\eqref{eq:cond-risk} at a fixed design point $x_0$. This difference induces two additional difficulties.
On one hand, the statistical properties of the estimator $\wh\theta(x_0)$ depend on the localized Jacobian
$J_\lambda(x_0;\theta)=\partial_\theta \eta_\lambda(x_0;\theta)$ in Assumption~\ref{ass:J_lambda_bound}.
The relevant conditional derivative function
$g_\theta(x)=\EE\{\partial_\theta\ell(Y;\theta)|X=x\}$ need not belong to $\cH$, even though $J_\lambda(x_0;\theta)=\EE\{w_{x_0,\lambda}(X)g_\theta(X)\}$ is scalar. As a result, standard uniform convergence results for KRR do not directly control its empirical approximation. To ensure that the empirical Jacobian consistently approximates its population counterpart, we establish the scaling condition
$\lambda^{-1}=o\left((N/\log N)^{m/d}\right)$ in Theorem~\ref{thm:thetahat-theta0-bound-twofold}.
On the other hand, the estimator $\wh\theta(x_0)$ couples the estimated localization weights $\wh w_{x_0,\lambda}$ with the estimating function $\ell(f(\cdot);\theta)$ evaluated on the same data in the empirical moment~\eqref{eq:eta-cf_twofold}, inducing a nontrivial dependence structure. We address this via cross-fitting: weights estimated on one fold are applied exclusively to the other. This decoupling ensures that the conditionally weighted moments behave as independent sums, allowing separate control of the moment estimation error and the weight estimation error in Theorem~\ref{thm:thetahat-theta0-bound-twofold}.

{
Finally, we analyze in Appendix~E a
no-split PPCI procedure that pools all labeled and unlabeled covariates
to construct the empirical RKHS localization weight
$\wh w_{x_0,\lambda}$. Although fully data-efficient, this approach
complicates the theoretical analysis. Since the weight is learned
from the pooled covariates, it is not independent of either the labeled
contributions
$\{\ell(Y_i;\theta)-\ell(f(X_i);\theta)\}_{i=1}^n$
or the unlabeled prediction contributions
$\{\ell(f(\wt X_u);\theta)\}_{u=1}^N$
in the empirical moment, creating shared-design dependence. We control
this dependence through an operator-resolvent expansion of
$\wh w_{x_0,\lambda}-w_{x_0,\lambda}$ and separate bounds for the
resulting interaction terms. The numerical comparison of the two-fold and no-split PPCI implementations in
Appendix~B.2 shows nearly identical statistical
performance, while the no-split implementation is substantially more
computationally demanding. Appendix~F.3
summarizes the technical novelties of this analysis and its differences
from classical KRR theory.
}

\smallskip
\smallskip
\noindent\textbf{Pointwise Minimax Lower Bound.}
We now establish an estimator-independent lower bound over a
regular Sobolev class.
Let $\cP$ denote the class of distributions
$(\rho_X,\rho_{Y|X},K)$ satisfying
Assumptions~\ref{ass:eigenfunctions_bound} and~\ref{ass:homo_res},
with the constants in these assumptions, including
$B_\eta\in(0,\infty)$, fixed uniformly over $P\in\cP$.
\begin{theorem}\label{thm:minimax-pointwise}
There exists a constant $c>0$, depending only on the fixed
constants defining $\cP$, such that for all sufficiently large $n$
and every $N\ge0$,
\begin{equation}\label{eq:minimax-pointwise-main}
\inf_{\wt\theta}\ \sup_{P\in\cP}
\EE_P\big[(\wt\theta-\theta_0(x_0))^2\big]
\ge c\,n^{-1+d/(2m)}.
\end{equation}
where the infimum is taken over all measurable estimators based on
$\{(X_i,Y_i)\}_{i=1}^n$ and
$\{\wt X_u\}_{u=1}^N$.
\end{theorem}

\noindent
If $N=n^{r_1}$ for some fixed $r_1\ge1$ satisfying $r_1>d/m$, then
the squared upper rate in
\eqref{eq:upper_bound_sobolev_rate_twofold} is
$n^{-1+d/(2m)}$ up to logarithmic factors and matches the minimax
lower rate in \eqref{eq:minimax-pointwise-main}. Thus, the
oracle-balanced PPCI estimator is pointwise minimax-rate optimal up
to logarithmic factors in this regime.

\subsection{Coverage Guarantees of Confidence Interval}\label{sec:asymptotic}
\noindent
We now establish asymptotic normality of the PPCI estimator
$\wh\theta(x_0)$ and the validity of the resulting confidence interval.

\begin{theorem}
\label{thm:thetahat-asymp-twofold}
Under the conditions of Theorem~\ref{thm:thetahat-theta0-bound-twofold}
and Assumptions~\ref{ass:B_bound_predictor}--\ref{ass:homo_res},
let
\begin{equation*}
    J_\lambda(x_0)=J_\lambda(x_0;\theta_\lambda(x_0)), \quad V(x_0):=\frac{1}{n}\sigma^2_{Y-f}(\theta_\lambda(x_0))+\frac{1}{N}\sigma^2_{f}(\theta_\lambda(x_0)),
\end{equation*}
 where $\sigma^2_{Y-f}(\theta)
    =\Var\big(w_{x_0,\lambda}(X)\{\ell(Y;\theta)-\ell(f(X);\theta)\}\big)$ and $ \sigma^2_f(\theta)=\Var\big(w_{x_0,\lambda}(X)\,\ell(f(X);\theta)\big)$.
Suppose
\begin{equation}
\label{eq:cond-clean-twofold}
n=o\left(\frac{N}{\log N}\right),
\quad
\lambda=o(n^{-1}),
\quad
\lambda^{-1}=o\left(\left(\frac{N}{\log N}\right)^{\frac{m}{d}}
\wedge\left(\frac{n}{(\log(n+N))^2}\right)^{\frac{2m}{d}}\right).
\end{equation}
Then, as $n,N\to\infty$,
\begin{equation*}
\frac{J_\lambda(x_0)\big(\wh\theta(x_0)-\theta_0(x_0)\big)}
{\sqrt{V(x_0)}}\rightsquigarrow N(0,1).
\end{equation*}
\end{theorem}

\noindent
We make three remarks on Theorem~\ref{thm:thetahat-asymp-twofold}. First, compared to Theorem~\ref{thm:thetahat-theta0-bound-twofold}, the asymptotic normality in Theorem~\ref{thm:thetahat-asymp-twofold} characterizes the limiting distribution rather than providing a worst-case upper bound. In Theorem~\ref{thm:thetahat-theta0-bound-twofold}, we explicitly bound the moment estimation error in~\eqref{eq:thetahat-theta0-bound-clean-twofold}. Taking the unlabeled data as an example, the variance of the empirical moment is $\sigma^2_f(\theta)/N=\Var\big(w_{x_0,\lambda}(X)\ell(f(X);\theta)\big)/N$. To construct the first term in \eqref{eq:thetahat-theta0-bound-clean-twofold}, we use Assumption~\ref{ass:B_bound_predictor} to decouple the estimating function from the localization weights, yielding the relaxation $\Var\big(w_{x_0,\lambda}(X)\ell(f(X);\theta)\big)/N \le \EE\big[\{w_{x_0,\lambda}(X)\ell(f(X);\theta)\}^2\big]/N$ that is further bounded by $B_f^2(\theta)\EE\big[w_{x_0,\lambda}(X)^2\big]/N = B_f^2(\theta)Q(x_0;\lambda)/N$, which is upper bounded by $B_f^2(\theta)D(x_0;\lambda)/N$. Consequently, the bounding constants $B_{Y-f}$ and $B_f$ explicitly appear in the first term of~\eqref{eq:thetahat-theta0-bound-clean-twofold}. In contrast, Theorem~\ref{thm:thetahat-asymp-twofold} establishes a central limit theorem where the exact variance components $\sigma^2_{Y-f}$ and $\sigma^2_f$ directly govern the limiting behavior. While Assumption~\ref{ass:B_bound_predictor} is still inherited to ensure finite moments, the conservative bounds $B_{Y-f}$ and $B_f$ are replaced by the precise localized-moment variance $V(x_0)$. Furthermore, Theorem~\ref{thm:thetahat-asymp-twofold} imposes Assumption~\ref{ass:homo_res} to guarantee that $V(x_0)$ is non-degenerate.

Second, the condition $n=o(N/\log N)$ ensures that the error from
estimating the localization weights is asymptotically negligible. The
condition $\lambda=o(n^{-1})$ imposes undersmoothing so that the
regularization bias is asymptotically negligible relative to the
leading moment-estimation error in
Theorem~\ref{thm:thetahat-theta0-bound-twofold}. To clarify the
scaling, suppose $N=n^{r_1}$ and $\lambda=n^{r_2}$ with $r_1>0$ and
$r_2<0$. These conditions are satisfied if
$r_1>\max\{1,d/m\}$ and
$-(m/d)\min\{r_1,2\}<r_2<-1$. The first requirement includes both
the abundant-unlabeled condition and the nonemptiness condition when
$1<r_1<2$. When $r_1\ge2$, nonemptiness follows directly from
$2m>d$. Compared with the estimation-optimal regime in
Section~\ref{sec:estimation-rates}, asymptotic normality strengthens
$r_1\ge1,\ r_1>d/m$ to $r_1>\max\{1,d/m\}$ and replaces the balance
$\lambda\asymp\log n/n$ with undersmoothing $\lambda=o(n^{-1})$.
Theorem~\ref{thm:minimax-pointwise} is estimator-independent
and therefore imposes no tuning requirement.

Finally, in practice, we replace $J_\lambda(x_0)$ and $V(x_0)$ by their estimates $\wh J_\lambda(x_0)$ in~\eqref{eq:H-hat-new_twofold} and
$\wh V(x_0)$ in~\eqref{eq:V-hat-new_twofold}, both evaluated at $\wh\theta(x_0)$,
yielding the PPCI confidence interval $\cC(x_0)$ in~\eqref{eq:CI-new_twofold}.
\begin{corollary}\label{cor:CI-coverage-twofold}
Under the conditions of Theorem~\ref{thm:thetahat-asymp-twofold},
$\wh J_\lambda(x_0)/J_\lambda(x_0)\to_p1$ and
$\wh V(x_0)/V(x_0)\to_p1$.
Moreover, for any $\alpha\in(0,1)$,
$\PP\big(\theta_0(x_0)\in\cC(x_0)\big)\to1-\alpha$,
where $\cC(x_0)$ is the PPCI confidence interval constructed in~\eqref{eq:CI-new_twofold}.
\end{corollary}
\noindent
This corollary validates the asymptotic coverage of the PPCI confidence interval.

{

\subsection{Tuning Procedure}
\label{sec:tuning_twofold}
\noindent
We now describe the tuning rule used to select the bandwidth $h$ and regularization level $\lambda$. For each candidate $(h,\lambda)$, write $K_{h,x_0}=K_h(\cdot,x_0)$ and $w_{x_0,h,\lambda}=(T_h+\lambda I)^{-1}K_{h,x_0}$. Define $D_h(\lambda)=\operatorname{Tr}\{(T_h+\lambda I)^{-1}T_h\}$, $D_h(x_0;\lambda)=w_{x_0,h,\lambda}(x_0)$, and $Q_h(x_0;\lambda)=\|w_{x_0,h,\lambda}\|_{L^2(\rho_X)}^2$.

For fixed finite sets $\mathcal C_h,\mathcal C_\lambda\subset(0,\infty)$, the numerical studies use $\mathcal G_h=\{c\operatorname{median}_{u\in\mathcal I}\|\widetilde X_u-x_0\|_2:c\in\mathcal C_h\}$ and $\mathcal G_\lambda=\{c/[n\log\log(n+e^e)]:c\in\mathcal C_\lambda\}$. The finite sets are fixed before the formal analysis. Scaling $h$ by the median distance adapts the same dimensionless envelope to the local geometry at each target point, while $n\lambda=O([\log\log(n+e^e)]^{-1})=o(1)$ uniformly over $\mathcal G_\lambda$, as required by \eqref{eq:cond-clean-twofold}. The bandwidth grid is also a locality envelope: overly large multiples may lower variance while changing the local target through excessive smoothing. Appendix~B.3 examines its upper endpoint.

For each $(h,\lambda)\in\mathcal G_h\times\mathcal G_\lambda$, compute the empirical analogues $\widehat D_h(\lambda)$, $\widehat D_h(x_0;\lambda)$, and $\widehat Q_h(x_0;\lambda)$ on the unlabeled fold $\mathcal I$, where
$\widehat Q_h(x_0;\lambda)=|\mathcal I|^{-1}\sum_{j\in\mathcal I}\widehat w_{x_0,h,\lambda}^2(\widetilde X_j)$.
We then define
\[
\mathcal A_{\rm stab}
=
\left\{(h,\lambda)\in\mathcal G_h\times\mathcal G_\lambda:
\widehat D_h(\lambda)\sqrt{\frac{\log|\mathcal I|}{|\mathcal I|}}\le c_{\rm op},\quad
\widehat D_h(x_0;\lambda)\frac{\log(n+|\mathcal I|)}{n\wedge|\mathcal I|}\le c_{\rm loc}
\right\}.
\]
These restrictions are finite-sample empirical analogues of the
stability condition in \eqref{eq:cond-upper-clean-twofold}. Under $N=n^{r_1}$ with $r_1>\max\{1,d/m\}$ and $2m>d$,
every sequence in the shrinking grid
$\lambda=c/[n\log\log(n+e^e)]$
satisfies the rate conditions in \eqref{eq:cond-clean-twofold}. We further define
\[
\mathcal A_{\rm bias}
=
\left\{(h,\lambda)\in\mathcal G_h\times\mathcal G_\lambda:
\left(
\frac{n\lambda[\widehat D_h(x_0;\lambda)-\widehat Q_h(x_0;\lambda)]_+}
{\widehat Q_h(x_0;\lambda)}
\right)^{1/2}
\le c_{\rm bias}
\right\}.
\]
This restriction controls the regularization bias relative to
the moment-estimation error scale in~\eqref{eq:thetahat-theta0-bound-clean-twofold}. Indeed,
\eqref{eq:thetahat-theta0-bound-clean-twofold} and
Assumption~\ref{ass:homo_res} imply
\[
\frac{\text{regularization-bias upper bound}}
{\text{lower moment-estimation error scale}}
\le
\frac{\|\eta(\cdot;\theta)\|_{\cH}}{\underline\sigma}
\left\{\frac{n\lambda(D_h-Q_h)}{Q_h}\right\}^{1/2}.
\]
The first factor is the unknown smoothness-to-noise ratio, and the
remaining factor is the diagnostic estimated by
the expression defining $\mathcal A_{\rm bias}$. Hence,
$c_{\rm bias}$ absorbs the unknown ratio.
Among candidates in $\mathcal A_{\rm stab}\cap\mathcal A_{\rm bias}$, we choose
\[
(\widehat h,\widehat\lambda)
=
\arg\min_{(h,\lambda)\in\mathcal A_{\rm stab}\cap\mathcal A_{\rm bias}}
\frac{\widehat Q_h(x_0;\lambda)^{1/2}}
{\left||\mathcal I|^{-1}\sum_{j\in\mathcal I}
\widehat w_{x_0,h,\lambda}(\widetilde X_j)\right|}.
\]
The criterion is a proxy for the weight-dependent component of the
moment-estimation error scale of $\wh\theta(x_0)$
in~\eqref{eq:thetahat-theta0-bound-clean-twofold}. Its numerator,
$\widehat Q_h(x_0;\lambda)^{1/2}
=\|\widehat w_{x_0,h,\lambda}\|_{L^2(\widehat\rho_{\mathcal I})}$,
is the empirical $L^2$ norm of the localization weight, where
$\widehat\rho_{\mathcal I}
=|\mathcal I|^{-1}\sum_{j\in\mathcal I}\delta_{\widetilde X_j}$.
The denominator is the absolute empirical mean of the weight. When
the derivative of $\ell$ is constant, this mean is proportional to the
localized Jacobian. For example, if
$\partial_\theta\ell(y;\theta)\equiv a_\theta\ne0$, then
$J_{h,\lambda}(x_0;\theta)
=\EE[w_{x_0,h,\lambda}(X)\partial_\theta\ell(Y;\theta)]
=a_\theta\EE[w_{x_0,h,\lambda}(X)]$, so the denominator is a proxy for
$|J_{h,\lambda}(x_0;\theta)|/|a_\theta|$.
For a general estimating function, the criterion remains a proxy rather than an exact estimate of the moment-estimation
error scale.

If $\mathcal A_{\rm stab}\cap\mathcal A_{\rm bias}$ is empty, define
$S_{\rm op}$, $S_{\rm loc}$, and $S_{\rm bias}$ as the three left
sides divided by their respective thresholds. We choose a minimizer
of $\max\{S_{\rm op},S_{\rm loc},S_{\rm bias}\}$, breaking ties by the
same criterion. We call this least-violation selection the fallback
rule. It is a numerical safeguard: all main experiments have no fallback events,
while Appendix~B.3 reports its frequency in
the sensitivity study.

For a fold of size $M$, constructing the Gram matrix
for each bandwidth costs $O(M^2d)$ operations, and one
eigendecomposition costs $O(M^3)$. Our implementation reuses this
decomposition across the $\lambda$ grid, so evaluating the full
tuning grid costs
$O\{|\mathcal G_h|(M^2d+M^3+|\mathcal G_\lambda|M^2)\}$.
The storage cost is $O(M^2)$ per bandwidth, or
$O(|\mathcal G_h|M^2)$ when all decompositions are cached as in the
reported experiments; see Appendix~A.2.
}

\section{Practical Implications}
\label{sec:practical}
\noindent In this section, we clarify how the ML predictor and the unlabeled data contribute to PPCI, and discuss how to sample labeled and unlabeled data under a limited budget.
\subsection{The Role of ML Predictor}
\noindent
The localized-moment variance $V(x_0)$ in Theorem~\ref{thm:thetahat-asymp-twofold} characterizes how the ML predictor $f$ affects estimation uncertainty; the estimator's asymptotic variance is $V(x_0)/J_\lambda(x_0)^2$. For the classical localized estimator based only on labeled data, the corresponding asymptotic variance is $\sigma_Y^2(\theta_\lambda(x_0))/\{nJ_\lambda(x_0)^2\}$, with $\sigma_Y^2(\theta)=
\Var\big(w_{x_0,\lambda}(X)\ell(Y;\theta)\big)$ \citep{Tsybakov_2009}. In contrast, the PPCI localized-moment variance decomposes as
$V(x_0)=n^{-1}\sigma_{Y-f}^2(\theta_\lambda(x_0))+N^{-1}\sigma_f^2(\theta_\lambda(x_0))$,
where $\sigma_{Y-f}^2$ and $\sigma_f^2$ are defined in Theorem~\ref{thm:thetahat-asymp-twofold}. When $N\gg n$, the second term $\sigma_f^2/N$ is negligible, and the variance is dominated by the residual component $\sigma_{Y-f}^2$. If the predictor $f$ is informative in the sense that $\sigma_{Y-f}^2(\theta)\ll\sigma_Y^2(\theta)$, then PPCI enjoys substantial variance reduction relative to the classical localized estimator. Thus, efficiency gains arise precisely when the predictor meaningfully reduces residual variability.

\begin{figure}[t]
    \centering
    \includegraphics[width=\linewidth,height=0.22\textheight]{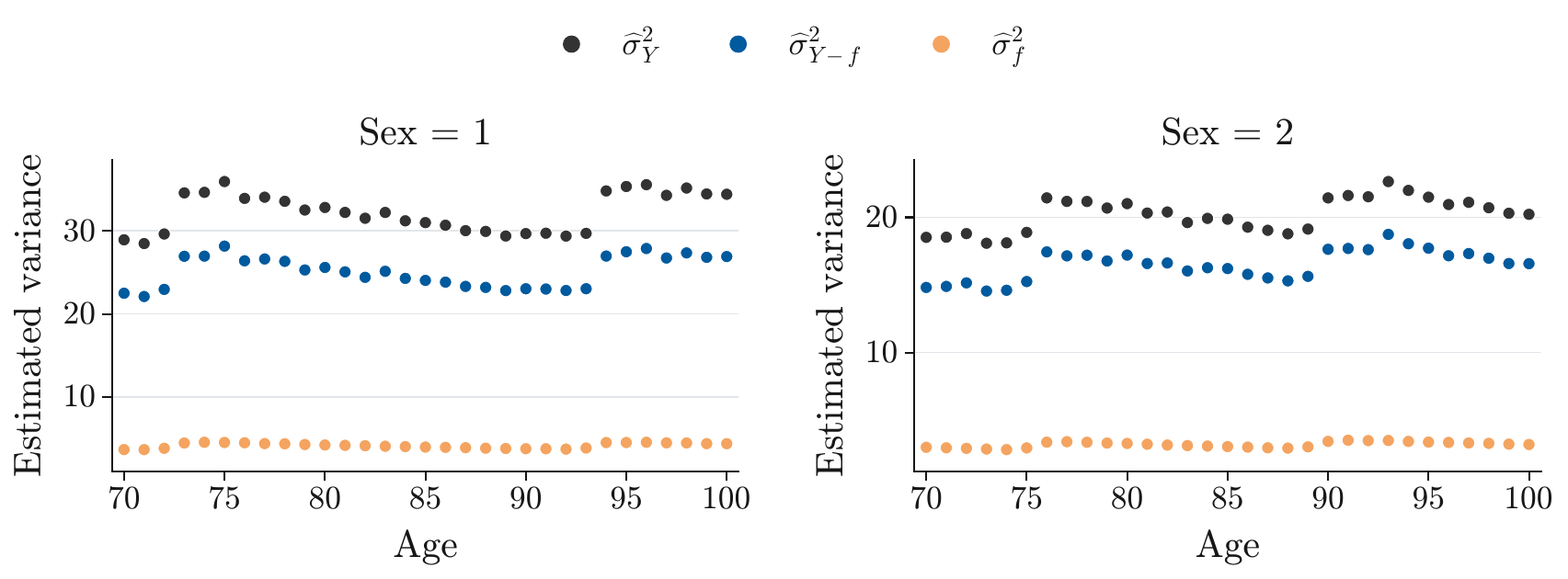}
    \caption{Empirical $\sigma^2_{Y-f}$, $\sigma^2_Y$, and $\sigma^2_f$ for the conditional mean at different (age, sex) target points in the census income data. Results are based on $1000$ replications.}
    \label{fig:income_sigma2}
    \Description{Two age-indexed panels, one per sex group, showing residual-contribution variance below outcome-contribution variance and a comparatively small prediction-contribution variance.}
\end{figure}

We illustrate this phenomenon using the census income data in Section~\ref{subsec:census}. Using XGBoost as the auxiliary predictor $f$, we estimate $\sigma_Y^2$, $\sigma_{Y-f}^2$, and $\sigma_f^2$. Figure~\ref{fig:income_sigma2} shows that $\sigma_{Y-f}^2$ is substantially smaller than $\sigma_Y^2$, especially for the male subgroup ($\text{sex}=1$, with $\text{sex}=2$ representing the female subgroup). In this regime, $\sigma_f^2/N$ is negligible, and PPCI achieves clear variance reduction.

The gain depends on the quality of the predictor $f$. If $f$ is
uninformative, the residual variance
$\sigma_{Y-f}^2(\theta)$ can be large. Motivated by the
variance-adaptive prediction weighting of PPI++
\citep{angelopoulos2023ppi++}, we introduce a point-specific
coefficient $\omega(x_0)\in[0,1]$ and modify the empirical moment
in~\eqref{eq:eta-cf_twofold} as
\begin{equation*}
\begin{aligned}
\widehat\eta_{\lambda,\omega}(x_0;\theta)
&= \frac{1}{n}\sum_{i=1}^n
\widehat w_{x_0,\lambda}(X_i)\ell(Y_i;\theta)
-\omega(x_0)\frac{1}{n}\sum_{i=1}^n
\widehat w_{x_0,\lambda}(X_i)\ell(f(X_i);\theta)\\
&\quad+\omega(x_0)\frac{1}{N}\sum_{u=1}^N
\widehat w_{x_0,\lambda}(\widetilde X_u)
\ell(f(\widetilde X_u);\theta).
\end{aligned}
\end{equation*}
{
When $\omega(x_0)=1$, this is the PPCI moment; when
$\omega(x_0)=0$, it reduces to the labeled-only localized moment.
The coefficient has a direct variance-minimization interpretation.
For a fixed $(x_0,\theta)$, let
$A_\theta=w_{x_0,\lambda}(X)\ell(Y;\theta)$,
$B_\theta=w_{x_0,\lambda}(X)\ell(f(X);\theta)$, and
$C_\theta=w_{x_0,\lambda}(\widetilde X)
\ell(f(\widetilde X);\theta)$. Minimizing
$n^{-1}\Var(A_\theta-\omega B_\theta)
+N^{-1}\omega^2\Var(C_\theta)$ over $\omega\in[0,1]$ gives
\begin{equation}
\omega^*(x_0;\theta)
=
\Pi_{[0,1]}\left\{
\frac{\Cov(A_\theta,B_\theta)}
{\Var(B_\theta)+(n/N)\Var(C_\theta)}
\right\},
\label{eq:omega-star-ppci-plus}
\end{equation}
where $\Pi_{[0,1]}$ denotes projection onto $[0,1]$.
For a fixed coefficient, the resulting localized-moment
variance is
$V_\omega(x_0)=n^{-1}\Var(A_\theta-\omega B_\theta)
+N^{-1}\omega^2\Var(C_\theta)$, so the parameter-scale standard
error is $\sqrt{V_\omega(x_0)}/|J_\lambda(x_0)|$.
We call the resulting point-specific procedure PPCI++. At the
population level, any fixed $\omega$ changes the variance but not the target
parameter. Implementation details are given in
Appendix~A.3.}

\subsection{The Role of Unlabeled Data}
\noindent
The unlabeled data $\{\wt X_u\}_{u=1}^N$ are incorporated to estimate the localization weights $\wh w_{x_0,\lambda}$. Recall from \eqref{eq:rkhs-localization-pop-x0_twofold} that the weight $w_{x_0,\lambda}$ relies on the population operator $T_K$. In Step 1 of Algorithm~\ref{alg:ppci_twofold}, we approximate $T_K$ using the empirical operator $\wh T_K$ constructed from unlabeled samples. As established in Theorem~\ref{thm:thetahat-theta0-bound-twofold}, the resulting weight estimation contributes an $O_p(\sqrt{D(x_0;\lambda)\log N/N})$ term to the total error. Thus, a large unlabeled sample enables a more accurate estimation of the underlying localization structure.
Moreover, unlabeled data directly reduce the localized-moment variance
$V(x_0) =n^{-1}\sigma_{Y-f}^2(\theta_\lambda(x_0))+N^{-1}\sigma_f^2(\theta_\lambda(x_0))$, and the term $\sigma_f^2(\theta)/N$ vanishes as $N$ increases.

\begin{figure}[t]
    \centering
    \includegraphics[width=\linewidth, height=0.22\textheight]{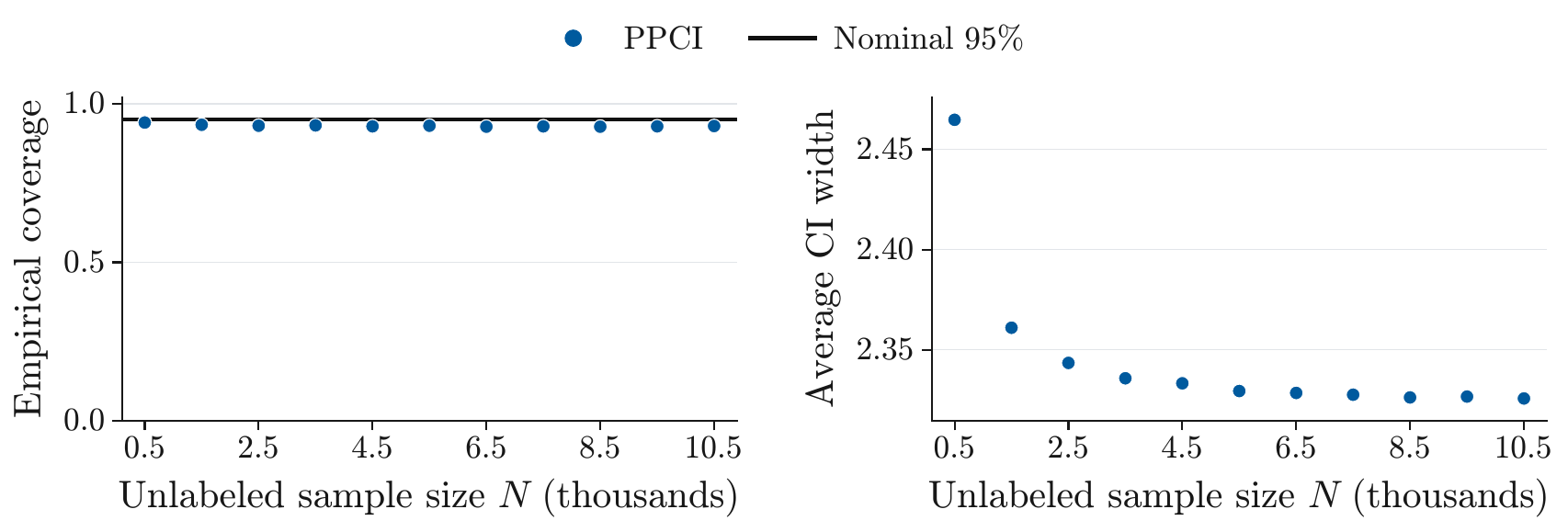}
    \caption{Empirical coverage rate and average width of nominal $95\%$ confidence intervals for the conditional mean at the target point ($\text{age}=70, \text{sex}=1$) in the census income data. The black horizontal line indicates the nominal $95\%$ coverage level. Results are based on $1000$ replications.}
    \label{fig:census-unlabeled-trend}
    \Description{As unlabeled sample size increases, empirical coverage remains near the 0.95 reference line while average interval width decreases.}
\end{figure}

We illustrate these effects by revisiting the census income data in Section~\ref{subsec:census}. Fixing a target point corresponding to age $70$ and sex $1$, we increase the unlabeled sample size $N$ from $500$ to $10,500$.
Figure~\ref{fig:census-unlabeled-trend} shows a clear pattern: as $N$ increases, the average confidence interval width decreases, while coverage remains stable. This result highlights the practical benefit of abundant unlabeled data in improving efficiency.

\subsection{Budget-Aware Sampling Design}
\label{subsec:budget_design}
\noindent
We study how to allocate labeled and unlabeled samples under a fixed budget to minimize the expected width of the PPCI confidence interval in \eqref{eq:CI-new_twofold}. Let $c_l>0$ and $c_u>0$ denote the per-sample costs of labeled and unlabeled data, and let $C>0$ be the total budget. Formally, the empirical budget-optimal allocation problem aims to solve
$\min_{n>0,N>0} \sqrt{\wh V(x_0)}/|\wh J_\lambda(x_0)|$
subject to $c_l n + c_u N \le C$.
While this empirical objective depends on the finite-sample estimates $\wh V(x_0)$ and $\wh J_\lambda(x_0)$, Corollary~\ref{cor:CI-coverage-twofold} shows that ${\sqrt{\wh V(x_0)}}/{|\wh J_\lambda(x_0)|} = {\sqrt{V(x_0)}}/{|J_\lambda(x_0)|}(1 + o_p(1))$. We focus on minimizing the deterministic leading-order term $\sqrt{V(x_0)}/|J_\lambda(x_0)|$. Since the population Jacobian $J_\lambda(x_0)$ depends solely on $(x_0, \lambda)$, minimizing the leading-order width is equivalent to minimizing the localized-moment variance $V(x_0)$ under the budget constraint.

\begin{proposition}\label{prop:budget_opt}
Fix $(x_0,\lambda)$ and suppose the two variance components are positive. Under the conditions of Theorem~\ref{thm:thetahat-asymp-twofold}, write $\sigma^2_{Y-f}=\sigma^2_{Y-f}(\theta_\lambda(x_0))$ and $\sigma^2_{f}=\sigma^2_{f}(\theta_\lambda(x_0))$. Then the budget-optimal allocation for $\min_{n>0,N>0} \sqrt{V(x_0)}/|J_\lambda(x_0)|$
subject to $c_l n + c_u N \le C$
is
$n^\star
=
\frac{C\,\sqrt{\sigma^2_{Y-f}/c_l}}
{\sqrt{\sigma^2_{Y-f}\,c_l}
+\sqrt{\sigma^2_{f}\,c_u}}$ and
$N^\star
=
\frac{C\,\sqrt{\sigma^2_{f}/c_u}}
{\sqrt{\sigma^2_{Y-f}\,c_l}
+\sqrt{\sigma^2_{f}\,c_u}}$,
with minimum variance $V_{\min}(C)
=
\frac{\big(\sqrt{\sigma^2_{Y-f}\,c_l}
+\sqrt{\sigma^2_{f}\,c_u}\big)^2}{C}.$
The boundary cases follow by continuity.
\end{proposition}
\noindent
In Proposition~\ref{prop:budget_opt}, the population variances $\sigma^2_{Y-f}$ and $\sigma^2_f$ depend on $(x_0,\lambda)$ and are unknown in practice.
To implement the optimal allocation, we adopt a two-stage design. First, we use a small fraction of the total budget to collect preliminary labeled and unlabeled samples and compute the empirical variances $\wh\sigma^2_{Y-f}$ and $\wh\sigma^2_f$ of the localized contributions.
We then plug these estimates into Proposition~\ref{prop:budget_opt} to determine the budget-optimal allocation $(n^\star, N^\star)$ under the remaining budget. Finally, we collect a new dataset of sizes $(n^\star, N^\star)$ and apply the full PPCI procedure in Algorithm~\ref{alg:ppci_twofold}. The preliminary observations are not reused in the confirmatory PPCI analysis.

\section{Experiments}\label{sec:experiments}
\noindent
We evaluate PPCI through simulation studies and two real-data
applications. Throughout, we implement the two-fold procedure in
Algorithm~\ref{alg:ppci_twofold} and select $(h,\lambda)$ independently
within each unlabeled training fold using
Section~\ref{sec:tuning_twofold}. After the candidate grids and
constants are fixed, every foldwise choice uses unlabeled covariates
only; no outcome, prediction, or coverage value enters tuning.
Appendix~B.1 reports the
dataset-specific grids and constants. For each target point, we
generate 50 unlabeled samples and reuse each pair of
fold-specific weights across 20 labeled draws, yielding
$50\times20=1000$ confidence intervals and 50 independent unlabeled
clusters.
Appendix~B.3 examines sensitivity to the
bias threshold, bandwidth envelope, and regularization envelope and
reports fallback frequencies under alternative configurations. The configurations used for the main results have no fallback events.

\subsection{Simulation Studies}\label{sec:simu-demo}
\noindent
We generate $X=(X_1,X_2,X_3)^\top\sim\mathrm{Unif}([0,1]^3)$ and $Y=m(X)+\varepsilon$, where $m(X)=\sin(2\pi X_1)+0.6\cos(2\pi X_2)+0.4\sin(2\pi X_3)+0.25\sin\{2\pi(X_1+X_2)\}$ and $\varepsilon\sim N(0,1)$ is independent of $X$. To represent a predictor trained before the inference sample is observed, we use the deterministic family $f_q(X)=qm(X)+(1-q)s(X)$, where $s(X)=\sqrt{1.5825}\sin(6\pi X_1)$; in particular, $f_q$ never uses the realized response. The main experiment takes $q=0.9$.
For each $x_0$, we use $n=400$ labeled and $N=10{,}000$ unlabeled observations with the Mat\'ern-$5/2$ kernel $K_h(x,z)=\{1+\sqrt{5}\|x-z\|_2/h+5\|x-z\|_2^2/(3h^2)\}\exp\{-\sqrt{5}\|x-z\|_2/h\}$. Appendix~B.1 reports the tuning configuration, and Appendix~B.3 gives the corresponding covariate-only sensitivity analysis.

\begin{figure}[t]
    \centering
    \includegraphics[width=\linewidth, height=0.18\textheight]{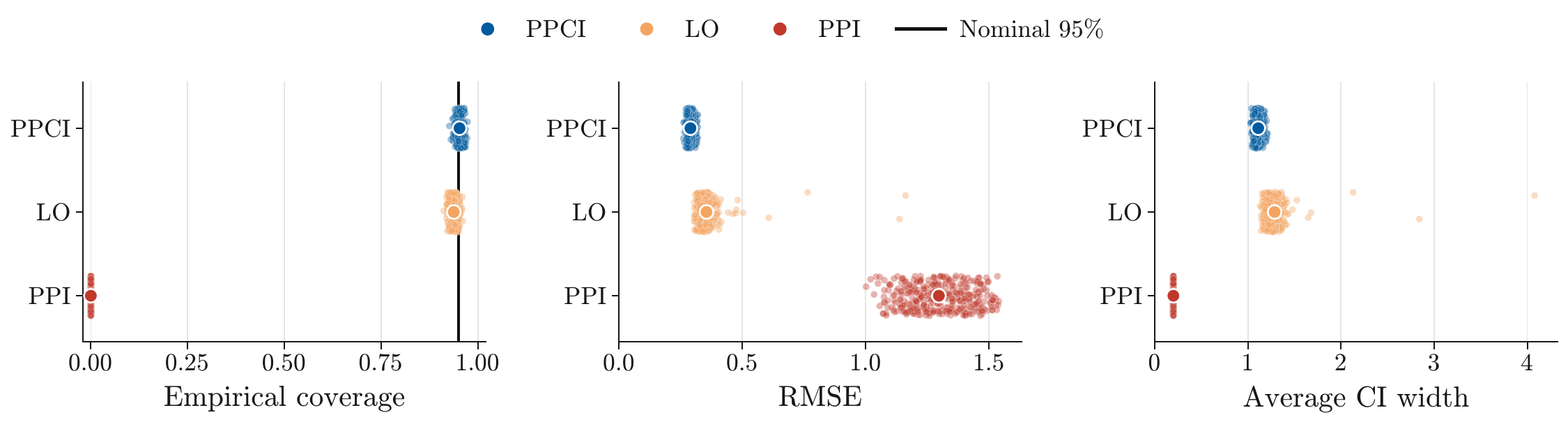}
    \caption{Distributions of empirical coverage rates, RMSE, and average widths of nominal $95\%$ confidence intervals for the conditional mean across 343 simulated target points. Results are based on $1000$ replications per target point.}
    \label{fig:simu-conditional-mean}
    \Description{Three distribution panels compare PPCI, labeled-only, and population-level PPI across target points; PPCI is concentrated near nominal coverage and has lower RMSE and width than labeled-only inference.}
\end{figure}

We compare PPCI with a labeled-only (LO) estimator using the same averaged RKHS localization weights but no prediction terms, and with classical PPI \citep{angelopoulos2023prediction}, which targets a marginal rather than conditional functional. The target points form a $7^3=343$ grid on $[0.70,0.85]^3$. As shown in Figure~\ref{fig:simu-conditional-mean}, PPCI has mean coverage $0.952$, RMSE $0.290$, and average width $1.111$, compared with $0.938$, $0.355$, and $1.287$ for LO. Thus, PPCI retains near-nominal coverage while reducing RMSE and width by $18.3\%$ and $13.7\%$, respectively. The population-level PPI comparator is out of scope for these pointwise conditional targets and consequently does not attain conditional coverage.

\subsection{Adaptation to Predictor Quality}\label{subsec:predictor-quality}
\noindent
We evaluate the PPCI++ procedure defined in Section~\ref{sec:practical} as predictor quality deteriorates. Using the same data-generating process and tuning rule, we take $n=500$, $N=10{,}000$, eight fixed target points, and the common $50\times20$ design, with $q\in\{0,0.5,0.9\}$. PPCI++ uses the normal Wald interval throughout.

\begin{figure}[t]
    \centering
    \includegraphics[width=0.85\linewidth, height=0.35\textheight]{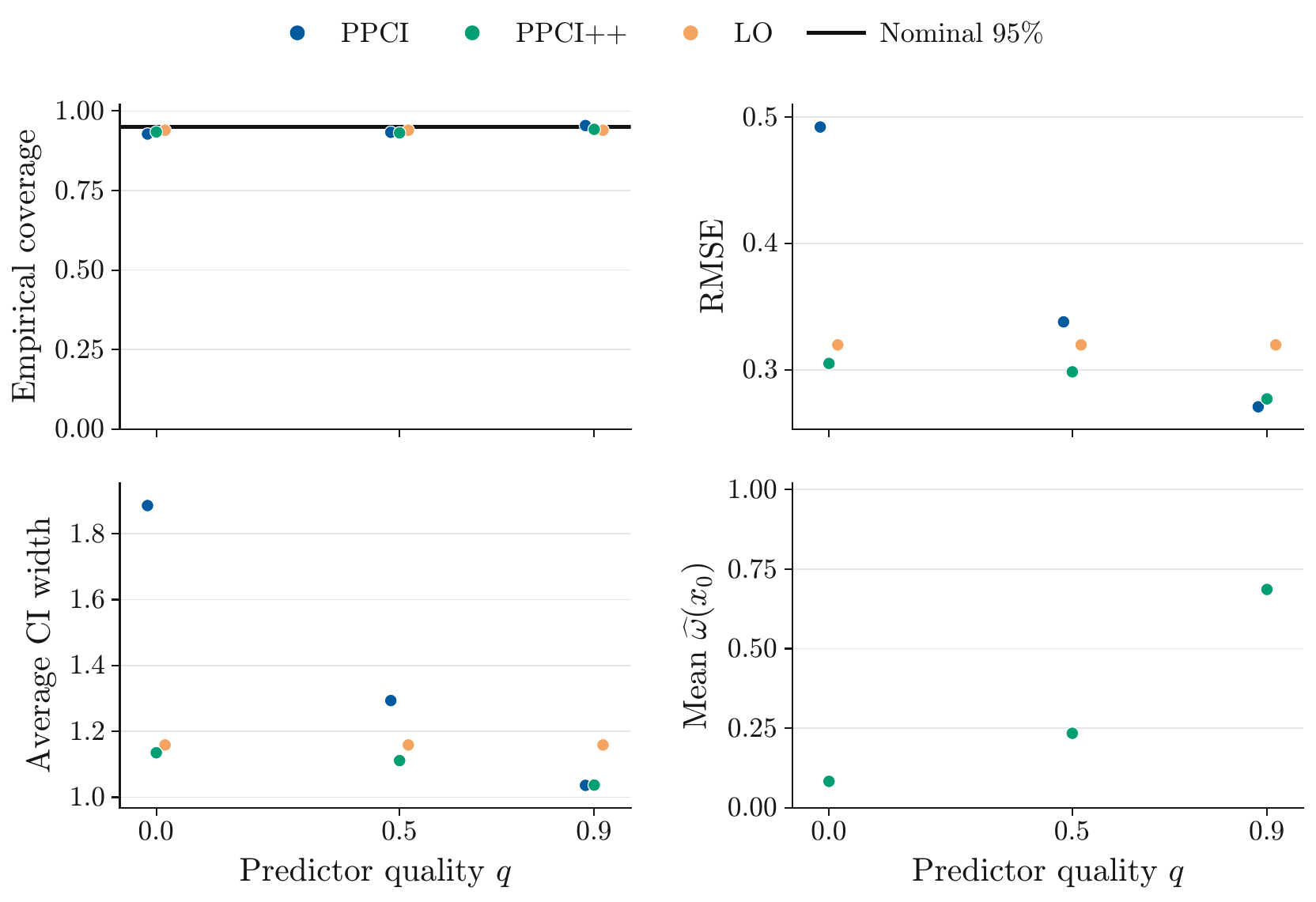}
    \caption{Empirical coverage, RMSE, average confidence-interval width, and mean $\widehat\omega(x_0)$ as the auxiliary predictor quality $q$ varies. Results are averaged over eight target points with $1000$ replications per point.}
    \label{fig:predictor-quality}
    \Description{Four panels trace performance as predictor quality declines; the adaptive PPCI++ coefficient decreases toward zero and PPCI++ approaches labeled-only RMSE and width.}
\end{figure}

Figure~\ref{fig:predictor-quality} shows that adaptation costs little when the predictor is strong: at $q=0.9$, PPCI and PPCI++ have RMSE $0.271$ and $0.277$ and nearly identical widths $1.036$ and $1.037$. As quality falls, the mean coefficient decreases from $0.686$ at $q=0.9$ to $0.234$ at $q=0.5$ and $0.083$ at $q=0$. At $q=0$, PPCI++ has coverage $0.934$, RMSE $0.305$, and width $1.135$, close to LO $(0.940,0.320,1.159)$ and substantially better than full-weight PPCI $(0.928,0.492,1.886)$. Thus, PPCI++ automatically retreats toward LO when predictions are uninformative while retaining the gain from informative predictions.

\subsection{Census Income Data}
\label{subsec:census}
\noindent
We study conditional mean income as a function of age and sex (Male$=1$, Female$=2$) using the California census data and precomputed XGBoost predictions from \citet{angelopoulos2023prediction}. The data contain $380{,}091$ observations. We rescale income by $10{,}000$ and analyze the two sex groups separately; their sample sizes are $187{,}471$ and $192{,}620$. For evaluation only, we construct smoothed finite-population reference targets by fitting a Nadaraya--Watson smoother with a Mat\'ern-$5/2$ kernel to each complete sex-specific sample. After standardization, $\theta_0(x_0)=\sum_iK_{\bar h}(X_i,x_0)Y_i/\sum_iK_{\bar h}(X_i,x_0)$, where $\bar h=\operatorname{median}_i\|X_i-x_0\|_2$.
This full-data smoother defines the finite-sample evaluation estimand; another reference smoother would define a slightly different target. Appendix~B.4 assesses this smoother choice using alternative bandwidths and a Gaussian kernel.

\begin{figure}[t]
    \centering
    \includegraphics[width=\linewidth, height=0.25\textheight]{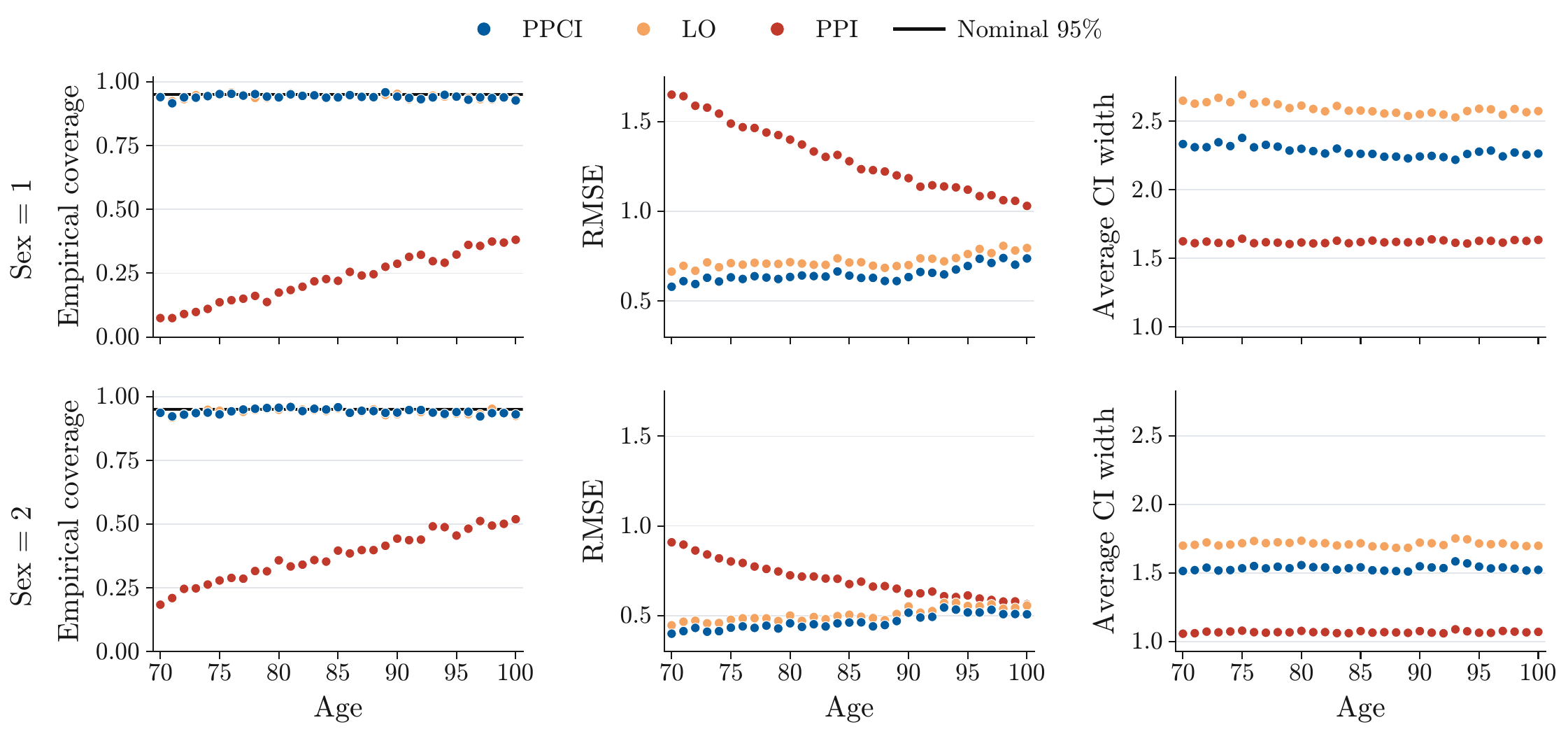}
    \caption{Empirical coverage rates, RMSE, and average widths of nominal $95\%$ confidence intervals for the conditional mean at different (age, sex) target points in the census income data. Results are based on $1000$ replications per target point.}
    \label{fig:census-conditional-mean}
    \Description{Two rows for the sex groups show coverage, RMSE, and interval width across ages 70--100; PPCI tracks labeled-only coverage with generally smaller error and width.}
\end{figure}

We consider every integer age from 70 to 100. At each $x_0=(\text{age},\text{sex})$, exact duplicates are excluded from the sampling pool before drawing $n=300$ labeled and $N=10{,}000$ unlabeled observations. The sex-specific covariate-only configurations are listed in Appendix~B.1. Although age is integer-valued, the kernel treats it as an ordered
covariate and borrows information from neighboring ages.

Figure~\ref{fig:census-conditional-mean} compares PPCI, LO, and population-level PPI. For sex$=1$, PPCI and LO both have mean coverage $0.941$, while PPCI reduces RMSE from $0.722$ to $0.648$ and width from $2.593$ to $2.280$. For sex$=2$, PPCI has coverage $0.941$, RMSE $0.466$, and width $1.536$, compared with $0.939$, $0.505$, and $1.712$ for LO. PPI is nearly insensitive to age because it targets the population mean, whereas PPCI and LO adapt to the local target.

\subsection{BlogFeedback Data}
\label{subsec:blog}
\noindent
We study conditional mean inference in the UCI BlogFeedback data \citep{buza2013feedback}. Each observation is a blog post, and the response is $Y=\log(1+\text{comments within 24 hours})$. The standardized covariate vector contains 280 temporal, activity, and content features. After preprocessing, $48{,}439$ observations remain. We train a LightGBM regressor \citep{ke2017lightgbm} on a training split of $33{,}872$ observations that is disjoint from the $14{,}517$-observation PPCI inference pool; all 50 target points and their outcomes are excluded from predictor fitting. The predictions on the held-out inference pool are denoted by $f(x)$.

Since most covariate profiles occur only once, we define the evaluation target using a Nadaraya--Watson smoother on the complete held-out inference population from which the repeated labeled and unlabeled samples are drawn. We select 50 target points and exclude exact duplicates of each $x_0$ from its sampling pool. At every point, $n=300$, $N=10{,}000$, and the Mat\'ern-$5/2$ kernel is used. Appendix~B.1 gives the BlogFeedback tuning grid and screening constants.
As in the income experiment, the smoother defines the finite-population reference target and is not used by the estimator.

\begin{figure}[t]
    \centering
    \includegraphics[width=\linewidth, height=0.15\textheight]{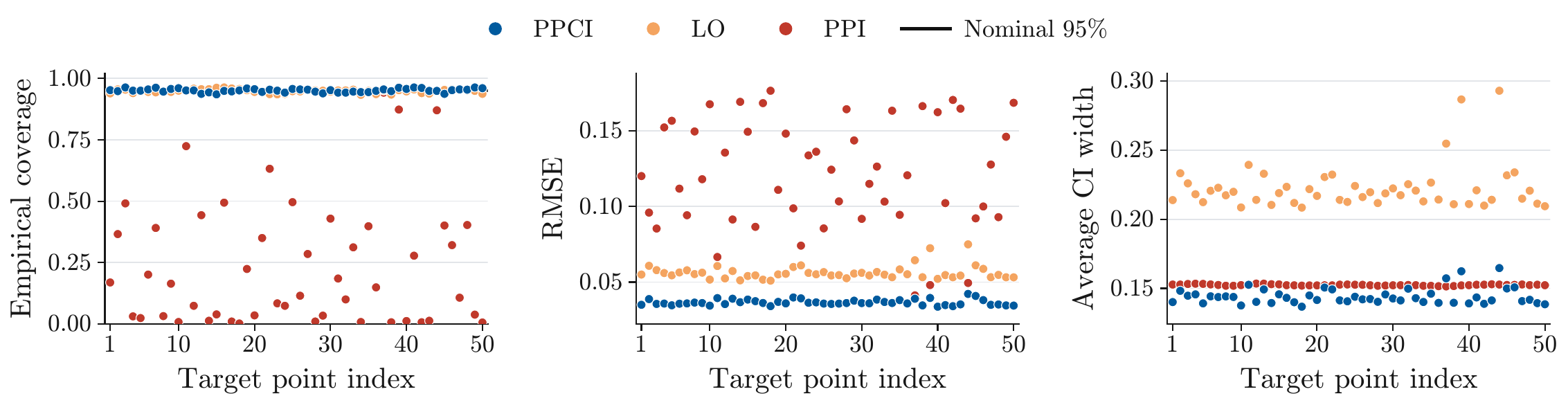}
    \caption{Empirical coverage rates, RMSE, and average widths of nominal $95\%$ confidence intervals for the conditional mean at 50 target points in the BlogFeedback data. Results are based on $1000$ replications per target point.}
    \label{fig:blog-conditional-mean}
    \Description{Coverage, RMSE, and interval-width scatterplots across 50 BlogFeedback target points show PPCI near nominal coverage with lower error and narrower intervals than labeled-only inference.}
\end{figure}

Figure~\ref{fig:blog-conditional-mean} shows that PPCI attains mean coverage $0.951$, RMSE $0.0366$, and width $0.1442$, compared with $0.948$, $0.0563$, and $0.2227$ for LO. PPCI reduces both RMSE and width by about $35\%$ while preserving calibration. Classical PPI again targets a global mean and is included only as an out-of-scope population-level reference for the conditional estimands.

\section{Conclusion}
\label{sec:conclusion}
\noindent
In this paper, we propose a prediction-powered framework for conditional inference that integrates labeled data, unlabeled covariates, and ML predictions. The procedure is explicitly tailored to a fixed target point $x_0$, with localization and uncertainty quantification designed at the pointwise level. We establish theoretical guarantees, including nonasymptotic error bounds, minimax-optimal rates in the abundant-unlabeled regime, and asymptotic normality, and show empirically that the resulting confidence intervals achieve near-nominal coverage while delivering substantial efficiency gains by leveraging ML predictors and unlabeled covariates.
We discuss a few directions for further investigation. PPCI++ provides a practical point-specific adaptation to predictor quality; a complete uniform theory for its estimated coefficient, including multivariate targets and richer local adaptation, remains open. Our current inference theory also relies on smooth estimating functions. Extending it to nonsmooth targets without smoothing approximations is another useful direction.

%\section*{Acknowledgement}
%\noindent
%The author thanks the Editor, Associate Editor, and anonymous reviewers for their invaluable feedback, which helped improve the paper.

%\noindent
%\emph{Conflict of interest}: We have no conflict of interest to disclose.

%\section*{Funding}
%\noindent
\section*{Data Availability Statement}
\noindent
The BlogFeedback data are available in the UCI Machine Learning Repository at \url{https://doi.org/10.24432/C58S3F}. The Census income data are available in the Prediction-Powered Inference (\texttt{ppi\_py}) repository at \url{https://github.com/aangelopoulos/ppi_py}, and can be accessed with the dataset identifier \texttt{census\_income}. The simulated data and code are available at \url{https://github.com/YS-stat/PPCI}.

%\section*{Supplementary material}
%\noindent
%Supplementary material is available online at \emph{Journal of the Royal Statistical Society: Series B}.

\baselineskip=22pt
\bibliographystyle{apa}
\bibliography{PPCI}

\newpage
\appendix
\noindent
The appendices are organized as follows:
\begin{itemize}
    \item Appendix~\ref{app:lcurve_details} gives the finite-dimensional implementation and computational cost of the empirical RKHS localization weights, together with the PPCI++ implementation.
    \item Appendix~\ref{app:additional-experiments} reports tuning
configurations and sensitivity analyses, together with the numerical
comparison of the two-fold and no-split implementations.
    \item Appendix~\ref{app:theory-extensions} discusses the
regularity assumptions, treats expected shortfall with an estimated
quantile, and develops the spectral source extension.
    \item Appendix~\ref{sec:comparisons} compares RKHS localization with alternative approaches, including Nadaraya--Watson localization.
    \item Appendix~\ref{app:ppci-no-split} develops the no-split PPCI
construction and its theoretical guarantees.
    \item Appendix~\ref{app:proofs-of-PPCI-combine} provides the detailed proofs and summarizes the main technical distinctions from classical KRR analysis.
\end{itemize}
Some symbols are reused across appendices to avoid unnecessary
proliferation; unless otherwise stated, each definition is local to
the section or subsection in which it appears.

\section{Implementation Details for Empirical Localization Weights}\label{app:lcurve_details}
\noindent
This appendix records the finite-dimensional identities used to compute the empirical RKHS localization weights in Algorithm~\ref{alg:ppci_twofold}, together with the implementation details for PPCI++. The weight identities are independent of any particular data-driven rule for selecting $\lambda$.

\subsection{Finite-Dimensional Formula}\label{app:weight_finite_formula}
\noindent
We illustrate the formula for the first fold $\cI_1$; the second fold is identical. Let $M_1=|\cI_1|$, let
$\Sigma_t^{(1)}\in\RR^{M_1\times M_1}$ be the Gram matrix on $\{\wt X_u:u\in\cI_1\}$, and let $k_{0,t}^{(1)}$ be the vector with entries $K(\wt X_u,x_0)$. Define
$\xi_\lambda^{(1)}=(\Sigma_t^{(1)}+M_1\lambda I_{M_1})^{-1}k_{0,t}^{(1)}$.
The empirical weight used in the PPCI estimating equation is $\wh w^{(1)}_{x_0,\lambda}=(\wh T_K^{(1)}+\lambda I)^{-1}K_{x_0}$,
where $\wh T_K^{(1)}=M_1^{-1}\sum_{u\in\cI_1}K_{\wt X_u}\otimes K_{\wt X_u}$. By the Woodbury identity,
\begin{equation}\label{eq:woodbury-weight-app}
\wh w^{(1)}_{x_0,\lambda}(x)
=
\frac{K(x,x_0)-K(x,\wt X_{\cI_1})^\top\xi_\lambda^{(1)}}{\lambda}.
\end{equation}
This differs from the noiseless KRR fitted function $K(x,\wt X_{\cI_1})^\top\xi_\lambda^{(1)}$, which solves a different normal equation with right-hand side $\wh T_K^{(1)}K_{x_0}$. For points in the same fold, \eqref{eq:woodbury-weight-app} reduces to $\wh w^{(1)}_{x_0,\lambda}(\wt X_u)=M_1(\xi_\lambda^{(1)})_u$ for $u\in\cI_1$,
since $(\Sigma_t^{(1)}+M_1\lambda I)\xi_\lambda^{(1)}=k_{0,t}^{(1)}$. In cross-fitting, however, the weight learned from $\cI_1$ is evaluated on $\cI_2$ and on the labeled covariates, so the full formula \eqref{eq:woodbury-weight-app} must be used.

\subsection{Computational Cost}\label{app:computation}
\noindent
For a fixed $\lambda$, factorizing
$\Sigma_t^{(1)}+M_1\lambda I$ and computing
$\xi_\lambda^{(1)}$ from scratch costs $O(M_1^3)$ after constructing
the Gram matrix. When a grid of candidate regularization values is
evaluated, our implementation instead computes the eigendecomposition
$\Sigma_t^{(1)}=UDU^\top$, where
$D=\mathrm{diag}(d_1,\ldots,d_{M_1})$, and precomputes
$\widetilde k=U^\top k_{0,t}^{(1)}$. Then
$\xi_\lambda^{(1)}
=U(D+M_1\lambda I)^{-1}\widetilde k$.
The eigendecomposition is performed only once for each bandwidth, and
each additional $\lambda$ requires $O(M_1^2)$ operations to
reconstruct $\xi_\lambda^{(1)}$. Thus, evaluating the complete
$\lambda$ grid costs
$O(M_1^3+|\mathcal G_\lambda|M_1^2)$ rather than
$O(|\mathcal G_\lambda|M_1^3)$. Scalar diagnostics alone could
alternatively be evaluated in $O(M_1)$ operations per $\lambda$ if
expressed entirely in the spectral basis.

Once $\xi_\lambda^{(1)}$ is available, the vector of weights on any
query set $Q$ is computed from
\[
\wh w^{(1)}_{x_0,\lambda}(Q)
=
\frac{K(Q,x_0)-K(Q,\wt X_{\cI_1})\xi_\lambda^{(1)}}{\lambda}.
\]
\subsection{PPCI++ Implementation}
\label{app:ppci-plus-implementation}
\noindent
PPCI++ follows Algorithm~\ref{alg:ppci_twofold} for
construction and tuning of the RKHS localization weights and for the
final Wald interval; only the prediction coefficient and its induced
contributions change. Split the labeled sample into folds
$\mathcal L_k$ of sizes $n_k$, with five folds in our experiments, and
let $\pi_k=n_k/n$. For each fold, use the complementary labeled
observations to compute a labeled-only pilot
$\widetilde\theta^{(-k)}(x_0)$ and the empirical covariance and
variances in~\eqref{eq:omega-star-ppci-plus}, evaluated at this pilot.
We set
\[
\widehat\omega_k(x_0)
=
\Pi_{[0,1]}
\left\{
\frac{\widehat{\Cov}_{-k}(A,B)}
{\widehat{\Var}_{-k}(B)
 +(n/N)\widehat{\Var}_U(C)+\rho_k}
\right\},
\]
where the training-fold scale-adaptive ridge is
\[
\rho_k
=
\varepsilon_\omega
\max\left\{
\widehat{\Var}_{-k}(A),
\widehat{\Var}_{-k}(B),
\frac{n}{N}\widehat{\Var}_U(C),
s_{\min}
\right\}.
\]
Here $\varepsilon_\omega=10^{-6}$ and $s_{\min}=10^{-12}$ are the
fixed numerical safeguards used in our experiments. The coefficient
$\widehat\omega_k(x_0)$ is applied only to observations in
$\mathcal L_k$, and the coefficient for the unlabeled contribution is
$\overline\omega=\sum_k\pi_k\widehat\omega_k$.

For the conditional mean, let $\widehat w_L$ denote the averaged
labeled-data weight from Algorithm~\ref{alg:ppci_twofold} and
$\widehat w_U$ the corresponding out-of-fold unlabeled weight. The
PPCI++ estimate is
\[
\widehat\theta_{++}(x_0)
=
\frac{
\overline\omega\,\mathbb P_N(\widehat w_Uf)
+\sum_k\pi_k\mathbb P_{\mathcal L_k}
 [\widehat w_L\{Y-\widehat\omega_k f\}]
}{
\overline\omega\,\mathbb P_N(\widehat w_U)
+\sum_k\pi_k(1-\widehat\omega_k)
 \mathbb P_{\mathcal L_k}(\widehat w_L)
}.
\]
The empirical Jacobian is the negative of this denominator; denote it
by $\widehat J_{++}(x_0)$. With
$Z_{ki}=\widehat w_L(X_i)
[\{Y_i-\widehat\theta_{++}\}
-\widehat\omega_k\{f(X_i)-\widehat\theta_{++}\}]$ and
$U_u=\widehat w_U(\widetilde X_u)
\{f(\widetilde X_u)-\widehat\theta_{++}\}$, the variance estimator is
\[
\widehat V_{++}(x_0)
=
\sum_k\frac{\pi_k^2}{n_k}
\widehat{\Var}_{\mathcal L_k}(Z_{ki})
+
\frac{\overline\omega^2}{N}
\widehat{\Var}_U(U_u).
\]
The corresponding standard error is
$\widehat V_{++}(x_0)^{1/2}/|\widehat J_{++}(x_0)|$.
For a general smooth estimating function, the same construction
replaces the mean residuals by the corresponding foldwise
contributions and differentiates the resulting moment for the
Jacobian. Thus, $\widehat\omega_k=0$ gives the labeled-only moment,
whereas $\widehat\omega_k=1$ gives the PPCI moment, without repeating
the unchanged steps of Algorithm~\ref{alg:ppci_twofold}.

\section{Experimental Settings and Additional Analyses}\label{app:additional-experiments}
\noindent
This section reports the experiment-specific configurations, the two-fold versus no-split comparison, and tuning and reference-smoother sensitivity analyses.

\subsection{Sample Sizes and Tuning Configurations}
\label{app:experiment-configurations}
\noindent
Table~\ref{tab:experiment-config} collects the sample sizes and all dataset-specific constants used with the common tuning rule in Section~\ref{sec:tuning_twofold}. The notation $L:[a,b]_{\log}$ denotes $L$ logarithmically spaced factors in $\mathcal C_\lambda$.

\begin{table}[t]
\centering
\caption{Sample sizes and covariate-only tuning configurations. The final column reports $(c_{\rm op},c_{\rm loc},c_{\rm bias})$.}
\label{tab:experiment-config}
\footnotesize
\setlength{\tabcolsep}{4pt}
\renewcommand{\arraystretch}{1.08}
\begin{tabular}{@{}lcccc@{}}
\toprule
Experiment & $(n,N,\#x_0)$ & $\mathcal C_h$ & $\mathcal C_\lambda$ & Screen constants \\
\midrule
Simulation & $(400,10{,}000,343)$ & $\{0.8,1.0,1.15,1.2\}$ & $35:[0.02,60]_{\log}$ & $(12,4,0.10)$ \\
Predictor quality & $(500,10{,}000,8)$ & $\{0.8,1.0,1.15,1.2\}$ & $35:[0.02,60]_{\log}$ & $(12,4,0.10)$ \\
Income (sex$=1$) & $(300,10{,}000,31)$ & $\{1.0,1.2,1.4\}$ & $41:[0.1,1000]_{\log}$ & $(12,4,22)$ \\
Income (sex$=2$) & $(300,10{,}000,31)$ & $\{1.0,1.2,1.4\}$ & $41:[0.1,1000]_{\log}$ & $(10,3,15)$ \\
BlogFeedback & $(300,10{,}000,50)$ & $\{0.8,0.9,1.0\}$ & $81:[0.1,10{,}000]_{\log}$ & $(12,4,300)$ \\
\bottomrule
\end{tabular}
\end{table}

\subsection{Two-fold versus No-split Comparison}
\label{app:split-comparison}
\noindent
We compare the two-fold implementation in Algorithm~\ref{alg:ppci_twofold} with a no-split implementation that constructs one RKHS localization weight from all pooled $n+N$ labeled and unlabeled covariates and evaluates the estimating equation on that shared design. The comparison uses the simulation design with $n=200$, $N=10{,}000$, five diagonal target points in $[0.70,0.85]^3$, and 50 paired replications per target point, for a total of 250 paired runs. Both implementations use the same tuning grids and constants, and are evaluated sequentially on the same paired data and the same NVIDIA RTX A6000 GPU.

\begin{table}[t]
\centering
\small
\setlength{\tabcolsep}{5pt}
\caption{Comparison of the two-fold and no-split PPCI implementations. Wall time and peak GPU memory are averages over the 250 paired runs; parenthetical values are ratios relative to the two-fold implementation.}
\label{tab:split-comparison}
\begin{tabular}{lrrrrrr}
\toprule
Implementation & Coverage & Bias & RMSE & Width & Wall time (s) & Peak memory (MiB) \\
\midrule
Two-fold & 0.964 & 0.0467 & 0.3465 & 1.3506 & 11.35 (1.00) & 2,680 (1.00) \\
No-split & 0.964 & 0.0472 & 0.3469 & 1.3524 & 37.15 (3.27) & 7,855 (2.93) \\
\bottomrule
\end{tabular}
\end{table}

The two implementations have essentially identical statistical performance: all 250 PPCI coverage indicators agree, and the changes in bias, RMSE, and interval width are negligible. The labeled-only estimator likewise has coverage $0.948$ under both implementations. In contrast, no-split is $3.27$ times slower and uses $2.93$ times as much peak GPU memory. This reflects the larger operator system formed from the full unlabeled sample. Thus, in this experiment, reusing the full sample provides no detectable accuracy gain, whereas two-fold cross-fitting preserves the independence structure used by the theory and substantially reduces computation.

\subsection{Tuning Sensitivity}
\label{app:tuning-sensitivity}
\noindent
We vary the three tuning choices that most directly affect the
selected localization weight: the bias threshold
$c_{\rm bias}$, the bandwidth cap, and the regularization-factor
envelope. Here $\mathcal C_\lambda$ denotes the factors $c$ in
$\lambda=c/[n\log\log(n+e^e)]$, while
$c_{\rm op}=12$ and $c_{\rm loc}=4$ remain fixed. We use the more
demanding simulation setting with $n=200$, $N=10{,}000$, $q=0.9$,
and the same $50\times20$ clustered replication design as in the
main experiments. The $c_{\rm bias}$ and bandwidth comparisons use
the same 125 target points and paired samples; the
regularization comparison uses 64 paired target points.
Table~\ref{tab:tuning-sensitivity} reports
equal-target-point averages, with minimum pointwise coverage
in parentheses.

\begin{table}[t]
\centering
\small
\setlength{\tabcolsep}{4pt}
\caption{Sensitivity to the bias threshold, bandwidth cap, and
regularization-factor envelope. Coverage is the mean (minimum across
target points); fallback is the proportion of tuning
decisions invoking the least-violation rule.}
\label{tab:tuning-sensitivity}
\begin{tabular}{llrrrrr}
\toprule
Component & Configuration & Coverage & Bias & RMSE & Width & Fallback \\
\midrule
$c_{\rm bias}$ & $0.06$ & 0.952 (0.930) & 0.0124 & 0.3876 & 1.4510 & 1.000 \\
                  & $0.08$ & 0.952 (0.931) & 0.0126 & 0.3862 & 1.4465 & 0.897 \\
                  & $0.10$ & 0.950 (0.925) & 0.0222 & 0.3629 & 1.3640 & 0.000 \\
                  & $0.12$ & 0.945 (0.909) & 0.0402 & 0.3392 & 1.2721 & 0.000 \\
                  & $0.14$ & 0.939 (0.891) & 0.0577 & 0.3245 & 1.2076 & 0.000 \\
\addlinespace
Bandwidth cap & $\{0.8,1.0,1.15,1.2\}h_0$ & 0.950 (0.925) & 0.0222 & 0.3629 & 1.3640 & 0.000 \\
              & $\{0.8,1.0,1.15,1.2,1.3,1.4\}h_0$ & 0.943 (0.897) & 0.0455 & 0.3247 & 1.2171 & 0.000 \\
\addlinespace
$\mathcal C_\lambda$ & $[0.02,60]$ & 0.949 (0.933) & 0.0220 & 0.3648 & 1.3651 & 0.000 \\
                   & $[0.01,120]$ & 0.949 (0.933) & 0.0234 & 0.3628 & 1.3572 & 0.000 \\
\bottomrule
\end{tabular}
\end{table}

The $c_{\rm bias}$ comparison shows the expected bias--variance
tradeoff. Values $0.06$ and $0.08$ are conservative but invoke
fallback in all or most tuning decisions. The prespecified
rule selects the smallest candidate with fallback rate at most
$5\%$, yielding $c_{\rm bias}=0.10$. Coverage, bias, RMSE, and width
are sensitivity diagnostics and do not enter foldwise selection.
Larger values shorten intervals but increase localization bias and
reduce mean and minimum coverage.

Extending the bandwidth grid to $1.4h_0$ reduces RMSE and width but
selects $1.4h_0$ throughout, approximately doubles the bias, and
lowers minimum coverage from $0.925$ to $0.897$. Thus, the $1.2h_0$
cap enforces locality rather than merely reducing computation.
Expanding $\mathcal C_\lambda$ from $[0.02,60]$ to $[0.01,120]$ has
negligible effects. The selected factors remain between $0.0235$ and
$0.0446$, with no endpoint selections or fallback, supporting the
more compact envelope $[0.02,60]$.
\subsection{Reference-Smoother Sensitivity}
\label{app:reference-smoother-sensitivity}
\noindent
The full-data smoothers for Income and BlogFeedback define finite-population reference targets used only for evaluation. We assess sensitivity by holding the PPCI and LO estimates, intervals, sampling design, and tuning fixed and changing only the reference used to score coverage and RMSE. We compare the Mat\'ern-$5/2$ Nadaraya--Watson reference at $0.8h_0$, $h_0$, and $1.2h_0$ with a Gaussian Nadaraya--Watson reference whose covariate-only bandwidth matches the local effective sample size at $h_0$. Writing $K^{\mathrm{G}}$ and $K^{\mathrm{M}}$ for the Gaussian and Mat\'ern-$5/2$ kernels, respectively, we choose $\widehat h_{\mathrm{G}}$ to satisfy
\[
\frac{\left\{\sum_i K^{\mathrm{G}}_{\widehat h_{\mathrm{G}}}(X_i,x_0)\right\}^2}{\sum_i\left\{K^{\mathrm{G}}_{\widehat h_{\mathrm{G}}}(X_i,x_0)\right\}^2}
=
\frac{\left\{\sum_i K^{\mathrm{M}}_{h_0}(X_i,x_0)\right\}^2}{\sum_i\left\{K^{\mathrm{M}}_{h_0}(X_i,x_0)\right\}^2}.
\]
All intervals retain the default tuning used in the main experiments; if a different smoother were adopted as the target definition, dedicated covariate-only tuning could further improve its finite-sample performance.

Reference shifts use all 62 Income and 50 BlogFeedback target points. Coverage uses eight prespecified targets per data set with $n=300$, $N=10{,}000$, and $15\times20=300$ intervals per target and method. RMSE and width use the formal 1000-replication summaries at all targets. Table~\ref{tab:reference-smoother-sensitivity} reports equal-target-point averages.

\begin{table}[t]
\centering
\footnotesize
\setlength{\tabcolsep}{3pt}
\caption{Sensitivity to the full-data reference smoother. Shift is the mean absolute target change from the main Mat\'ern-$5/2$ reference at $h_0$. Coverage uses eight targets and 300 intervals per target and method; RMSE and width use the formal 1000-replication summaries at all targets.}
\label{tab:reference-smoother-sensitivity}
\begin{tabular}{llrrrrrrr}
\toprule
Data & Reference & Shift & \multicolumn{2}{c}{Coverage} & \multicolumn{2}{c}{RMSE} & \multicolumn{2}{c}{Width} \\
\cmidrule(lr){4-5}\cmidrule(lr){6-7}\cmidrule(lr){8-9}
 & & & PPCI & LO & PPCI & LO & PPCI & LO \\
\midrule
Income & Mat\'ern, $0.8h_0$ & 0.1976 & 0.9329 & 0.9329 & 0.5181 & 0.5782 & 1.9079 & 2.1528 \\
       & Mat\'ern, $h_0$    & 0      & 0.9429 & 0.9375 & 0.5571 & 0.6134 & 1.9079 & 2.1528 \\
       & Mat\'ern, $1.2h_0$ & 0.1681 & 0.9233 & 0.9242 & 0.6404 & 0.6898 & 1.9079 & 2.1528 \\
       & Gaussian, ESS matched & 0.0630 & 0.9400 & 0.9408 & 0.5364 & 0.5946 & 1.9079 & 2.1528 \\
\addlinespace
BlogFeedback & Mat\'ern, $0.8h_0$ & 0.0200 & 0.9175 & 0.9404 & 0.0425 & 0.0603 & 0.1442 & 0.2227 \\
             & Mat\'ern, $h_0$    & 0      & 0.9533 & 0.9404 & 0.0366 & 0.0563 & 0.1442 & 0.2227 \\
             & Mat\'ern, $1.2h_0$ & 0.0159 & 0.9213 & 0.9288 & 0.0397 & 0.0583 & 0.1442 & 0.2227 \\
             & Gaussian, ESS matched & 0.0016 & 0.9554 & 0.9433 & 0.0367 & 0.0564 & 0.1442 & 0.2227 \\
\bottomrule
\end{tabular}
\end{table}

The ESS-matched Gaussian reference leaves coverage and RMSE nearly unchanged. The broader $\pm20\%$ bandwidth perturbations affect absolute coverage, especially for BlogFeedback, because they define different finite-population targets. Across all references, PPCI retains lower RMSE than LO and intervals that are $11.4\%$ shorter for Income and $35.2\%$ shorter for BlogFeedback. Width is constant across rows, as rescoring does not alter the intervals. Thus the conclusions are insensitive to kernel family, while bandwidth sensitivity reflects the evaluation estimand's smoothing scale.

\section{Assumption Verification and Theoretical Extensions}
\label{app:theory-extensions}
\noindent
This section supplements the theory in Section~\ref{sec:theory}. We first verify the main regularity assumptions for the conditional functionals considered in this paper, then analyze the effect of estimating the conditional quantile used for expected shortfall, and finally extend the RKHS membership condition to a spectral source class.

\subsection{Verification of Assumptions in Section~\ref{sec:theory}}
\label{app:appendix_assumptions_verification}
\noindent
Assumptions~\ref{ass:B_bound_predictor} and \ref{ass:deriv-bdd-jac}:
First, consider the residual term $\ell(y;\theta)-\ell(f(x);\theta)$. In each running example, the same $\theta$-dependent term appears in both evaluations of the estimating function and cancels in their difference. For the conditional mean and binary log-odds, the difference is $y-f(x)$. For conditional expected shortfall, it becomes $y^\ast-f^\ast(x)$, where $f^\ast(x)$ predicts $Y^\ast$. In all cases, $G_{Y-f}(\theta)\equiv 0$, and $B_{Y-f}(\theta)<\infty$ provided responses and predictions are bounded.
Second, for the predictor term $\ell(f(x);\theta)$, $\partial_\theta \ell(f(x);\theta)$ equals either $-1$ or $-p(\theta)(1-p(\theta))$, where $p(\theta)=(1+e^{-\theta})^{-1}$. Since $\Theta_0$ is compact, these derivatives are uniformly bounded. If the predictions are bounded, then $B_f(\theta)$ is also finite. Hence both assumptions hold. The interval $\Theta_0$ can be chosen with the required positive margin whenever the scalar target lies in the interior of its parameter space. For binary log-odds, this corresponds to keeping the conditional success probability away from $0$ and $1$.

Assumption~\ref{ass:J_lambda_bound}:
For the conditional mean and expected shortfall, $\partial_\theta\ell(\cdot;\theta)=-1$, so $|J_\lambda(x_0;\theta)|=\EE[w_{x_0,\lambda}(X)]=1_\lambda(x_0)$, where $1_\lambda$ is the Tikhonov-regularized approximation of the constant function $1$ in \eqref{eqn:Tikhonov-regu_twofold}. Since $1\in H^m(\cX)$ for a compact $\cX$, standard regularization theory for integral operators~\citep{caponnetto2007optimal} implies that $1_\lambda(x_0)$ is strictly bounded away from zero and infinity for sufficiently small $\lambda>0$. For the binary log-odds, $\partial_\theta\ell(\cdot;\theta)=-p(\theta)\{1-p(\theta)\}$ with $p(\theta)=(1+e^{-\theta})^{-1}$, yielding $|J_\lambda(x_0;\theta)|=p(\theta)\{1-p(\theta)\}\,1_\lambda(x_0)$. If $\PP(Y=1| X=x_0)$ is bounded away from $0$ and $1$, then $p(\theta)\{1-p(\theta)\}$ is bounded away from zero in a neighborhood of $\theta_0(x_0)$, verifying the assumption.
Compactness of $\Theta_0$ makes these bounds uniform in $\theta$, and the preceding approximation is uniform for all sufficiently small $\lambda$ in the admissible range.

Assumption~\ref{ass:eigenfunctions_bound} has two logically distinct parts. For the conditional mean and expected shortfall, $\eta(x;\theta)$ equals a smooth conditional mean minus the constant $\theta$; for binary log-odds it has the analogous conditional-probability form. Thus, if the corresponding conditional function belongs to $H^m(\cX)$, then $\sup_{\theta\in\Theta_0}\|\eta(\cdot;\theta)\|_{\cH}<\infty$ since $\Theta_0$ is compact and constants belong to $H^m(\cX)$. This remains a smoothness condition on the unknown data-generating function. The eigenvalue decay and bounded-eigenfunction clauses are structural kernel--design conditions for the chosen Sobolev-equivalent kernel.

Assumption~\ref{ass:homo_res}:
For conditional mean, it reduces to $\Var(Y| X)\ge \underline{\sigma}^2$.  For binary log-odds, it becomes $\PP(Y=1| X)\PP(Y=0| X)\ge \underline{\sigma}^2$.
For conditional expected shortfall, it translates to $\Var(Y^\ast| X)\ge \underline{\sigma}^2$, meaning that the lower-tail distribution of $Y$ below its $\tau$-quantile remains non-degenerate.

\subsection{Conditional Expected Shortfall with an Estimated Quantile}
\label{app:es-estimated-quantile}
\noindent
In practice, $\mathrm{Quant}_\tau$ is estimated from labeled data and
used to construct the surrogate response. Let
$\psi(y,q)=q+\tau^{-1}(y-q)\mathbbm{1}\{y\le q\}$, so that
$Y^\ast=\psi\{Y,\mathrm{Quant}_\tau(X)\}$. For
$h_x(q)=\EE\{\psi(Y,q)|X=x\}$,
\[
h_x'(q)=1-\tau^{-1}F_{Y|X}(q|x),
\qquad
h_x'\{\mathrm{Quant}_\tau(x)\}=0,
\]
and
\[
h_x(q)-h_x\{\mathrm{Quant}_\tau(x)\}
=
-\tau^{-1}\int_{\mathrm{Quant}_\tau(x)}^q
\{F_{Y|X}(u|x)-\tau\}\,du.
\]
Thus, if the conditional density is bounded near
$\mathrm{Quant}_\tau(x)$ and
$\widehat{\mathrm{Quant}}_\tau(x)$ lies in that neighborhood with
probability tending to one, the substitution produces conditional
bias of order
$\{\widehat{\mathrm{Quant}}_\tau(x)
-\mathrm{Quant}_\tau(x)\}^2$ rather than a first-order error.

We remove the resulting sampling dependence by cross-fitting over the
labeled observations. Specifically,
$\widehat{\mathrm{Quant}}_\tau^{(-k)}$ is trained outside labeled fold
$k$, and its surrogate responses are evaluated on fold $k$; the folds
$\cI_1$ and $\cI_2$ in Algorithm~\ref{alg:ppci_twofold} continue to
cross-fit the unlabeled-data localization weights. Writing
$\Delta_{\mathrm{Quant}}
=\widehat{\mathrm{Quant}}_\tau-\mathrm{Quant}_\tau$, a sufficient
condition for retaining oracle first-order inference is
\[
\EE\{|w_{x_0,\lambda}(X)|
\Delta_{\mathrm{Quant}}(X)^2
|\widehat{\mathrm{Quant}}_\tau\}
+
\left[
n^{-1}\EE\{w_{x_0,\lambda}(X)^2
\Delta_{\mathrm{Quant}}(X)^2
|\widehat{\mathrm{Quant}}_\tau\}
\right]^{1/2}
=
o_p\{\sqrt{V(x_0)}\}.
\]
This requirement is local and weighted. With bounded weights and
root-$n$ scaling, it follows from
$\|\Delta_{\mathrm{Quant}}\|_{L^2}=o_p(n^{-1/4})$.
Vanishing weighted bias and fluctuation terms suffice for consistency,
whereas the displayed condition preserves the oracle limiting
distribution, variance estimator, and confidence interval. Otherwise,
the quadratic remainder may affect the first-order limit and cause
undercoverage unless it is explicitly accounted for.
\subsection{Beyond RKHS}
\label{sec:spectral-source-src}
\noindent
The RKHS in Assumption~\ref{ass:eigenfunctions_bound} is needed to construct the localization weight in \eqref{eq:rkhs-localization-pop-x0_twofold}, but it need not be a correctly specified model for the conditional moment. Let $L_K:L^2(\rho_X)\to L^2(\rho_X)$ be the integral operator $(L_Kq)(x)=\int K(x,x')q(x')\,d\rho_X(x')$, with eigenpairs $(\mu_j,\phi_j)$ from Assumption~\ref{ass:eigenfunctions_bound}. This is the classical spectral source-condition framework for Tikhonov regularization and kernel learning~\citep{engl1996regularization,bauer2007regularization,smale2007learning}.

\begin{assumption}[Spectral source condition]
\label{ass:spectral-source-src}
Retain the kernel clauses of Assumption~\ref{ass:eigenfunctions_bound}, but replace $\eta(\cdot;\theta)\in\cH$ by
\[
\eta(\cdot;\theta)=L_K^r q_\theta,
\qquad
r\in\left(\frac{d}{4m},\frac12\right],
\qquad
\sup_{\theta\in\Theta_0}\|q_\theta\|_{L^2(\rho_X)}<\infty,
\]
where $q_\theta$ is taken in the closed span of $\{\phi_j\}_{j\ge1}$.
\end{assumption}

Writing $q_{\theta,j}=\langle q_\theta,\phi_j\rangle_{L^2(\rho_X)}$, the series $\sum_j\mu_j^r q_{\theta,j}\phi_j$ converges absolutely and uniformly since $\sup_x\sum_j\mu_j^{2r}\phi_j(x)^2<\infty$ when $r>d/(4m)$ under the Sobolev eigenvalue decay~\citep{fischer2020sobolev}. It therefore defines the pointwise representative used in $\eta(x_0;\theta_0(x_0))=0$. At $r=1/2$, $\|q_\theta\|_{L^2(\rho_X)}=\|\eta(\cdot;\theta)\|_{\cH}$ under the usual spectral identification; for $r<1/2$, the source class permits $\eta(\cdot;\theta)\notin\cH$~\citep{bauer2007regularization,caponnetto2007optimal}. At the endpoint $r=d/(4m)$, the present assumptions do not ensure point evaluation or the uniform bound needed below; logarithmic refinements require stronger local spectral conditions.

The Tikhonov characterization \eqref{eqn:Tikhonov-regu_twofold} remains valid for an $L^2(\rho_X)$ target~\citep{smale2007learning}. Moreover,
\[
\eta_\lambda(x_0;\theta)-\eta(x_0;\theta)
=-\lambda\sum_{j\ge1}\frac{\mu_j^r}{\mu_j+\lambda}q_{\theta,j}\phi_j(x_0),
\]
so, with
\[
B_r^2(x_0;\lambda)
:=\lambda^2\sum_{j\ge1}\frac{\mu_j^{2r}}{(\mu_j+\lambda)^2}\phi_j(x_0)^2,
\]
we have $|\eta_\lambda(x_0;\theta)-\eta(x_0;\theta)|\le \|q_\theta\|_{L^2(\rho_X)}B_r(x_0;\lambda)$ and $B_r(x_0;\lambda)=O\{\lambda^{r-d/(4m)}\}$. For $r=1/2$, $B_{1/2}^2(x_0;\lambda)=\lambda\{D(x_0;\lambda)-Q(x_0;\lambda)\}$, recovering \eqref{eq:bias-bound}.

Only one other proof input changes. In Step~4 of the proof of Lemma~\ref{lem:I3-clean}, the mean embedding $\mu(\theta)=\EE[K_X\eta(X;\theta)]$ remains an element of $\cH$, but the identity $\mu(\theta)=T_K\eta(\cdot;\theta)$ is no longer available. Instead,
\[
b_\lambda(x;\theta)
:=\{(T_K+\lambda I)^{-1}\mu(\theta)\}(x)
=\sum_{j\ge1}\frac{\mu_j^{r+1}}{\mu_j+\lambda}q_{\theta,j}\phi_j(x),
\]
and hence $\|b_\lambda(\cdot;\theta)\|_\infty\le C_{r,K}\|q_\theta\|_{L^2(\rho_X)}$, where $C_{r,K}:=\sup_x\{\sum_j\mu_j^{2r}\phi_j(x)^2\}^{1/2}<\infty$. Steps~5--8 of that proof use only this sup-norm bound, so Lemma~\ref{lem:I23_split_bounds_v2} keeps the order $\sqrt{D(x_0;\lambda)\log N/N}$ with $\kappa\|\eta(\cdot;\theta)\|_{\cH}$ replaced by $C_{r,K}\|q_\theta\|_{L^2(\rho_X)}$. The $\cH$-norm part of \eqref{eqn:b_bound} is not uniform over this source class when $r<1/2$, but it is not used in the subsequent Bernstein bounds.

Consequently, under Assumptions~\ref{ass:B_bound_predictor}--\ref{ass:J_lambda_bound} and Assumption~\ref{ass:spectral-source-src}, with the same regularization condition and probability as in Theorem~\ref{thm:thetahat-theta0-bound-twofold}, the bound \eqref{eq:thetahat-theta0-bound-clean-twofold} holds after the two replacements
\begin{align*}
\kappa\|\eta(\cdot;\theta_\lambda(x_0))\|_{\cH}
\longmapsto
C_{r,K}\|q_{\theta_\lambda(x_0)}\|_{L^2(\rho_X)},
\\
\|\eta(\cdot;\theta_0(x_0))\|_{\cH}
\sqrt{\lambda\{D(x_0;\lambda)-Q(x_0;\lambda)\}}
\longmapsto
\|q_{\theta_0(x_0)}\|_{L^2(\rho_X)}B_r(x_0;\lambda).
\end{align*}
Let $\sN_r$ denote $\sN$ after the first replacement. Combining Proposition~\ref{prop:Dx0-order-H} with the order of $B_r$, the standard source--capacity balancing calculation~\citep{caponnetto2007optimal} gives $\lambda\asymp\{\sN_r/\|q_{\theta_0(x_0)}\|_{L^2(\rho_X)}\}^{1/r}$ and
\[
|\wh\theta(x_0)-\theta_0(x_0)|
\lesssim
\frac{1}{c_J}
\|q_{\theta_0(x_0)}\|_{L^2(\rho_X)}^{d/(4mr)}
\sN_r^{\,1-d/(4mr)}.
\]
This rate applies only when the balanced choice also satisfies the full regularization condition in Theorem~\ref{thm:thetahat-theta0-bound-twofold}; $r>d/(4m)$ alone does not guarantee that condition. At $r=1/2$, the balance and both exponents reduce to \eqref{eq:upper_bound_sobolev_rate_twofold}.

For inference, retain the first and third conditions of \eqref{eq:cond-clean-twofold} and replace $\lambda=o(n^{-1})$ by
\[
\frac{\sqrt n\,\|q_{\theta_0(x_0)}\|_{L^2(\rho_X)}B_r(x_0;\lambda)}
{\sqrt{Q(x_0;\lambda)}}\longrightarrow0.
\]
Under the Sobolev orders, a sufficient form is $\sqrt n\lambda^r\to0$, equivalently $\lambda=o(n^{-1/(2r)})$, and the conclusions of Theorem~\ref{thm:thetahat-asymp-twofold} and Corollary~\ref{cor:CI-coverage-twofold} continue to hold. The third condition remains $\lambda\gg\max\{(\log N/N)^{m/d},\,((\log(n+N))^2/n)^{2m/d}\}$. Thus, if $N=n^{r_1}$ and $\lambda=n^{r_2}$, then, ignoring logarithms, the full window is $r_1>1$ and $-(m/d)\min\{r_1,2\}<r_2<-1/(2r)$. It is nonempty under $r>d/(4m)$ when $r_1\ge2$; when $1<r_1<2$, the stronger condition $r>d/(2mr_1)$ is required.

The PPCI estimator and confidence interval are unchanged for a prescribed $(h,\lambda)$. Two scope limitations remain: Theorem~\ref{thm:minimax-pointwise} proves optimality only over the RKHS class, so the rate above is only an upper rate for $r<1/2$; and the tuning rule in Section~\ref{sec:tuning_twofold} is calibrated to $B_{1/2}^2=\lambda(D-Q)$ and the grid $\lambda\asymp\{n\log\log(n+e^e)\}^{-1}$, so a uniform guarantee over $r<1/2$ would require a source-aware bias certificate and a smaller-$\lambda$ grid.

\section{Comparisons with Alternative Methods}\label{sec:comparisons}
\noindent
In this section, we compare the proposed PPCI with some related approaches. We highlight the challenges of conditional inference and clarify the specific advantages of PPCI.
\subsection{Nadaraya--Watson PPCI and RKHS Localization}
\label{app:nw-localization}
\noindent
This subsection gives the complete two-fold NW-PPCI construction, establishes its smoothing-bias and efficiency properties, develops the variational characterization of the RKHS localization weight, and reports numerical comparisons that clarify the different bias--variance mechanisms of the two localization schemes. Relatedly, \citet{gu2024local} construct prediction-powered estimators of a conditional mean and its gradient using local multivariable regression. Motivated by this local-smoothing perspective, we consider the following NW-PPCI benchmark.

\subsubsection{Two-fold NW-PPCI construction}
\noindent
Let $K_h^{\rm NW}(x,x_0)=h^{-d}\mathcal K\left((x-x_0)/h\right)$,
where $\mathcal K:\RR^d\to[0,\infty)$ is bounded and integrable, satisfies $\mathcal K(u)=\mathcal K(-u)$, $\int_{\RR^d}\mathcal K(u)\,du=1$ and $\int_{\RR^d}u\mathcal K(u)\,du=0$, and has finite second-moment matrix
$M_{\mathcal K}=\int_{\RR^d}uu^\top\mathcal K(u)\,du$.
The population normalized NW weight is
\[
w_h^{\rm NW}(x;x_0)
=
\frac{K_h^{\rm NW}(x,x_0)}
{\EE\{K_h^{\rm NW}(X,x_0)\}},
\]
and the corresponding localized moment is $\eta_h^{\rm NW}(x_0;\theta)=\EE\{w_h^{\rm NW}(X;x_0)\ell(Y;\theta)\}$.
Let $\theta_h^{\rm NW}(x_0)$ denote the locally unique solution to
$\eta_h^{\rm NW}(x_0;\theta)=0$. For fixed $h$, this is generally a smoothed target rather than the pointwise parameter $\theta_0(x_0)$.

Split the unlabeled indices into two equal-sized folds $\cI_1$ and $\cI_2$. For $m\in\{1,2\}$, estimate the normalizing denominator using only fold $\cI_m$:
\[
\widehat\mu_{m,h}(x_0)
=
\frac{1}{|\cI_m|}
\sum_{u\in\cI_m}
K_h^{\rm NW}(\wt X_u,x_0),
\]
and define
\[
\widehat w_{m,h}^{\rm NW}(x;x_0)
=
\frac{K_h^{\rm NW}(x,x_0)}
{\widehat\mu_{m,h}(x_0)}.
\]
For the labeled observations, use the averaged weight
\[
\widehat w_{L,h}^{\rm NW}(x;x_0)
=
\frac{1}{2}
\left\{
\widehat w_{1,h}^{\rm NW}(x;x_0)
+
\widehat w_{2,h}^{\rm NW}(x;x_0)
\right\}.
\]
For an unlabeled observation in fold $\cI_m$, use the weight trained on the opposite fold:
\[
\widehat w_{U,h}^{\rm NW}(\wt X_u;x_0)
=
\widehat w_{3-m,h}^{\rm NW}(\wt X_u;x_0),
\qquad
u\in\cI_m.
\]

The two-fold NW-PPCI estimating equation is
\begin{equation}
\label{eq:nw-ppci-moment}
\begin{aligned}
\widehat\eta_h^{\rm NW}(x_0;\theta)
={}&
\frac{1}{n}
\sum_{i=1}^n
\widehat w_{L,h}^{\rm NW}(X_i;x_0)
\{\ell(Y_i;\theta)-\ell(f(X_i);\theta)\}
\\
&+
\frac{1}{N}
\sum_{m=1}^2
\sum_{u\in\cI_m}
\widehat w_{3-m,h}^{\rm NW}(\wt X_u;x_0)
\ell(f(\wt X_u);\theta).
\end{aligned}
\end{equation}
Its solution is denoted by $\widehat\theta_h^{\rm NW}(x_0)$: $
\widehat\eta_h^{\rm NW}
\{x_0;\widehat\theta_h^{\rm NW}(x_0)\}
=
0$.
Define the empirical Jacobian
\[
\widehat J_h^{\rm NW}(x_0)
=
\left.
\partial_\theta
\widehat\eta_h^{\rm NW}(x_0;\theta)
\right|_{\theta=\widehat\theta_h^{\rm NW}(x_0)}.
\]
For labeled observations, define
\[
\widehat\zeta_{i,h}^{\rm NW}(x_0)
=
\widehat w_{L,h}^{\rm NW}(X_i;x_0)
\left[
\ell\{Y_i;\widehat\theta_h^{\rm NW}(x_0)\}
-
\ell\{f(X_i);\widehat\theta_h^{\rm NW}(x_0)\}
\right],
\]
and, for $u\in\cI_m$, define
\[
\widehat\zeta_{u,h}^{\rm NW}(x_0)
=
\widehat w_{3-m,h}^{\rm NW}(\wt X_u;x_0)
\ell\{f(\wt X_u);\widehat\theta_h^{\rm NW}(x_0)\}.
\]
Let $\widehat\sigma_{Y-f,h}^{2,{\rm NW}}(x_0)$ and
$\widehat\sigma_{f,h}^{2,{\rm NW}}(x_0)$ be the sample variances of the labeled and unlabeled contributions, respectively, and set
\[
\widehat V_h^{\rm NW}(x_0)
=
\frac{\widehat\sigma_{Y-f,h}^{2,{\rm NW}}(x_0)}{n}
+
\frac{\widehat\sigma_{f,h}^{2,{\rm NW}}(x_0)}{N}.
\]
The resulting confidence interval is
\[
\mathcal C_h^{\rm NW}(x_0)
=
\left[
\widehat\theta_h^{\rm NW}(x_0)
\pm
z_{1-\alpha/2}
\frac{
\{\widehat V_h^{\rm NW}(x_0)\}^{1/2}
}{
|\widehat J_h^{\rm NW}(x_0)|
}
\right].
\]
For fixed $h$, this interval is naturally centered at $\theta_h^{\rm NW}(x_0)$. Coverage of $\theta_0(x_0)$ requires the smoothing bias
$\theta_h^{\rm NW}(x_0)-\theta_0(x_0)$ to be negligible relative to the stochastic scale.

\subsubsection{Interior smoothing bias and prediction-powered efficiency}

\begin{proposition}[Interior bias and efficiency of NW-PPCI]
\label{prop:nw-localization-limit}
Let $p_X=d\rho_X/dx$. Suppose that $x_0$ is an interior point of $\cX$, that $0<\rho_0\le p_X(x)\le\rho_1<\infty$ on $\cX$, and that $p_X$ is twice continuously differentiable near $x_0$. Suppose also that $x\mapsto\eta(x;\theta)$ is twice continuously differentiable near $x_0$ uniformly for $\theta$ in a neighborhood of $\theta_0(x_0)$. Write
$\eta_0(x)=\eta(x;\theta_0(x_0))$ and
$J_0(x_0)=\partial_\theta\eta(x_0;\theta_0(x_0))\ne0$, and define
\[
b_{\rm NW}(x_0)
=
-\frac{1}{2J_0(x_0)}
\left[
\operatorname{tr}
\{M_{\mathcal K}\nabla_x^2\eta_0(x_0)\}
+
2\nabla_x\eta_0(x_0)^\top
M_{\mathcal K}\nabla_x\log p_X(x_0)
\right].
\]
Then
\[
\theta_h^{\rm NW}(x_0)-\theta_0(x_0)
=
h^2b_{\rm NW}(x_0)+o(h^2),
\]
and
\[
\EE\{w_h^{\rm NW}(X;x_0)^2\}
=
\frac{
\int_{\RR^d}\mathcal K(u)^2\,du
}{
p_X(x_0)h^d
}
\{1+o(1)\}.
\]
If $b_{\rm NW}(x_0)\ne0$, $N\ge n$, and Assumption~\ref{ass:homo_res} holds, centered pointwise inference for $\theta_0(x_0)$ requires $nh^{d+4}\to0$,
in addition to $nh^d\to\infty$ and $Nh^d\to\infty$.

Moreover, consider the conditional mean model
$Y=m(X)+\varepsilon$, where $\EE(\varepsilon| X)=0$, $m$ is twice continuously differentiable near $x_0$, and
$\sigma^2(x)=\Var(\varepsilon| X=x)$ is continuous near $x_0$ with $\sigma^2(x_0)>0$. Suppose that the predictor satisfies $f=m$ and that $N/n\to\kappa\in(1,\infty]$. Then the relative variance reduction of NW-PPCI over its NW labeled-only counterpart is $O(h^2)$ when $\nabla m(x_0)\ne0$, and is $O(h^4)$ when $\nabla m(x_0)=0$ but $\nabla_x^2m(x_0)\ne0$.
\end{proposition}

\begin{proof}[Proof of Proposition~\ref{prop:nw-localization-limit}]
Write $p=p_X$, $\eta_0(x)=\eta(x;\theta_0(x_0))$, and
$J_0=J_0(x_0)$. At an interior point, a change of variables gives
\[
\eta_h^{\rm NW}(x_0;\theta)
=
\frac{
\int_{\RR^d}
\mathcal K(u)
\eta(x_0+hu;\theta)
p(x_0+hu)\,du
}{
\int_{\RR^d}
\mathcal K(u)p(x_0+hu)\,du
}.
\]
Using $\int_{\RR^d}u\mathcal K(u)\,du=0$, a second-order expansion of the numerator and denominator yields
\[
\begin{aligned}
&
\eta_h^{\rm NW}
\{x_0;\theta_0(x_0)\}
-
\eta\{x_0;\theta_0(x_0)\}
\\
&\quad=
\frac{h^2}{2}
\left[
\operatorname{tr}
\{M_{\mathcal K}\nabla_x^2\eta_0(x_0)\}
+
2\nabla_x\eta_0(x_0)^\top
M_{\mathcal K}\nabla_x\log p(x_0)
\right]
+
o(h^2).
\end{aligned}
\]
Since $\eta(x_0;\theta_0(x_0))=0$, a first-order expansion in $\theta$ around $\theta_0(x_0)$ gives
\[
0
=
\eta_h^{\rm NW}
\{x_0;\theta_0(x_0)\}
+
J_0
\{\theta_h^{\rm NW}(x_0)-\theta_0(x_0)\}
+
o(h^2),
\]
which implies
$\theta_h^{\rm NW}(x_0)-\theta_0(x_0)=h^2b_{\rm NW}(x_0)+o(h^2)$.

Next, $\EE\{K_h^{\rm NW}(X,x_0)\}=p(x_0)+O(h^2)$,
whereas
\[
\EE\{K_h^{\rm NW}(X,x_0)^2\}
=
h^{-d}p(x_0)
\int_{\RR^d}\mathcal K(u)^2\,du
+
o(h^{-d}).
\]
Therefore,
\[
\EE\{w_h^{\rm NW}(X;x_0)^2\}
=
\frac{
\int_{\RR^d}\mathcal K(u)^2\,du
}{
p(x_0)h^d
}
\{1+o(1)\}.
\]
Under Assumption~\ref{ass:homo_res}, the labeled stochastic scale is bounded below by a constant multiple of $(nh^d)^{-1/2}$. If $b_{\rm NW}(x_0)\ne0$, negligibility of the $O(h^2)$ smoothing bias requires
\[
\frac{h^2}{(nh^d)^{-1/2}}
=
\sqrt{nh^{d+4}}
\longrightarrow0.
\]
Thus, $nh^{d+4}\to0$ is necessary for centered pointwise inference. The conditions $nh^d\to\infty$ and $Nh^d\to\infty$ ensure a growing effective labeled sample and consistent estimation of the normalizing denominator, respectively.

For the efficiency result, define
\[
A_h
=
w_h^{\rm NW}(X;x_0)\varepsilon,
\qquad
B_h
=
w_h^{\rm NW}(X;x_0)
\{m(X)-\theta_h^{\rm NW}(x_0)\}.
\]
Since $\EE(A_h| X)=0$, we have $\operatorname{Cov}(A_h,B_h)=0$. Hence, $V_{\rm LO}^{\rm NW}(h)=\frac{\Var(A_h)+\Var(B_h)}{n}$,
while
\[
V_{\rm PPCI}^{\rm NW}(h)
=
\frac{\Var(A_h)}{n}
+
\frac{\Var(B_h)}{N}.
\]
It follows that
\[
1-
\frac{
V_{\rm PPCI}^{\rm NW}(h)
}{
V_{\rm LO}^{\rm NW}(h)
}
=
\left(1-\frac{n}{N}\right)
\frac{
\Var(B_h)
}{
\Var(A_h)+\Var(B_h)
}.
\]
The continuity and nondegeneracy of $\sigma^2(x)$ near $x_0$ imply
$\Var(A_h)\asymp h^{-d}$.

If $\nabla m(x_0)\ne0$, then $m(x_0+hu)-m(x_0)=h\nabla m(x_0)^\top u+O(h^2)$,
so $\Var(B_h)=O(h^{2-d})$ and the relative gain is $O(h^2)$. If
$\nabla m(x_0)=0$ but $\nabla_x^2m(x_0)\ne0$, then
\[
m(x_0+hu)-m(x_0)
=
\frac{h^2}{2}
u^\top\nabla_x^2m(x_0)u
+
o(h^2),
\]
so $\Var(B_h)=O(h^{4-d})$ and the relative gain is $O(h^4)$.
\end{proof}

\noindent\textbf{Implications for bandwidth selection.}
The interior NW bias is generically $O(h^2)$, while the variance is of order $(nh^d)^{-1}$. The MSE-oriented bandwidth therefore balances $h^4$ and $(nh^d)^{-1}$, giving
$h\asymp n^{-1/(d+4)}$. At this bandwidth, $h^2\asymp(nh^d)^{-1/2}$,
so the smoothing bias and standard error have the same order. Pointwise inference consequently requires a smaller bandwidth than the conventional MSE-oriented choice.

A local-constant NW estimator also does not automatically exploit smoothness beyond second order. Its root-MSE rate remains
$n^{-2/(d+4)}$. If the Sobolev class has pointwise smoothness
$s=m-d/2>2$, this is slower than the rate $n^{-s/(2s+d)}$ obtained from the RKHS localization analysis. Higher-order kernels or local-polynomial estimators can improve the bias order, but their equivalent weights are generally signed and require additional choices of polynomial order and boundary correction~\citep{fan1996local}.

\subsubsection{Variational characterization and comparison with NW}
\noindent
For a fixed base kernel $K_h$, let $T_h$ denote its covariance operator and write $K_{h,x_0}=K_h(\cdot,x_0)$. For $w\in\cH$, define $\mathcal B_h(w)=\sup_{\|g\|_{\cH}\le1}|g(x_0)-\EE\{w(X)g(X)\}|$ and $Q_h(w)=\|w\|_{L^2(\rho_X)}^2$. By the reproducing property,
\[
\mathcal B_h(w)=\sup_{\|g\|_{\cH}\le1}|\langle K_{h,x_0}-T_hw,g\rangle_{\cH}|=\|K_{h,x_0}-T_hw\|_{\cH}.
\]
Consider
\[
\mathcal L_{h,\lambda}(w)=\mathcal B_h(w)^2+\lambda Q_h(w).
\]

\begin{proposition}[Variational and constrained optimality]
\label{prop:rkhs-localization-variational-nw}
For every $\lambda>0$, the RKHS localization weight $w_{x_0,h,\lambda}=(T_h+\lambda I)^{-1}K_{h,x_0}$ is the unique minimizer of $\mathcal L_{h,\lambda}$ over $\cH$. Moreover,
\[
\begin{aligned}
\mathcal B_h(w_{x_0,h,\lambda})^2
&=\lambda\{D_h(x_0;\lambda)-Q_h(x_0;\lambda)\},\\
\min_{w\in\cH}\mathcal L_{h,\lambda}(w)
&=\lambda D_h(x_0;\lambda).
\end{aligned}
\]
If $\tau_\lambda=\mathcal B_h(w_{x_0,h,\lambda})$, then
\[
w_{x_0,h,\lambda}\in\arg\min_{\substack{w\in\cH:\ \mathcal B_h(w)\le\tau_\lambda}}Q_h(w).
\]
When NW uses the same nonnegative base kernel $K_h$, its normalized weight $w_h^{\rm NW}=K_{h,x_0}/\EE\{K_h(X,x_0)\}$ belongs to $\cH$ and satisfies
\[
\mathcal L_{h,\lambda}(w_{x_0,h,\lambda})\le\mathcal L_{h,\lambda}(w_h^{\rm NW}).
\]
\end{proposition}

\begin{proof}[Proof of Proposition~\ref{prop:rkhs-localization-variational-nw}]
For any $w,u\in\cH$,
\[
\begin{aligned}
\frac{1}{2}
\frac{d}{dt}
\mathcal L_{h,\lambda}(w+tu)
\bigg|_{t=0}
&=
-
\langle K_{h,x_0}-T_hw,T_hu\rangle_{\cH}
+
\lambda
\langle w,u\rangle_{L^2(\rho_X)}
\\
&=
-
\langle
K_{h,x_0}-T_hw-\lambda w,
u
\rangle_{L^2(\rho_X)},
\end{aligned}
\]
where we use
$\langle v,T_hu\rangle_{\cH}
=
\langle v,u\rangle_{L^2(\rho_X)}$.

At a stationary point, $\langle K_{h,x_0}-T_hw-\lambda w,u\rangle_{L^2(\rho_X)}=0$ for every $u\in\cH$. Let
$r=K_{h,x_0}-T_hw-\lambda w$. Since $T_h:\cH\to\cH$, we have $r\in\cH$. Taking $u=r$ gives
$\|r\|_{L^2(\rho_X)}^2=0$. Under $\cH=H^m(\cX)$ with $m>d/2$, functions in $\cH$ are continuous, and the positive lower bound on the design density makes the embedding
$\cH\hookrightarrow L^2(\rho_X)$ injective. Hence $r=0$ and $(T_h+\lambda I)w=K_{h,x_0}$.
Furthermore, $\langle v,(T_h+\lambda I)v\rangle_{\cH}=\|v\|_{L^2(\rho_X)}^2+\lambda\|v\|_{\cH}^2\ge\lambda\|v\|_{\cH}^2$, so $T_h+\lambda I$ is strictly positive and boundedly invertible. Therefore, $w=(T_h+\lambda I)^{-1}K_{h,x_0}$
is the unique minimizer.

The Tikhonov equation gives $K_{h,x_0}-T_hw_{x_0,h,\lambda}=\lambda w_{x_0,h,\lambda}$,
so $\mathcal B_h(w_{x_0,h,\lambda})^2=\lambda^2\|w_{x_0,h,\lambda}\|_{\cH}^2$. Using $D_h(x_0;\lambda)=Q_h(x_0;\lambda)+\lambda\|w_{x_0,h,\lambda}\|_{\cH}^2$, we obtain $\mathcal B_h(w_{x_0,h,\lambda})^2=\lambda\{D_h(x_0;\lambda)-Q_h(x_0;\lambda)\}$.
Consequently,
\[
\begin{aligned}
\min_{w\in\cH}\mathcal L_{h,\lambda}(w)
&=
\mathcal B_h(w_{x_0,h,\lambda})^2
+
\lambda Q_h(x_0;\lambda)
\\
&=
\lambda
\{D_h(x_0;\lambda)-Q_h(x_0;\lambda)\}
+
\lambda Q_h(x_0;\lambda)
\\
&=
\lambda D_h(x_0;\lambda).
\end{aligned}
\]

To establish constrained optimality, suppose that some $w\in\cH$ satisfies
$\mathcal B_h(w)\le\tau_\lambda$ and
$Q_h(w)<Q_h(w_{x_0,h,\lambda})$. Then $\mathcal L_{h,\lambda}(w)<\mathcal L_{h,\lambda}(w_{x_0,h,\lambda})$, contradicting the variational optimality of $w_{x_0,h,\lambda}$. Hence, $w_{x_0,h,\lambda}\in\arg\min_{\mathcal B_h(w)\le\tau_\lambda}Q_h(w)$.

Finally, if NW uses the same base kernel $K_h$, then
$w_h^{\rm NW}$ is a scalar multiple of $K_{h,x_0}$ and therefore belongs to $\cH$. The variational inequality follows by taking
$w=w_h^{\rm NW}$ as a feasible candidate.
\end{proof}

\noindent\textbf{Connection to the estimation bound.}
Let
$R_\eta=\|\eta(\cdot;\theta_0(x_0))\|_{\cH}$ and define
\[
\mathcal E_{n,N}
=
\sqrt{\log(n+N)}
\left\{
\frac{B_{Y-f}(\theta_\lambda(x_0))}{\sqrt n}
+
\frac{B_f(\theta_\lambda(x_0))}{\sqrt N}
\right\}
+
\kappa\|\eta(\cdot;\theta_\lambda(x_0))\|_{\cH}
\sqrt{\frac{\log N}{N}}.
\]
Up to universal constants, the upper bound
\eqref{eq:thetahat-theta0-bound-clean-twofold} has the form
\[
|\widehat\theta(x_0)-\theta_0(x_0)|
\lesssim
\frac{1}{c_J}
\left[
\mathcal E_{n,N}
\sqrt{D_h(x_0;\lambda)}
+
R_\eta
\sqrt{
\lambda\{D_h(x_0;\lambda)-Q_h(x_0;\lambda)\}
}
\right].
\]
At
$\lambda\asymp(\mathcal E_{n,N}/R_\eta)^2$, both terms are bounded by a constant multiple of
$R_\eta\{\lambda D_h(x_0;\lambda)\}^{1/2}$. Proposition~\ref{prop:rkhs-localization-variational-nw} therefore gives
\[
|\widehat\theta(x_0)-\theta_0(x_0)|
\lesssim
\frac{R_\eta}{c_J}
\left[
\min_{w\in\cH}
\mathcal L_{h,\lambda}(w)
\right]^{1/2}.
\]

\noindent\textbf{Connection to pointwise inference.}
For a generic weight $w$, the worst-case localization bias is bounded by
$R_\eta\mathcal B_h(w)$, while $Q_h(w)^{1/2}/\sqrt n$ gives the weight-dependent labeled stochastic scale. Centered inference therefore requires
\[
\frac{\sqrt n\,\mathcal B_h(w)}{Q_h(w)^{1/2}}\longrightarrow0.
\]
For the RKHS localization weight, the squared ratio is
\[
\frac{n\lambda\{D_h(x_0;\lambda)-Q_h(x_0;\lambda)\}}{Q_h(x_0;\lambda)}.
\]
By Proposition~\ref{prop:Dx0-order-H}, this ratio is of order $n\lambda$, which gives the condition $\lambda=o(n^{-1})$ in
\eqref{eq:cond-clean-twofold}. At the bias level
$\tau_\lambda=\mathcal B_h(w_{x_0,h,\lambda})$, Proposition~\ref{prop:rkhs-localization-variational-nw} shows that the RKHS localization weight minimizes $Q_h(w)$. Thus, when the variance and localized-Jacobian factors induced by $\ell$ are comparable, it also minimizes the leading variance and confidence-interval width among RKHS weights satisfying the same worst-case bias tolerance.

\subsubsection{Numerical comparisons}

\noindent\textbf{Data-generating process and methods.}
We consider
\[
X\sim\operatorname{Unif}[-1,1],
\qquad
Y=bX+4X^2+\varepsilon,
\qquad
\varepsilon\sim N(0,1),
\]
with perfect predictor $f(X)=bX+4X^2$ and target
$\theta_0(0)=0$. We take $n=200$, $N=5000$,
$b\in\{0,2\}$, and 500 paired Monte Carlo replications.

NW uses the Matérn--$5/2$ kernel. NW-CV selects its bandwidth by
leave-one-out cross-validation on an independent labeled sample,
whereas NW-US uses $h_{\rm US}=h_{\rm CV}n^{-0.05}$. The NW
normalizing denominator is estimated by two-fold cross-fitting over
the unlabeled sample.
PPCI with the RKHS localization weight uses the
median-scaled bandwidth grid, shrinking regularization grid,
bias-admissibility threshold $c_{\rm bias}=0.18$, stability constants
$(c_{\rm op},c_{\rm loc})=(12,4)$, and the fallback rule in
Section~\ref{sec:tuning_twofold}.

\noindent\textbf{Data-driven comparison.}
Table~\ref{tab:nw-rkhs-localization-data-driven} compares the
fully data-driven methods.
\begin{table}[t]
\centering
\small
\setlength{\tabcolsep}{6pt}
\caption{Data-driven comparison of NW-PPCI and PPCI with the
RKHS localization weight over 500 paired replications.}
\label{tab:nw-rkhs-localization-data-driven}
\begin{tabular}{clrrrr}
\toprule
$b$ & Method & Coverage & Bias & RMSE & Width \\
\midrule
0 & NW-CV         & 0.964 &  0.0302 & 0.1818 & 0.7110 \\
0 & NW-US         & 0.970 &  0.0145 & 0.2027 & 0.7958 \\
0 & RKHS localization & 0.970 & -0.0005 & 0.1674 & 0.6808 \\
\addlinespace
2 & NW-CV         & 0.960 &  0.0288 & 0.1833 & 0.7171 \\
2 & NW-US         & 0.970 &  0.0135 & 0.2044 & 0.8014 \\
2 & RKHS localization & 0.970 & -0.0005 & 0.1674 & 0.6813 \\
\bottomrule
\end{tabular}
\end{table}
The mean selected bandwidths satisfy
$\bar h_{\rm CV}\in[0.081,0.083]$,
$\bar h_{\rm US}\in[0.063,0.064]$, and
$\bar h_{\rm RKHS}\approx0.600$. NW-US trades lower bias for larger
RMSE and width, whereas RKHS localization attains the smallest
absolute bias, RMSE, and width. Writing $\mathrm{VR}$ for LO-relative variance
reduction,
$(\mathrm{VR}_{\rm CV},\mathrm{VR}_{\rm US},\mathrm{VR}_{\rm RKHS})$
equals $(0.14\%,0.13\%,0.86\%)$ for $b=0$ and
$(3.02\%,1.80\%,7.01\%)$ for $b=2$. This contrast agrees with the
$(O(h^4),O(h^2))$ orders for $(b=0,b=2)$ in
Proposition~\ref{prop:nw-localization-limit}.

\noindent\textbf{Same-bandwidth mechanism comparison.}
At each fixed $h\in\{0.08,0.20,0.40,0.60\}$, we compare
NW-PPCI and PPCI using the same Matérn--$5/2$ kernel; RKHS
localization selects only $\lambda$.
For a fitted weight $\widehat w$, define $
\widehat\Delta_2(\widehat w)
=|
\frac{
\mathbb P_N\{\widehat w(X)X^2\}
}{
\mathbb P_N\{\widehat w(X)\}
}|$.
This diagnostic measures quadratic localization error rather
than the worst-case RKHS quantity $\mathcal B_h(w)$.

\begin{table}[t]
\centering
\small
\setlength{\tabcolsep}{5pt}
\caption{Same-bandwidth comparison for $b=0$ over 500 paired
replications. Here $\widehat\Delta_2$ is the normalized empirical
quadratic localization error.}
\label{tab:nw-rkhs-localization-same-h-full}
\begin{tabular}{ccrrrr}
\toprule
$h$ & Method & Coverage & Bias & RMSE & $\widehat\Delta_2$ \\
\midrule
0.08 & NW            & 0.956 &  0.0260 & 0.1766 & 0.007690 \\
0.08 & RKHS localization & 0.980 &  0.0023 & 0.3806 & 0.000047 \\
\addlinespace
0.20 & NW            & 0.652 &  0.1851 & 0.2175 & 0.047563 \\
0.20 & RKHS localization & 0.968 &  0.0019 & 0.2601 & 0.000127 \\
\addlinespace
0.40 & NW            & 0.000 &  0.5879 & 0.5943 & 0.147705 \\
0.40 & RKHS localization & 0.952 & -0.0049 & 0.1960 & 0.000259 \\
\addlinespace
0.60 & NW            & 0.000 &  0.8638 & 0.8675 & 0.216623 \\
0.60 & RKHS localization & 0.942 & -0.0070 & 0.1676 & 0.000459 \\
\bottomrule
\end{tabular}
\end{table}

At $h=0.08$, NW has lower RMSE because its narrow neighborhood
limits curvature bias, whereas RKHS localization incurs a variance
cost from balancing. As $h$ increases, NW bias rises rapidly and coverage
deteriorates. At $h=0.60$,
$\widehat\Delta_2(\widehat w_{\rm RKHS})/
\widehat\Delta_2(\widehat w_{\rm NW})\approx1/472$ and
$\mathbb P_N\{\widehat w_{\rm RKHS}<0\}=41.8\%$. Results for $b=2$
follow the same pattern.

\noindent\textbf{Curvature stress test.}
We fix $h=0.60$ and $b=0$ and vary the quadratic coefficient
in $Y=AX^2+\varepsilon$ over $A\in\{2,4,8\}$, thereby isolating
sensitivity to curvature. For $A=2,4,8$, the NW
$(\mathrm{Bias},\mathrm{RMSE})$ pairs are
$(0.431,0.438)$, $(0.864,0.867)$, and $(1.730,1.732)$, respectively.
By contrast, RKHS localization has
$\mathrm{Bias}\in[-0.0079,-0.0051]$ and $\mathrm{RMSE}\approx0.168$.
The approximately proportional growth of NW bias with $A$ is
consistent with quadratic smoothing bias, while the RKHS localization
weight balances curvature over the wider neighborhood.

\noindent\textbf{Analytic rate verification.}
We also verify the rate statements independently of the main
performance comparison. Let $\widehat\alpha_b$ denote the fitted
log--log slope over the prespecified lower $40\%$ of the bandwidth
grid. Then $(\widehat\alpha_0,\widehat\alpha_2)
=(3.9995,1.9951)\approx(4,2)$, matching the $O(h^4)$ and $O(h^2)$
orders in Proposition~\ref{prop:nw-localization-limit}. For a fixed
uniform NW weight, the empirical bias pair
$(0.1213,0.1218)\approx(0.12,0.12)$, while the empirical--analytic
variance discrepancy lies in $[4.8\%,9.6\%]$.

\noindent\textbf{NW bandwidth-grid sensitivity.}
To assess boundary sensitivity, we expand the NW grid from
$[0.05,0.8]$ to $[0.02,0.8]$. The pair (NW-CV lower-bound selection
rate, NW-US clipping rate) changes as
$(8\%,18\%)\to(0.4\%,1\%)$. Nevertheless, bias, RMSE, and coverage do
not materially improve, while RMSE and width increase slightly. The
wider grid therefore changes boundary selection without altering the
substantive comparison.

\noindent\textbf{RKHS localization tuning diagnostics.}
Across all 500 replications and every
$h\in\{0.08,0.20,0.40,0.60\}$, the RKHS procedure satisfies all three
tuning criteria without fallback. Its bias-admissibility statistic
lies in $[0.1669,0.1794]\subset[0,c_{\rm bias})$, with
$c_{\rm bias}=0.18$. Thus, the comparison does not rely on the
least-violation safeguard.
\subsection{Why Localize First?}
\label{subsec:why-localize-first}
\noindent
A natural strategy for conditional inference is to first estimate the
entire conditional-moment function and then evaluate it at a query point
$x_0$. Classical prediction-powered inference (PPI)
\citep{angelopoulos2023prediction}, originally developed for a global
parameter $\theta_0$ satisfying $\EE[\ell(Y;\theta_0)]=0$, suggests one
such construction: form a prediction-powered pseudo-outcome and regress it
on $X$ using kernel ridge regression (KRR). We call this estimate-first,
evaluate-later procedure PPI+KRR. PPCI instead fixes $x_0$ first and
calibrates localization specifically for that point. This distinction is
important for conditional inference.

Suppose that $N'=n+N$ covariates $\{\bar X_j\}_{j=1}^{N'}$ are observed,
with only $n$ corresponding outcomes available. Let $O_j$ indicate whether
the outcome is observed, independently of $(\bar X_j,Y_j)$, with
$\PP(O_j=1)=\pi=n/N'$. For a fixed $\theta$, define
\begin{equation*}
Z_j(\theta)
=
\ell(f(\bar X_j);\theta)
+
\frac{O_j}{\pi}
\left\{
\ell(Y_j;\theta)-\ell(f(\bar X_j);\theta)
\right\}.
\end{equation*}
Since $\EE\{Z(\theta)|X=x\}=\eta(x;\theta)$, PPI+KRR estimates the
conditional moment function by
\begin{equation*}
\wh g_{\gamma,\theta}
=
\underset{g\in\cH}{\arg\min}
\left[
\frac{1}{N'}
\sum_{j=1}^{N'}
\{Z_j(\theta)-g(\bar X_j)\}^2
+
\gamma\|g\|_{\cH}^2
\right].
\end{equation*}
Writing $\Sigma=K(\bar X,\bar X)$ and $k_x=K(\bar X,x)$ gives
$\wh g_{\gamma,\theta}(x)
=k_x^\top(\Sigma+N'\gamma I_{N'})^{-1}\bar Z(\theta)$, and the estimator at
$x_0$ solves $\wh g_{\gamma,\wh\theta_\gamma(x_0)}(x_0)=0$. Both PPI+KRR and PPCI impose smoothness on
$\eta(\cdot;\theta)$, while $f$ affects only the pseudo-outcome noise:
\begin{equation}\label{eq:ppi-VarZ}
\Var\{Z(\theta)|X=x\}
=
\frac{1}{\pi}\Var\{\ell(Y;\theta)|X=x\}
+
\frac{1-\pi}{\pi}
\left[
\eta(x;\theta)-\ell(f(x);\theta)
\right]^2.
\end{equation}

\begin{proof}[Proof of~\eqref{eq:ppi-VarZ}]
Conditionally on $X=x$, let $D = \ell(Y;\theta) - \ell(f(x);\theta)$. Since $\ell(f(x);\theta)$ is deterministic given $X=x$, the conditional variance of $Z(\theta)$ is
\begin{equation*}
\Var\big(Z(\theta) | X=x\big)
= \Var\left( \ell(f(x);\theta) + \frac{O}{\pi} D | X=x \right)
= \frac{1}{\pi^2} \Var(OD | X=x).
\end{equation*}
For $\Var(OD | X=x)$, we have $\Var(OD | X=x) = \EE[O^2 D^2 | X=x] - \big(\EE[OD | X=x]\big)^2$. Since the observation indicator $O \in \{0, 1\}$ is a Bernoulli random variable, we have $O^2 = O$. Furthermore, $O$ is independent of $(X, Y)$ with $\EE[O] = \pi$. Therefore, the conditional moments are:
\begin{equation*}
\begin{aligned}
\EE[O^2 D^2 | X=x] &= \EE[O] \EE[D^2 | X=x] = \pi \EE[D^2 | X=x], \\
\big(\EE[OD | X=x]\big)^2 &= \big(\EE[O] \EE[D | X=x]\big)^2 = \pi^2 \big(\EE[D | X=x]\big)^2.
\end{aligned}
\end{equation*}
Substituting these into the variance formula, and applying the identity $\EE[D^2 | X=x] = \Var(D | X=x) + \big(\EE[D | X=x]\big)^2$, we obtain:
\begin{equation*}
\begin{aligned}
\Var(OD | X=x)
&= \pi \EE[D^2 | X=x] - \pi^2 \big(\EE[D | X=x]\big)^2 \\
&= \pi \Big( \Var(D | X=x) + \big(\EE[D | X=x]\big)^2 \Big) - \pi^2 \big(\EE[D | X=x]\big)^2 \\
&= \pi \Var(D | X=x) + \pi(1-\pi) \big(\EE[D | X=x]\big)^2.
\end{aligned}
\end{equation*}
Noting that $\Var(D | X=x) = \Var\big(\ell(Y;\theta) | X=x\big)$ and $\EE[D | X=x] = \eta(x;\theta) - \ell(f(x);\theta)$, we have
\begin{equation*}
\begin{aligned}
\Var\big(Z(\theta) | X=x\big)
&= \frac{1}{\pi^2} \Big[ \pi \Var(D | X=x) + \pi(1-\pi) \big(\EE[D | X=x]\big)^2 \Big] \\
&= \frac{1}{\pi} \Var\big(\ell(Y;\theta) | X=x\big) + \frac{1-\pi}{\pi} \big(\eta(x;\theta) - \ell(f(x);\theta)\big)^2,
\end{aligned}
\end{equation*}
which is exactly~\eqref{eq:ppi-VarZ}.
\end{proof}

For a fixed kernel and regularization level, PPI+KRR is closely connected
to the RKHS localization weight used by PPCI. At the population level,
$g_{\gamma,\theta}=(T_K+\gamma I)^{-1}T_K\eta(\cdot;\theta)$, so $
g_{\gamma,\theta}(x_0)
=
\EE[
\{(T_K+\gamma I)^{-1}K_{x_0}\}(X)\eta(X;\theta)]$.
Thus evaluating KRR at $x_0$ implicitly induces an RKHS localization
weight. The difference is calibration: PPI+KRR selects one
global $\gamma$ to estimate the full function, whereas PPCI fixes $x_0$
first and selects the point-specific $\lambda$.

This distinction matters for standard prediction-oriented KRR tuning.
The pseudo-outcome variance is inflated by
$1/\pi=N'/n$, so the integrated KRR risk is, up to constants,
$\gamma+D(\gamma)/n$. Under
Assumption~\ref{ass:eigenfunctions_bound},
$D(\gamma)\asymp\gamma^{-d/(2m)}$, so risk-consistent cross-validation or
generalized cross-validation yields $\gamma_{\rm GCV}
\asymp
n^{-2m/(2m+d)}$.
This rate is generally too
large for centered pointwise inference.
To see this, consider finitely many interior query points
$\mathcal X_0=\{x_0^{(1)},\ldots,x_0^{(L)}\}$ and marginal pointwise
coverage at each point. Write
$D_\ell(\gamma)=D(x_0^{(\ell)};\gamma)$ and
$Q_\ell(\gamma)=Q(x_0^{(\ell)};\gamma)$. The worst-case RKHS
regularization bias at $x_0^{(\ell)}$ is bounded by
$\|\eta(\cdot;\theta)\|_{\cH}
\{\gamma(D_\ell(\gamma)-Q_\ell(\gamma))\}^{1/2}$, while in the
abundant-unlabeled regime the leading stochastic scale is proportional to
$\{Q_\ell(\gamma)/n\}^{1/2}$. Hence the squared standardized-bias proxy is $
\frac{
n\gamma\{D_\ell(\gamma)-Q_\ell(\gamma)\}
}{
Q_\ell(\gamma)
}$.

Suppose, under the Sobolev eigenvalue condition, that
$\{D_\ell(\gamma)-Q_\ell(\gamma)\}/Q_\ell(\gamma)
\to\kappa_\ell\in(0,\infty)$. Then the proxy is
$n\gamma\kappa_\ell\{1+o(1)\}$. Since $
n\gamma_{\rm GCV}
\asymp
n^{d/(2m+d)}
\longrightarrow\infty$,
globally prediction-optimal PPI+KRR does not ensure negligible pointwise
regularization bias. A confidence interval centered at the GCV-tuned KRR
estimator therefore need not attain nominal pointwise coverage uniformly
over the RKHS class.

Bootstrap methods for KRR
\citep{singh2023kernel,yu2025estimation}, based on Gaussian and
multiplier-bootstrap approximations
\citep{chernozhukov2016empirical}, do not remove this issue. They
approximate the stochastic distribution for a given regularization level
but do not eliminate the difference between the penalized target
$g_{\gamma,\theta}(x_0)$ and the unregularized target
$\eta(x_0;\theta)$. For example, \citet{yu2025estimation} requires this
difference to be negligible relative to the stochastic scale and achieves
it through undersmoothing. Thus inference for the unregularized target
still requires negligible regularization bias.

Replacing GCV by undersmoothing controls bias, but introduces a
different compromise. A common $\gamma$ valid at all query points must,
under the calibration in Section~\ref{sec:tuning_twofold}, satisfy
\begin{equation*}
\gamma_{\rm global}
\lesssim
\frac{c_{\rm bias}^2}
{n\max_{1\le k\le L}\kappa_k}.
\end{equation*}
Thus the regularization level is dictated by the most difficult query
point. Since $Q_\ell(\gamma)\asymp\gamma^{-d/(2m)}$, whereas the largest
point-specific regularization satisfying the same bias criterion is of
order $c_{\rm bias}^2/(n\kappa_\ell)$, the resulting width inflation at
$x_0^{(\ell)}$ is
\begin{equation*}
\frac{
\operatorname{Width}_{\rm global}(x_0^{(\ell)})
}{
\operatorname{Width}_{\rm pointwise}(x_0^{(\ell)})
}
\asymp
\left(
\frac{
\max_{1\le k\le L}\kappa_k
}{
\kappa_\ell
}
\right)^{d/(4m)}
\ge 1.
\end{equation*}
Hence GCV can leave excessive pointwise bias, whereas one globally
undersmoothed $\gamma$ can make inference at easier points unnecessarily
variable. The same tension appears in stability constraints, since smaller
$\gamma$ increases effective dimension and pointwise leverage.

PPCI avoids this global compromise. For each $x_0^{(\ell)}$, it separately
constructs the RKHS localization weight and selects
$\wh\lambda_\ell$ using
$\wh D_h(x_0^{(\ell)};\lambda)$ and
$\wh Q_h(x_0^{(\ell)};\lambda)$. The shrinking regularization grid provides
undersmoothing, while bias and stability restrictions are imposed
pointwise. PPCI can therefore use stronger undersmoothing where
localization is difficult while retaining more stable regularization at
easier points.

\subsection{Conformal Prediction}\label{sec:CP}
\noindent
Conformal prediction provides distribution-free guarantees for predicting $Y$ given $x$ by constructing prediction sets with marginal coverage \citep{vovk2005algorithmic,lei2018distribution}, and more recently, conditional coverage guarantees \citep{gibbs2025conformal}. However, its focus is on predictive coverage for the realized outcome $Y$. Our objective is different: we aim to conduct inference for conditional functionals. For instance, in the conditional mean setting, the target is $\theta_0(x_0)=\EE[Y | X=x_0]$, which may differ substantially from the observed outcome $Y$ when the noise in response is large.

\begin{figure}[t]
  \centering
  \includegraphics[width=\textwidth, height=0.20\textheight]{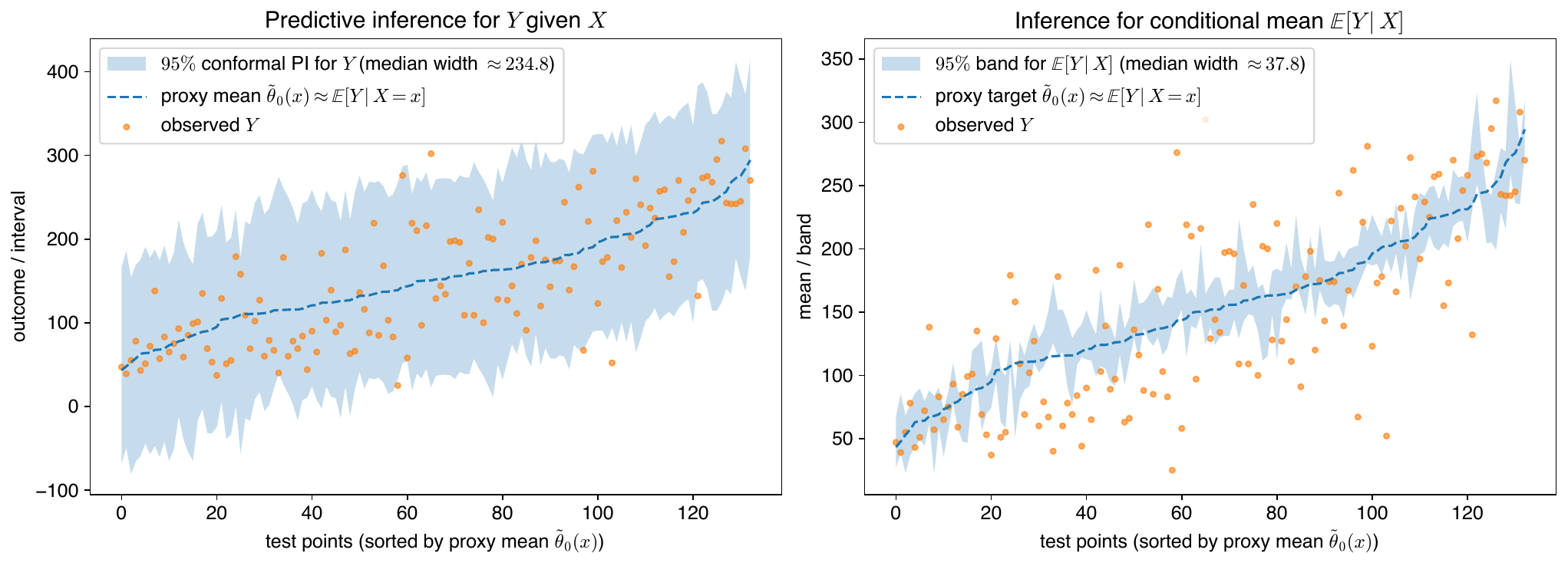}
  \caption{Diabetes progression data illustration. Left: split conformal
  predictive intervals for the outcome $Y$ at test covariates $X$. Right: a bootstrap uncertainty band for the conditional mean
  $\theta_0(x)=\EE[Y| X=x]$.}
  \label{fig:cp-vs-ci-diabetes}
  \Description{Side-by-side conceptual intervals contrast uncertainty for a future response with uncertainty for the conditional mean at the same covariate value.}
\end{figure}
The bootstrap panel is a conceptual illustration of the different inferential target; it does not use an auxiliary predictor or unlabeled data and is not presented as a performance benchmark against PPCI.

This distinction is practically important. In clinical decision support, practitioners need more than a range of plausible responses $Y$; they require a reliable, patient-specific risk estimate $\EE[Y| X=x_0]$ together with its uncertainty~\citep{begoli2019need}.
To illustrate, consider the diabetes progression dataset with $n=442$ patients and $d=10$ baseline covariates \citep{efron2004least}, where $Y$ measures disease progression one year after baseline. Let $\theta_0(x)=\EE[Y| X=x]$ denote the conditional mean. Figure~\ref{fig:cp-vs-ci-diabetes} compares (left) a split conformal prediction interval for $Y$ \citep{lei2018distribution} and (right) a bootstrap uncertainty band for $\theta_0(x)$, both evaluated at the same target points. The predictive intervals remain wide across covariates, reflecting substantial response variability, whereas the uncertainty band for $\theta_0(x)$ is considerably tighter. Consequently, response-level coverage can be overly conservative for decision-making. For example, two patients may have clearly different estimated mean progression, yet their predictive intervals for $Y$ may be wide and heavily overlapping. In such cases, predictive coverage gives an ambiguous prognosis, while conditional inference for $\theta_0(x)$ provides uncertainty quantification that is directly aligned with action.

\subsection{Combining PPCI with Synthetic Surrogates}
\label{subsec:synth-surrogate}
\noindent Synthetic-surrogate (SynSurr) methods
\citep{gronsbell2024another,mccaw2024synthetic} exploit the association between the
covariate and surrogate predictions to improve inference for a global target. Following the PPCI localization step, SynSurr can be extended to conditional targets.
Suppose that a surrogate $\wh Y$ is observed for both the labeled
and unlabeled samples. For a fixed query point $x_0$, PPCI first constructs
the RKHS localization weight
$w_{x_0,\lambda}=(T_K+\lambda I)^{-1}K_{x_0}$ and targets the localized
moment $\EE[w_{x_0,\lambda}(X)\ell(Y;\theta)]$. For any scalar coefficient
$\omega$, this moment admits the decomposition
\begin{equation*}
\EE\big[w_{x_0,\lambda}(X)\ell(Y;\theta)\big]
=
\EE\big[
w_{x_0,\lambda}(X)
\{\ell(Y;\theta)-\omega\ell(\wh Y;\theta)\}
\big]
+
\omega\EE\big[
w_{x_0,\lambda}(X)\ell(\wh Y;\theta)
\big].
\end{equation*}

Suppressing the fold indices of the cross-fitted weights, the corresponding
empirical localized moment can be written as
\begin{equation*}
\begin{aligned}
\wh\eta_{\lambda,\mathrm{Syn}}(x_0;\theta)
={}&
\frac{1}{n}
\sum_{i=1}^{n}
\wh w_{x_0,\lambda}(X_i)
\left\{
\ell(Y_i;\theta)
-
\wh\omega(x_0)
\ell(\wh Y_i;\theta)
\right\}
\\
&+
\frac{\wh\omega(x_0)}{n+N}
\left\{
\sum_{i=1}^{n}
\wh w_{x_0,\lambda}(X_i)\ell(\wh Y_i;\theta)
+
\sum_{u=1}^{N}
\wh w_{x_0,\lambda}(\wt X_u)\ell(\wh Y_u;\theta)
\right\}.
\end{aligned}
\end{equation*}
The resulting estimator is defined by
$\wh\eta_{\lambda,\mathrm{Syn}}(x_0;\wh\theta_{\mathrm{Syn}}(x_0))=0$.
The coefficient $\omega$ controls the efficiency gain. Temporarily treating
the localization weight as fixed, the variance-minimizing population
coefficient at the target parameter is
\begin{equation*}
\omega^\ast(x_0)
=
\frac{
\Cov\left(
w_{x_0,\lambda}(X)\ell(Y;\theta),
w_{x_0,\lambda}(X)\ell(\wh Y;\theta)
\right)
}{
\Var\left(
w_{x_0,\lambda}(X)\ell(\wh Y;\theta)
\right)
}.
\end{equation*}
Relative to the localized labeled-only moment, the corresponding variance
reduction is
\begin{equation*}
\frac{N}{n(n+N)}
\frac{
\Cov\left(
w_{x_0,\lambda}(X)\ell(Y;\theta),
w_{x_0,\lambda}(X)\ell(\wh Y;\theta)
\right)^2
}{
\Var\left(
w_{x_0,\lambda}(X)\ell(\wh Y;\theta)
\right)
}.
\end{equation*}
Hence, the synthetic surrogate improves efficiency whenever its localized
contribution is correlated with the localized outcome contribution. If this correlation
vanishes, then $\omega^\ast(x_0)=0$, and the procedure reduces to
localized labeled-only inference. In
practice, $\omega^\ast(x_0)$ can be estimated on an independent fold and
inserted into the cross-fitted estimating equation.

\section{PPCI Without Sample-Splitting}
\label{app:ppci-no-split}
\noindent
As established in Section~\ref{sec:PPCI}, the cross-fitted PPCI estimator defined there and analyzed in Section~\ref{sec:theory} achieves the desired asymptotic normality and variance reduction. Nevertheless, the cross-fitting partition prevents the fold-specific empirical weights from fully exploiting the entire pool of available covariates. To address this, we further investigate a variant of the PPCI procedure without sample-splitting in this appendix. Specifically, we construct a global empirical localization weight, $\wh w_{x_0,\lambda}$, by pooling all available covariates from both the labeled and unlabeled datasets. The remainder of this appendix is organized as follows: Appendix~\ref{app:PPCI_method} introduces this full-sample methodology. Appendix~\ref{app:PPCI_theory} provides the detailed theoretical guarantees for this full-sample pointwise estimator. Appendix~\ref{app:proofs-of-PPCI-nosplit} contains the formal proofs of the theory presented in Appendix~\ref{app:PPCI_theory}, and Appendix~\ref{app:proofs-of-lemmas-nosplit} provides the proofs of the technical lemmas utilized in Appendix~\ref{app:proofs-of-PPCI-nosplit}.

\subsection{Methodology}\label{app:PPCI_method}
\noindent
To construct the empirical localization weight corresponding to~\eqref{eq:rkhs-localization-pop-x0_twofold} in Section~\ref{sec:localization_twofold},
let $N':=n+N$ and define the pooled covariate sample by
$\bar X_j=X_j$ for $j=1,\dots,n$ and $\bar X_j=\wt X_{j-n}$ for $j=n+1,\dots,N'$.
The empirical integral operator of kernel $K$ is
$\wh T_K := \frac{1}{N'}\sum_{j=1}^{N'} K_{\bar X_j}\otimes K_{\bar X_j}$,
and the empirical localization weight is
$\wh w_{x_0,\lambda} := (\wh T_K+\lambda I)^{-1}K(x_0,\cdot)$.
Equivalently, let $\Sigma\in\RR^{N'\times N'}$ be the pooled Gram matrix
$\Sigma_{jk}=K(\bar X_j,\bar X_k)$ and let $k_0\in\RR^{N'}$ satisfy
$(k_0)_j=K(\bar X_j,x_0)$.

Set $\xi_\lambda=(\Sigma+\lambda N' I_{N'})^{-1}k_0$. The empirical RKHS localization weight corresponding to \eqref{eq:rkhs-localization-pop-x0_twofold} has the finite-dimensional form
\begin{equation}\label{eq:finite-nosplit-weight}
\wh w_{x_0,\lambda}(x)
=
\frac{K(x,x_0)-K(x,\bar X)^\top\xi_\lambda}{\lambda}.
\end{equation}
At the pooled sample points, this simplifies to
\[
\big(\wh w_{x_0,\lambda}(\bar X_1),\dots,\wh w_{x_0,\lambda}(\bar X_{N'})\big)^\top
=N'\xi_\lambda,
\]
which follows from $(\Sigma+N'\lambda I)\xi_\lambda=k_0$.

The empirical localized moment is obtained by replacing the expectations in~\eqref{eqn:Mc_twofold} with sample averages and using $\wh w_{x_0,\lambda}$. Specifically,
\begin{equation}
\label{eq:Mhatc-def}
\wh \eta_\lambda(x_0;\theta) =\frac{1}{n}\sum_{i=1}^n \wh w_{x_0,\lambda}(X_i)\big\{\ell(Y_i;\theta)-\ell(f(X_i);\theta)\big\} +\frac{1}{N}\sum_{u=1}^N \wh w_{x_0,\lambda}(\wt X_u)\ell\big(f(\wt X_u);\theta\big).
\end{equation}
The prediction-powered conditional estimator $\wh\theta(x_0)$ is defined as the solution to
\begin{equation}\label{eq:thetahat-x0}
\wh \eta_\lambda(x_0;\wh\theta(x_0))=0.
\end{equation}
Based on the prediction-powered estimator $\wh\theta(x_0)$ defined as the solution to~\eqref{eq:thetahat-x0}, we construct conditional inference for $\theta_0(x_0)$ at target point $x_0$.
First, we compute the empirical Jacobian of the localized moment~\eqref{eq:Mhatc-def} evaluated at $\wh\theta(x_0)$:
\begin{equation}\label{eq:H-hat-new}
  \wh J_\lambda(x_0) := \partial_\theta \wh \eta_\lambda(x_0;\theta)  \big|_{\theta=\wh\theta(x_0)}.
\end{equation}
Second, define the empirical contributions from the labeled and unlabeled samples to \eqref{eq:Mhatc-def} as
$\wh\zeta_i(x_0):=\wh w_{x_0,\lambda}(X_i)\big\{\ell(Y_i;\wh\theta(x_0))-\ell(f(X_i);\wh\theta(x_0))\big\}$ with $i=1,\dots,n$,
and $\wh\zeta_u(x_0):=\wh w_{x_0,\lambda}(\wt X_u)\,\ell(f(\wt X_u);\wh\theta(x_0))$, with $u=1,\dots,N$.
Let $\wh\sigma^2_{Y-f}(x_0)$ and $\wh\sigma^2_{f}(x_0)$ denote the sample variances of
$\{\wh\zeta_i(x_0)\}_{i=1}^n$ and $\{\wh\zeta_u(x_0)\}_{u=1}^N$, respectively.
The variance estimator is then
\begin{equation}\label{eq:V-hat-new}
  \wh V(x_0):=\frac{1}{n}\wh\sigma^2_{Y-f}(x_0)+\frac{1}{N}\wh\sigma^2_{f}(x_0).
\end{equation}
Finally, the $(1-\alpha)$ conditional confidence interval for $\theta_0(x_0)$ is constructed as
\begin{equation}\label{eq:CI-new}
  \cC(x_0):=\Big(\wh\theta(x_0)  \pm z_{1-\alpha/2}\, \sqrt{\wh V(x_0)}/\big|\wh J_\lambda(x_0)\big|\Big),
\end{equation}
where $z_{1-\alpha/2}$ denotes the $(1-\alpha/2)$ quantile of the standard normal distribution. Note that, to avoid notational clutter, we reuse the notations from Section~\ref{sec:PPCI} (e.g., $\wh\theta(x_0)$, $\wh\eta_\lambda(x_0;\theta)$, $\wh J_\lambda(x_0)$, $\wh V(x_0)$, and $\cC(x_0)$), but we emphasize that these quantities are constructed here without sample-splitting and with the empirical weight in \eqref{eq:finite-nosplit-weight}. {As in the main text, $\lambda$ is treated as fixed in the theoretical analysis.}

\subsection{Theoretical Guarantees}\label{app:PPCI_theory}
\noindent
For PPCI without sample-splitting, we establish a nonasymptotic upper bound for the PPCI estimator
$\wh\theta(x_0)$ in \eqref{eq:thetahat-x0}.
\begin{theorem}
\label{thm:thetahat-theta0-bound}
Let $x_0$ be an interior point of $\cX$, and let $\theta_\lambda(x_0)$ denote the locally unique solution in $\Theta_0$ to $\eta_\lambda(x_0;\theta)=0$. Suppose Assumptions~\ref{ass:sampling-scheme}--\ref{ass:eigenfunctions_bound} hold. Assume $\lambda=\lambda_{n,N}\to0$ and that the regularization parameter satisfies
\begin{equation*}
\lambda^{-1} = o\left( \left(\frac{N'}{\log N'}\right)^{\frac{m}{d}} \wedge
\left(\frac{n\wedge N}{\log N'}\right)^{\frac{2m}{d}} \right).
\end{equation*}
Then for all sufficiently large $n,N$, with probability at least
$1 - 30 (N')^{-1}$ with $N'=n+N$,
\begin{equation}
\label{eq:thetahat-theta0-bound-clean}
\begin{aligned}
|\wh\theta(x_0)-\theta_0(x_0)|\le
&
\underbrace{\frac{8\sqrt{2}}{c_J}\sqrt{Q(x_0;\lambda)\log N'}
\left(\frac{B_{Y-f}(\theta_\lambda(x_0))}{\sqrt n}+\frac{B_f(\theta_\lambda(x_0))}{\sqrt N}\right)}_{\text{moment estimation error}}
\\
&\quad+\underbrace{\frac{32}{c_J}\,\kappa\,\|\eta(\cdot;\theta_\lambda(x_0))\|_{\cH}\,\sqrt{\frac{D(x_0;\lambda)\log N'}{N'}}}_{\text{weight estimation error}}
\\
&\quad+
\underbrace{\frac{1}{c_J}\,\|\eta(\cdot;\theta_0(x_0))\|_{\cH}\,\sqrt{\lambda\{D(x_0;\lambda)-Q(x_0;\lambda)\}}}_{\text{regularization bias}}.
\end{aligned}
\end{equation}
Moreover, define the stochastic-error scale
\begin{equation*}
\sN := \sqrt{\log N'} \left( \frac{8\sqrt{2}B_{Y-f}(\theta_\lambda(x_0))}{\sqrt n} + \frac{8\sqrt{2}B_{f}(\theta_\lambda(x_0))}{\sqrt N} + \frac{32\kappa\|\eta(\cdot;\theta_\lambda(x_0))\|_{\cH}}{\sqrt{N'}} \right)
\end{equation*}
which collects the contributions from labeled and unlabeled data. The formal oracle balance takes
$\lambda\asymp\{\sN/\|\eta(\cdot;\theta_0(x_0))\|_{\cH}\}^2$. If a sequence of this order also satisfies the preceding stability condition, then there exists $C_0>0$ such that
\begin{equation}\label{eqn:upper_bound_sobolev_rate}
|\wh\theta(x_0) - \theta_0(x_0)|\le\frac{2\sqrt{C_0}}{c_J}\,\|\eta(\cdot;\theta_0(x_0))\|_{\cH}^{\frac{d}{2m}}\,\sN^{\,1-\frac{d}{2m}}.
\end{equation}
\end{theorem}
\noindent
Compared to Theorem~\ref{thm:thetahat-theta0-bound-twofold}, the error bound in \eqref{eq:thetahat-theta0-bound-clean} shares a very similar structure, consisting of the moment estimation error, the weight estimation error, and the regularization bias. The primary distinction lies in the weight estimation error: specifically, it scales at the rate of $(N')^{-1/2}$ for the non-split estimator, as opposed to $N^{-1/2}$ in Theorem~\ref{thm:thetahat-theta0-bound-twofold}. This improvement arises because the non-split procedure fully exploits all available covariates from both the labeled and unlabeled datasets ($N'=n+N$) to estimate the empirical localization weight.

Establishing Theorem~\ref{thm:thetahat-theta0-bound} involves substantial technical challenges.
The key difficulty is the shared-design dependence between the labeled and unlabeled parts:
the empirical localization weight $\wh w_{x_0,\lambda}$ is learned from the pooled covariates
$\{\bar X_j\}_{j=1}^{N'}$, and is therefore not independent of
$\{\ell(Y_i;\theta)-\ell(f(X_i);\theta)\}_{i=1}^n$ nor $\{\ell(f(\wt X_u);\theta)\}_{u=1}^N$.
As a result, controlling the interaction term that couples random objects built from overlapping data is nontrivial.
We use the identity
$\Delta w=-(T_K+\lambda I)^{-1/2}(I+A(\lambda))^{-1}A(\lambda)(T_K+\lambda I)^{-1/2}K_{x_0}$
to express the interaction $I_2=\langle\Delta w,\wh\mu-\mu\rangle_{\cH}$ as a weak bilinear form. The leading term is a weighted V-statistic: its diagonal component has only $O(N')$ terms, while its off-diagonal component is degenerate after conditioning on the appropriate independent coordinates. The higher-order resolvent remainder is controlled by the event $\|A(\lambda)\|_{\op}\le1/2$. This argument keeps the analysis aligned with the empirical RKHS localization weight in \eqref{eq:finite-nosplit-weight}.

Specifically, suppose $N=n^{r_1}$ for fixed $r_1\ge1$ with $r_1>d/m$, and write $\lambda=n^{r_2}$ at the polynomial scale. The oracle balance has $r_2=-1$ up to logarithmic factors. Then the upper rate in \eqref{eqn:upper_bound_sobolev_rate} is $n^{-1/2+d/(4m)}$ up to logarithmic factors; after squaring, it matches the minimax lower rate in Theorem~\ref{thm:minimax-pointwise}.

Next, we establish the asymptotic properties of the pointwise estimator
$\wh\theta(x_0)$ for the conditional target $\theta_0(x_0)$.
\begin{theorem}
\label{thm:thetahat-asymp-new}
Under the settings of Theorem~\ref{thm:thetahat-theta0-bound} and Assumptions~\ref{ass:B_bound_predictor}--\ref{ass:homo_res},
define $J_\lambda(x_0)=J_\lambda(x_0;\theta_\lambda(x_0))$ and $
V(x_0):=\frac{1}{n}\sigma^2_{Y-f}(\theta_\lambda(x_0))
+\frac{1}{N}\sigma^2_{f}(\theta_\lambda(x_0))$,
where $
\sigma^2_{Y-f}(\theta)
    =\Var(w_{x_0,\lambda}(X)(\ell(Y;\theta)-\ell(f(X);\theta)))$ and $ \sigma^2_f(\theta)=\Var(w_{x_0,\lambda}(X)\,\ell(f(X);\theta))$.
Suppose that
\begin{align}
     n = o\left( \frac{N}{\log N} \right), \quad
    \lambda = o(n^{-1}), \quad \lambda^{-1} = o\left(\left( \frac{N'}{\log N'} \right)^{\frac{m}{d}}\wedge\left( \frac{n}{(\log N')^2} \right)^{\frac{2m}{d}} \right). \label{eq:cond-clean}
\end{align}
Then, as $n,N\to\infty$,
\[
\frac{J_\lambda(x_0)\big(\wh\theta(x_0)-\theta_0(x_0)\big)}{\sqrt{V(x_0)}}
\ \rightsquigarrow\
N(0,1).
\]
\end{theorem}
\noindent
To make the scaling requirements in~\eqref{eq:cond-clean} transparent, suppose that $N=n^{r_1}$ and $\lambda=n^{r_2}$
for some exponents $r_1>0$ and $r_2<0$. Ignoring logarithmic factors that do not change the
polynomial exponents, the conditions in~\eqref{eq:cond-clean} are satisfied provided that
\[
r_1>\max\{1,d/m\},
\quad
-\frac{m}{d}\min\{r_1,2\}\ <\ r_2\ <\ -1.
\]
Since $r_1>1$ ensures that the total sample size $N'=n+N$ is asymptotically dominated by the unlabeled sample size $N$, substituting $N$ with $N'$ does not alter the dominant polynomial exponents. Consequently, the scaling requirements for $r_1$ and $r_2$ remain exactly identical to those in the cross-fitting setting in Section~\ref{sec:theory}.

Theorem~\ref{thm:thetahat-asymp-new} suggests a natural construction of pointwise conditional
confidence intervals at $x_0$; the corresponding plug-in procedure is described in
equation~\eqref{eq:CI-new}. We now formalize the resulting asymptotic coverage.
In practice, we approximate
$J_\lambda(x_0)$ by the empirical Jacobian $\wh J_\lambda(x_0)$ in
\eqref{eq:H-hat-new}, and $V(x_0)$ by $\wh V(x_0)$ in \eqref{eq:V-hat-new}, both evaluated at
$\wh\theta(x_0)$. This leads to the confidence interval $\cC(x_0)$ in~\eqref{eq:CI-new}.
The next corollary states consistency of these plug-in quantities and asymptotic coverage.

\begin{corollary}\label{cor:CI-coverage-new}
Under the
setting of Theorem~\ref{thm:thetahat-asymp-new}, $\wh J_\lambda(x_0)/J_\lambda(x_0)\to_p1$ and
  $\wh V(x_0)/V(x_0)\to_p1$,
where the probability is taken over the labeled and unlabeled samples. Moreover,
for any $\alpha\in(0,1)$, as $n,N\to\infty$,
\[
  \PP(
    \theta_0(x_0)
    \in \cC(x_0))
  \to
  1-\alpha.
\]
\end{corollary}

\subsection{Proofs of Theorems in Appendix~\ref{app:PPCI_theory}}\label{app:proofs-of-PPCI-nosplit}
\noindent
We first review the notation used below. Let $x_0$ be an interior point of $\cX$.
For $\theta\in\Theta$, define
\[
r(y,x;\theta):=\ell(y;\theta)-\ell(f(x);\theta),
\qquad
u(x;\theta):=\ell(f(x);\theta).
\]
Let $\{(X_i,Y_i)\}_{i=1}^n$ be labeled data and $\{\wt X_u\}_{u=1}^N$ be
unlabeled covariates, mutually independent, with $X_i,\wt X_u\stackrel{i.i.d.}{\sim}\rho_X$.
Let $N':=n+N$ and let $\{\bar X_j\}_{j=1}^{N'}$ denote the pooled covariates
containing both $\{X_i\}$ and $\{\wt X_u\}$.

\noindent\textbf{RKHS operators and population weights.} Recall the covariance operator $T_K:=\EE[K_X\otimes K_X]:\cH\to\cH$ and define its pooled empirical version $\wh T_K:=N'^{-1}\sum_{j=1}^{N'}K_{\bar X_j}\otimes K_{\bar X_j}$.
Fix $\lambda>0$ and define the population and empirical localization weights
\[
w_{x_0,\lambda}:=(T_K+\lambda I)^{-1}K_{x_0},
\qquad
\wh w_{x_0,\lambda}:=(\wh T_K+\lambda I)^{-1}K_{x_0},
\qquad
\Delta w:=\wh w_{x_0,\lambda}-w_{x_0,\lambda}.
\]
Define leverage and effective dimension
\[
D(x_0;\lambda):=\langle K_{x_0},(T_K+\lambda I)^{-1}K_{x_0}\rangle_{\cH},
\qquad
D(\lambda):=\Tr\big((T_K+\lambda I)^{-1}T_K\big).
\]

\noindent\textbf{RKHS embeddings.}  Define the empirical and population RKHS embeddings
\begin{align*}
\wh{\mu}(\theta)
&:=
\frac{1}{n}\sum_{i=1}^{n} r(Y_i,X_i;\theta)\,K_{X_i}
+
\frac{1}{N}\sum_{u=1}^{N} u(\wt X_u;\theta)\,K_{\wt X_u}
\in\cH, \\
\mu(\theta)
&:=
\EE[r(Y,X;\theta)\,K_X]
+
\EE[u(X;\theta)\,K_X]
=
\EE[\ell(Y;\theta)\,K_X]
\in\cH.
\end{align*}
Recall that $\eta(x;\theta):=\EE[\ell(Y;\theta)| X=x]\in\cH$. Then
$\mu(\theta)=\EE[\eta(X;\theta)K_X]=T_K\,\eta(\cdot;\theta)$.

\noindent\textbf{Localized PPI moments.}
Define the population localized PPI moment and its empirical estimator:
\begin{align*}
\eta_\lambda(x_0;\theta)
&:=
\EE\big[w_{x_0,\lambda}(X)\,r(Y,X;\theta)\big]
+
\EE\big[w_{x_0,\lambda}(X)\,u(X;\theta)\big]
=
\EE\big[w_{x_0,\lambda}(X)\ell(Y;\theta)\big],\\
\wh \eta_\lambda(x_0;\theta)
&:=
\frac{1}{n}\sum_{i=1}^n
    \wh w_{x_0,\lambda}(X_i)\,
    r(Y_i,X_i;\theta)
  +
  \frac{1}{N}\sum_{u=1}^N
    \wh w_{x_0,\lambda}(\wt X_u)\,
    u(\wt X_u;\theta).
\end{align*}
Let the true target be defined by $\eta(x_0;\theta_0(x_0))=0$, and write
$\theta_0:=\theta_0(x_0)$.
Define the estimator $\wh\theta=\wh\theta(x_0)$ as any solution to
$\wh \eta_\lambda(x_0;\wh\theta)=0$.
Define the oracle estimator
\[
\wt \eta_\lambda(x_0;\theta)
:=
\frac{1}{n}\sum_{i=1}^n
    w_{x_0,\lambda}(X_i)\,r(Y_i,X_i;\theta)
  +
  \frac{1}{N}\sum_{u=1}^N
    w_{x_0,\lambda}(\wt X_u)\,u(\wt X_u;\theta).
\]

\noindent\textbf{Decomposition.}
Then, with the definitions of $\wh\mu(\theta)$ and $\mu(\theta)$, the reproducing
property yields
\begin{align*}
\wt \eta_\lambda(x_0;\theta)=\langle w_{x_0,\lambda},\wh\mu(\theta)\rangle_{\cH},\quad
\wh \eta_\lambda(x_0;\theta)=\langle \wh w_{x_0,\lambda},\wh\mu(\theta)\rangle_{\cH},
\quad
\eta_\lambda(x_0;\theta)=\langle w_{x_0,\lambda},\mu(\theta)\rangle_{\cH}.
\end{align*}
Hence the exact decomposition
\begin{equation}\label{eq:I123-decomp}
\wh \eta_\lambda(x_0;\theta)-\eta_\lambda(x_0;\theta)
=
I_1(x_0;\theta)+I_2(x_0;\theta)+I_3(x_0;\theta),
\end{equation}
where $I_1(x_0;\theta):=\langle w_{x_0,\lambda},\wh\mu(\theta)-\mu(\theta)\rangle_{\cH}$, $I_2(x_0;\theta):=\langle \Delta w,\wh\mu(\theta)-\mu(\theta)\rangle_{\cH}$, and $I_3(x_0;\theta):=\langle \Delta w,\mu(\theta)\rangle_{\cH}$.

\noindent\textbf{Regularization bias.} Since $\eta(x_0;\theta_0)=0$, we have
\[
\eta_\lambda(x_0;\theta_0)=\eta_\lambda(x_0;\theta_0)-\eta(x_0;\theta_0)
=\big((R_\lambda-I)\eta(\cdot;\theta_0)\big)(x_0),
\]
where $R_\lambda:=(T_K+\lambda I)^{-1}T_K$.
Using $K_{x_0}-T_Kw_{x_0,\lambda}=\lambda w_{x_0,\lambda}$, the identity $D(x_0;\lambda)-Q(x_0;\lambda)=\lambda\|w_{x_0,\lambda}\|_{\cH}^2$, and Cauchy--Schwarz,
\begin{equation}\label{eq:bias-bound}
|\eta_\lambda(x_0;\theta_0)|
\le
\|\eta(\cdot;\theta_0)\|_{\cH}\sqrt{\lambda\{D(x_0;\lambda)-Q(x_0;\lambda)\}}.
\end{equation}

\noindent\textbf{Jacobians.} Define the population and empirical Jacobians
\[
J_\lambda(x_0;\theta)
:=\partial_\theta \eta_\lambda(x_0;\theta)
=\EE\big[w_{x_0,\lambda}(X)\,\partial_\theta\ell(Y;\theta)\big],
\]
and
\[
\wh J_\lambda(x_0;\theta)
:=\partial_\theta\wh \eta_\lambda(x_0;\theta)=
\frac{1}{n}\sum_{i=1}^n
\wh w_{x_0,\lambda}(X_i)\,\partial_\theta r(Y_i,X_i;\theta)
+\frac{1}{N}\sum_{u=1}^N
\wh w_{x_0,\lambda}(\wt X_u)\,\partial_\theta u(\wt X_u;\theta).
\]

\noindent\textbf{Technical lemmas.}
Before proceeding to the main proofs of the theorems in Appendix~\ref{app:ppci-no-split}, we first provide three technical lemmas that will be frequently invoked throughout our analysis. Lemma~\ref{lem:A_bound_1/2_sobolev}, in particular, plays a foundational role in our theoretical framework. It sharply characterizes the concentration properties of $A(\lambda):=(T_K+\lambda I)^{-1/2}(\wh T_K-T_K)(T_K+\lambda I)^{-1/2}$.

\begin{lemma}[Sobolev interpolation and scaling bounds]\label{lem:sobolev_rate}
Assume that $x_0$ lies in the interior of
$\cX$. Then there exist constants $C_0>0$, $R_0>0$, and $C_1>0$ such that
\[
\sup_{\|v\|_{\cH}\le R\|v\|_{L_2}}
\frac{|v(x_0)|}{\|v\|_{L_2}}
>
C_0 R^{\frac{d}{2m}},
\qquad\text{for all }R\ge R_0,
\]
and
\[
|v(x_0)|\le \|v\|_{\infty}
\le
C_1\,\|v\|_{L_2}^{1-\frac{d}{2m}}\,\|v\|_{\cH}^{\frac{d}{2m}}.
\]
\end{lemma}

\begin{proof}[Proof of Lemma~\ref{lem:sobolev_rate}]
The second inequality follows from Theorem 3.8 of~\cite{adams2003sobolev}, together with the norm equivalence between $\cH$ and $H^m(\cX)$ and between $L^2(\rho_X)$ and Lebesgue $L^2$ on $\cX$. We give the short construction needed for the first inequality. Choose a fixed $\psi\in C_c^\infty(\mathbb R^d)$ supported in the unit ball with $\psi(0)=1$. Since $x_0$ is interior, $v_h(x)=\psi\{(x-x_0)/h\}$ is supported in $\cX$ for all sufficiently small $h$. Standard Sobolev scaling gives
$\|v_h\|_{L^2(\rho_X)}\asymp h^{d/2}$ and
$\|v_h\|_{\cH}\asymp h^{d/2-m}$. Hence
$\|v_h\|_{\cH}/\|v_h\|_{L^2(\rho_X)}\le C h^{-m}$, while
$|v_h(x_0)|/\|v_h\|_{L^2(\rho_X)}\ge c h^{-d/2}$.
Taking $h=(2C/R)^{1/m}$ for sufficiently large $R$ places $v_h$ in the displayed constraint set and yields the lower bound $cR^{d/(2m)}$.
\end{proof}

\begin{lemma}[Bounds for $\|w_{x_0,\lambda}\|_{\cH}^2$]\label{lem:fHnorm-bds}
Under the settings of Proposition~\ref{prop:Dx0-order-H} and Assumption~\ref{ass:eigenfunctions_bound}, there exist constants $0<c_w\le1$ depending only on the kernel--design spectral constants, the RKHS--Sobolev and $L^2$ norm-equivalence constants, the density bounds, the domain, and the fixed interior target point $x_0$ such that for all $\lambda\in(0, \lambda_0]$,
\begin{equation}\label{eq:fHnorm-two-sided}
c_w\,\frac{D(x_0;\lambda)}{\lambda}
\ \le\
\|w_{x_0,\lambda}\|_{\cH}^2
\ \le\
\frac{D(x_0;\lambda)}{\lambda}.
\end{equation}
\end{lemma}

\begin{lemma}[Population and empirical local roots]\label{lem:local-root-existence}
Suppose Assumptions~\ref{ass:B_bound_predictor}, \ref{ass:J_lambda_bound}, and~\ref{ass:eigenfunctions_bound} hold and $\lambda\to0$. Then, for all sufficiently small $\lambda$, the population equation $\eta_\lambda(x_0;\theta)=0$ has a unique solution $\theta_\lambda(x_0)$ in the $\delta_\theta$-neighborhood of $\theta_0(x_0)$. Moreover, let $\widehat\eta_\lambda$ denote either the two-fold or no-split empirical moment and $\widehat J_\lambda=\partial_\theta\widehat\eta_\lambda$. On any sequence of events on which
\[
|\widehat\eta_\lambda(x_0;\theta_\lambda(x_0))|=o(1),
\qquad
\inf_{\theta\in\Theta_0}|\widehat J_\lambda(x_0;\theta)|\ge c_J/2,
\]
the empirical equation has a unique local root $\widehat\theta(x_0)\in\Theta_0$ for all sufficiently large sample sizes.
\end{lemma}

\begin{proof}[Proof of Lemma~\ref{lem:local-root-existence}]
The Tikhonov bias bound, the uniform RKHS radius, and Proposition~\ref{prop:Dx0-order-H} give
$|\eta_\lambda(x_0;\theta_0(x_0))|\le B_\eta\sqrt{\lambda\{D(x_0;\lambda)-Q(x_0;\lambda)\}}=o(1)$.
Since the continuous derivative $J_\lambda(x_0;\theta)$ is bounded away from zero on the interval $\Theta_0$, it has a constant sign there. The interior margin in Assumption~\ref{ass:B_bound_predictor}, the intermediate value theorem, and the mean-value theorem therefore give a unique population root satisfying $|\theta_\lambda(x_0)-\theta_0(x_0)|\le c_J^{-1}|\eta_\lambda(x_0;\theta_0(x_0))|=o(1)$.
For the empirical equation, continuity and the lower bound on $|\widehat J_\lambda|$ likewise imply a constant derivative sign. Put
$e_n=|\widehat\eta_\lambda(x_0;\theta_\lambda(x_0))|$ and $r_n=3e_n/c_J$. For all sufficiently large sample sizes, the interval $[\theta_\lambda(x_0)-r_n,\theta_\lambda(x_0)+r_n]$ lies in $\Theta_0$, and the empirical moment has opposite signs at its two endpoints. Hence it has a root in this interval; the derivative bound makes that root unique.
\end{proof}

\begin{lemma}
\label{lem:A_bound_1/2_sobolev}
For a bounded linear operator $A$ on $\cH$, $\|A\|_{\op}$ denotes its operator norm. Write $\tau:=\frac{d}{2m}\in(0,1)$.
Then there exists a constant $C_a := C_1^2\,\tau^\tau(1-\tau)^{1-\tau}$
such that, for every $\delta\in(0,1]$, with probability at least $1-\delta$,
\begin{align}
\label{eq:A_bound_sobolev_general}
\|A(\lambda)\|_{\op}\le
\frac{4\,C_a}{3\,N'\,\lambda^{\tau}}\,
\log\Big(\frac{4\,\lambda^{\tau}D(\lambda)}{\delta}\Big)
+
\sqrt{\frac{2\,C_a}{N'\,\lambda^{\tau}}\,
\log\Big(\frac{4\,\lambda^{\tau}D(\lambda)}{\delta}\Big)}.
\end{align}
Moreover, under Proposition~\ref{prop:Dx0-order-H}, $D(\lambda)\asymp \lambda^{-\tau}$,
taking $\delta=(N')^{-1}$ yields
\begin{equation}\label{eq:A_bound_sobolev_log_simplify}
\|A(\lambda)\|_{\op}
\ \le\
\frac{8\,C_a}{3\,N'\,\lambda^{\tau}}\,
\log N'
+
\sqrt{\frac{4\,C_a}{N'\,\lambda^{\tau}}\,
\log N'}.
\end{equation}
Let $\sE$ denote the event on which \eqref{eq:A_bound_sobolev_log_simplify} holds. If the regularization satisfies
\begin{equation}\label{eq:gamma_rate_for_Ahalf}
\lambda\ \ge\
\Big(\frac{256\,C_a\,\log N'}{N'}\Big)^{1/\tau}
=
\Big(\frac{256\,C_a\,\log N'}{N'}\Big)^{2m/d},
\end{equation}
then $\sE$ implies $\|A(\lambda)\|_{\op}\le1/2$, and $\PP(\sE)\ge 1-(N')^{-1}$ for all sufficiently large $N'$.
\end{lemma}

\noindent\textbf{Bounds for $I_1$, $I_2$, and $I_3$.}
In the following, we derive detailed error bounds for the three components in \eqref{eq:I123-decomp}, showing that $I_1$ and $I_3$ constitute the leading terms, whereas $I_2$ is negligible and can be dominated by $I_1$ and $I_3$. First, we have the following results for $I_1$:
\begin{lemma}[Oracle sampling error]\label{lem:I1_clean}
Under Assumption~\ref{ass:B_bound_predictor}, for the oracle sampling error $I_1(x_0;\theta):=\wt \eta_\lambda(x_0;\theta)-\eta_\lambda(x_0;\theta)$,
if the regularization level satisfies
\begin{equation}\label{eq:lambda-lower-I1}
\lambda\ \ge\ \Big(\frac{8}{9}\,C_1^2\,\frac{\log N'}{n\wedge N}\Big)^{2m/d},
\end{equation}
then with probability at least $1-8(N')^{-1}$,
\begin{equation}\label{eq:I1-hp-bound}
|I_1(x_0;\theta)|
\ \le\
2\sqrt{2\log N'}\,
\sqrt{Q(x_0;\lambda)}\,
\left(
\frac{B_{Y-f}(\theta)}{\sqrt n}
+
\frac{B_{f}(\theta)}{\sqrt N}
\right).
\end{equation}
\end{lemma}

To control $I_2$, the main difficulty is that it couples two random objects built from overlapping data: the empirical RKHS localization weight is learned from the pooled covariates of size $n+N$, while the empirical moments use the same labeled and unlabeled covariates. We avoid pointwise leave-one-out control of $\wh w_{x_0,\lambda}$. Instead, we expand $\Delta w$ through the resolvent identity and control the resulting weak bilinear form by a diagonal/off-diagonal decomposition.

\begin{lemma}[High-probability bound for $\|\Delta w\|_{L^2(\rho_X)}$]\label{lem:delta_f_L2_bound}
Under Assumption~\ref{ass:eigenfunctions_bound} and the event $\sE$ in Lemma~\ref{lem:A_bound_1/2_sobolev}. Write $\Delta w:=\wh w_{x_0,\lambda}-w_{x_0,\lambda}$.
Then there exists an absolute constant $c_0>0$ such that with probability at least
$1-(2N')^{-1}$,
\begin{equation}\label{eq:Deltaf-L2-hp-clean}
\|\Delta w\|_{L^2(\rho_X)}
\le
c_0\,B_\phi\,
\sqrt{\frac{D(\lambda)\,D(x_0;\lambda)\,\log N'}{N'}}.
\end{equation}
\end{lemma}

\begin{lemma}[High-probability interaction bound for $I_2$]
\label{lem:I2-clean-hp-refined}
Under Assumptions~\ref{ass:B_bound_predictor}, \ref{ass:eigenfunctions_bound}, and the event $\sE$ from Lemma~\ref{lem:A_bound_1/2_sobolev}, for $I_2(x_0;\theta)
:=
\big\langle \Delta w,\ \wh\mu(\theta)-\mu(\theta)\big\rangle_{\cH}$, with probability at least $1-5(N')^{-1}$,
\begin{align}\label{eq:I2-total-hp-refined}
|I_2(x_0;\theta)|
\lesssim\;&
D(\lambda)\sqrt{D(x_0;\lambda)}
\sqrt{\frac{\log N'}{N'}}
\left(
\frac{B_{Y-f}(\theta)}{\sqrt n}
+
\frac{B_f(\theta)}{\sqrt N}
\right)
\notag\\
&+(B_{Y-f}(\theta)+B_f(\theta))D(\lambda)\sqrt{D(x_0;\lambda)}\frac{\log N'}{N'}
\notag\\
&+D(\lambda)^{3/2}\sqrt{D(x_0;\lambda)}
\frac{(\log N')^{3/2}}{N'}
\left(
\frac{B_{Y-f}(\theta)}{\sqrt n}
+
\frac{B_f(\theta)}{\sqrt N}
\right).
\end{align}
\end{lemma}

For $I_3$, we have the following results:
\begin{lemma}[Operator-induced term $I_3$]\label{lem:I3-clean}
Under Assumption~\ref{ass:eigenfunctions_bound}, recall that $I_3(x_0;\theta):=\langle \wh w_{x_0,\lambda}-w_{x_0,\lambda},\ \mu(\theta)\rangle_{\cH}$.
On the event $\sE$ in Lemma~\ref{lem:A_bound_1/2_sobolev}, in addition, assume the regularization level satisfies $D(\lambda)\sqrt{\frac{\log N'}{N'}}\to 0$.
Then for sufficiently large $N'$, with probability at least
$1-2(N')^{-1}$,
\[
|I_3(x_0;\theta)|
\ \le\
8\,\kappa \|\eta(\cdot;\theta)\|_{\cH}\sqrt{\frac{D(x_0;\lambda)\log N'}{N'}}.
\]
\end{lemma}

\noindent\textbf{Jacobian stability and invertibility.}
We analyze the stability and invertibility of the empirical Jacobian.
The proof follows the same decomposition strategy as in bounding $I_1$, $I_2$, and $I_3$, but requires a more careful treatment of the Jacobian-specific terms.
\begin{lemma}[Jacobian stability and invertibility]
\label{lem:Hhat-H-J123}
Under Assumptions~\ref{ass:deriv-bdd-jac}, \ref{ass:J_lambda_bound}, and~\ref{ass:eigenfunctions_bound}, suppose that the regularization level satisfies
\begin{equation}\label{eq:D-regime-J}
D(\lambda)\sqrt{\frac{\log N'}{N'}}\ \to\ 0,
\qquad\text{and}\qquad
\frac{D(x_0;\lambda)\log N'}{n\wedge N}\ \to\ 0.
\end{equation}
Then on the event $\sE$ from
Lemma~\ref{lem:A_bound_1/2_sobolev}, for all sufficiently large $N'$, with probability at least $1-14(N')^{-1}$
\begin{equation}\label{eq:H-invert-final}
\inf_{\theta\in \Theta_0}\big|\wh J_\lambda(x_0;\theta)\big|\ \ge\ c_J/2
\end{equation}
On $\sE$ and the same event of probability at least $1-14(N')^{-1}$, $\sup_{\theta\in \Theta_0}\big|\wh J_\lambda(x_0;\theta)^{-1}\big|
\ \le\ \frac{2}{c_J}$.
\end{lemma}

\noindent\textbf{Proof of Theorem~\ref{thm:thetahat-theta0-bound}.} With Lemma~\ref{lem:Hhat-H-J123}, we provide the following proof.
\begin{proof}[Proof of Theorem~\ref{thm:thetahat-theta0-bound}]
The moment and Jacobian events assembled in Steps~2--3 satisfy Lemma~\ref{lem:local-root-existence}; hence the locally unique empirical root used below exists on the same high-probability event for all sufficiently large sample sizes.
By definition, $\wh\theta(x_0)-\theta_0(x_0)=(\wh\theta(x_0)-\theta_\lambda(x_0))+(\theta_\lambda(x_0)-\theta_0(x_0))$.
We bound these two terms separately.

\noindent\textit{Step 1: linear expansion around $\theta_\lambda(x_0)$.}
Since $\wh\theta\in \Theta_0$ solves $\wh \eta_\lambda(x_0;\wh\theta(x_0))=0$ and
$\theta_\lambda(x_0)\in\Theta_0$ solves $\eta_\lambda(x_0;\theta_\lambda(x_0))=0$, the mean-value theorem yields
the existence of some $\wt\theta$ between $\wh\theta(x_0)$ and $\theta_\lambda(x_0)$ such that
\[
0
=
\wh \eta_\lambda(x_0;\wh\theta(x_0))
=
\wh \eta_\lambda(x_0;\theta_\lambda(x_0))
+
\wh J_\lambda(x_0;\wt\theta)\,(\wh\theta(x_0)-\theta_\lambda(x_0)),
\]
hence
\begin{equation}\label{eq:theta-stoch-linear-clean}
\wh\theta(x_0)-\theta_\lambda(x_0)
=
-\wh J_\lambda(x_0;\wt\theta)^{-1}\,\wh \eta_\lambda(x_0;\theta_\lambda(x_0)).
\end{equation}

\noindent\textit{Step 2: control $\wh J_\lambda^{-1}$ on $\Theta_0$.}
Let $\sE$ be the high-probability event from
Lemma~\ref{lem:A_bound_1/2_sobolev}; on this event, \eqref{eq:A_bound_sobolev_log_simplify} holds and $\|A(\lambda)\|_{\op}\le1/2$. By Lemma~\ref{lem:Hhat-H-J123}, on $\sE$ and for all large $N'$,
with probability at least $1-14(N')^{-1}$,
\[
\inf_{\theta\in \Theta_0}|\wh J_\lambda(x_0;\theta)|\ge c_J/2,
\qquad\text{and therefore}\qquad
\sup_{\theta\in \Theta_0}|\wh J_\lambda(x_0;\theta)^{-1}|\le 2/c_J.
\]
Define the event $\sE_J:=\{\inf_{\theta\in \Theta_0}|\wh J_\lambda(x_0;\theta)|\ge c_J/2\}$.
Then, using $\PP(\sE^c)\le (N')^{-1}$ from Lemma~\ref{lem:A_bound_1/2_sobolev},
\begin{equation}\label{eq:prob-step2}
\PP(\sE\cap\sE_J)\ \ge\ 1-15(N')^{-1}.
\end{equation}
On $\sE\cap\sE_J$, using \eqref{eq:theta-stoch-linear-clean},
\begin{equation}\label{eq:theta-stoch-bound-clean}
|\wh\theta(x_0)-\theta_\lambda(x_0)|
\le
\frac{2}{c_J}\,|\wh \eta_\lambda(x_0;\theta_\lambda(x_0))|.
\end{equation}

\noindent\textit{Step 3: control $\wh \eta_\lambda(x_0;\theta_\lambda(x_0)$.}
Since $\eta_\lambda(x_0;\theta_\lambda(x_0))=0$, the exact decomposition \eqref{eq:I123-decomp} gives
\[
\wh \eta_\lambda(x_0;\theta_\lambda(x_0))
-
\eta_\lambda(x_0;\theta_\lambda(x_0))
=
I_1(x_0;\theta_\lambda(x_0))+I_2(x_0;\theta_\lambda(x_0))+I_3(x_0;\theta_\lambda(x_0)),
\]
hence
\begin{equation}\label{eq:Mhat-at-thetalambda-clean}
|\wh \eta_\lambda(x_0;\theta_\lambda(x_0))|
\le
|I_1(x_0;\theta_\lambda(x_0))|
+|I_2(x_0;\theta_\lambda(x_0))|
+|I_3(x_0;\theta_\lambda(x_0))|.
\end{equation}

\noindent\textit{Step 3a: high-probability events.}
Let $\sE_{I_1}$ be the event on which Lemma~\ref{lem:I1_clean} yields \eqref{eq:I1-hp-bound}
(at $\theta=\theta_\lambda(x_0)$), let $\sE_{I_2}$ be the event on which Lemma~\ref{lem:I2-clean-hp-refined}
yields \eqref{eq:I2-total-hp-refined} (on $\sE$), and let $\sE_{I_3}$ be the event on which
Lemma~\ref{lem:I3-clean} yields its bound (on $\sE$).
Then for all large $N'$,
\[
\PP(\sE_{I_1}^c)\le 8(N')^{-1},\qquad
\PP(\sE_{I_2}^c| \sE)\le 5(N')^{-1},\qquad
\PP(\sE_{I_3}^c| \sE)\le 2(N')^{-1}.
\]
Combining with \eqref{eq:prob-step2} and a union bound gives the global event $\sG:=\sE\cap\sE_J\cap\sE_{I_1}\cap\sE_{I_2}\cap\sE_{I_3}$, satisfying $\PP(\sG)\ge1-30(N')^{-1}$.

\noindent\textit{Step 3b: plug in $I_1$ and $I_3$, and isolate $I_2$.}
On $\sG$, Lemma~\ref{lem:I1_clean} and Lemma~\ref{lem:I3-clean} yield
\begin{align}
|I_1(x_0;\theta_\lambda(x_0))|
&\le
2\sqrt{2\log N'}\,
\sqrt{Q(x_0;\lambda)}\,
\left(
\frac{B_{Y-f}(\theta_\lambda(x_0))}{\sqrt n}
+
\frac{B_{f}(\theta_\lambda(x_0))}{\sqrt N}
\right),
\label{eq:I_1-plug}\\
|I_3(x_0;\theta_\lambda(x_0))|
&\le
8\,\kappa \|\eta(\cdot;\theta_\lambda(x_0))\|_{\cH}\sqrt{\frac{D(x_0;\lambda)\log N'}{N'}}.
\label{eq:I_3-plug}
\end{align}

Also, on $\sG$ we have the interaction bound from Lemma~\ref{lem:I2-clean-hp-refined}:
\begin{align}
\begin{split}
|I_2(x_0;\theta_\lambda(x_0))|
&\lesssim
D(\lambda)\sqrt{D(x_0;\lambda)}
\sqrt{\frac{\log N'}{N'}}
\left(
\frac{B_{Y-f}(\theta_\lambda(x_0))}{\sqrt n}
+
\frac{B_f(\theta_\lambda(x_0))}{\sqrt N}
\right)
\\&\quad+
(B_{f}(\theta_\lambda(x_0))+B_{Y-f}(\theta_\lambda(x_0)))D(\lambda)\sqrt{D(x_0;\lambda)}\frac{\log N'}{N'}
\\&\quad+
D(\lambda)^{3/2}\sqrt{D(x_0;\lambda)}
\frac{(\log N')^{3/2}}{N'}
\left(
\frac{B_{Y-f}(\theta_\lambda(x_0))}{\sqrt n}
+
\frac{B_f(\theta_\lambda(x_0))}{\sqrt N}
\right).
\label{eq:I_2-plug}
\end{split}
\end{align}

\noindent\textit{Step 3c: $I_2$ is dominated by $I_1$ and $I_3$.}

Using Proposition~\ref{prop:Dx0-order-H}, $D(\lambda)\asymp D(x_0;\lambda)$. The first term in \eqref{eq:I_2-plug} equals the leading $I_1$ scale multiplied by $D(\lambda)/\sqrt{N'}$ up to logarithmic factors, and is therefore $o(1)$ times \eqref{eq:I_1-plug} under $D(\lambda)\sqrt{\log N'/N'}\to0$. The third term in \eqref{eq:I_2-plug} is smaller than the first term since $\sqrt{D(\lambda)}\log N'/\sqrt{N'}\to0$,
which follows from the same regularization regime. For the diagonal term, since $n\wedge N\le N'$,
\[
\frac{D(\lambda)\sqrt{D(x_0;\lambda)}\log N'/N'}
{\sqrt{D(x_0;\lambda)\log N'}\{n^{-1/2}+N^{-1/2}\}}
\le
D(\lambda)\sqrt{\frac{\log N'}{N'}}
=o(1).
\]
Thus the diagonal component is also absorbed by the $I_1$ bound, and $I_2(x_0;\theta_\lambda(x_0))$ is lower-order than the sum of the $I_1$ and $I_3$ contributions in \eqref{eq:I_1-plug}--\eqref{eq:I_3-plug}.

Hence, for all sufficiently large $N'$,
\begin{equation}\label{eq:I_123-simplify}
|I_1|+|I_2|+|I_3|
\ \le\
2|I_1|+2|I_3|
\qquad\text{on }\sG.
\end{equation}

\noindent\textit{Step 3d: the stochastic bound.}
Combining \eqref{eq:theta-stoch-bound-clean}, \eqref{eq:Mhat-at-thetalambda-clean},
\eqref{eq:I_1-plug}--\eqref{eq:I_3-plug}, and \eqref{eq:I_123-simplify}, we get on $\sG$:
\begin{align}
\begin{split}
|\wh\theta(x_0)-\theta_\lambda(x_0)|
&\le
\frac{2}{c_J}\big(|I_1|+|I_2|+|I_3|\big)
\ \le\
\frac{4}{c_J}\big(|I_1|+|I_3|\big)\\
&\le
\frac{8\sqrt{2}}{c_J}\,\sqrt{Q(x_0;\lambda)\log N'}\,
\left(
\frac{B_{Y-f}(\theta_\lambda(x_0))}{\sqrt n}
+
\frac{B_{f}(\theta_\lambda(x_0))}{\sqrt N}
\right)\\
&\qquad\quad+
\frac{32}{c_J}\,\kappa \|\eta(\cdot;\theta_\lambda(x_0))\|_{\cH}\,
\sqrt{\frac{D(x_0;\lambda)\log N'}{N'}}.
\label{eq:thetahat-thetalambda-clean}
\end{split}
\end{align}

\noindent\textit{Step 4: bias bound for $\theta_\lambda-\theta_0$.}
Since $\eta_\lambda(x_0;\theta_\lambda(x_0))=0$ and $\theta_0\in\Theta_0$, a mean-value expansion for
the population map $\eta_\lambda$ yields the existence of $\bar\theta$ between
$\theta_0(x_0)$ and $\theta_\lambda(x_0)$ such that
\[
0
=
\eta_\lambda(x_0;\theta_\lambda(x_0))
=
\eta_\lambda(x_0;\theta_0(x_0))+J_\lambda(x_0;\bar\theta)\,(\theta_\lambda(x_0)-\theta_0(x_0)).
\]
By Assumption~\ref{ass:J_lambda_bound}, $|J_\lambda(x_0;\bar\theta)|\ge c_J$, hence
\[
|\theta_\lambda(x_0)-\theta_0(x_0)|
\le
\frac{1}{c_J}\,|\eta_\lambda(x_0;\theta_0(x_0))|.
\]
Finally, by \eqref{eq:bias-bound},
\begin{equation}\label{eq:theta-bias-final-clean}
|\theta_\lambda(x_0)-\theta_0(x_0)|
\le
\frac{1}{c_J}\,\|\eta(\cdot;\theta_0(x_0))\|_{\cH}\sqrt{\lambda\{D(x_0;\lambda)-Q(x_0;\lambda)\}}.
\end{equation}

On the global event $\sG$,
combine \eqref{eq:thetahat-thetalambda-clean} and \eqref{eq:theta-bias-final-clean} to obtain
\eqref{eq:thetahat-theta0-bound-clean}.
\end{proof}

\noindent\textbf{Proof of Theorem~\ref{thm:thetahat-asymp-new}.} We proceed to prove the asymptotic result in Theorem~\ref{thm:thetahat-asymp-new}.
We first establish the following Lemma~\ref{lem:V-lb}, which provides a lower bound for the variance. We then give the proof of Theorem~\ref{thm:thetahat-asymp-new}.
\begin{lemma}\label{lem:V-lb}
Under Assumptions~\ref{ass:eigenfunctions_bound} and~\ref{ass:homo_res},
let $V(x_0;\theta)
:=\frac{1}{n}\sigma^2_{Y-f}(\theta)+\frac{1}{N}\sigma^2_{f}(\theta)$.
There exist constants $\lambda_0\in(0,1)$ and $c_0>0$ independent of $\lambda$
such that for all $\lambda\in(0,\lambda_0]$,
\begin{equation}\label{eq:c-lambda-lb}
V(x_0;\theta)\ \ge\ \frac{\underline\sigma^2}{n}Q(x_0;\lambda)
\ \ge\ \Big(\frac{\underline\sigma^2\,c_0}{2}\Big)\,\frac{D(x_0;\lambda)}{n}.
\end{equation}
\end{lemma}

\begin{lemma}[Relative consistency of the plug-in variance]\label{lem:variance-ratio}
Under the conditions of Theorem~\ref{thm:thetahat-asymp-new} for the no-split procedure, or Theorem~\ref{thm:thetahat-asymp-twofold} for the two-fold procedure, the corresponding plug-in variance estimator satisfies
\[
|\widehat V(x_0)-V(x_0)|=o_p\{Q(x_0;\lambda)/n\},
\qquad
\widehat V(x_0)/V(x_0)\to_p1.
\]
For either procedure, $\widehat J_\lambda(x_0)/J_\lambda(x_0)\to_p1$.
\end{lemma}
\begin{proof}[Proof of Theorem~\ref{thm:thetahat-asymp-new}]
First, we prove
\begin{equation}\label{eq:asymp-normal-H-new}
\frac{J_\lambda(x_0)\big(\wh\theta(x_0)-\theta_\lambda(x_0)\big)}{\sqrt{V(x_0)}}
\ \to\
N(0,1).
\end{equation}

\noindent\textit{Step 1: local linearization at $\theta_\lambda(x_0)$.}
As in the upper bound proof (Theorem~\ref{thm:thetahat-theta0-bound}), a mean-value expansion gives
\begin{equation}\label{eq:lin-asymp-new}
\wh\theta(x_0)-\theta_\lambda(x_0)
=
-\wh J_\lambda(x_0;\wt\theta(x_0))^{-1}\,
\wh \eta_\lambda(x_0;\theta_\lambda(x_0)),
\end{equation}
for some $\wt\theta(x_0)$ between $\wh\theta(x_0)$ and $\theta_\lambda(x_0)$.
Since $\eta_\lambda(x_0;\theta_\lambda(x_0))=0$, the decomposition \eqref{eq:I123-decomp} yields
\begin{equation}\label{eq:M-decomp-asymp-new}
\wh \eta_\lambda(x_0;\theta_\lambda(x_0))
=
I_1(x_0;\theta_\lambda(x_0))
+
I_2(x_0;\theta_\lambda(x_0))
+
I_3(x_0;\theta_\lambda(x_0)).
\end{equation}

\noindent\textit{Step 2: Jacobian stability.}
By Lemma~\ref{lem:Hhat-H-J123} and the rate conditions \eqref{eq:cond-clean},
\[
\sup_{\theta\in \Theta_0}\big|\wh J_\lambda(x_0;\theta)-J_\lambda(x_0;\theta)\big|
\ \to_p\ 0,
\qquad\text{and}\qquad
\sup_{\theta\in \Theta_0}\big|\wh J_\lambda(x_0;\theta)^{-1}\big|=O_p(1),
\]
where $J_\lambda(x_0;\theta):=\partial_\theta \eta_\lambda(x_0;\theta)$ and
$\inf_{\theta\in \Theta_0}|J_\lambda(x_0;\theta)|\ge c_J$ by Assumption~\ref{ass:J_lambda_bound}.
In particular, since $\wt\theta(x_0)\in \Theta_0$ w.p.\ tending to one,
\begin{equation}\label{eq:Hhat-inv-to-Hinv-new}
\wh J_\lambda(x_0;\wt\theta(x_0))^{-1}\ \to_p\ J_\lambda(x_0)^{-1}.
\end{equation}
Combining \eqref{eq:lin-asymp-new}--\eqref{eq:M-decomp-asymp-new},
\begin{equation}\label{eq:lin-asymp-2-new}
\wh\theta(x_0)-\theta_\lambda(x_0)
=
-\wh J_\lambda(x_0;\wt\theta(x_0))^{-1}I_1(x_0;\theta_\lambda(x_0))
+R_{n,N},
\end{equation}
where
\[
R_{n,N}
:=
-\wh J_\lambda(x_0;\wt\theta(x_0))^{-1}\{I_2(x_0;\theta_\lambda(x_0))+I_3(x_0;\theta_\lambda(x_0))\}.
\]
Thus, to prove \eqref{eq:asymp-normal-H-new}, it suffices to show:
\begin{equation}\label{eq:goal-CLT-I1-new}
\frac{I_1(x_0;\theta_\lambda(x_0))}{\sqrt{V(x_0)}}\to N(0,1),
\end{equation}
and
\begin{equation}\label{eq:goal-negligible-new}
\frac{I_2(x_0;\theta_\lambda(x_0))}{\sqrt{V(x_0)}}=o_p(1),
\qquad
\frac{I_3(x_0;\theta_\lambda(x_0))}{\sqrt{V(x_0)}}=o_p(1).
\end{equation}

\noindent\textit{Step 3: Lindeberg--Feller CLT for $I_1/\sqrt{V(x_0)}$.}
Define the centered summands
\[
\zeta_{i,n}
:=
\frac{1}{n}\Big\{w_{x_0,\lambda}(X_i)\,r(Y_i,X_i;\theta_\lambda(x_0))
-\EE[w_{x_0,\lambda}(X)\,r(Y,X;\theta_\lambda(x_0))]\Big\},
\qquad i=1,\dots,n,
\]
and
\[
\zeta_{n+u,N}
:=
\frac{1}{N}\Big\{w_{x_0,\lambda}(\wt X_u)\,u(\wt X_u;\theta_\lambda(x_0))
-\EE[w_{x_0,\lambda}(X)\,u(X;\theta_\lambda(x_0))]\Big\},
\qquad u=1,\dots,N.
\]
Then $I_1(x_0;\theta_\lambda(x_0))=\sum_{k=1}^{n+N}\zeta_k$ is a sum of independent, mean-zero terms and
\[
\sum_{k=1}^{n+N}\Var(\zeta_k)
=
\frac{1}{n}\sigma^2_{Y-f}(\theta_\lambda(x_0))+\frac{1}{N}\sigma^2_{f}(\theta_\lambda(x_0))
=
V(x_0).
\]
By the Lindeberg--Feller theorem, it suffices to verify that for every $\varepsilon>0$,
\begin{equation}\label{eq:Lindeberg-new}
\frac{1}{V(x_0)}\sum_{k=1}^{n+N}\EE\big[\zeta_k^2\,\mathbbm{1}\{|\zeta_k|>\varepsilon \sqrt{V(x_0)}\}\big]\to 0.
\end{equation}
Under Assumption~\ref{ass:eigenfunctions_bound}, for any $x\in\cX$,
\begin{align}
|w_{x_0,\lambda}(x)|^2
&=
\big|\langle (T_K+\lambda I)^{-1}K_{x_0},\ K_x\rangle_{\cH}\big|^2
=
\big|\langle K_{x_0},\ (T_K+\lambda I)^{-1}K_x\rangle_{\cH}\big|^2
\notag\\
&\le
\langle K_{x_0},(T_K+\lambda I)^{-1}K_{x_0}\rangle_{\cH}\,
\langle K_x,(T_K+\lambda I)^{-1}K_x\rangle_{\cH}
=
D(x_0;\lambda)\,D(x;\lambda).
\label{eq:f-CS-leverage}
\end{align}
Moreover, in the spectral expansion, $D(x;\lambda)\le
B_\phi^2\,D(\lambda)$,
so combining with \eqref{eq:f-CS-leverage} yields
\begin{equation}\label{eq:f-sup-sharp}
\|w_{x_0,\lambda}\|_\infty
\ \le\
B_\phi\,\sqrt{D(x_0;\lambda)\,D(\lambda)}.
\end{equation}
Under Assumption~\ref{ass:B_bound_predictor},
$|r(Y,X;\theta_\lambda(x_0))|\le B_{Y-f}(\theta_\lambda(x_0))$ and $|u(X;\theta_\lambda(x_0))|\le B_{f}(\theta_\lambda(x_0))$.
Hence
\[
\max_{1\le i\le n}|\zeta_{i,n}|
\le
\frac{2\,\|w_{x_0,\lambda}\|_\infty\,B_{Y-f}(\theta_\lambda(x_0))}{n},
\qquad
\max_{1\le u\le N}|\zeta_{n+u,N}|
\le
\frac{2\,\|w_{x_0,\lambda}\|_\infty\,B_{f}(\theta_\lambda(x_0))}{N}.
\]
Therefore, it is enough to show
\begin{equation}\label{eq:max-to-zero-new}
\frac{\max_{1\le k\le n+N}|\zeta_k|}{\sqrt{V(x_0)}}\to 0,
\end{equation}
since \eqref{eq:max-to-zero-new} implies the indicator in \eqref{eq:Lindeberg-new} is eventually zero
uniformly in $k$, yielding \eqref{eq:Lindeberg-new}.
By Lemma~\ref{lem:V-lb} and $\lambda\to 0$, we have $c(\lambda)\ge c_0$ for all large $n$,
and therefore
\[
V(x_0)
\ge
\Big(\frac{\underline\sigma^2\,c_0}{2}\Big)\,\frac{D(x_0;\lambda)}{n}.
\]
Using \eqref{eq:f-sup-sharp} and Lemma~\ref{lem:V-lb},
\begin{align*}
\frac{\max_{1\le i\le n}|\zeta_{i,n}|}{\sqrt{V(x_0)}}
&\lesssim
\frac{\|w_{x_0,\lambda}\|_\infty/n}{\sqrt{D(x_0;\lambda)/n}}
\ \le\
\frac{B_\phi\sqrt{D(x_0;\lambda)D(\lambda)}/n}{\sqrt{D(x_0;\lambda)/n}}
=
B_\phi\,\sqrt{\frac{D(\lambda)}{n}}
\ \to\ 0,
\\
\frac{\max_{1\le u\le N}|\zeta_{n+u,N}|}{\sqrt{V(x_0)}}
&\lesssim
\frac{\|w_{x_0,\lambda}\|_\infty/N}{\sqrt{D(x_0;\lambda)/n}}
\ \le\
\frac{B_\phi\sqrt{D(x_0;\lambda)D(\lambda)}/N}{\sqrt{D(x_0;\lambda)/n}}
=
B_\phi\,\frac{\sqrt{n\,D(\lambda)}}{N}
\ \to\ 0,
\end{align*}
by the condition \eqref{eq:cond-clean}. Thus \eqref{eq:max-to-zero-new} holds,
which proves the Lindeberg condition \eqref{eq:Lindeberg-new} and hence establishes
\eqref{eq:goal-CLT-I1-new}.

\noindent\textit{Step 4: $I_2$ and $I_3$ are negligible.}
From Lemma~\ref{lem:V-lb}, we have $\sqrt{V(x_0)}\gtrsim \sqrt{D(x_0;\lambda)/n}$.

By Lemma~\ref{lem:I2-clean-hp-refined}, the leading part of $I_2$ satisfies
\[
|I_2(x_0;\theta_\lambda(x_0))|
\lesssim_p
D(\lambda)\sqrt{D(x_0;\lambda)}
\sqrt{\frac{\log N'}{N'}}
\left(\frac{1}{\sqrt n}+\frac{1}{\sqrt N}\right)
+\text{lower-order terms}.
\]
Using Lemma~\ref{lem:V-lb},
\[
\frac{|I_2(x_0;\theta_\lambda(x_0))|}{\sqrt{V(x_0)}}
\lesssim_p
D(\lambda)\sqrt{\frac{\log N'}{N'}}
\left(1+\sqrt{\frac nN}\right)+o_p(1)
\to 0,
\label{eq:I2-negl-new}
\]
by the condition~\eqref{eq:cond-clean}. Hence $I_2(x_0;\theta_\lambda(x_0))/\sqrt{V(x_0)}=o_p(1)$.

Using Lemma~\ref{lem:I3-clean},
\[
|I_3(x_0;\theta_\lambda(x_0))|
\ \lesssim_p\
\kappa\|\eta(\cdot;\theta_\lambda(x_0))\|_{\cH}\sqrt{\frac{D(x_0;\lambda)\log N'}{N'}},
\]
and thus
\begin{align*}
\frac{|I_3(x_0;\theta_\lambda(x_0))|}{\sqrt{V(x_0)}}
&\lesssim_p
\kappa\|\eta(\cdot;\theta_\lambda(x_0))\|_{\cH}
\sqrt{\frac{D(x_0;\lambda)\log N'}{N'}}\cdot \sqrt{\frac{n}{D(x_0;\lambda)}}
\\&=
\kappa\|\eta(\cdot;\theta_\lambda(x_0))\|_{\cH}\sqrt{\frac{n\log N'}{N'}}
\to 0,
\end{align*}
by the condition~\eqref{eq:cond-clean}. Hence $I_3(x_0;\theta_\lambda(x_0))/\sqrt{V(x_0)}=o_p(1)$. Then \eqref{eq:goal-negligible-new} holds.

\noindent\textit{Step 5: conclude.}
Multiply \eqref{eq:lin-asymp-2-new} by $J_\lambda(x_0)/\sqrt{V(x_0)}$:
\begin{align*}
\frac{J_\lambda(x_0)\big(\wh\theta(x_0)-\theta_\lambda(x_0)\big)}{\sqrt{V(x_0)}}
=&
-\frac{J_\lambda(x_0)}{\wh J_\lambda(x_0;\wt\theta(x_0))}\cdot\frac{I_1(x_0;\theta_\lambda(x_0))}{\sqrt{V(x_0)}}
\\&-\frac{J_\lambda(x_0)}{\wh J_\lambda(x_0;\wt\theta(x_0))}\cdot\frac{I_2(x_0;\theta_\lambda(x_0))+I_3(x_0;\theta_\lambda(x_0))}{\sqrt{V(x_0)}}.
\end{align*}
By \eqref{eq:Hhat-inv-to-Hinv-new}, $\frac{J_\lambda(x_0)}{\wh J_\lambda(x_0;\wt\theta(x_0))}\to_p 1$.
By \eqref{eq:goal-CLT-I1-new}, $\frac{I_1(x_0;\theta_\lambda(x_0))}{\sqrt{V(x_0)}}\to N(0,1)$.
By \eqref{eq:goal-negligible-new}, $\frac{I_2(x_0;\theta_\lambda(x_0))+I_3(x_0;\theta_\lambda(x_0))}{\sqrt{V(x_0)}}=o_p(1)$.
Slutsky's theorem yields~\eqref{eq:asymp-normal-H-new}.

\noindent\textit{Step 6: bias is negligible.}
By the population bias bound \eqref{eq:bias-bound} in Theorem~\ref{thm:thetahat-theta0-bound},
\[
|\theta_\lambda(x_0)-\theta_0(x_0)|
\ \le\
\frac{1}{c_J}\,\|\eta(\cdot;\theta_0(x_0))\|_{\cH}\sqrt{\lambda\{D(x_0;\lambda)-Q(x_0;\lambda)\}}.
\]
Divide by the standard error scale $\sqrt{V(x_0)}/|J_\lambda(x_0)|$ and use Lemma~\ref{lem:V-lb}:
\[
\frac{|J_\lambda(x_0)|\,|\theta_\lambda(x_0)-\theta_0(x_0)|}{\sqrt{V(x_0)}}
\ \lesssim\
\frac{\sqrt{\lambda\{D(x_0;\lambda)-Q(x_0;\lambda)\}}}{\sqrt{Q(x_0;\lambda)/n}}
\asymp
\sqrt{n\lambda}
\ \to\ 0.
\]
Hence,
\begin{align*}
\frac{J_\lambda(x_0)\big(\wh\theta(x_0)-\theta_0(x_0)\big)}{\sqrt{V(x_0)}}
&=
\frac{J_\lambda(x_0)\big(\wh\theta(x_0)-\theta_\lambda(x_0)\big)}{\sqrt{V(x_0)}}
+\frac{J_\lambda(x_0)\big(\theta_\lambda(x_0)-\theta_0(x_0)\big)}{\sqrt{V(x_0)}}
\\&=
\frac{J_\lambda(x_0)\big(\wh\theta(x_0)-\theta_\lambda(x_0)\big)}{\sqrt{V(x_0)}}+o(1).
\end{align*}
Combining with \eqref{eq:asymp-normal-H-new} concludes the proof.
\end{proof}

\noindent\textbf{Proof of Corollary~\ref{cor:CI-coverage-new}.} We also provide the following proof.
\begin{proof}[Proof of Corollary~\ref{cor:CI-coverage-new}]
All probabilities and expectations below are taken over the samples, with $x_0$ treated as fixed.
The proof proceeds in two steps.

\noindent\textit{Step 1: relative consistency of the plug-in quantities.}
Lemma~\ref{lem:variance-ratio} applies to the pooled-design no-split weight and gives
$\wh V(x_0)/V(x_0)\to_p1$ and $\wh J_\lambda(x_0)/J_\lambda(x_0)\to_p1$. Its proof uses the resolvent energy bound for the empirical weight and therefore does not invoke an iid law of large numbers for a weight trained on the same covariates.

\noindent\textit{Step 2: asymptotic coverage.}
Define the studentized statistic
\[
  T_{n,N}(x_0)
  :=
  \frac{
    \wh J_\lambda(x_0)\big(\wh\theta(x_0)-\theta_0(x_0)\big)
  }{
    \sqrt{\wh V(x_0)}
  }.
\]
By Step~1, both plug-in ratios converge to one, so Slutsky's theorem yields
$T_{n,N}(x_0)\to N(0,1)$. The event $\{\theta_0(x_0)\in\cC(x_0)\}$ is equivalent to
$\{|T_{n,N}(x_0)|\le z_{1-\alpha/2}\}$, hence
\[
  \PP\big(\theta_0(x_0)\in\cC(x_0)\big)
  =
  \PP\big(|T_{n,N}(x_0)|\le z_{1-\alpha/2}\big)
  \ \to\ 1-\alpha.
\]
This proves the corollary.
\end{proof}

\subsection{Proofs of Technical Lemmas in Appendix~\ref{app:proofs-of-PPCI-nosplit}}\label{app:proofs-of-lemmas-nosplit}
\begin{proof}[Proof of Lemma~\ref{lem:fHnorm-bds}]
The Tikhonov energy identity gives the exact relation $\lambda\|w_{x_0,\lambda}\|_{\cH}^2=D(x_0;\lambda)-Q(x_0;\lambda)$.
Proposition~\ref{prop:Dx0-order-H} states that
$D(x_0;\lambda)-Q(x_0;\lambda)\asymp D(x_0;\lambda)$. Dividing by $\lambda$ therefore yields
$\|w_{x_0,\lambda}\|_{\cH}^2\asymp D(x_0;\lambda)/\lambda$, as claimed.
\end{proof}

\begin{proof}[Proof of Lemma~\ref{lem:A_bound_1/2_sobolev}]
\textit{Step 1: rewrite $A(\lambda)$.}
Define the preconditioned feature map
\[
\psi_\lambda(x)
\ :=\
(T_K+\lambda I)^{-1/2}K_x
\ \in\ \cH,
\qquad
K_x:=K(x,\cdot).
\]
Then, using $T_K=\EE[K_X\otimes K_X]$ and $\wh T_K=\frac1{N'}\sum_{i=1}^{N'}K_{X_i}\otimes K_{X_i}$,
\begin{align*}
(T_K+\lambda I)^{-1/2}\wh T_K(T_K+\lambda I)^{-1/2}
&=
\frac1{N'}\sum_{i=1}^{N'} \psi_\lambda(X_i)\otimes \psi_\lambda(X_i),\\
(T_K+\lambda I)^{-1/2}T_K(T_K+\lambda I)^{-1/2}
&=
\EE\big[\psi_\lambda(X)\otimes \psi_\lambda(X)\big].
\end{align*}
Therefore
\begin{equation}\label{eq:A_as_empirical_average}
\|A(\lambda)\|_{\op}
=
\Big\|
\frac1{N'}\sum_{i=1}^{N'}\psi_\lambda(X_i)\otimes\psi_\lambda(X_i)
-
\EE[\psi_\lambda\otimes\psi_\lambda]
\Big\|_{\op}
=
\Big\|\frac1{N'}\sum_{i=1}^{N'} Z_i\Big\|_{\op},
\end{equation}
where the summands are i.i.d.\ centered self-adjoint operators
\[
Z_i
:=
\psi_\lambda(X_i)\otimes\psi_\lambda(X_i)
-
\EE[\psi_\lambda(X)\otimes\psi_\lambda(X)],
\qquad
\EE[Z_i]=0.
\]

\noindent\textit{Step 2: uniform bound.}
For each $x\in\cX$,
\[
\|\psi_\lambda(x)\|_{\cH}^2
=
\big\langle K_x,(T_K+\lambda I)^{-1}K_x\big\rangle_{\cH}
=:D(x;\lambda).
\]
In particular,
\begin{align*}
\|\psi_\lambda(x)\otimes\psi_\lambda(x)\|_{\op}
=
\|\psi_\lambda(x)\|_{\cH}^2
=
D(x;\lambda),
\\
\|Z_i\|_{\op}\le
\|\psi_\lambda(X_i)\otimes\psi_\lambda(X_i)\|_{\op}
+
\|\EE[\psi_\lambda\otimes\psi_\lambda]\|_{\op}.
\end{align*}
Since $\EE[\psi_\lambda\otimes\psi_\lambda] = T_K(T_K+\lambda I)^{-1} \preceq I$,
we may bound $\|\EE[\psi_\lambda\otimes\psi_\lambda]\|_{\op}\le 1 \le \sup_{x\in\cX} D(x;\lambda)$
for all sufficiently small $\lambda$, and hence
\begin{equation}\label{eq:Zi_op_bound_by_supD}
\|Z_i\|_{\op}\ \le\ 2\,\sup_{x\in\cX}D(x;\lambda).
\end{equation}
Thus, to apply a Bernstein-type inequality to \eqref{eq:A_as_empirical_average},
it remains to control $\sup_x D(x;\lambda)$ sharply.

\noindent\textit{Step 3: upper bound for $\sup_x D(x;\lambda)$.}
Using the variational characterization of $D(x;\lambda)$,
one has for each fixed $x\in\cX$,
\begin{align}
D(x;\lambda)
=
\sup_{v\in\cH\setminus\{0\}}
\frac{\langle K_x,v\rangle_{\cH}^2}{\langle v,(T_K+\lambda I)v\rangle_{\cH}}=
\sup_{v\in\cH\setminus\{0\}}
\frac{v(x)^2}{\|v\|_{L^2(\rho_X)}^2+\lambda\|v\|_{\cH}^2}.
\label{eq:Dx_variational}
\end{align}
By Lemma~\ref{lem:sobolev_rate},
for every $v\in\cH$ and $x\in\cX$,
\begin{equation}\label{eq:interp_use}
|v(x)|^2
\ \le\
C_1^2\,\|v\|_{L^2(\rho_X)}^{2(1-\tau)}\,\|v\|_{\cH}^{2\tau}.
\end{equation}
Write $a:=\|v\|_{L^2(\rho_X)}^2$ and $b:=\|v\|_{\cH}^2$. Then \eqref{eq:interp_use} gives
$v(x)^2\le C_1^2\,a^{1-\tau}b^{\tau}$, and \eqref{eq:Dx_variational} implies
\begin{equation}\label{eq:Dx_reduce_ab}
D(x;\lambda)
\ \le\
C_1^2\,
\sup_{a>0,b>0}\frac{a^{1-\tau}b^\tau}{a+\lambda b}.
\end{equation}
To evaluate the supremum in \eqref{eq:Dx_reduce_ab}, we can normalize the denominator by setting $a+\lambda b=1$, yielding
\[
\sup_{a>0,b>0}\frac{a^{1-\tau}b^\tau}{a+\lambda b}
=
\sup_{\substack{a>0,b>0\\ a+\lambda b=1}}
a^{1-\tau}b^\tau.
\]
Substituting $b=(1-a)/\lambda$ for $a\in(0,1)$, we maximize the one-dimensional function:
\[
\phi(a):=
a^{1-\tau}\Big(\frac{1-a}{\lambda}\Big)^\tau
=
\lambda^{-\tau}\,a^{1-\tau}(1-a)^\tau.
\]
Taking the derivative of $\log\phi(a)$ yields:
\[
\frac{d}{da}\log\phi(a)
=
\frac{1-\tau}{a}-\frac{\tau}{1-a}=0
\quad\implies\quad
a^\star=1-\tau.
\]
Hence,
\begin{equation}\label{eq:phi_max}
\sup_{\substack{a>0,b>0\\ a+\lambda b=1}}
a^{1-\tau}b^\tau
=
\lambda^{-\tau}\,(1-\tau)^{1-\tau}\tau^\tau.
\end{equation}
Combining \eqref{eq:Dx_reduce_ab} and \eqref{eq:phi_max} yields the uniform leverage-score bound:
\begin{equation}\label{eq:supD_bound}
\sup_{x\in\cX}D(x;\lambda)
\ \le\
C_1^2\,\tau^\tau(1-\tau)^{1-\tau}\,\lambda^{-\tau}
\ =\
C_a\,\lambda^{-\tau}.
\end{equation}

\noindent\textit{Step 4: apply  Bernstein inequality.}
We now plug \eqref{eq:supD_bound} into a standard Bernstein inequality for sums of
independent self-adjoint random operators.

First, \eqref{eq:Zi_op_bound_by_supD} and \eqref{eq:supD_bound} give the uniform almost sure bound $\|Z_i\|_{\op}\le2C_a\lambda^{-\tau}$.
Second, define the second-moment operator
\[
\Sigma_\lambda
:=
\EE[\psi_\lambda(X)\otimes\psi_\lambda(X)]
=
(T_K+\lambda I)^{-1/2}T_K(T_K+\lambda I)^{-1/2}.
\]
Its trace equals the effective dimension:
\[
\Tr(\Sigma_\lambda)
=
\Tr\big(T_K(T_K+\lambda I)^{-1}\big)
=
D(\lambda).
\]
Then for any $\delta\in(0,1]$, with probability at least $1-\delta$,
\[
\Big\|\frac1{N'}\sum_{i=1}^{N'} Z_i\Big\|_{\op}
\ \le\
\frac{4\,\sup_x D(x;\lambda)}{3N'}\,
\log\Big(\frac{4\,D(\lambda)}{\delta}\Big)
+
\sqrt{\frac{2\,\sup_x D(x;\lambda)}{N'}\,
\log\Big(\frac{4\,D(\lambda)}{\delta}\Big)}.
\]
We now replace $\sup_x D(x;\lambda)$ by \eqref{eq:supD_bound}.
Since our bound scales as $\lambda^{-\tau}$, it is convenient
to rewrite the logarithmic factor in terms of $\lambda^\tau D(\lambda)$ by absorbing the $\lambda^{-\tau}$ inside the logarithm (adjusting universal constants appropriately), which yields exactly \eqref{eq:A_bound_sobolev_general}.

\noindent\textit{Step 5: specialize to $\delta=(N')^{-1}$.}
Take $\delta=(N')^{-1}$ so the confidence is $1-(N')^{-1}$. Let
\[
L_\lambda
:=
\log\Big(\frac{4\lambda^\tau D(\lambda)}{\delta}\Big)
=
\log\Big(4\lambda^\tau D(\lambda)\,N'\Big).
\]
Under $D(\lambda)\le C_D\lambda^{-\tau}$ (Proposition~\ref{prop:Dx0-order-H}), we have $\lambda^\tau D(\lambda)\le C_D$. Thus,
\[
L_\lambda
\le
\log(4C_D N')
=
\log(4C_D)+\log N'.
\]
For all sufficiently large $N'$ (such that $N' \ge 4C_D$), we have $L_\lambda\le 2\log N'$.
Plugging this upper bound into \eqref{eq:A_bound_sobolev_general} directly yields \eqref{eq:A_bound_sobolev_log_simplify}.

\noindent\textit{Step 6: ensure the right-hand side is $<1/2$.}
Let $U := \frac{C_a\,L_\lambda}{N'\lambda^\tau}$. Then \eqref{eq:A_bound_sobolev_general} reads $\|A(\lambda)\|_{\op}\ \le\ \frac{4}{3}U+\sqrt{2U}$.
To ensure $\|A(\lambda)\|_{\op} \le 1/2$, it suffices to require $U \le 1/128$ (since $\frac{4}{3}(\frac{1}{128}) + \sqrt{\frac{2}{128}} < 1/2$).
Using $L_\lambda\le 2\log N'$ from Step~5, the condition $U \le 1/128$ is implied by
\[
\frac{C_a\,(2\log N')}{N'\lambda^\tau}\ \le\ \frac{1}{128},
\qquad\text{i.e.}\qquad
N'\lambda^\tau\ \ge\ 256\,C_a\,\log N',
\]
which is exactly \eqref{eq:gamma_rate_for_Ahalf}. Under \eqref{eq:gamma_rate_for_Ahalf},
we obtain that the event $\sE$ holds with probability at least $1-(N')^{-1}$.
This completes the proof.
\end{proof}

\begin{proof}[Proof of Lemma~\ref{lem:I1_clean}]
Write
\begin{align*}
  I_1(x_0;\theta)
  =
  \underbrace{
    \frac{1}{n}\sum_{i=1}^n
      \Big(
        w_{x_0,\lambda}(X_i)r(Y_i,X_i;\theta)
        -
        \EE\big[w_{x_0,\lambda}(X)r(Y,X;\theta)\big]
      \Big)
  }_{=:U_n(x_0;\theta)}
  \\+
  \underbrace{
    \frac{1}{N}\sum_{u=1}^N
      \Big(
        w_{x_0,\lambda}(\wt X_u)u(\wt X_u;\theta)
        -
        \EE\big[w_{x_0,\lambda}(X)u(X;\theta)\big]
      \Big)
  }_{=:U_N(x_0;\theta)}.
\end{align*}

\noindent\textit{Step 1: bound for $\|w_{x_0,\lambda}\|_\infty$.}  Using the spectral decomposition of $T_K$, one verifies that
\begin{align}\label{eqn:expectation_f}
  \EE\big[w_{x_0,\lambda}(X)^2\big]
  =
  \big\langle w_{x_0,\lambda}, T_K w_{x_0,\lambda}\big\rangle_{\cH}
  =Q(x_0;\lambda)
  \le D(x_0;\lambda).
\end{align}
By~Lemma~\ref{lem:fHnorm-bds}, $\|w_{x_0,\lambda}\|_{\cH}^2
\le
\frac{D(x_0;\lambda)}{\lambda}$.
Moreover, by~\eqref{eqn:expectation_f},
\[
\|w_{x_0,\lambda}\|_{L_2(\rho_X)}^2
=
\EE\big[w_{x_0,\lambda}(X)^2\big]
=Q(x_0;\lambda)
\le
D(x_0;\lambda).
\]
Applying Lemma~\ref{lem:sobolev_rate} to $v=w_{x_0,\lambda}$ and combining the two bounds above yields
the sharper deterministic sup-norm control
\begin{equation}\label{eq:f-sup-I1}
\|w_{x_0,\lambda}\|_\infty
\le
C_1\,
\Big(\sqrt{D(x_0;\lambda)}\Big)^{1-\frac{d}{2m}}
\Big(\sqrt{D(x_0;\lambda)/\lambda}\Big)^{\frac{d}{2m}}
=
C_1\,\sqrt{D(x_0;\lambda)}\,\lambda^{-\frac{d}{4m}}.
\end{equation}

\noindent\textit{Step 2: Bernstein for $U_n$.}
Recall that
\[
Z_{R,i}
:=
w_{x_0,\lambda}(X_i)\,r(Y_i,X_i;\theta)
-\EE\big[w_{x_0,\lambda}(X)\,r(Y,X;\theta)\big],
\qquad i=1,\dots,n,
\]
so that $U_n=n^{-1}\sum_{i=1}^n Z_{R,i}$ and $\EE[Z_{R,i}]=0$.
By Assumption~\ref{ass:B_bound_predictor}, $|r(Y,X;\theta)|\le B_{Y-f}(\theta)$ almost surely, hence by~\eqref{eq:f-sup-I1},
\[
|Z_{R,i}|
\le
2\,\|w_{x_0,\lambda}\|_\infty\,B_{Y-f}(\theta)
\le
2\,C_1\,B_{Y-f}(\theta)\,\sqrt{D(x_0;\lambda)}\,\lambda^{-\frac{d}{4m}}.
\]
Moreover, exactly as before,
\[
\Var(Z_{R,1})
=
\Var\big(w_{x_0,\lambda}(X)\,r(Y,X;\theta)\big)
\le
B_{Y-f}^2(\theta)\,Q(x_0;\lambda).
\]
Applying Bernstein's inequality to $\{Z_{R,i}\}_{i=1}^n$ gives: with probability at least $1-\delta$,
\begin{align}
\begin{split}
|U_n|
&\le
\sqrt{\frac{2\,\Var(Z_{R,1})\,\log(2/\delta)}{n}}
+\frac{2\,\|Z_{R,1}\|_\infty\,\log(2/\delta)}{3n}\\
&\le
B_{Y-f}(\theta)\sqrt{\frac{2\,Q(x_0;\lambda)\,\log(4/\delta)}{n}}
+\frac{4}{3}\,B_{Y-f}(\theta)\,\|w_{x_0,\lambda}\|_\infty\,\frac{\log(4/\delta)}{n}\\
&\le
B_{Y-f}(\theta)\sqrt{\frac{2\,Q(x_0;\lambda)\,\log(4/\delta)}{n}}
+\frac{4}{3}\,C_1\,B_{Y-f}(\theta)\,\sqrt{D(x_0;\lambda)}\,\lambda^{-\frac{d}{4m}}\,
\frac{\log(4/\delta)}{n}.
\label{eq:Un-bernstein-pre}
\end{split}
\end{align}
Let $\delta\in(0,1)$ and suppose the regularization level satisfies~\eqref{eq:lambda-lower-I1}.
Using $n\wedge N\le n$,
\[
\lambda^{d/(2m)}\ \ge\ \frac{8}{9}\,C_1^2\,\frac{\log(4/\delta)}{n}
\quad\to\quad
\frac{4}{3}\,C_1\,\lambda^{-\frac{d}{4m}}\,\frac{\log(4/\delta)}{n}
\ \le\
\sqrt{\frac{2\log(4/\delta)}{n}}.
\]
Combining this with Proposition~\ref{prop:Dx0-order-H}, which gives $D(x_0;\lambda)\lesssim Q(x_0;\lambda)$ for small $\lambda$, absorbs the linear Bernstein term into the $Q$-variance term after adjusting constants. Thus, with probability at least $1-\delta$,
\begin{equation}\label{eq:Un-clean}
|U_n|
\le
2\,B_{Y-f}(\theta)\sqrt{\frac{2\,Q(x_0;\lambda)\,\log(4/\delta)}{n}}
=
2\sqrt{2\log(4/\delta)}\,\sqrt{Q(x_0;\lambda)}\,\frac{B_{Y-f}(\theta)}{\sqrt n}.
\end{equation}

\noindent\textit{Step 3: Bernstein for $U_N$.}
Define i.i.d.\ centered summands
\[
Z_{U,u}
:=
w_{x_0,\lambda}(\wt X_u)\,u(\wt X_u;\theta)
-\EE\big[w_{x_0,\lambda}(X)\,u(X;\theta)\big],
\qquad u=1,\dots,N,
\]
so that $U_N=N^{-1}\sum_{u=1}^N Z_{U,u}$ and $\EE[Z_{U,u}]=0$.
By Assumption~\ref{ass:B_bound_predictor}, $|u(X;\theta)|\le B_{f}(\theta)$, and by~\eqref{eq:f-sup-I1},
\[
|Z_{U,u}|
\le
2\,\|w_{x_0,\lambda}\|_\infty\,B_{f}(\theta)
\le
2\,C_1\,B_{f}(\theta)\,\sqrt{D(x_0;\lambda)}\,\lambda^{-\frac{d}{4m}},
\quad
\Var(Z_{U,1})
\le
B_{f}^2(\theta)\,Q(x_0;\lambda).
\]
Repeating the Bernstein argument from Step 2 and using again \eqref{eq:lambda-lower-I1} with $n\wedge N\le N$ gives: with probability at least $1-\delta$,
\begin{equation}\label{eq:UN-clean}
|U_N|
\le
2\sqrt{2\log(4/\delta)}\,\sqrt{Q(x_0;\lambda)}\,\frac{B_{f}(\theta)}{\sqrt N}.
\end{equation}

\noindent\textit{Step 4.}
By independence and a union bound, with probability at least $1-2\delta$,
\[
|I_1|
\le
|U_n|+|U_N|
\le
2\sqrt{2\log(4/\delta)}\,\sqrt{Q(x_0;\lambda)}
\left(
\frac{B_{Y-f}(\theta)}{\sqrt n}
+
\frac{B_{f}(\theta)}{\sqrt N}
\right).
\]
In particular, taking $\delta=4(N')^{-1}$ yields $\log(4/\delta)=\log N'$ and a bound holding with probability at least
$1-8(N')^{-1}$,
\[
|I_1|
\le
2\sqrt{2\log N'}\,\sqrt{Q(x_0;\lambda)}
\left(
\frac{B_{Y-f}(\theta)}{\sqrt n}
+
\frac{B_{f}(\theta)}{\sqrt N}
\right),
\]
under the condition \eqref{eq:lambda-lower-I1}.
\end{proof}

\begin{proof}[Proof of Lemma~\ref{lem:delta_f_L2_bound}]
\textit{Step 1: compare $\|\Delta w\|_{L^2(\rho_X)}$ to a $T_\lambda^{1/2}$-norm.}
Using $\wh w=\wh T_\lambda^{-1}K_{x_0}$ and $w=T_\lambda^{-1}K_{x_0}$,
the resolvent identity gives the exact equality
\[
\Delta w
=
\wh T_\lambda^{-1}K_{x_0}-T_\lambda^{-1}K_{x_0}
=
-\wh T_\lambda^{-1}(\wh T_K-T_K)\,T_\lambda^{-1}K_{x_0}
=
-\wh T_\lambda^{-1}(\wh T_K-T_K)\,w.
\]
Recall that for any $h\in\cH$, $
\|h\|_{L^2(\rho_X)}^2
=
\EE[h(X)^2]
=
\langle h,T_K h\rangle_{\cH}$.
Since $0\preceq T_K\preceq T_\lambda:=T_K+\lambda I$, we have
\[
\|h\|_{L^2(\rho_X)}^2
=
\langle h,T_K h\rangle_{\cH}
\le
\langle h,T_\lambda h\rangle_{\cH}
=
\|T_\lambda^{1/2}h\|_{\cH}^2.
\]
Applying this to $h=\Delta w$ yields
\begin{equation}\label{eq:L2-to-Tlambda}
\|\Delta w\|_{L^2(\rho_X)}
\le
\|T_\lambda^{1/2}\Delta w\|_{\cH}.
\end{equation}

\noindent\textit{Step 2: on $\sE$, control the resolvent factor.}

On $\sE$, write
\[
\wh T_\lambda
=
T_\lambda+\wh T_K-T_K
=
T_\lambda^{1/2}\Big(I+T_\lambda^{-1/2}(\wh T_K-T_K)T_\lambda^{-1/2}\Big)T_\lambda^{1/2}
=
T_\lambda^{1/2}(I+A(\lambda))T_\lambda^{1/2}.
\]
The condition $\|A(\lambda)\|_{\op}\le 1/2$ implies
$\|(I+A(\lambda))^{-1}\|_{\op}\le 2$. Hence
\[
T_\lambda^{1/2}\wh T_\lambda^{-1}T_\lambda^{1/2}
=
(I+A(\lambda))^{-1},
\qquad
\|T_\lambda^{1/2}\wh T_\lambda^{-1}T_\lambda^{1/2}\|_{\op}\le 2.
\]

Combining with Step 1, on $\sE$ we obtain
\begin{align*}
\|T_\lambda^{1/2}\Delta w\|_{\cH}
&=
\big\|T_\lambda^{1/2}\wh T_\lambda^{-1}(\wh T_K-T_K)w\big\|_{\cH}\\
&=
\big\|(T_\lambda^{1/2}\wh T_\lambda^{-1}T_\lambda^{1/2})
\cdot T_\lambda^{-1/2}(\wh T_K-T_K)w\big\|_{\cH}\\
&\le
2\,\big\|T_\lambda^{-1/2}(\wh T_K-T_K)w\big\|_{\cH}.
\end{align*}
Together with \eqref{eq:L2-to-Tlambda}, this yields on $\sE$:
\begin{equation}\label{eq:Deltaf-L2-vs-zeta}
\|\Delta w\|_{L^2(\rho_X)}
\le
2\,\big\|T_\lambda^{-1/2}(\wh T_K-T_K)w\big\|_{\cH}.
\end{equation}

\noindent\textit{Step 3: reduce to a Hilbert-space average.}
Define the mean-zero $\cH$-valued random element
\[
\zeta(X)
:=
T_\lambda^{-1/2}\Big((K_X\otimes K_X)-T_K\Big)w
=
T_\lambda^{-1/2}\Big(K_X w(X)-T_K w\Big).
\]
Then
\[
T_\lambda^{-1/2}(\wh T_K-T_K)w
=
\frac1{N'}\sum_{j=1}^{N'}\zeta(\bar X_j),
\]
and therefore, on $\sE$,
\begin{equation}\label{eq:Deltaf-L2-vs-zeta-avg}
\|\Delta w\|_{L^2(\rho_X)}
\le
2\Big\|\frac1{N'}\sum_{j=1}^{N'}\zeta(\bar X_j)\Big\|_{\cH}.
\end{equation}
Thus it suffices to control the Hilbert-space average
$N'^{-1}\sum_{j=1}^{N'}\zeta(\bar X_j)$ by a Hilbert-space Bernstein inequality.
The two required inputs are (i) a variance proxy $\EE\|\zeta(X)\|_{\cH}^2$
and (ii) a uniform bound on $\|\zeta(X)\|_{\cH}$.

\noindent\textit{Step 4: variance proxy.}
Let $\Delta K:=K_X\otimes K_X-T_K$ for $X\sim\rho_X$. Then
\[
\EE\|\zeta(X)\|_{\cH}^2
=
\EE\Big\langle \Delta K w,\ T_\lambda^{-1}\Delta K w\Big\rangle_{\cH}.
\]
Let $a:=K_X w(X)\in\cH$ and $b:=T_K w=\EE[a]\in\cH$. Then $\Delta K w=a-b$ and
\begin{align*}
\EE\langle \Delta K w,\ T_\lambda^{-1} \Delta K w\rangle_{\cH}
&=
\EE\langle a-b,\ T_\lambda^{-1}(a-b)\rangle_{\cH}
\\&=
\EE\langle a,T_\lambda^{-1}a\rangle_{\cH}
-\langle b,T_\lambda^{-1}b\rangle_{\cH}
\ \le\
\EE\langle a,T_\lambda^{-1}a\rangle_{\cH},
\end{align*}
where we dropped the nonnegative term $\langle b,T_\lambda^{-1}b\rangle_{\cH}$.
By reproducing,
\begin{align*}
\langle a,T_\lambda^{-1}a\rangle_{\cH}
&=
\langle K_X w(X),\ T_\lambda^{-1} K_X w(X)\rangle_{\cH}
\\&=
w(X)^2\,\langle K_X,\ T_\lambda^{-1} K_X\rangle_{\cH}
=
w(X)^2\,D(X;\lambda).
\end{align*}
Therefore, $\EE\|\zeta(X)\|_{\cH}^2
\le\EE[w(X)^2\,D(X;\lambda)]$. Under Assumption~\ref{ass:eigenfunctions_bound},
\[
D(X;\lambda)
=
\sum_{j\ge1}\frac{\mu_j}{\mu_j+\lambda}\phi_j(X)^2
\le
\Big(\sup_{j\ge1}\|\phi_j\|_\infty^2\Big)
\sum_{j\ge1}\frac{\mu_j}{\mu_j+\lambda}
\le
B_\phi^2\,D(\lambda),
\]
hence $\EE\|\zeta(X)\|_{\cH}^2\le
B_\phi^2 D(\lambda)\EE[w(X)^2]$.
Finally, since $\EE[w(X)^2]=\langle w,T_K w\rangle_{\cH}\le\langle w,T_\lambda w\rangle_{\cH}$ and
$T_\lambda w=K_{x_0}$, we get $\EE[w(X)^2]\le \langle w,K_{x_0}\rangle_{\cH}=w(x_0)=D(x_0;\lambda)$.
Combining yields
\begin{equation}\label{eq:zeta-variance-proxy-L2}
\EE\|\zeta(X)\|_{\cH}^2
\le
B_\phi^2\,D(\lambda)D(x_0;\lambda).
\end{equation}

\noindent\textit{Step 5: uniform bound.}
By triangle inequality,
\[
\|\zeta(X)\|_{\cH}
=
\|T_\lambda^{-1/2}(K_X w(X)-T_K w)\|_{\cH}
\le
\|T_\lambda^{-1/2}K_X\|_{\cH}\,|w(X)|
+\|T_\lambda^{-1/2}T_K w\|_{\cH}.
\]
First,
\[
\|T_\lambda^{-1/2}K_X\|_{\cH}^2
=
\langle K_X,\ T_\lambda^{-1}K_X\rangle_{\cH}
=
D(X;\lambda)
\le
B_\phi^2\,D(\lambda).
\]
Also,
\begin{align*}
|w(X)|
&=
|\langle w,K_X\rangle_{\cH}|
=
|\langle T_\lambda^{-1/2}K_{x_0},\ T_\lambda^{-1/2}K_X\rangle_{\cH}|
\\&\le
\|T_\lambda^{-1/2}K_{x_0}\|_{\cH}\,\|T_\lambda^{-1/2}K_X\|_{\cH}
=
\sqrt{D(x_0;\lambda)}\,\sqrt{D(X;\lambda)}.
\end{align*}
Thus
\begin{align*}
\|T_\lambda^{-1/2}K_X\|_{\cH}\,|w(X)|
&\le
\sqrt{D(X;\lambda)}\cdot \sqrt{D(x_0;\lambda)}\cdot \sqrt{D(X;\lambda)}
\\&=
\sqrt{D(x_0;\lambda)}\,D(X;\lambda)
\le
\sqrt{D(x_0;\lambda)}\,B_\phi^2 D(\lambda).
\end{align*}
For the second term, use $0\preceq T_K\preceq T_\lambda$, implying
$\|T_\lambda^{-1/2}T_K T_\lambda^{-1/2}\|_{\op}\le 1$. Hence
\[
\|T_\lambda^{-1/2}T_K w\|_{\cH}
=
\|(T_\lambda^{-1/2}T_K T_\lambda^{-1/2})(T_\lambda^{-1/2}K_{x_0})\|_{\cH}
\le
\|T_\lambda^{-1/2}K_{x_0}\|_{\cH}
=
\sqrt{D(x_0;\lambda)}.
\]
Putting together, for all $X$,
\begin{equation}\label{eq:zeta-uniform-bound-L2}
\|\zeta(X)\|_{\cH}
\le
\sqrt{D(x_0;\lambda)}\big( B_\phi^2 D(\lambda)+1\big).
\end{equation}

\noindent\textit{Step 6: Hilbert-space Bernstein.}
By \eqref{eq:zeta-variance-proxy-L2} and \eqref{eq:zeta-uniform-bound-L2}, a Bernstein inequality for sums of independent, bounded Hilbert-space random elements
then yields: for any $\delta\in(0,1)$, with probability at least $1-\delta$,
\begin{equation}\label{eq:zeta-bernstein-L2}
\Big\|\frac1{N'}\sum_{j=1}^{N'}\zeta(\bar X_j)\Big\|_{\cH}
\le
c\left[
\sqrt{\frac{B_\phi^2 D(\lambda)D(x_0;\lambda)\,\log(1/\delta)}{N'}}
+
\frac{\sqrt{D(x_0;\lambda)}(B_\phi^2 D(\lambda)+1)\,\log(1/\delta)}{N'}
\right],
\end{equation}
for a universal constant $c>0$.
Choose $\delta=1/(2N')$ in \eqref{eq:zeta-bernstein-L2}; then
$\log(1/\delta)\asymp \log N'$. On the intersection $\sE\cap\{\eqref{eq:zeta-bernstein-L2}\text{ holds}\}$,
by \eqref{eq:Deltaf-L2-vs-zeta-avg},
\[
\|\Delta w\|_{L^2(\rho_X)}
\le
c'\left[
B_\phi\sqrt{\frac{D(\lambda)D(x_0;\lambda)\,\log N'}{N'}}
+
\sqrt{D(x_0;\lambda)}\frac{(B_\phi^2 D(\lambda)+1)\,\log N'}{N'}
\right]
\]
for another universal constant $c'>0$.
Under $\lambda\ge C(\log N'/N')^{2m/d}$,
the second term is dominated by the first, so after adjusting constants we obtain
\eqref{eq:Deltaf-L2-hp-clean}.
Finally, on event $\sE$, the Bernstein failure (with $\delta=1/(2N')$) is $\le 1/(2N')$.

\end{proof}

\begin{proof}[Proof of Lemma~\ref{lem:I2-clean-hp-refined}]
Throughout the proof we work on the event $\sE$. Write $T_\lambda=T_K+\lambda I$ and define
\[
A(\lambda)=T_\lambda^{-1/2}(\wh T_K-T_K)T_\lambda^{-1/2},
\qquad
q_0:=T_\lambda^{-1/2}K_{x_0}.
\]
Let $v(\theta):=T_\lambda^{-1/2}\{\wh\mu(\theta)-\mu(\theta)\}$.
The whitening identity gives
\[
\Delta w
=
-T_\lambda^{-1/2}(I+A(\lambda))^{-1}A(\lambda)q_0.
\]
Therefore
\begin{equation}\label{eq:I2-resolvent-expand}
I_2(x_0;\theta)
=-\langle A(\lambda)q_0,v(\theta)\rangle_{\cH}
+\langle (I+A(\lambda))^{-1}A(\lambda)^2q_0,v(\theta)\rangle_{\cH}
=:T_1+T_2.
\end{equation}
We control these two terms separately.

\noindent\textit{Step 1: the leading bilinear term.}
For $x\in\cX$, set $\psi_\lambda(x)=T_\lambda^{-1/2}K_x$ and $w(x)=w_{x_0,\lambda}(x)$. Then
\[
A(\lambda)q_0=\frac1{N'}\sum_{j=1}^{N'}a_j,
\qquad
a_j:=w(\bar X_j)\psi_\lambda(\bar X_j)-\EE[w(X)\psi_\lambda(X)],
\]
and $v(\theta)=n^{-1}\sum_{i=1}^n b_i^R+N^{-1}\sum_{u=1}^N b_u^U$,
where
\[
b_i^R:=r(Y_i,X_i;\theta)\psi_\lambda(X_i)-\EE[r(Y,X;\theta)\psi_\lambda(X)],
\quad
b_u^U:=u(\wt X_u;\theta)\psi_\lambda(\wt X_u)-\EE[u(X;\theta)\psi_\lambda(X)].
\]
Hence $T_1$ can be written as $T_1=T_{1,R}+T_{1,U}$, where
\[
T_{1,R}:=-\frac1{N'n}\sum_{j=1}^{N'}\sum_{i=1}^n
\langle a_j,b_i^R\rangle_{\cH},
\qquad
T_{1,U}:=-\frac1{N'N}\sum_{j=1}^{N'}\sum_{u=1}^N
\langle a_j,b_u^U\rangle_{\cH}.
\]
We detail the bound for $T_{1,R}$; the proof for $T_{1,U}$ is identical with $n$ and
$B_{Y-f}(\theta)$ replaced by $N$ and $B_f(\theta)$. Since the pooled sample satisfies
$\bar X_i=X_i$ for $i=1,\ldots,n$, the only terms in $T_{1,R}$ where $a_j$ and $b_i^R$
are constructed from the same labeled observation are those with $j=i$. We therefore
decompose
\[
T_{1,R}=T_{1,R}^{\mathrm{diag}}+T_{1,R}^{\mathrm{off}},
\]
where
\[
T_{1,R}^{\mathrm{diag}}
:=
-\frac1{N'n}\sum_{i=1}^n\langle a_i,b_i^R\rangle_{\cH},
\qquad
T_{1,R}^{\mathrm{off}}
:=
-\frac1{N'n}\sum_{i=1}^n\sum_{\substack{1\le j\le N'\\ j\neq i}}
\langle a_j,b_i^R\rangle_{\cH}.
\]

We first bound the diagonal term. Using
\[
\|\psi_\lambda(x)\|_{\cH}^2=D(x;\lambda)\le B_\phi^2D(\lambda),
\quad
|w(x)|\le B_\phi\sqrt{D(\lambda)D(x_0;\lambda)},
\]
and $|r(Y,X;\theta)|\le B_{Y-f}(\theta)$, we have
\[
\EE\{\langle a_i,b_i^R\rangle_{\cH}^2\}
\lesssim
B_{Y-f}^2(\theta)D^2(\lambda)D(x_0;\lambda),
\]
with the corresponding envelope bounded by a constant multiple of
$B_{Y-f}(\theta)D^{3/2}(\lambda)\sqrt{D(x_0;\lambda)}$. Applying Bernstein's inequality
to the centered summands, and bounding the mean by the same second-moment scale, gives
with probability at least $1-(2N')^{-1}$,
\begin{equation}\label{eq:T1-diag-bound}
|T_{1,R}^{\mathrm{diag}}|
\lesssim
B_{Y-f}(\theta)D(\lambda)\sqrt{D(x_0;\lambda)}\frac{\log N'}{N'}.
\end{equation}

Next consider the off-diagonal term. We further split it according to whether the index
$j$ comes from the labeled or unlabeled part of the pooled covariates:
\[
T_{1,R}^{\mathrm{off}}
=
T_{1,R}^{\mathrm{off},L}
+
T_{1,R}^{\mathrm{off},U},
\]
where
\[
T_{1,R}^{\mathrm{off},L}
:=
-\frac1{N'n}\sum_{i=1}^n\sum_{\substack{1\le j\le n\\j\neq i}}
\langle a_j,b_i^R\rangle_{\cH},
\qquad
T_{1,R}^{\mathrm{off},U}
:=
-\frac1{N'n}\sum_{i=1}^n\sum_{j=n+1}^{N'}
\langle a_j,b_i^R\rangle_{\cH}.
\]
For $T_{1,R}^{\mathrm{off},U}$, conditional on the labeled data, the vector
$\sum_{i=1}^n b_i^R$ is fixed, while $\{a_j:j=n+1,\ldots,N'\}$ are independent
mean-zero Hilbert-valued random elements. Moreover, a Hilbert-space Bernstein inequality
gives
\[
\left\|\sum_{i=1}^n b_i^R\right\|_{\cH}
\lesssim
B_{Y-f}(\theta)\sqrt{nD(\lambda)\log N'}
\]
with probability at least $1-(2N')^{-1}$. Conditional on this event, applying Bernstein's
inequality to the scalar variables
$\langle a_j,\sum_{i=1}^n b_i^R\rangle_{\cH}$ yields, with conditional probability at least
$1-(2N')^{-1}$,
\[
|T_{1,R}^{\mathrm{off},U}|
\lesssim
B_{Y-f}(\theta)D(\lambda)\sqrt{D(x_0;\lambda)}
\sqrt{\frac{\log N'}{N'n}}.
\]
For $T_{1,R}^{\mathrm{off},L}$, the kernel
$\langle a_j,b_i^R\rangle_{\cH}$ is canonical for $j\neq i$, since
\[
\EE\{\langle a_j,b_i^R\rangle_{\cH}| \bar X_j\}=0,
\qquad
\EE\{\langle a_j,b_i^R\rangle_{\cH}| X_i,Y_i\}=0.
\]
Thus $T_{1,R}^{\mathrm{off},L}$ is an ordered degenerate U-statistic. The standard bounded
degenerate U-statistic concentration inequality gives, with probability at least
$1-(2N')^{-1}$,
\[
|T_{1,R}^{\mathrm{off},L}|
\lesssim
B_{Y-f}(\theta)D(\lambda)\sqrt{D(x_0;\lambda)}
\frac{\log N'}{N'}.
\]
Combining the two off-diagonal pieces and taking a union bound,
\begin{equation}\label{eq:T1-off-bound}
|T_{1,R}^{\mathrm{off}}|
\lesssim
B_{Y-f}(\theta)D(\lambda)\sqrt{D(x_0;\lambda)}
\left\{
\sqrt{\frac{\log N'}{N'n}}+\frac{\log N'}{N'}
\right\}
\end{equation}
with probability at least $1-3(2N')^{-1}$.

Repeating the same diagonal/off-diagonal argument for $T_{1,U}$, now with the diagonal
indices given by $j=n+u$, yields with probability at least $1-2(N')^{-1}$,
\begin{equation}\label{eq:T1U-bound}
|T_{1,U}|
\lesssim
B_f(\theta)D(\lambda)\sqrt{D(x_0;\lambda)}
\left\{
\sqrt{\frac{\log N'}{N'N}}+\frac{\log N'}{N'}
\right\}.
\end{equation}
Combining \eqref{eq:T1-diag-bound}--\eqref{eq:T1U-bound}, we obtain with probability at
least $1-4(N')^{-1}$,
\begin{align}\label{eq:T1-final-bound}
|T_1|
\lesssim\;&
D(\lambda)\sqrt{D(x_0;\lambda)}
\sqrt{\frac{\log N'}{N'}}
\left(\frac{B_{Y-f}(\theta)}{\sqrt n}+\frac{B_f(\theta)}{\sqrt N}\right)
\notag\\
&+(B_{Y-f}(\theta)+B_f(\theta))D(\lambda)\sqrt{D(x_0;\lambda)}\frac{\log N'}{N'}.
\end{align}

\noindent\textit{Step 2: the higher-order resolvent remainder.}
On $\sE$, $\|(I+A(\lambda))^{-1}\|_{\op}\le2$. Hence
\[
|T_2|
\le
2\|A(\lambda)\|_{\op}\,\|A(\lambda)q_0\|_{\cH}\,\|v(\theta)\|_{\cH}.
\]
Lemma~\ref{lem:A_bound_1/2_sobolev} gives
\[
\|A(\lambda)\|_{\op}
\lesssim \sqrt{\frac{D(\lambda)\log N'}{N'}},
\]
and the same Hilbert-space Bernstein argument used in Lemma~\ref{lem:delta_f_L2_bound} yields
\[
\|A(\lambda)q_0\|_{\cH}
\lesssim
\sqrt{\frac{D(\lambda)D(x_0;\lambda)\log N'}{N'}}.
\]
Finally,
\[
\|v(\theta)\|_{\cH}
\lesssim
\sqrt{D(\lambda)\log N'}
\left(\frac{B_{Y-f}(\theta)}{\sqrt n}+\frac{B_f(\theta)}{\sqrt N}\right).
\]
Therefore, with probability at least $1-(N')^{-1}$,
\begin{equation}\label{eq:T2-final-bound}
|T_2|
\lesssim
D(\lambda)^{3/2}\sqrt{D(x_0;\lambda)}
\frac{(\log N')^{3/2}}{N'}
\left(\frac{B_{Y-f}(\theta)}{\sqrt n}+\frac{B_f(\theta)}{\sqrt N}\right).
\end{equation}
Combining \eqref{eq:I2-resolvent-expand}, \eqref{eq:T1-final-bound}, and \eqref{eq:T2-final-bound} gives \eqref{eq:I2-total-hp-refined}. The stated probability follows from the union bound $4(N')^{-1}+(N')^{-1}=5(N')^{-1}$ over the events used above.
\end{proof}

\begin{proof}[Proof of Lemma~\ref{lem:I3-clean}]
\textit{Step 1: whitening identity.}
Recall
\[
  A(\lambda)
  :=
  (T_K+\lambda I)^{-1/2}\,(\wh T_K-T_K)\,(T_K+\lambda I)^{-1/2}.
\]
Then
\[
\wh T_K+\lambda I=(T_K+\lambda I)^{1/2}\big(I+A(\lambda)\big)(T_K+\lambda I)^{1/2},
\]
and hence
\begin{equation}\label{eq:inv-whiten-B2}
(\wh T_K+\lambda I)^{-1}
=(T_K+\lambda I)^{-1/2}\big(I+A(\lambda)\big)^{-1}(T_K+\lambda I)^{-1/2}.
\end{equation}
Let
\[
u:=(T_K+\lambda I)^{-1/2}K_{x_0},
\qquad
v:=(T_K+\lambda I)^{-1/2}\mu(\theta).
\]
Note that $\|u\|_{\cH}^2=\langle K_{x_0},(T_K+\lambda I)^{-1}K_{x_0}\rangle_{\cH}=D(x_0;\lambda)$.
Combining $\wh w_{x_0,\lambda}=(\wh T_K+\lambda I)^{-1}K_{x_0}$ with \eqref{eq:inv-whiten-B2} gives the exact identity
\[
\wh w_{x_0,\lambda}
=
(T_K+\lambda I)^{-1/2}\big(I+A(\lambda)\big)^{-1}u,
\qquad
w_{x_0,\lambda}=(T_K+\lambda I)^{-1/2}u,
\]
hence
\[
\wh w_{x_0,\lambda}-w_{x_0,\lambda}
=
(T_K+\lambda I)^{-1/2}\Big(\big(I+A(\lambda)\big)^{-1}-I\Big)u
=
-(T_K+\lambda I)^{-1/2}\big(I+A(\lambda)\big)^{-1}A(\lambda)u,
\]
and therefore
\begin{equation}\label{eq:I3-exact}
I_3(x_0;\theta)
=
-\big\langle \big(I+A(\lambda)\big)^{-1}A(\lambda)u,\ v\big\rangle_{\cH}.
\end{equation}

\noindent\textit{Step 2: stability event and  decomposition.}
Recall that $\sE:=\{\|A(\lambda)\|_{\op}\le 1/2\}$, on $\sE$, $\|(I+A(\lambda))^{-1}\|_{\op}\le 2$ and we can expand
\[
(I+A(\lambda))^{-1}A(\lambda)
=
A(\lambda)-(I+A(\lambda))^{-1}A(\lambda)^2.
\]
Plugging into \eqref{eq:I3-exact} yields the exact decomposition on $\sE$:
\[
I_3(x_0;\theta)
=
-\underbrace{\langle A(\lambda)u,\ v\rangle_{\cH}}_{=:T_1}
+
\underbrace{\Big\langle (I+A(\lambda))^{-1}A(\lambda)^2u,\ v\Big\rangle_{\cH}}_{=:T_2}.
\]
Consequently, on $\sE$,
\begin{equation}\label{eq:I3-T1T2}
|I_3(x_0;\theta)|
\le |T_1|+|T_2|
\le |T_1| + 2\,\big|\langle A(\lambda)^2u,\ v\rangle_{\cH}\big|
= |T_1| + 2\,\big|\langle A(\lambda)u,\ A(\lambda)v\rangle_{\cH}\big|.
\end{equation}

\noindent\textit{Step 3: reduce $T_1=\langle A(\lambda)u,v\rangle$ to empirical mean.}
Let $g_x:=(T_K+\lambda I)^{-1/2}K_x\in\cH$. Then
\[
A(\lambda)
=
\frac1{N'}\sum_{i=1}^{N'}\big(g_{\bar X_i}\otimes g_{\bar X_i}-\EE[g_X\otimes g_X]\big),
\]
and therefore
\[
\langle A(\lambda)u,\ v\rangle_{\cH}
=
\frac1{N'}\sum_{i=1}^{N'}
\Big(
\langle g_{\bar X_i},u\rangle_{\cH}\,\langle g_{\bar X_i},v\rangle_{\cH}
-
\EE[\langle g_X,u\rangle_{\cH}\,\langle g_X,v\rangle_{\cH}]
\Big).
\]
Define the deterministic functions
\begin{align}
a(x)&:=\langle g_x,u\rangle_{\cH}
=\langle K_x,(T_K+\lambda I)^{-1}K_{x_0}\rangle_{\cH}
=w_{x_0,\lambda}(x),
\label{eq:a-def-B2}
\\
b(x)&:=\langle g_x,v\rangle_{\cH}
=\langle K_x,(T_K+\lambda I)^{-1}\mu(\theta)\rangle_{\cH}.
\label{eq:b-def-B2}
\end{align}
Then
\begin{equation}\label{eq:scalar-mean}
T_1
=
\langle A(\lambda)u,\ v\rangle_{\cH}
=
\frac1{N'}\sum_{i=1}^{N'}\Big(Z_i-\EE[Z]\Big),
\qquad
Z_i:=a(\bar X_i)b(\bar X_i),
\end{equation}
and $\{Z_i\}_{i=1}^{N'}$ are i.i.d.

\noindent\textit{Step 4: bound $b(\cdot)$.}
By $\mu(\theta)=T_K\eta(\cdot;\theta)$,
\[
b(\cdot)
=(T_K+\lambda I)^{-1}\mu(\theta)
=(T_K+\lambda I)^{-1}T_K\eta(\cdot;\theta).
\]
In the eigenbasis of $T_K$, $(T_K+\lambda I)^{-1}T_K$ has eigenvalues $\mu_j/(\mu_j+\lambda)\in[0,1]$,
so $\|(T_K+\lambda I)^{-1}T_K\|_{\op}\le 1$. Hence
\begin{equation}\label{eqn:b_bound}
\|b\|_{\cH}\le \|\eta(\cdot;\theta)\|_{\cH},
\qquad
\|b\|_\infty \le \kappa\|\eta(\cdot;\theta)\|_{\cH}.
\end{equation}

\noindent\textit{Step 5: variance proxy.}
First, $\Var(Z)\le \EE[Z^2]\le \|b\|_\infty^2\EE[a(X)^2]$.
Since $a(X)=w_{x_0,\lambda}(X)$,
\begin{align*}
\EE[a(X)^2]
&=\langle w_{x_0,\lambda},T_Kw_{x_0,\lambda}\rangle_{\cH}
\\&\le \langle w_{x_0,\lambda},(T_K+\lambda I)w_{x_0,\lambda}\rangle_{\cH}
\\&=\langle w_{x_0,\lambda},K_{x_0}\rangle_{\cH}
=w_{x_0,\lambda}(x_0)
=D(x_0;\lambda).
\end{align*}
Combining with \eqref{eqn:b_bound} gives
\begin{equation}\label{eq:VarZ}
\Var(Z)\le \kappa^2 \|\eta(\cdot;\theta)\|_\cH^2\,D(x_0;\lambda).
\end{equation}

\noindent\textit{Step 6: envelope bound.}
By Cauchy--Schwarz in $\cH$,
\begin{equation}\label{eq:a-pointwise}
|a(x)|
=|\langle g_x,u\rangle_{\cH}|
\le \|g_x\|_{\cH}\,\|u\|_{\cH}
=\sqrt{D(x;\lambda)}\,\sqrt{D(x_0;\lambda)}.
\end{equation}
Under Assumption~\ref{ass:eigenfunctions_bound}, we have for every $x\in\cX$,
\[
D(x;\lambda)
=\sum_{j\ge1}\frac{\mu_j}{\mu_j+\lambda}\phi_j(x)^2
\le B_\phi^2\sum_{j\ge1}\frac{\mu_j}{\mu_j+\lambda}
= B_\phi^2D(\lambda),
\]
and hence \eqref{eq:a-pointwise} implies
\begin{equation}\label{eq:a-infty}
\|a\|_\infty
\le B_\phi\,\sqrt{D(\lambda)}\,\sqrt{D(x_0;\lambda)}.
\end{equation}
Therefore, using $\|Z-\EE Z\|_\infty\le 2\|Z\|_\infty\le 2\|a\|_\infty\|b\|_\infty$ and \eqref{eqn:b_bound},
\begin{equation}\label{eq:envZ}
\|Z-\EE Z\|_\infty
\le
2\,\kappa \|\eta(\cdot;\theta)\|_\cH\,B_\phi\,\sqrt{D(\lambda)D(x_0;\lambda)}.
\end{equation}

\noindent\textit{Step 7: scalar Bernstein for $T_1=\langle A(\lambda)u,v\rangle$.}
Apply Bernstein's inequality to the i.i.d.\ mean in \eqref{eq:scalar-mean}.
For any $t\ge1$, with probability at least $1-2e^{-t}$,
\[
|T_1|
\le
\sqrt{\frac{2\,\Var(Z)\,t}{N'}}
+
\frac{2\,\|Z-\EE Z\|_\infty\,t}{3N'}.
\]
Plugging \eqref{eq:VarZ} and \eqref{eq:envZ} yields
\begin{align}
|T_1|
&\le
\sqrt{2}\,\kappa \|\eta(\cdot;\theta)\|_{\cH}\sqrt{\frac{D(x_0;\lambda)t}{N'}}
+ \frac{4}{3}\,\kappa \|\eta(\cdot;\theta)\|_\cH\,B_\phi \sqrt{D(\lambda)D(x_0;\lambda)}\,\frac{t}{N'}
\notag\\
&=
\sqrt{2}\,\kappa \|\eta(\cdot;\theta)\|_{\cH}\sqrt{\frac{D(x_0;\lambda)t}{N'}}
\left(
1+\frac{4\sqrt{2}B_\phi}{3}\sqrt{\frac{D(\lambda)t}{N'}}
\right).
\label{eq:T1-bern}
\end{align}

\noindent\textit{Step 8: the remainder term $T_2$.}
We bound $T_2$ using Cauchy--Schwarz in $\cH$:
\begin{equation}\label{eq:quad-cs}
\big|\langle A(\lambda)u,\ A(\lambda)v\rangle_{\cH}\big|
\le \|A(\lambda)u\|_{\cH}\,\|A(\lambda)v\|_{\cH}.
\end{equation}
Next, note that $A(\lambda)u$ and $A(\lambda)v$ are themselves empirical means of centered $\cH$-valued random variables.
Indeed, using $g_x=(T_K+\lambda I)^{-1/2}K_x$ and \eqref{eq:a-def-B2}--\eqref{eq:b-def-B2},
\begin{align}
A(\lambda)u
&=
\frac1{N'}\sum_{i=1}^{N'}\Big(a(\bar X_i)\,g_{\bar X_i}-\EE[a(X)g_X]\Big),
\label{eq:Au-mean}\\
A(\lambda)v
&=
\frac1{N'}\sum_{i=1}^{N'}\Big(b(\bar X_i)\,g_{\bar X_i}-\EE[b(X)g_X]\Big).
\label{eq:Av-mean}
\end{align}
Applying a Hilbert-valued Bernstein inequality to \eqref{eq:Au-mean}--\eqref{eq:Av-mean} with the same parameter $t\ge1$
and using \eqref{eq:a-infty}, \eqref{eqn:b_bound}, and $\|g_x\|_{\cH}^2=D(x;\lambda)\le B_\phi^2D(\lambda)$ under Assumption~\ref{ass:eigenfunctions_bound},
one obtains: with probability at least $1-4e^{-t}$,
\begin{align}
\|A(\lambda)u\|_{\cH}
&\le
C_u\,\sqrt{D(x_0;\lambda)}\left(
\sqrt{\frac{D(\lambda)t}{N'}} + \frac{D(\lambda)t}{N'}
\right),
\label{eq:Au-bound}\\
\|A(\lambda)v\|_{\cH}
&\le
C_v\,\kappa\|\eta(\cdot;\theta)\|_{\cH}\left(
\sqrt{\frac{D(\lambda)t}{N'}} + \frac{D(\lambda)t}{N'}
\right),
\label{eq:Av-bound}
\end{align}
for absolute constants $C_u,C_v$ depending only on $(B_\phi,\kappa)$.
Combining \eqref{eq:quad-cs}--\eqref{eq:Av-bound} yields that, with probability at least $1-4e^{-t}$,
\begin{equation}\label{eq:T2-rate}
\big|\langle A(\lambda)u,\ A(\lambda)v\rangle_{\cH}\big|
\le
C_{uv}\,\kappa\|\eta(\cdot;\theta)\|_{\cH}\sqrt{D(x_0;\lambda)}
\left(
\sqrt{\frac{D(\lambda)t}{N'}} + \frac{D(\lambda)t}{N'}
\right)^2,
\end{equation}
where $C_{uv}:=C_uC_v$.
Set $\rho_{N'}(t)\ :=\ D(\lambda)\sqrt{\frac{t}{N'}}$.
Under $D(\lambda)\sqrt{\frac{\log N'}{N'}}\to0$ and $t\asymp \log N'$, we have $\rho_{N'}(t)\to0$.
In particular $\sqrt{\frac{D(\lambda)t}{N'}}\to0$.
Then
\[
\left(
\sqrt{\frac{D(\lambda)t}{N'}} + \frac{D(\lambda)t}{N'}
\right)^2
=
\left(s+s^2\right)^2
\le 4s^2
=4\,\frac{D(\lambda)t}{N'}
\qquad\text{whenever }s:=\sqrt{\frac{D(\lambda)t}{N'}}\to0,
\]
so \eqref{eq:T2-rate} implies the quadratic contribution satisfies
\begin{align}
\begin{split}
\label{eq:T2-simplified}
2\,\big|\langle A(\lambda)u,\ A(\lambda)v\rangle_{\cH}\big|
&\le
8C_{uv}\,\kappa\|\eta(\cdot;\theta)\|_{\cH}\sqrt{D(x_0;\lambda)}\,
\frac{D(\lambda)t}{N'}
\\&=
8C_{uv}\,\kappa\|\eta(\cdot;\theta)\|_{\cH}\sqrt{\frac{D(x_0;\lambda)t}{N'}}\,
\rho_{N'}(t).
\end{split}
\end{align}
Thus, once $N'$ is large enough so that $\rho_{N'}(t)\to0$, the remainder term in \eqref{eq:I3-T1T2} is controlled by the same
$\sqrt{\frac{D(x_0;\lambda)t}{N'}}$ scale as $T_1$, and can be absorbed into the overall constant
multiplying the leading rate.

\noindent\textit{Step 9.}
Choose $t:=\log(4N')$, so that $2e^{-t}\le (N')^{-1}$ and $4e^{-t}\le (N')^{-1}$.
Then
the scalar Bernstein event in Step 7 holds with failure probability at most $(N')^{-1}$,
and the Hilbert-valued Bernstein event in Step 8 holds with failure probability at most $(N')^{-1}$.
By a union bound, their intersection holds with probability at least $1-2(N')^{-1}$ on event $\sE$.

On this intersection, first use \eqref{eq:T1-bern} to absorb the linear Bernstein term:
for $N'$ large enough so that $\sqrt{\frac{D(\lambda)t}{N'}}\to c_0$ and $\frac{4\sqrt{2}B_\phi}{3}c_0\le 1$,
\begin{equation}\label{eq:T1-clean}
|T_1|
\le
2\sqrt{2}\,\kappa \|\eta(\cdot;\theta)\|_{\cH}\sqrt{\frac{D(x_0;\lambda)t}{N'}}.
\end{equation}
Next, \eqref{eq:T2-simplified} and the condition $8C_{uv}\rho_{N'}(t)\le \sqrt{2}$ yield $2\big|\langle A(\lambda)u,\ A(\lambda)v\rangle_{\cH}\big|
\le |T_1|$.
Plugging these into \eqref{eq:I3-T1T2} gives, for all $N'\ge N'_0$,
\begin{align*}
|I_3(x_0;\theta)|\le
|T_1| + 2|\langle A(\lambda)u,A(\lambda)v\rangle_{\cH}|
\le
4\sqrt{2}\,\kappa \|\eta(\cdot;\theta)\|_{\cH}\sqrt{\frac{D(x_0;\lambda)t}{N'}},
\end{align*}
Using $t=\log(4N')\le \log N' + \log 4$, we have
\begin{align*}
|I_3(x_0;\theta)|
\le
8\,\kappa \|\eta(\cdot;\theta)\|_{\cH}\sqrt{\frac{D(x_0;\lambda)\log N'}{N'}},
\end{align*}
This completes the proof.
\end{proof}
\begin{proof}[Proof of Lemma~\ref{lem:Hhat-H-J123}]
First, we introduce the derivative analogues of $\wh\mu(\theta)$ and $\mu(\theta)$:
\begin{align*}
\wh\mu^1(\theta)
&:=
\frac{1}{n}\sum_{i=1}^{n}\partial_\theta r(Y_i,X_i;\theta)\,K_{X_i}
+
\frac{1}{N}\sum_{u=1}^{N}\partial_\theta u(\wt X_u;\theta)\,K_{\wt X_u}
\ \in\ \cH,\\
\mu^1(\theta)
&:=
\EE\big[\partial_\theta r(Y,X;\theta)\,K_X\big]
+\EE\big[\partial_\theta u(X;\theta)\,K_X\big]
=\EE\big[\partial_\theta\ell(Y;\theta)\,K_X\big]
\ \in\ \cH.
\end{align*}
By the reproducing property,
\begin{equation}\label{eq:H-repr-J}
J_\lambda(x_0;\theta)=\langle w_{x_0,\lambda},\mu^1(\theta)\rangle_{\cH},
\qquad
\wh J_\lambda(x_0;\theta)=\langle \wh w_{x_0,\lambda},\wh\mu^1(\theta)\rangle_{\cH}.
\end{equation}

\noindent\textit{Step 1: exact decomposition.}
From \eqref{eq:H-repr-J},
\[
\wh J_\lambda(x_0;\theta)-J_\lambda(x_0;\theta)
=
\langle \wh w_{x_0,\lambda},\wh\mu^1(\theta)\rangle_{\cH}
-\langle w_{x_0,\lambda},\mu^1(\theta)\rangle_{\cH}.
\]
For any $\theta\in \Theta_0$, define
\begin{align*}
J_1(x_0;\theta):&=\langle w_{x_0,\lambda},\wh\mu^1(\theta)-\mu^1(\theta)\rangle_{\cH},\\\
J_2(x_0;\theta):&=\langle \Delta w,\wh\mu^1(\theta)-\mu^1(\theta)\rangle_{\cH},\\
J_3(x_0;\theta):&=\langle \Delta w,\mu^1(\theta)\rangle_{\cH}.
\end{align*}
Then the exact decomposition holds:
\begin{equation}\label{eq:Hhat-H-decomp-J}
\wh J_\lambda(x_0;\theta)-J_\lambda(x_0;\theta)=J_1(x_0;\theta)+J_2(x_0;\theta)+J_3(x_0;\theta).
\end{equation}

\noindent\textit{Step 2: $J_1$ follows $I_1$.}
Write $J_1$ as the sum of two independent centered empirical means:
\begin{align*}
J_1(x_0;\theta)
=&
\Big\{\frac1n\sum_{i=1}^n w_{x_0,\lambda}(X_i)\partial_\theta r(Y_i,X_i;\theta)
-\EE[w_{x_0,\lambda}(X)\partial_\theta r(Y,X;\theta)]\Big\}\\
&+
\Big\{\frac1N\sum_{u=1}^N w_{x_0,\lambda}(\wt X_u)\partial_\theta u(\wt X_u;\theta)
-\EE[w_{x_0,\lambda}(X)\partial_\theta u(X;\theta)]\Big\}.
\end{align*}
Under Assumption~\ref{ass:deriv-bdd-jac}, for $\theta\in \Theta_0$,
\[
|\partial_\theta r(Y,X;\theta)|\le G_{Y-f}(\theta),
\qquad
|\partial_\theta u(X;\theta)|\le G_{f}(\theta).
\]
As in Lemma~\ref{lem:I1_clean}, $\EE[w_{x_0,\lambda}(X)^2]
\le D(x_0;\lambda)$.
Applying the same scalar Bernstein argument as in Lemma~\ref{lem:I1_clean}
with envelopes $G_{Y-f}(\theta),G_{f}(\theta)$ and the same choice $t\asymp\log N'$
yields: with probability at least
$1-8(N')^{-1}$,
\begin{equation}\label{eq:J1-hp-J}
\sup_{\theta\in \Theta_0}|J_1(x_0;\theta)|
\ \le\
2\sqrt{2\log N'}\,
\sqrt{D(x_0;\lambda)}
\left(
\frac{G_{Y-f}(\theta)}{\sqrt n}+\frac{G_{f}(\theta)}{\sqrt N}
\right).
\end{equation}

\noindent\textit{Step 3: $J_2$ follows $I_2$ on $\sE$.}

On the stability event $\sE$ from Lemma~\ref{lem:A_bound_1/2_sobolev},
the proof of Lemma~\ref{lem:I2-clean-hp-refined} applies verbatim after replacing
$B_{Y-f}(\theta),B_{f}(\theta)$ by the uniform derivative envelopes
$G_{Y-f}(\theta),G_{f}(\theta)$. Thus, on $\sE$, with probability at least $1-5(N')^{-1}$,
\begin{align}\label{eq:J2-hp-J}
\sup_{\theta\in \Theta_0}|J_2(x_0;\theta)|
\lesssim\;&
D(\lambda)\sqrt{D(x_0;\lambda)}
\sqrt{\frac{\log N'}{N'}}
\left(
\frac{G_{Y-f}(\theta)}{\sqrt n}+
\frac{G_{f}(\theta)}{\sqrt N}
\right)
\notag\\
&+(G_{f}(\theta)+G_{Y-f}(\theta))D(\lambda)\sqrt{D(x_0;\lambda)}\frac{\log N'}{N'}
\notag\\
&+D(\lambda)^{3/2}\sqrt{D(x_0;\lambda)}
\frac{(\log N')^{3/2}}{N'}
\left(
\frac{G_{Y-f}(\theta)}{\sqrt n}+
\frac{G_{f}(\theta)}{\sqrt N}
\right).
\end{align}

\noindent\textit{Step 4: $J_3$ on $\sE$.}
We now rework the $I_3$-type argument for $J_3(x_0;\theta)$.
The key difference from Lemma~\ref{lem:I3-clean} is that $\mu^1(\theta)$ cannot be
written as $T_K\eta(\cdot;\theta)$ for some $\eta(\cdot;\theta)\in\cH$; instead we can write
$\mu^1(\theta)=\EE[g_\theta(X)K_X]$ for a function $g_\theta\in L^2(\rho_X)$, which
introduces an extra $\sqrt{D(\lambda)}$ factor in the envelope.
Recall
\[
A(\lambda)
:=
(T_K+\lambda I)^{-1/2}\,(\wh T_K-T_K)\,(T_K+\lambda I)^{-1/2},
\qquad
u:=(T_K+\lambda I)^{-1/2}K_{x_0}.
\]
On $\sE=\{\|A(\lambda)\|_{\op}\le 1/2\}$, we have $\|(I+A(\lambda))^{-1}\|_{\op}\le 2$ and
\[
\wh w_{x_0,\lambda}-w_{x_0,\lambda}
=-(T_K+\lambda I)^{-1/2}(I+A(\lambda))^{-1}A(\lambda)u.
\]
Let $v_1(\theta):=(T_K+\lambda I)^{-1/2}\mu^1(\theta)$.
Then the exact identity holds:
\[
J_3(x_0;\theta)
=
-\big\langle (I+A(\lambda))^{-1}A(\lambda)u,\ v_1(\theta)\big\rangle_{\cH}.
\]
Expanding as in Lemma~\ref{lem:I3-clean}, on $\sE$ we can write $(I+A)^{-1}A
=
A-(I+A)^{-1}A^2$,
so
\[
J_3(x_0;\theta)
=
-\underbrace{\langle A(\lambda)u,\ v_1(\theta)\rangle_{\cH}}_{=:T_{1,J}(\theta)}
+\underbrace{\langle (I+A(\lambda))^{-1}A(\lambda)^2u,\ v_1(\theta)\rangle_{\cH}}_{=:T_{2,J}(\theta)}.
\]
Hence on $\sE$,
\begin{equation}\label{eq:J3-split-J}
|J_3(x_0;\theta)|
\le |T_{1,J}(\theta)| + |T_{2,J}(\theta)|
\le |T_{1,J}(\theta)| + 2\,|\langle A(\lambda)^2u,\ v_1(\theta)\rangle_{\cH}|.
\end{equation}
Let $g_x:=(T_K+\lambda I)^{-1/2}K_x\in\cH$ so that
\[
A(\lambda)=\frac1{N'}\sum_{j=1}^{N'}\Big(g_{\bar X_j}\otimes g_{\bar X_j}-\EE[g_X\otimes g_X]\Big).
\]
Exactly as in the $I_3$ proof, define
\[
a(x):=\langle g_x,u\rangle_{\cH}=w_{x_0,\lambda}(x),
\qquad
b_1(x;\theta):=\langle g_x,v_1(\theta)\rangle_{\cH}
=\big\langle K_x,(T_K+\lambda I)^{-1}\mu^1(\theta)\big\rangle_{\cH}.
\]
Then
\begin{equation}\label{eq:T1J-scalar}
T_{1,J}(\theta)
=
\langle A(\lambda)u,\ v_1(\theta)\rangle_{\cH}
=
\frac1{N'}\sum_{j=1}^{N'}\Big(Z_j(\theta)-\EE[Z(\theta)]\Big),
\qquad
Z_j(\theta):=a(\bar X_j)\,b_1(\bar X_j;\theta),
\end{equation}
and $\{Z_j(\theta)\}_{j=1}^{N'}$ are i.i.d.\ for each fixed $\theta$.
Define the conditional derivative function $g_\theta(x):=\EE[\partial_\theta\ell(Y;\theta)|X=x]\in L^2(\rho_X)$.
Under Assumption~\ref{ass:deriv-bdd-jac}, $|g_\theta(x)|\le G_{Y-f}(\theta)+G_{f}(\theta)$ for $\theta\in\Theta_0$, hence
$\|g_\theta\|_{L^2(\rho_X)}\le G_{Y-f}(\theta)+G_f(\theta)$ uniformly over $\theta\in\Theta_0$.
Using the Mercer expansion with $(\mu_j,\phi_j)$ and the standard identity
$\mu^1(\theta)=\EE[g_\theta(X)K_X]$, one obtains
\[
b_1(x;\theta)
=
\sum_{j\ge1}\frac{\mu_j}{\mu_j+\lambda}\,\langle g_\theta,\phi_j\rangle_{L^2(\rho_X)}\,\phi_j(x).
\]
Therefore, by Cauchy--Schwarz,
\begin{align}
\begin{split}
|b_1(x;\theta)|
&\le
\|g_\theta\|_{L^2(\rho_X)}\,
\Big(\sum_{j\ge1}\frac{\mu_j}{\mu_j+\lambda}\,\phi_j(x)^2\Big)^{1/2}
\\&=
\|g_\theta\|_{L^2(\rho_X)}\sqrt{D(x;\lambda)}
\le B_\phi\,\|g_\theta\|_{L^2(\rho_X)}\sqrt{D(\lambda)}.
\label{eq:b1-infty}
\end{split}
\end{align}
We will also use the companion bound:
in the RKHS eigenbasis $\{\psi_j\}$ with $\psi_j:=\sqrt{\mu_j}\phi_j\in\cH$,
\[
\mu^1(\theta)=\sum_{j\ge1}\sqrt{\mu_j}\,\langle g_\theta,\phi_j\rangle_{L^2}\,\psi_j,
\quad\to\quad
\|v_1(\theta)\|_{\cH}^2
=
\sum_{j\ge1}\frac{\mu_j}{\mu_j+\lambda}\,\langle g_\theta,\phi_j\rangle_{L^2}^2
\le \|g_\theta\|_{L^2(\rho_X)}^2,
\]
so for $\theta\in\Theta_0$,
\begin{equation}\label{eq:v1-L2}
\|v_1(\theta)\|_{\cH}\ \le\ G_{Y-f}(\theta)+G_{f}(\theta)<\infty.
\end{equation}
Using \eqref{eq:b1-infty}, $\Var(Z(\theta))\le \EE[Z(\theta)^2]\le \|b_1(\cdot;\theta)\|_\infty^2\,\EE[a(X)^2]$.
As in Lemma~\ref{lem:I3-clean}, $\EE[a(X)^2]\le D(x_0;\lambda)$.
Hence, uniformly over $\theta\in \Theta_0$,
\begin{equation}\label{eq:VarZ-T1J}
\Var(Z(\theta))
\ \le\
B_\phi^2\,\|g_\theta\|_{L^2(\rho_X)}^2\,D(\lambda)\,D(x_0;\lambda).
\end{equation}
For the envelope, use $\|Z-\EE Z\|_\infty\le 2\|Z\|_\infty\le 2\|a\|_\infty\|b_1\|_\infty$.
As in Lemma~\ref{lem:I3-clean},
\[
\|a\|_\infty=\|w_{x_0,\lambda}\|_\infty
\le B_\phi\sqrt{D(\lambda)D(x_0;\lambda)},
\]
and combining with \eqref{eq:b1-infty} yields, uniformly over $\theta\in \Theta_0$,
\begin{equation}\label{eq:EnvZ-T1J}
\|Z(\theta)-\EE Z(\theta)\|_\infty
\ \le\
C\,B_\phi^2\,\|g_\theta\|_{L^2(\rho_X)}\,D(\lambda)\,\sqrt{D(x_0;\lambda)}.
\end{equation}
Apply Bernstein to \eqref{eq:T1J-scalar} with $t=\log(4N')$.
With probability at least $1-2e^{-t}$, for each fixed $\theta$,
\[
|T_{1,J}(\theta)|
\le
\sqrt{\frac{2\,\Var(Z(\theta))\,t}{N'}}
+\frac{2\,\|Z(\theta)-\EE Z(\theta)\|_\infty\,t}{3N'}.
\]
Using \eqref{eq:VarZ-T1J}--\eqref{eq:EnvZ-T1J} and $\|g_\theta\|_{L^2}\le G_{Y-f}(\theta)+G_{f}(\theta)$ gives,
after absorbing constants,
\begin{align}
|T_{1,J}(\theta)|
&\le
C\,B_\phi\,(G_{Y-f}(\theta)+G_{f}(\theta))
\sqrt{\frac{D(\lambda)D(x_0;\lambda)\,t}{N'}}
\Bigg(
1+ C'\,B_\phi\sqrt{\frac{D(\lambda)t}{N'}}
\Bigg).
\label{eq:T1J-bern}
\end{align}
Under \eqref{eq:D-regime-J}, with $t\asymp\log N'$, we have
$\sqrt{D(\lambda)t/N'}\to 0$, so the parenthetical factor in \eqref{eq:T1J-bern}
is bounded by (say) $2$ for large $N'$, yielding
\begin{equation}\label{eq:T1J-clean}
|T_{1,J}(\theta)|
\ \le\
C\,B_\phi\,(G_{Y-f}(\theta)+G_{f}(\theta))
\sqrt{\frac{D(\lambda)D(x_0;\lambda)\,\log N'}{N'}}.
\end{equation}

From \eqref{eq:J3-split-J} and Cauchy--Schwarz,
\[
|\langle A(\lambda)^2u,\ v_1(\theta)\rangle_{\cH}|
\le \|A(\lambda)^2u\|_{\cH}\,\|v_1(\theta)\|_{\cH}
\le \|A(\lambda)\|_{\op}\,\|A(\lambda)u\|_{\cH}\,\|v_1(\theta)\|_{\cH}.
\]
On $\sE$, $\|A(\lambda)\|_{\op}\le 1/2$, hence
\begin{equation}\label{eq:T2J-reduce}
|T_{2,J}(\theta)|
\le
2\,|\langle A(\lambda)^2u,\ v_1(\theta)\rangle_{\cH}|
\le
\|A(\lambda)u\|_{\cH}\,\|v_1(\theta)\|_{\cH}.
\end{equation}
It remains to bound $\|A(\lambda)u\|_{\cH}$. As in Step 8 of the $I_3$ proof,
$A(\lambda)u$ is a centered empirical mean of $\cH$-valued random variables:
\[
A(\lambda)u=\frac1{N'}\sum_{j=1}^{N'}\Big(a(\bar X_j)\,g_{\bar X_j}-\EE[a(X)g_X]\Big),
\qquad g_x:=(T_K+\lambda I)^{-1/2}K_x.
\]
Applying a Hilbert-valued Bernstein inequality
and using $\|g_x\|_{\cH}^2=D(x;\lambda)\le B_\phi^2D(\lambda)$ together with
$\|a\|_\infty\le B_\phi\sqrt{D(\lambda)D(x_0;\lambda)}$,
we obtain: for $t=\log(4N')$, with probability at least $1-2e^{-t}$,
\begin{equation}\label{eq:Au-J}
\|A(\lambda)u\|_{\cH}
\le
C_u\,\sqrt{D(x_0;\lambda)}
\left(
\sqrt{\frac{D(\lambda)t}{N'}}+\frac{D(\lambda)t}{N'}
\right),
\end{equation}
for an absolute constant $C_u$ depending only on $(B_\phi,\kappa)$.
Combining \eqref{eq:T2J-reduce}, \eqref{eq:Au-J}, and \eqref{eq:v1-L2} yields
\begin{align}
|T_{2,J}(\theta)|
&\le
C_u\,(G_{Y-f}(\theta)+G_{f}(\theta))\,\sqrt{D(x_0;\lambda)}
\left(
\sqrt{\frac{D(\lambda)t}{N'}}+\frac{D(\lambda)t}{N'}
\right).
\label{eq:T2J-bound}
\end{align}
Under \eqref{eq:D-regime-J} (with $t\asymp\log N'$), the second term is $o(1)$ and is dominated by
the leading $\sqrt{D(\lambda)t/N'}$ term. In particular, for all large $N'$,
\begin{equation}\label{eq:T2J-clean}
|T_{2,J}(\theta)|
\ \le\
C\,(G_{Y-f}(\theta)+G_{f}(\theta))\sqrt{\frac{D(\lambda)D(x_0;\lambda)\log N'}{N'}}.
\end{equation}

Combine \eqref{eq:J3-split-J}, \eqref{eq:T1J-clean}, and \eqref{eq:T2J-clean} (with $t=\log(4N')$).
On $\sE$, with probability at least $1-(N')^{-1}$, we obtain
\begin{equation}\label{eq:J3-hp-J}
|J_3(x_0;\theta)|
\ \le\
C\,B_\phi\,(G_{Y-f}(\theta)+G_{f}(\theta))
\sqrt{\frac{D(\lambda)D(x_0;\lambda)\log N'}{N'}}.
\end{equation}
Finally, note that if one wishes to express this in a worst-case leverage form, then using
$\sup_x D(x;\lambda)\le B_\phi^2D(\lambda)$ and $D(x_0;\lambda)\le B_\phi^2D(\lambda)$ gives
\[
|J_3(x_0;\theta)|
\ \le\
C\,B_\phi^2\,(G_{Y-f}(\theta)+G_{f}(\theta))\,D(\lambda)\sqrt{\frac{\log N'}{N'}},
\]
which tends to $0$ under \eqref{eq:D-regime-J}.

\noindent\textit{Step 5: probability bookkeeping and uniform event.}
Let $\sE_{J1}$ be the event on which \eqref{eq:J1-hp-J} holds, $\sE_{J2}$ the event on which
\eqref{eq:J2-hp-J} holds, and $\sE_{J3}$ the event on which \eqref{eq:J3-hp-J} holds (all uniformly over $\theta\in \Theta_0$).
By Steps 2--4, there exists an absolute constant $C>0$ such that
\[
\PP(\sE_{J1}^c)\le 8(N')^{-1},\qquad
\PP(\sE_{J2}^c| \sE)\le 5(N')^{-1},\qquad
\PP(\sE_{J3}^c| \sE)\le (N')^{-1}.
\]
Therefore, by a union bound,
\begin{equation}\label{eq:prob-union-J}
\PP\big(\sE_{J1}\cap\sE_{J2}\cap\sE_{J3}|\sE\big)
\ \ge\ 1-14(N')^{-1}.
\end{equation}

\noindent\textit{Step 6: uniform control of $\wh J_\lambda-J_\lambda$.}
On the intersection event in \eqref{eq:prob-union-J}, combine
\eqref{eq:Hhat-H-decomp-J} with \eqref{eq:J1-hp-J}, \eqref{eq:J2-hp-J}, and \eqref{eq:J3-hp-J} to obtain

\begin{align*}
&\sup_{\theta\in \Theta_0}\big|\wh J_\lambda(x_0;\theta)-J_\lambda(x_0;\theta)\big|\\
\lesssim&
\sqrt{D(x_0;\lambda)\log N'}
\left(
\frac{1}{\sqrt n}+\frac{1}{\sqrt N}
\right)
+
D(\lambda)\sqrt{D(x_0;\lambda)}
\sqrt{\frac{\log N'}{N'}}
\left(
\frac{1}{\sqrt n}+\frac{1}{\sqrt N}
\right)\\
&+D(\lambda)\sqrt{D(x_0;\lambda)}\frac{\log N'}{N'}
+D(\lambda)^{3/2}\sqrt{D(x_0;\lambda)}
\frac{(\log N')^{3/2}}{N'}
\left(
\frac{1}{\sqrt n}+\frac{1}{\sqrt N}
\right)\\
&+
B_\phi
\sqrt{\frac{D(\lambda)D(x_0;\lambda)\log N'}{N'}}.
\end{align*}

By the scaling conditions in \eqref{eq:D-regime-J}, each term on the right-hand side
tends to $0$ as $N'\to\infty$ (using also $n\wedge N\to\infty$ and $N'\ge n\wedge N$).
Hence there exists $N'_0$ such that for all $N'\ge N'_0$,
\[
\sup_{\theta\in \Theta_0}\big|\wh J_\lambda(x_0;\theta)-J_\lambda(x_0;\theta)\big|
\ \le\ c_J/2
\qquad\text{on }\ \sE\cap\sE_{J1}\cap\sE_{J2}\cap\sE_{J3}.
\]
Therefore, on the same event and for all $\theta\in \Theta_0$,
\[
|\wh J_\lambda(x_0;\theta)|
\ge
|J_\lambda(x_0;\theta)|-\big|\wh J_\lambda(x_0;\theta)-J_\lambda(x_0;\theta)\big|
\ge c_J-c_J/2=c_J/2,
\]
where we used Assumption~\ref{ass:J_lambda_bound}. This proves \eqref{eq:H-invert-final}.
Finally, $\sup_{\theta\in \Theta_0}|\wh J_\lambda(x_0;\theta)^{-1}|
\le \frac{2}{c_J}$.
\end{proof}

\begin{proof}[Proof of Lemma~\ref{lem:V-lb}]
\textit{Step 1: reduce to labeled variance.}
By independence of the labeled and unlabeled samples,
\[
V(x_0;\theta)
=\frac{1}{n}\sigma^2_{Y-f}(\theta)+\frac{1}{N}\sigma^2_{f}(\theta)
\ \ge\ \frac{1}{n}\sigma^2_{Y-f}(\theta).
\]

\noindent\textit{Step 2.}
By the law of total variance,
\begin{align*}
\sigma^2_{Y-f}(\theta)
&=
\Var\big(w_{x_0,\lambda}(X)\,r\big)\\
&=
\EE\Big[w_{x_0,\lambda}(X)^2\,\Var(r| X)\Big]
+\Var\Big(w_{x_0,\lambda}(X)\,\EE[r| X]\Big)\\
&\ge
\EE\Big[w_{x_0,\lambda}(X)^2\,\Var(r| X)\Big]
\ \ge\
\underline\sigma^2\,Q(x_0;\lambda),
\end{align*}
where the last inequality uses Assumption~\ref{ass:homo_res}.

\noindent\textit{Step 3: compare $Q$ and $D$.}
Proposition~\ref{prop:Dx0-order-H} gives a constant $c_Q>0$ such that
$Q(x_0;\lambda)\ge c_QD(x_0;\lambda)$ for all sufficiently small $\lambda$.
Combining this comparison with Step~2 yields
\[
V(x_0;\theta)
\ge\frac{\underline\sigma^2}{n}Q(x_0;\lambda)
\ge\frac{\underline\sigma^2c_Q}{n}D(x_0;\lambda),
\]
which proves \eqref{eq:c-lambda-lb} after relabeling the constant.
\end{proof}

\begin{proof}[Proof of Lemma~\ref{lem:variance-ratio}]
We give a common argument and indicate the only step that differs between the two implementations. Here $\mathbb P_m$ denotes the empirical average over the indicated sample or subsample of size $m$. Let $\widetilde V(x_0)$ denote the same sample-variance formula as $\widehat V(x_0)$, but evaluated at the population weight $w_{x_0,\lambda}$ and the population center $\theta_\lambda(x_0)$. For either bounded labeled or unlabeled contribution $R$, spectral Cauchy--Schwarz gives
\[
\|w_{x_0,\lambda}\|_\infty^2
\le D(x_0;\lambda)\sup_{x\in\cX}D(x;\lambda)
\lesssim Q(x_0;\lambda)D(\lambda).
\]
Consequently,
$\EE[\{w_{x_0,\lambda}(X)R\}^4]\lesssim Q(x_0;\lambda)^2D(\lambda)$, and the usual second-moment calculation for a sample variance yields
\[
|\widetilde V(x_0)-V(x_0)|
=O_p\left[
\frac{Q(x_0;\lambda)}{n}\sqrt{\frac{D(\lambda)}{n}}
+\frac{Q(x_0;\lambda)}{N}\sqrt{\frac{D(\lambda)}{N}}
\right]
=o_p\{Q(x_0;\lambda)/n\}
\]
under the stated rate conditions.

It remains to replace the population weight and center by their empirical versions. Bounded derivatives of $\ell$ and $\widehat\theta(x_0)-\theta_\lambda(x_0)=o_p(1)$ show that replacing the center changes each unscaled sample variance by $o_p\{Q(x_0;\lambda)\}$. For the two-fold procedure, condition on the opposite weight-training fold. The foldwise $L^2(\rho_X)$ weight-error bound in Lemma~\ref{lem:delta_f_L2_bound}, followed by conditional Bernstein, gives
$\mathbb P_{n}(\widehat w-w)^2=o_p(Q)$ for the labeled average and the analogous out-of-fold bounds for both unlabeled folds.

For the no-split procedure, the evaluation covariates also train the weight, so an iid law of large numbers is not applicable. Let $N'=n+N$, $\widehat T_K$ be the pooled empirical operator, $\Delta w=\widehat w-w$, and $T_\lambda=T_K+\lambda I$. On the operator-stability event, the resolvent identity gives
\[
\|T_\lambda^{1/2}\Delta w\|_{\cH}^2
\le4D(x_0;\lambda)\|A(\lambda)\|_{\rm op}^2
=O_p\left\{Q(x_0;\lambda)D(\lambda)\frac{\log N'}{N'}\right\}.
\]
Moreover,
$\mathbb P_{N'}(\Delta w)^2=\langle\Delta w,\widehat T_K\Delta w\rangle_{\cH}$ is bounded by a constant multiple of the same energy norm. Since the labeled and unlabeled empirical quadratic errors are nonnegative,
\[
\mathbb P_n(\Delta w)^2\le\frac{N'}n\mathbb P_{N'}(\Delta w)^2=o_p(Q),
\qquad
\mathbb P_N(\Delta w)^2\le\frac{N'}N\mathbb P_{N'}(\Delta w)^2=o_p(Q),
\]
where the last step uses $D(\lambda)\log N'/n=o(1)$ and its unlabeled analogue. Cauchy--Schwarz controls all cross terms, and empirical centering in the sample variances is handled identically. Therefore
$|\widehat V-\widetilde V|=o_p(Q/n)$ for both procedures. Combining the two steps proves the first assertion, and Lemma~\ref{lem:V-lb} gives $V(x_0)\ge\underline\sigma^2Q(x_0;\lambda)/n$, hence $\widehat V/V\to_p1$. Finally, the existing uniform Jacobian approximation and Assumption~\ref{ass:J_lambda_bound} imply $\widehat J_\lambda/J_\lambda\to_p1$.
\end{proof}

\section{Proofs of Theoretical Results in Sections~\ref{sec:theory} and \ref{sec:practical}}\label{app:proofs-of-PPCI-combine}
\noindent
This appendix provides the complete theoretical proofs for the results presented in Sections~\ref{sec:theory} and \ref{sec:practical} under the two-fold sample-splitting setting, along with the proofs of all supporting technical lemmas. While the overarching proof strategy for the error bound (Theorem~\ref{thm:thetahat-theta0-bound-twofold}) and asymptotic normality (Theorem~\ref{thm:thetahat-asymp-twofold}) closely follows the non-split framework established in Appendix~\ref{app:proofs-of-PPCI-nosplit}, the technical details are carefully adapted here to accommodate the cross-fitting structure. Furthermore, we include the proofs for Propositions~\ref{prop:Dx0-order-H} and~\ref{prop:budget_opt} and the minimax lower bound (Theorem~\ref{thm:minimax-pointwise}). The remainder of this appendix is organized as follows: we first proceed to prove the main theorems and proposition in Appendix~\ref{app:proofs-of-PPCI}, and detail the proofs of the required technical lemmas in Appendix~\ref{app:proofs-of-lemmas}. Finally, Appendix~\ref{subsec:technical-novelties} provides a detailed discussion on our technical novelties and a comprehensive comparison with classical KRR theory.

\subsection{Proofs of Theorems in Sections~\ref{sec:theory} and \ref{sec:practical}}\label{app:proofs-of-PPCI}
\noindent
We begin by introducing the notation that differs from that used in Appendix~\ref{app:proofs-of-PPCI-nosplit}.

\noindent\textbf{Two-fold splitting on unlabeled covariates.}
Randomly partition $\{1,\dots,N\}$ into two disjoint folds $\cI_1$ and $\cI_2$
(with $|\cI_1|=|\cI_2|=N/2$ for simplicity).
For each $m\in\{1,2\}$, define the fold-specific empirical covariance operator
\[
\wh T_K^{(m)}:=\frac{1}{|\cI_m|}\sum_{u\in\cI_m} K_{\wt X_u}\otimes K_{\wt X_u},
\]
and the corresponding empirical localization weight
\[
\wh w^{(m)}_{x_0,\lambda}:=\big(\wh T_K^{(m)}+\lambda I\big)^{-1}K_{x_0}\in\cH,
\qquad
\Delta w^{(m)}:=\wh w^{(m)}_{x_0,\lambda}-w_{x_0,\lambda}.
\]
For labeled data, we apply the averaged weight
\[
\overline w_{x_0,\lambda}:=\frac12\big(\wh w^{(1)}_{x_0,\lambda}+\wh w^{(2)}_{x_0,\lambda}\big),
\qquad
\Delta \overline w:=\overline w_{x_0,\lambda}-w_{x_0,\lambda}
=\frac12\big(\Delta w^{(1)}+\Delta w^{(2)}\big).
\]
For unlabeled data, we evaluate out-of-fold: for $u\in\cI_m$, we weight $\wt X_u$ using $\wh w^{(3-m)}_{x_0,\lambda}$.

\noindent\textbf{Localized PPI moments.}
Recall that $\eta_\lambda(x_0;\theta)=\EE[w_{x_0,\lambda}(X)\ell(Y;\theta)]$.
Define its two-fold cross-fitted empirical estimator is
\begin{equation}\label{eq:eta-cf-notation}
\wh\eta_{\lambda}(x_0;\theta)
:=
\frac{1}{n}\sum_{i=1}^n \overline{w}_{x_0,\lambda}(X_i)\, r(Y_i,X_i;\theta)
+
\frac{1}{N}\sum_{m=1}^2 \sum_{u\in\cI_m}
\wh w^{(3-m)}_{x_0,\lambda}(\wt X_u)\,u(\wt X_u;\theta).
\end{equation}
We also define the oracle version that uses the population weight $w_{x_0,\lambda}$,
\[
\wt \eta_\lambda(x_0;\theta)
:=
\frac{1}{n}\sum_{i=1}^n w_{x_0,\lambda}(X_i)\,r(Y_i,X_i;\theta)
+
\frac{1}{N}\sum_{u=1}^N w_{x_0,\lambda}(\wt X_u)\,u(\wt X_u;\theta).
\]
The true target is defined by $\eta(x_0;\theta_0(x_0))=0$, and the estimator $\wh\theta(x_0)$ is any solution to
$\wh\eta_\lambda(x_0;\wh\theta(x_0))=0$.

\noindent\textbf{Decomposition.}
The reproducing property gives
$\wt \eta_\lambda(x_0;\theta)=\langle w_{x_0,\lambda},\wh\mu(\theta)\rangle_{\cH}$
and
$\eta_\lambda(x_0;\theta)=\langle w_{x_0,\lambda},\mu(\theta)\rangle_{\cH}$.
Therefore,
\[
\wt \eta_\lambda(x_0;\theta)-\eta_\lambda(x_0;\theta)
=
\underbrace{\langle w_{x_0,\lambda},\wh\mu(\theta)-\mu(\theta)\rangle_{\cH}}_{=:I_1(x_0;\theta)}.
\]
Next, write the additional error induced by estimating the localization weight as
\[
\wh\eta_\lambda(x_0;\theta)-\wt\eta_\lambda(x_0;\theta)
=
\frac{1}{n}\sum_{i=1}^n \Delta\overline w(X_i)\,r(Y_i,X_i;\theta)
+
\frac{1}{N}\sum_{m=1}^2\sum_{u\in\cI_m} \Delta w^{(3-m)}(\wt X_u)\,u(\wt X_u;\theta).
\]
We split this term into a mean component and a centered fluctuation component.
Define
\[
I^{\mathrm{split}}_3(x_0;\theta)
:=
\langle \Delta\overline w,\mu(\theta)\rangle_{\cH},
\]
and
\[
I^{\mathrm{split}}_2(x_0;\theta)
:=
\big(\wh\eta_\lambda(x_0;\theta)-\wt\eta_\lambda(x_0;\theta)\big)
-
I^{\mathrm{split}}_3(x_0;\theta).
\]
Combining the above identities yields the three-term decomposition
\[
\wh \eta_\lambda(x_0;\theta)-\eta_\lambda(x_0;\theta)
=
I_1(x_0;\theta)
+
I^{\mathrm{split}}_2(x_0;\theta)
+
I^{\mathrm{split}}_3(x_0;\theta).
\]

\noindent\textbf{Jacobians.}
Recall that $J_\lambda(x_0;\theta)
=\EE[w_{x_0,\lambda}(X)\partial_\theta\ell(Y;\theta)]$
and define the empirical Jacobian corresponding to \eqref{eq:eta-cf-notation},
\[
\wh J_\lambda(x_0;\theta)
:=\partial_\theta\wh \eta_\lambda(x_0;\theta)
=
\frac{1}{n}\sum_{i=1}^n
\overline w_{x_0,\lambda}(X_i)\,\partial_\theta r(Y_i,X_i;\theta)
+\frac{1}{N}\sum_{m=1}^2\sum_{u\in\cI_m}
\wh w^{(3-m)}_{x_0,\lambda}(\wt X_u)\,\partial_\theta u(\wt X_u;\theta).
\]

\noindent\textbf{Operator stability with two-fold splitting.}
We next present a technical lemma establishing operator stability under the two-fold sample-splitting regime, which is instrumental for our subsequent theoretical analysis. This result can be viewed as the cross-fitted counterpart to Lemma~\ref{lem:A_bound_1/2_sobolev}. Since its proof follows identical arguments to those used for the non-split version, it is omitted here.
\begin{lemma}
\label{lem:A_bound_1/2_sobolev_split}
Write $\tau:=\frac{d}{2m}\in(0,1)$.
For each fold $m\in\{1,2\}$, define
\[
A^{(m)}(\lambda)
:=
(T_K+\lambda I)^{-1/2}\big(\wh T_K^{(m)}-T_K\big)(T_K+\lambda I)^{-1/2}.
\]
Then there exists a constant $C_a:=C_1^2\,\tau^\tau(1-\tau)^{1-\tau}$ such that,
for every $\delta\in(0,1]$, with probability at least $1-\delta$,
\[
\|A^{(m)}(\lambda)\|_{\op}\le
\frac{4\,C_a}{3\,|\cI_m|\,\lambda^{\tau}}\,
\log\Big(\frac{4\,\lambda^{\tau}D(\lambda)}{\delta}\Big)
+
\sqrt{\frac{2\,C_a}{|\cI_m|\,\lambda^{\tau}}\,
\log\Big(\frac{4\,\lambda^{\tau}D(\lambda)}{\delta}\Big)}.
\]
Moreover, under Proposition~\ref{prop:Dx0-order-H}, $D(\lambda)\asymp \lambda^{-\tau}$,
taking $\delta=|\cI_m|^{-1}$ yields
\[
\|A^{(m)}(\lambda)\|_{\op}
\le
\frac{8\,C_a}{3\,|\cI_m|\,\lambda^{\tau}}\,
\log |\cI_m|
+
\sqrt{\frac{4\,C_a}{|\cI_m|\,\lambda^{\tau}}\,
\log |\cI_m|}.
\]
Let $\sE_m$ denote the event on which the preceding display holds. If the regularization satisfies
\[
\lambda\ \ge\
\Big(\frac{256\,C_a\,\log |\cI_m|}{|\cI_m|}\Big)^{1/\tau}
=
\Big(\frac{256\,C_a\,\log |\cI_m|}{|\cI_m|}\Big)^{2m/d},
\]
then $\sE_m$ implies $\|A^{(m)}(\lambda)\|_{\op}\le1/2$, and $\PP(\sE_m)\ge 1-|\cI_m|^{-1}$.
In particular, under two-fold equal splitting $|\cI_1|=|\cI_2|=N/2$, defining $\sE:=\sE_1\cap\sE_2$,
a union bound gives $\PP(\sE)\ge 1-4/N$.
\end{lemma}

\noindent\textbf{Proof of Proposition~\ref{prop:Dx0-order-H}.} First, we provide the proof of Proposition~\ref{prop:Dx0-order-H} in Section~\ref{sec:theory}.
\begin{proof}[Proof of Proposition~\ref{prop:Dx0-order-H}]
For the Sobolev space $H^m$, the eigenvalues of $T_K$ satisfy $\mu_j(T_K)\asymp j^{-2m/d}$~\cite{fischer2020sobolev}, and a direct calculation gives $D(\lambda)\asymp\lambda^{-d/(2m)}$. We next analyze $D(x_0;\lambda)$.
For $h,g\in\cH$, by the definition of a rank-one operator and the reproducing property
$\langle h,K_x\rangle_{\cH}=h(x)$,
\[
\big\langle h,(K_X\otimes K_X)g\big\rangle_{\cH}
=
\langle h,K_X\rangle_{\cH}\,\langle K_X,g\rangle_{\cH}
=
h(X)\,g(X).
\]
Taking expectation yields $\langle h,T_K g\rangle_{\cH}
=
\EE[h(X)g(X)]$.
In particular,
\begin{equation}\label{eq:TK-quadratic}
\langle h,T_K h\rangle_{\cH}
=
\EE[h(X)^2]
=
\|h\|_{L^2(\rho_X)}^2.
\end{equation}

\noindent\textit{Step 1: variational representation of $D(x_0;\lambda)$.}
Let $T_\lambda:=T_K+\lambda I$ on $\cH$, which is self-adjoint and strictly positive.
For any fixed $g\in\cH$,
\begin{equation}\label{eq:basic-variational}
\langle g,T_\lambda^{-1}g\rangle_{\cH}
=
\sup_{h\in\cH,\ h\neq 0}\frac{\langle h,g\rangle_{\cH}^2}{\langle h,T_\lambda h\rangle_{\cH}}
=
\sup_{\langle h,T_\lambda h\rangle_{\cH}\le 1}\langle h,g\rangle_{\cH}^2.
\end{equation}
Apply \eqref{eq:basic-variational} with $g=K_{x_0}$.
Since $\langle h,K_{x_0}\rangle_{\cH}=h(x_0)$ and by \eqref{eq:TK-quadratic},
\[
\langle h,T_\lambda h\rangle_{\cH}
=
\langle h,T_K h\rangle_{\cH}+\lambda\|h\|_{\cH}^2
=
\|h\|_{L^2(\rho_X)}^2+\lambda\|h\|_{\cH}^2.
\]
Therefore
\begin{equation}\label{eq:Dx0-variational}
D(x_0;\lambda)
=
\sup_{h\in\cH,\ h\neq 0}
\frac{h(x_0)^2}{\|h\|_{L^2(\rho_X)}^2+\lambda\|h\|_{\cH}^2}.
\end{equation}

\noindent\textit{Step 2: lower bound $D(x_0;\lambda)\gtrsim \lambda^{-d/(2m)}$.}
Since the density of $\rho_X$ is bounded above/below on $\cX$,
$\|h\|_{L^2(\rho_X)}$ is equivalent to the Lebesgue $L^2$ norm on $\cX$:
\[
\rho_0^{1/2}\|h\|_{L^2(dx)}\le \|h\|_{L^2(\rho_X)}\le \rho_1^{1/2}\|h\|_{L^2(dx)}.
\]
Hence, throughout the proof we may invoke Sobolev inequalities stated with $L^2(dx)$
and replace them by $L^2(\rho_X)$ at the cost of constants depending only on $\rho_0,\rho_1$.
We keep the notation $\|h\|_{L^2(\rho_X)}$ for consistency.

Let $R:=\lambda^{-1/2}$ and consider the subset
$\mathscr{C}_R:=\{h\in\cH:\ \|h\|_{\cH}\le R\|h\|_{L^2(\rho_X)}\}$.
For any $h\in\sC_R$,
\[
\|h\|_{L^2(\rho_X)}^2+\lambda\|h\|_{\cH}^2
\le
\|h\|_{L^2(\rho_X)}^2+\lambda R^2\|h\|_{L^2(\rho_X)}^2
=
2\|h\|_{L^2(\rho_X)}^2.
\]
Plugging this restriction into \eqref{eq:Dx0-variational} yields
\begin{align}
\begin{split}
D(x_0;\lambda)
&\ge
\sup_{h\in\sC_R,\ h\neq 0}
\frac{h(x_0)^2}{\|h\|_{L^2(\rho_X)}^2+\lambda\|h\|_{\cH}^2}
\ \ge\
\frac12\sup_{h\in\sC_R,\ h\neq 0}\frac{h(x_0)^2}{\|h\|_{L^2(\rho_X)}^2}\\
&=
\frac12\Bigg(\sup_{h\in\sC_R,\ h\neq 0}\frac{|h(x_0)|}{\|h\|_{L^2(\rho_X)}}\Bigg)^2.
\label{eq:Dx0-lb-to-sup}
\end{split}
\end{align}
Now invoke Lemma~\ref{lem:sobolev_rate} and choose $\lambda_0\in(0,1)$ so that $R=\lambda^{-1/2}\ge R_0$ whenever $\lambda\le \lambda_0$.
Combining with \eqref{eq:Dx0-lb-to-sup} gives, for $\lambda\in(0,\lambda_0]$,
\[
D(x_0;\lambda)
\ \ge\
\frac12 C_0^2\,R^{d/m}
\asymp\lambda^{-d/(2m)}.
\]

\noindent\textit{Step 3: upper bound $D(x_0;\lambda)\lesssim \lambda^{-d/(2m)}$.}
Fix $h\in\cH$, $h\neq 0$. Set $t:=\|h\|_{\cH}/\|h\|_{L^2(\rho_X)}\in(0,\infty)$.
By the second inequality in Lemma~\ref{lem:sobolev_rate},
\[
|h(x_0)|
\le
\|h\|_\infty
\le
C_1\,\|h\|_{L^2(\rho_X)}^{1-\frac{d}{2m}}\,
\|h\|_{\cH}^{\frac{d}{2m}}
=
C_1\,\|h\|_{L^2(\rho_X)}\,t^{\frac{d}{2m}}.
\]
Therefore,
\begin{align}
\frac{h(x_0)^2}{\|h\|_{L^2(\rho_X)}^2+\lambda\|h\|_{\cH}^2}
&\le
C\,\frac{\|h\|_{L^2(\rho_X)}^2\,t^{d/m}}{\|h\|_{L^2(\rho_X)}^2(1+\lambda t^2)}
=
C\,\frac{t^{d/m}}{1+\lambda t^2}.
\label{eq:Dx0-ub-to-t}
\end{align}
Taking the supremum over $h\neq 0$ in \eqref{eq:Dx0-variational} yields
\[
D(x_0;\lambda)\ \le\ C\sup_{t>0}\frac{t^{d/m}}{1+\lambda t^2}.
\]
Let $s:=\sqrt{\lambda}\,t$. Then
$\sup_{t>0}\frac{t^{d/m}}{1+\lambda t^2}=\lambda^{-d/(2m)}\sup_{s>0}\frac{s^{d/m}}{1+s^2}$.
Since $m>d/2$, we have $d/m\in(0,2)$ and hence $\sup_{s>0}s^{d/m}/(1+s^2)<\infty$.
Thus, for all $\lambda\in(0,\lambda_0]$,
$D(x_0;\lambda)\ \lesssim\lambda^{-d/(2m)}$.

\noindent\textit{Step 4: order of $Q(x_0;\lambda)$ and $D(x_0;\lambda)-Q(x_0;\lambda)$.}
Write $\tau=d/(2m)\in(0,1)$ and
$D_j=\mu_j\phi_j(x_0)^2/(\mu_j+\lambda)$, so that
\[
Q_j=D_j\frac{\mu_j}{\mu_j+\lambda},
\qquad
(D-Q)_j=D_j\frac{\lambda}{\mu_j+\lambda}.
\]
The upper bounds $Q\le D$ and $D-Q\le D$ are immediate. For the lower bounds, fix $0<a<1<b$ and partition the spectrum according to $\mu_j<a\lambda$, $a\lambda\le\mu_j\le b\lambda$, and $\mu_j>b\lambda$. The eigenvalue decay and $\sup_j\|\phi_j\|_\infty\le B_\phi$ imply
\[
\sum_{\mu_j<a\lambda}D_j
\le \frac{B_\phi^2}{\lambda}\sum_{\mu_j<a\lambda}\mu_j
\le C_La^{1-\tau}\lambda^{-\tau},
\qquad
\sum_{\mu_j>b\lambda}D_j
\le B_\phi^2\#\{j:\mu_j>b\lambda\}
\le C_Hb^{-\tau}\lambda^{-\tau}.
\]
By Step~2, $D(x_0;\lambda)\ge c_D\lambda^{-\tau}$. Choose fixed $a$ and $b$, depending only on the displayed constants, so that
$C_La^{1-\tau}\le c_D/4$ and $C_Hb^{-\tau}\le c_D/4$.
The middle band then contributes at least $(c_D/2)\lambda^{-\tau}$ to $D$. On this fixed band,
$Q_j\ge\{a/(1+a)\}D_j$ and $(D-Q)_j\ge\{1/(1+b)\}D_j$. Thus
$Q(x_0;\lambda)\gtrsim\lambda^{-\tau}$ and
$D(x_0;\lambda)-Q(x_0;\lambda)\gtrsim\lambda^{-\tau}$, with no additional local spectral-mass assumption.

\end{proof}

\noindent\textbf{Bounds for $I_2^{\mathrm{split}}$ and $I_3^{\mathrm{split}}$.}
Since the term $I_1$ remains identical to that in Appendix~\ref{app:proofs-of-PPCI-nosplit}, we now provide the error bounds for the remaining two components, $I_2^{\mathrm{split}}$ and $I_3^{\mathrm{split}}$, in the following lemma.

\begin{lemma}[Two-fold splitting bounds for $I_2^{\mathrm{split}}$ and $I_3^{\mathrm{split}}$]
\label{lem:I23_split_bounds_v2}
Fix $x_0\in\cX$ and $\theta\in\Theta$.
Let $\sE_m$ be the event from Lemma~\ref{lem:A_bound_1/2_sobolev_split} and set $\sE:=\sE_1\cap\sE_2$.

\noindent{(i)}
Assume additionally $D(\lambda)\sqrt{\frac{\log N}{N}}\to 0$.
Then for all sufficiently large $N$, with probability at least $1-\frac{12}{N}$,
\[
|I_3^{\mathrm{split}}(x_0;\theta)|
\le
8\sqrt{2}\kappa\|\eta(\cdot;\theta)\|_{\cH}
\sqrt{\frac{D(x_0;\lambda)\log N}{N}}.
\]
\noindent{(ii)}
Let $c_0>0$ be the absolute constant in Lemma~\ref{lem:delta_f_L2_bound}.
Then with probability at least $
1-\frac{6}{N}-\frac{2}{n^2}-\frac{4}{N^2}$,
we have
\[
|I_2^{\mathrm{split}}(x_0;\theta)|
\le
B_\phi\sqrt{D(\lambda)D(x_0;\lambda)}
\left[
(2c_0+4)B_{Y-f}(\theta)\frac{\log n}{n}
+
\Big(c_0B_{Y-f}(\theta)+(4c_0+8)B_f(\theta)\Big)\frac{\log N}{N}
\right].
\]
\end{lemma}

\noindent\textbf{Jacobian stability and invertibility.}
We analyze the stability and invertibility of the empirical Jacobian under two-fold splitting.

\begin{lemma}
\label{lem:Hhat-H-J123_split}
Under Assumptions~\ref{ass:deriv-bdd-jac}, \ref{ass:J_lambda_bound}, and \ref{ass:eigenfunctions_bound}, suppose that the regularization level satisfies
\begin{equation}\label{eq:D-regime-J-split}
D(\lambda)\sqrt{\frac{\log N}{N}}\to 0,
\qquad
\frac{D(x_0;\lambda)\log N}{n\wedge N}\to 0.
\end{equation}
Let $\sE_m$ be the event from
Lemma~\ref{lem:A_bound_1/2_sobolev_split} and set $\sE:=\sE_1\cap\sE_2$.
Set $G_{Y-f}:=\sup_{\theta\in\Theta_0}G_{Y-f}(\theta)$ and $G_f:=\sup_{\theta\in\Theta_0}G_{f}(\theta)$,
where $G_{Y-f}(\theta),G_{f}(\theta)$ are the derivative envelopes in Assumption~\ref{ass:deriv-bdd-jac}.
Then on the event $\sE$, for all sufficiently large $N$, with probability at least $1-\frac{26}{N}-\frac{2}{n^2}-\frac{4}{N^2}$,
we have
\begin{equation}\label{eq:H-invert-final-split}
\inf_{\theta\in\Theta_0}\big|\wh J_\lambda(x_0;\theta)\big|\ge c_J/2.
\end{equation}
Consequently, on $\sE$ and the same event, $\sup_{\theta\in\Theta_0}\big|\wh J_\lambda(x_0;\theta)^{-1}\big|
\le \frac{2}{c_J}$.
\end{lemma}

\noindent\textbf{Proof of Theorem~\ref{thm:thetahat-theta0-bound-twofold}.} We provide the following proof.
\begin{proof}[Proof of Theorem~\ref{thm:thetahat-theta0-bound-twofold}]
The moment and Jacobian events assembled in Steps~2--3 satisfy Lemma~\ref{lem:local-root-existence}; hence the locally unique two-fold empirical root exists on the same high-probability event for all sufficiently large sample sizes.
\noindent\textit{Step 1: linear expansion around $\theta_\lambda(x_0)$.}
By definition,
\[
\wh\theta(x_0)-\theta_0(x_0)
=
(\wh\theta(x_0)-\theta_\lambda(x_0))+(\theta_\lambda(x_0)-\theta_0(x_0)).
\]
Since $\wh\theta(x_0)\in\Theta_0$ solves $\wh\eta_\lambda(x_0;\wh\theta(x_0))=0$ and $\theta_\lambda(x_0)\in\Theta_0$ solves $\eta_\lambda(x_0;\theta_\lambda(x_0))=0$, the mean-value theorem yields the existence of some $\wt\theta$ between $\wh\theta(x_0)$ and $\theta_\lambda(x_0)$ such that
\[
0
=
\wh\eta_\lambda(x_0;\wh\theta(x_0))
=
\wh\eta_\lambda(x_0;\theta_\lambda(x_0))
+
\wh J_\lambda(x_0;\wt\theta)\big(\wh\theta(x_0)-\theta_\lambda(x_0)\big),
\]
hence
\begin{equation}
\label{eq:theta-stoch-linear-twofold}
\wh\theta(x_0)-\theta_\lambda(x_0)
=
-\wh J_\lambda(x_0;\wt\theta)^{-1}\wh\eta_\lambda(x_0;\theta_\lambda(x_0)).
\end{equation}

\noindent\textit{Step 2: control $(\wh J_\lambda)^{-1}$ on $\Theta_0$.}
Let $\sE_m$ and $\sE:=\sE_1\cap\sE_2$ be as in Lemma~\ref{lem:A_bound_1/2_sobolev_split}. By that lemma, $\PP(\sE^c)\le\PP(\sE_1^c)+\PP(\sE_2^c)\le4/N$.
By Lemma~\ref{lem:Hhat-H-J123_split}, on $\sE$ and for all sufficiently large $N$, with probability at least $1-\frac{26}{N}-\frac{2}{n^2}-\frac{4}{N^2}$,
\[
\inf_{\theta\in\Theta_0}\big|\wh J_\lambda(x_0;\theta)\big|\ge c_J/2,
\qquad
\sup_{\theta\in\Theta_0}\big|\wh J_\lambda(x_0;\theta)^{-1}\big|\le\frac{2}{c_J}.
\]
Define $\sE_J:=\{\inf_{\theta\in\Theta_0}|\wh J_\lambda(x_0;\theta)|\ge c_J/2\}$.
Then
\begin{equation}
\label{eq:prob-step2-twofold}
\PP(\sE\cap\sE_J)\ge 1-\frac{4}{N}-\frac{26}{N}-\frac{2}{n^2}-\frac{4}{N^2}=1-\frac{30}{N}-\frac{2}{n^2}-\frac{4}{N^2}.
\end{equation}
On $\sE\cap\sE_J$, using \eqref{eq:theta-stoch-linear-twofold},
\begin{equation}
\label{eq:theta-stoch-bound-twofold}
|\wh\theta(x_0)-\theta_\lambda(x_0)|
\le
\frac{2}{c_J}\big|\wh\eta_\lambda(x_0;\theta_\lambda(x_0))\big|.
\end{equation}

\noindent\textit{Step 3: control $\wh\eta_\lambda(x_0;\theta_\lambda(x_0))$ by $I_1+I_2^{\mathrm{split}}+I_3^{\mathrm{split}}$.}
Since $\eta_\lambda(x_0;\theta_\lambda(x_0))=0$, the split-version exact decomposition gives
\[
\wh\eta_\lambda(x_0;\theta_\lambda(x_0))
-\eta_\lambda(x_0;\theta_\lambda(x_0))
=
I_1(x_0;\theta_\lambda(x_0))+I_2^{\mathrm{split}}(x_0;\theta_\lambda(x_0))+I_3^{\mathrm{split}}(x_0;\theta_\lambda(x_0)).
\]
Therefore,
\begin{equation}
\label{eq:Mhat-at-thetalambda-twofold}
\big|\wh\eta_\lambda(x_0;\theta_\lambda(x_0))\big|
\le
|I_1(x_0;\theta_\lambda(x_0))|
+|I_2^{\mathrm{split}}(x_0;\theta_\lambda(x_0))|
+|I_3^{\mathrm{split}}(x_0;\theta_\lambda(x_0))|.
\end{equation}

\noindent\textit{Step 3a: high-probability events and probability bookkeeping.}
Let $\sE_{I_1}$ be the event on which Lemma~\ref{lem:I1_clean} yields its bound at $\theta=\theta_\lambda(x_0)$,
let $\sE_{I_2}$ be the event on which Lemma~\ref{lem:I23_split_bounds_v2}(ii) holds at $\theta=\theta_\lambda(x_0)$,
and let $\sE_{I_3}$ be the event on which Lemma~\ref{lem:I23_split_bounds_v2}(i) holds at $\theta=\theta_\lambda(x_0)$.
Then for all sufficiently large $n,N$,
\[
\PP(\sE_{I_1}^c)\le\frac{8}{n+N},
\qquad
\PP(\sE_{I_2}^c)\le\frac{6}{N}+\frac{2}{n^2}+\frac{4}{N^2},
\qquad
\PP(\sE_{I_3}^c)\le\frac{12}{N}.
\]
Define the global event $\sG:=\sE\cap\sE_J\cap\sE_{I_1}\cap\sE_{I_2}\cap\sE_{I_3}$.
Using \eqref{eq:prob-step2-twofold} and the union bound together with $\frac{8}{n+N}\le\frac{8}{N}$, we get
\[
\PP(\sG)
\ge
1-\left(\frac{30}{N}+\frac{2}{n^2}+\frac{4}{N^2}\right)-\frac{8}{N}-\left(\frac{6}{N}+\frac{2}{n^2}+\frac{4}{N^2}\right)-\frac{12}{N}
=
1-\frac{56}{N}-\frac{4}{n^2}-\frac{8}{N^2}.
\]

\noindent\textit{Step 3b: plug in $I_1$ and $I_3^{\mathrm{split}}$.}
On $\sG$, Lemma~\ref{lem:I1_clean} gives
\begin{equation}
\label{eq:I1-plug-twofold}
|I_1(x_0;\theta_\lambda(x_0))|
\le
2\sqrt{2\log(n+N)}\sqrt{Q(x_0;\lambda)}
\left(\frac{B_{Y-f}(\theta_\lambda(x_0))}{\sqrt n}+\frac{B_f(\theta_\lambda(x_0))}{\sqrt N}\right).
\end{equation}
Also, Lemma~\ref{lem:I23_split_bounds_v2}(i) gives
\begin{equation}
\label{eq:I3-plug-twofold}
|I_3^{\mathrm{split}}(x_0;\theta_\lambda(x_0))|
\le
8\sqrt{2}\kappa\|\eta(\cdot;\theta_\lambda(x_0))\|_{\cH}\sqrt{\frac{D(x_0;\lambda)\log N}{N}}.
\end{equation}

\noindent\textit{Step 3c: absorb $I_2^{\mathrm{split}}$.}
Under Proposition~\ref{prop:Dx0-order-H} and the regime in \eqref{eq:D-regime-J-split}, the bound in Lemma~\ref{lem:I23_split_bounds_v2}(ii) implies that on $\sG$ and for all sufficiently large $n,N$,
\[
|I_2^{\mathrm{split}}(x_0;\theta_\lambda(x_0))|
\le
|I_1(x_0;\theta_\lambda(x_0))|+|I_3^{\mathrm{split}}(x_0;\theta_\lambda(x_0))|.
\]
Therefore, on $\sG$ and for all sufficiently large $n,N$,
\[
|I_1(x_0;\theta_\lambda(x_0))|+|I_2^{\mathrm{split}}(x_0;\theta_\lambda(x_0))|+|I_3^{\mathrm{split}}(x_0;\theta_\lambda(x_0))|
\le
2|I_1(x_0;\theta_\lambda(x_0))|+2|I_3^{\mathrm{split}}(x_0;\theta_\lambda(x_0))|.
\]
Combining with \eqref{eq:theta-stoch-bound-twofold} and \eqref{eq:Mhat-at-thetalambda-twofold} yields, on $\sG$,
\[
|\wh\theta(x_0)-\theta_\lambda(x_0)|
\le
\frac{4}{c_J}\left(|I_1(x_0;\theta_\lambda(x_0))|+|I_3^{\mathrm{split}}(x_0;\theta_\lambda(x_0))|\right).
\]
Plugging in \eqref{eq:I1-plug-twofold} and \eqref{eq:I3-plug-twofold} gives
\begin{equation}
\label{eq:thetahat-thetalambda-twofold}
\begin{aligned}
|\wh\theta(x_0)-\theta_\lambda(x_0)|
\le
&
\frac{8\sqrt{2}}{c_J}\sqrt{Q(x_0;\lambda)\log(n+N)}
\left(\frac{B_{Y-f}(\theta_\lambda(x_0))}{\sqrt n}+\frac{B_f(\theta_\lambda(x_0))}{\sqrt N}\right)
\\
&+
\frac{32\sqrt{2}}{c_J}\kappa\|\eta(\cdot;\theta_\lambda(x_0))\|_{\cH}\sqrt{\frac{D(x_0;\lambda)\log N}{N}}.
\end{aligned}
\end{equation}

\noindent\textit{Step 4: bias bound for $\theta_\lambda(x_0)-\theta_0(x_0)$.}
Since $\eta_\lambda(x_0;\theta_\lambda(x_0))=0$ and $\theta_0(x_0)\in\Theta_0$, a mean-value expansion yields the existence of $\bar\theta$ between $\theta_0(x_0)$ and $\theta_\lambda(x_0)$ such that
\[
0
=
\eta_\lambda(x_0;\theta_\lambda(x_0))
=
\eta_\lambda(x_0;\theta_0(x_0))+J_\lambda(x_0;\bar\theta)\big(\theta_\lambda(x_0)-\theta_0(x_0)\big).
\]
By Assumption~\ref{ass:J_lambda_bound}, $|J_\lambda(x_0;\bar\theta)|\ge c_J$, hence
\[
|\theta_\lambda(x_0)-\theta_0(x_0)|
\le
\frac1{c_J}|\eta_\lambda(x_0;\theta_0(x_0))|.
\]
By the regularization bias bound \eqref{eq:bias-bound},
\begin{equation}
\label{eq:theta-bias-final-twofold}
|\theta_\lambda(x_0)-\theta_0(x_0)|
\le
\frac{1}{c_J}\|\eta(\cdot;\theta_0(x_0))\|_{\cH}\sqrt{\lambda\{D(x_0;\lambda)-Q(x_0;\lambda)\}}.
\end{equation}

On $\sG$, combine \eqref{eq:thetahat-thetalambda-twofold} and \eqref{eq:theta-bias-final-twofold} to obtain \eqref{eq:thetahat-theta0-bound-clean-twofold}. The probability statement is exactly $\PP(\sG)\ge 1-\frac{56}{N}-\frac{4}{n^2}-\frac{8}{N^2}$ from Step 3(a). This completes the proof.
\end{proof}

\noindent\textbf{Proof of Theorem~\ref{thm:thetahat-asymp-twofold}.} We proceed to prove the asymptotic result in Theorem~\ref{thm:thetahat-asymp-twofold}.

\begin{proof}[Proof of Theorem~\ref{thm:thetahat-asymp-twofold}]
First, we prove
\begin{equation}\label{eq:asymp-normal-twofold-center}
\frac{J_\lambda(x_0)\big(\wh\theta(x_0)-\theta_\lambda(x_0)\big)}{\sqrt{V(x_0)}}
\to
N(0,1).
\end{equation}

\noindent\textit{Step 1: local linearization at $\theta_\lambda(x_0)$.}
As in the two-fold upper bound proof (Theorem~\ref{thm:thetahat-theta0-bound-twofold}), a mean-value expansion gives
\begin{equation}\label{eq:lin-asymp-twofold}
\wh\theta(x_0)-\theta_\lambda(x_0)
=
-\wh J_\lambda(x_0;\wt\theta(x_0))^{-1}
\wh \eta_\lambda(x_0;\theta_\lambda(x_0)),
\end{equation}
for some $\wt\theta(x_0)$ between $\wh\theta(x_0)$ and $\theta_\lambda(x_0)$, where
$\wh\eta_\lambda$ is the two-fold estimating equation in Theorem~\ref{thm:thetahat-theta0-bound-twofold}.
Since $\eta_\lambda(x_0;\theta_\lambda(x_0))=0$, the two-fold decomposition yields
\begin{equation}\label{eq:M-decomp-asymp-twofold}
\wh \eta_\lambda(x_0;\theta_\lambda(x_0))
=
I_1(x_0;\theta_\lambda(x_0))
+
I_2^{\mathrm{split}}(x_0;\theta_\lambda(x_0))
+
I_3^{\mathrm{split}}(x_0;\theta_\lambda(x_0)).
\end{equation}

\noindent\textit{Step 2: Jacobian stability.}
Under the rate conditions \eqref{eq:cond-clean-twofold}, the split Jacobian estimator is stable on $\Theta_0$ and invertible with probability tending to one. In particular,
\begin{equation}\label{eq:Hhat-inv-to-Hinv-twofold}
\wh J_\lambda(x_0;\wt\theta(x_0))^{-1}\to_p J_\lambda(x_0)^{-1},
\qquad
\sup_{\theta\in\Theta_0}\big|\wh J_\lambda(x_0;\theta)^{-1}\big|=O_p(1).
\end{equation}
Combining \eqref{eq:lin-asymp-twofold} and \eqref{eq:M-decomp-asymp-twofold},
\begin{equation}\label{eq:lin-asymp-2-twofold}
\wh\theta(x_0)-\theta_\lambda(x_0)
=
-\wh J_\lambda(x_0;\wt\theta(x_0))^{-1}I_1(x_0;\theta_\lambda(x_0))
+R_{n,N},
\end{equation}
where
\[
R_{n,N}
:=
-\wh J_\lambda(x_0;\wt\theta(x_0))^{-1}
\Big(I_2^{\mathrm{split}}(x_0;\theta_\lambda(x_0))+I_3^{\mathrm{split}}(x_0;\theta_\lambda(x_0))\Big).
\]
Thus, to prove \eqref{eq:asymp-normal-twofold-center}, it suffices to show
\begin{equation}\label{eq:goal-CLT-I1-twofold}
\frac{I_1(x_0;\theta_\lambda(x_0))}{\sqrt{V(x_0)}}\to N(0,1),
\end{equation}
and
\begin{equation}\label{eq:goal-negligible-twofold}
\frac{I_2^{\mathrm{split}}(x_0;\theta_\lambda(x_0))}{\sqrt{V(x_0)}}=o_p(1),
\qquad
\frac{I_3^{\mathrm{split}}(x_0;\theta_\lambda(x_0))}{\sqrt{V(x_0)}}=o_p(1).
\end{equation}

\noindent\textit{Step 3: Lindeberg--Feller CLT for $I_1/\sqrt{V(x_0)}$.}
Since $I_1$ remains unaffected by the two-fold splitting procedure, \eqref{eq:goal-CLT-I1-twofold} has already been established in the proof of Theorem~\ref{thm:thetahat-asymp-new} in Appendix~\ref{app:proofs-of-PPCI-nosplit}.

\noindent\textit{Step 4: $I_2^{\mathrm{split}}$ and $I_3^{\mathrm{split}}$ are negligible.}
By Lemma~\ref{lem:V-lb}, $\sqrt{V(x_0)}\gtrsim \sqrt{D(x_0;\lambda)/n}$.
By Lemma~\ref{lem:I23_split_bounds_v2}(i),
\[
|I_3^{\mathrm{split}}(x_0;\theta_\lambda(x_0))|
=
O_p\left(\kappa\|\eta(\cdot;\theta_\lambda(x_0))\|_{\cH}\sqrt{\frac{D(x_0;\lambda)\log N}{N}}\right),
\]
hence
\[
\frac{|I_3^{\mathrm{split}}(x_0;\theta_\lambda(x_0))|}{\sqrt{V(x_0)}}
=
O_p\left(\kappa\|\eta(\cdot;\theta_\lambda(x_0))\|_{\cH}\sqrt{\frac{n\log N}{N}}\right)
\to 0
\]
by \eqref{eq:cond-clean-twofold}.
By Lemma~\ref{lem:I23_split_bounds_v2}(ii),
\[
|I_2^{\mathrm{split}}(x_0;\theta_\lambda(x_0))|
=
O_p\left(
B_\phi\sqrt{D(\lambda)D(x_0;\lambda)}
\left(\frac{\log n}{n}+\frac{\log N}{N}\right)
\right),
\]
and therefore
\[
\frac{|I_2^{\mathrm{split}}(x_0;\theta_\lambda(x_0))|}{\sqrt{V(x_0)}}
=
O_p\left(
B_\phi\sqrt{D(\lambda)}
\left(\frac{\log n}{\sqrt n}+\frac{\sqrt n\log N}{N}\right)
\right)
\to 0
\]
by \eqref{eq:cond-clean-twofold}. This proves \eqref{eq:goal-negligible-twofold}.

\noindent\textit{Step 5: conclude.}
Multiply \eqref{eq:lin-asymp-2-twofold} by $J_\lambda(x_0)/\sqrt{V(x_0)}$:
\begin{align*}
\frac{J_\lambda(x_0)\big(\wh\theta(x_0)-\theta_\lambda(x_0)\big)}{\sqrt{V(x_0)}}
=&
-\frac{J_\lambda(x_0)}{\wh J_\lambda(x_0;\wt\theta(x_0))}
\frac{I_1(x_0;\theta_\lambda(x_0))}{\sqrt{V(x_0)}}
\\
&-
\frac{J_\lambda(x_0)}{\wh J_\lambda(x_0;\wt\theta(x_0))}
\frac{I_2^{\mathrm{split}}(x_0;\theta_\lambda(x_0))+I_3^{\mathrm{split}}(x_0;\theta_\lambda(x_0))}{\sqrt{V(x_0)}}.
\end{align*}
By \eqref{eq:Hhat-inv-to-Hinv-twofold},
$\frac{J_\lambda(x_0)}{\wh J_\lambda(x_0;\wt\theta(x_0))}\to_p 1$.
By \eqref{eq:goal-CLT-I1-twofold},
$\frac{I_1(x_0;\theta_\lambda(x_0))}{\sqrt{V(x_0)}}\to N(0,1)$.
By \eqref{eq:goal-negligible-twofold},
$\frac{I_2^{\mathrm{split}}(x_0;\theta_\lambda(x_0))+I_3^{\mathrm{split}}(x_0;\theta_\lambda(x_0))}{\sqrt{V(x_0)}}=o_p(1)$.
Slutsky's theorem yields \eqref{eq:asymp-normal-twofold-center}.

\noindent\textit{Step 6: bias is negligible.} Following the same arguments as in Step 6 of the proof of Theorem~\ref{thm:thetahat-asymp-new} in Appendix~\ref{app:proofs-of-PPCI-nosplit}, we obtain
\[
\frac{J_\lambda(x_0)\big(\wh\theta(x_0)-\theta_0(x_0)\big)}{\sqrt{V(x_0)}}
=
\frac{J_\lambda(x_0)\big(\wh\theta(x_0)-\theta_\lambda(x_0)\big)}{\sqrt{V(x_0)}}+o(1),
\]
and combining with \eqref{eq:asymp-normal-twofold-center} completes the proof.
\end{proof}

\noindent\textbf{Proof of Corollary~\ref{cor:CI-coverage-twofold}.} We also provide the following proof.

\begin{proof}[Proof of Corollary~\ref{cor:CI-coverage-twofold}]
Lemma~\ref{lem:variance-ratio}, applied conditionally fold by fold, gives
$\wh V(x_0)/V(x_0)\to_p1$ and $\wh J_\lambda(x_0)/J_\lambda(x_0)\to_p1$.
Combining these ratios with Theorem~\ref{thm:thetahat-asymp-twofold} and Slutsky's theorem yields
\[
\frac{\wh J_\lambda(x_0)\{\wh\theta(x_0)-\theta_0(x_0)\}}
{\sqrt{\wh V(x_0)}}\rightsquigarrow N(0,1),
\]
which is equivalent to the asserted coverage.
\end{proof}

\noindent\textbf{Proof of Theorem~\ref{thm:minimax-pointwise}.} In the following, we prove the minimax result (Theorem~\ref{thm:minimax-pointwise}) in Section~\ref{sec:theory}.
We first recall the following useful lemma from~\cite{Tsybakov_2009}, and then give the proof of Theorem~\ref{thm:minimax-pointwise}.
\begin{lemma}[Section~2.2 in~\cite{Tsybakov_2009}]\label{lem:lecam-minimax}
Let $P_0,P_1$ be two distributions and let $\theta(P)$ be a real-valued parameter.
Write $\Delta:=|\theta(P_1)-\theta(P_0)|$.
Then for any estimator $\wt\theta$,
\[
\max\Big\{\EE_{P_0}\big[(\wt\theta-\theta(P_0))^2\big],\ \EE_{P_1}\big[(\wt\theta-\theta(P_1))^2\big]\Big\}
\ \ge\
\frac{\Delta^2}{4}\Big(1-\TV(P_0,P_1)\Big).
\]
Moreover, by Pinsker's inequality, $\TV(P_0,P_1)\le \sqrt{\KL(P_0\|P_1)/2}$.
\end{lemma}
\begin{proof}[Proof of Theorem~\ref{thm:minimax-pointwise}]
It suffices to restrict $\cP$ to a Gaussian conditional-mean submodel with estimating function $\ell(y;\theta)=y-\theta$. Fix one admissible covariate distribution and kernel satisfying Assumption~\ref{ass:eigenfunctions_bound}, set $Y=m(X)+\varepsilon$ with $\varepsilon\sim N(0,\underline\sigma^2)$ independent of $X$, and take the same auxiliary predictor $f\equiv0$ under both alternatives. Then Assumption~\ref{ass:homo_res} holds with equality, while the unlabeled marginal and the supplied predictor are identical under the two alternatives. Let $\lambda_n=c_\lambda/n$ for a fixed $c_\lambda>0$, and define $h_n=w_{x_0,\lambda_n}/\sqrt{D(x_0;\lambda_n)}$.
The energy identity
$D(x_0;\lambda)=Q(x_0;\lambda)+\lambda\|w_{x_0,\lambda}\|_{\cH}^2$
implies $\EE\{h_n(X)^2\}=Q(x_0;\lambda_n)/D(x_0;\lambda_n)\le1$ and
$\|h_n\|_{\cH}\le\lambda_n^{-1/2}$. Moreover,
$h_n(x_0)=\sqrt{D(x_0;\lambda_n)}\asymp n^{d/(4m)}$ by Proposition~\ref{prop:Dx0-order-H}.

Set $\delta_n=c_0/\sqrt n$ and consider the alternatives
$m_\pm=\pm\delta_n h_n$, with $c_0>0$ chosen below. Their centered conditional moments at the target satisfy
\[
\|m_\pm(\cdot)-m_\pm(x_0)\|_{\cH}
\le \delta_n\{\|h_n\|_{\cH}+|h_n(x_0)|\|1\|_{\cH}\}
\le \frac{c_0}{\sqrt{c_\lambda}}+O\{c_0n^{-1/2+d/(4m)}\}.
\]
Since $m>d/2$ and $1\in H^m(\cX)$, choosing $c_0$ sufficiently small places both alternatives in $\cP$ for all large $n$. The unlabeled observations have exactly the same distribution under the two alternatives. For the labeled observations,
\[
\KL(P_+^{(n,N)}\|P_-^{(n,N)})
=\frac{2n\delta_n^2}{\underline\sigma^2}\EE\{h_n(X)^2\}
\le\frac{2c_0^2}{\underline\sigma^2}.
\]
Choose $c_0$ smaller if necessary so that Pinsker's inequality gives $\TV(P_+^{(n,N)},P_-^{(n,N)})\le1/4$. The pointwise separation is
$2\delta_nh_n(x_0)=2\delta_n\sqrt{D(x_0;\lambda_n)}$. Lemma~\ref{lem:lecam-minimax} and Proposition~\ref{prop:Dx0-order-H} therefore yield
\[
\inf_{\widetilde\theta}\sup_{P\in\cP}
\EE_P\{\widetilde\theta-\theta_0(x_0)\}^2
\ge \frac34\delta_n^2D(x_0;\lambda_n)
\gtrsim n^{-1+d/(2m)},
\]
which proves the claim.
\end{proof}

\noindent\textbf{Proof of Proposition~\ref{prop:budget_opt}.} We prove the optimal sampling result in Section~\ref{sec:practical}.

\begin{proof}[Proof of Proposition~\ref{prop:budget_opt}]
By the Cauchy-Schwarz inequality, we have
\begin{align*}
V(x_0) \cdot C &= \left( \frac{\sigma^2_{Y-f}}{n} + \frac{\sigma^2_{f}}{N} \right) (c_l n + c_u N) \\
&\ge \left( \frac{\sigma_{Y-f}}{\sqrt{n}} \sqrt{c_l n} + \frac{\sigma_{f}}{\sqrt{N}} \sqrt{c_u N} \right)^2 = \left( \sqrt{\sigma^2_{Y-f}\,c_l} + \sqrt{\sigma^2_{f}\,c_u} \right)^2.
\end{align*}
Dividing both sides by $C$ yields the minimum variance lower bound $V_{\min}(C)$. The equality condition for Cauchy-Schwarz holds if and only if
\[
\frac{\sigma_{Y-f} / \sqrt{n}}{\sqrt{c_l n}} = \frac{\sigma_{f} / \sqrt{N}}{\sqrt{c_u N}}
\quad \implies \quad
\frac{n}{N} = \frac{\sigma_{Y-f} / \sqrt{c_l}}{\sigma_{f} / \sqrt{c_u}}.
\]
Substituting this back into the active budget constraint $c_l n + c_u N = C$ yields the optimal sample allocations $n^\star$ and $N^\star$, completing the proof.
\end{proof}

\subsection{Proofs of Technical Lemmas in Appendix~\ref{app:proofs-of-PPCI}}\label{app:proofs-of-lemmas}
\begin{proof}[Proof of Lemma~\ref{lem:I23_split_bounds_v2}]
\noindent\textit{Step 1.}
From the definition of $\wh\eta_\lambda$ and $\wt\eta_\lambda$,
\[
\wh\eta_\lambda(x_0;\theta)-\wt\eta_\lambda(x_0;\theta)
=
\frac1n\sum_{i=1}^n \Delta\overline w(X_i)r(Y_i,X_i;\theta)
+
\frac1N\sum_{m=1}^2\sum_{u\in\cI_m}\Delta w^{(3-m)}(\wt X_u)u(\wt X_u;\theta).
\]
Also $\mu(\theta)=\EE[\ell(Y;\theta)K_X]=\EE[r(Y,X;\theta)K_X]+\EE[u(X;\theta)K_X]$, hence
\[
I_3^{\mathrm{split}}(x_0;\theta)
=
\EE[\Delta\overline w(X)r(Y,X;\theta)]+\EE[\Delta\overline w(X)u(X;\theta)].
\]
Since $|\cI_1|=|\cI_2|=N/2$ and $\Delta\overline w=\frac12(\Delta w^{(1)}+\Delta w^{(2)})$,
\[
\EE[\Delta\overline w(X)u(X;\theta)]
=
\frac12\sum_{m=1}^2 \EE[\Delta w^{(m)}(X)u(X;\theta)].
\]
On the other hand, for each $m$, conditional on the training fold $\{\wt X_u:u\in\cI_{3-m}\}$,
the function $\Delta w^{(3-m)}$ is measurable and $\{\wt X_u:u\in\cI_m\}$ are i.i.d. and independent of that training fold, so
\[
\EE\left[\frac{2}{N}\sum_{u\in\cI_m}\Delta w^{(3-m)}(\wt X_u)u(\wt X_u;\theta)\Big|\{\wt X_u:u\in\cI_{3-m}\}\right]
=
\EE[\Delta w^{(3-m)}(X)u(X;\theta)].
\]
Averaging over $m=1,2$ shows that subtracting $I_3^{\mathrm{split}}$ from $\wh\eta_\lambda-\wt\eta_\lambda$
indeed removes the (conditional) mean, leaving a sum of centered empirical fluctuations. This is exactly $I_2^{\mathrm{split}}$ by definition.

\noindent\textit{Step 2: $\|\Delta w^{(m)}\|_\infty$ bound on $\sE_m$.}
Fix $m\in\{1,2\}$ and work on $\sE_m$.
From the whitening identity,
\[
(\wh T_K^{(m)}+\lambda I)^{-1}
=
(T_K+\lambda I)^{-1/2}\big(I+A^{(m)}(\lambda)\big)^{-1}(T_K+\lambda I)^{-1/2},
\]
so on $\sE_m$, $\|(I+A^{(m)}(\lambda))^{-1}\|_{\op}\le 2$ implies the operator inequality
$(\wh T_K^{(m)}+\lambda I)^{-1}\preceq 2(T_K+\lambda I)^{-1}$.
Therefore, for any $x\in\cX$,
\begin{align*}
|\wh w^{(m)}_{x_0,\lambda}(x)|
&=
|\langle K_x,(\wh T_K^{(m)}+\lambda I)^{-1}K_{x_0}\rangle_{\cH}|
\\&\le
\sqrt{\langle K_x,(\wh T_K^{(m)}+\lambda I)^{-1}K_x\rangle_{\cH}}
\sqrt{\langle K_{x_0},(\wh T_K^{(m)}+\lambda I)^{-1}K_{x_0}\rangle_{\cH}}
\le
2\sqrt{D(x;\lambda)D(x_0;\lambda)}.
\end{align*}
Also $|w_{x_0,\lambda}(x)|\le \sqrt{D(x;\lambda)D(x_0;\lambda)}$.
Hence on $\sE_m$,
\[
\|\Delta w^{(m)}\|_\infty
\le
3\sup_{x\in\cX}\sqrt{D(x;\lambda)D(x_0;\lambda)}
\le
3B_\phi\sqrt{D(\lambda)D(x_0;\lambda)},
\]
where we used Assumption~\ref{ass:eigenfunctions_bound} to get $\sup_x D(x;\lambda)\le B_\phi^2D(\lambda)$.
On $\sE=\sE_1\cap\sE_2$ this implies
\[
\|\Delta\overline w\|_\infty
\le
\frac12(\|\Delta w^{(1)}\|_\infty+\|\Delta w^{(2)}\|_\infty)
\le
3B_\phi\sqrt{D(\lambda)D(x_0;\lambda)}.
\]

\noindent\textit{Step 3: $L^2$ bounds for $\Delta w^{(m)}$ and $\Delta\overline w$.}
Apply Lemma~\ref{lem:delta_f_L2_bound} fold-wise with $N'=|\cI_m|=N/2$ and $\wh T_K=\wh T_K^{(m)}$.
Then on $\sE_m$, with probability at least $1-(2|\cI_m|)^{-1}=1-\frac1N$,
\[
\|\Delta w^{(m)}\|_{L^2(\rho_X)}
\le
c_0B_\phi\sqrt{\frac{D(\lambda)D(x_0;\lambda)\log|\cI_m|}{|\cI_m|}}
\le
c_0B_\phi\sqrt{\frac{2D(\lambda)D(x_0;\lambda)\log N}{N}}.
\]
By the triangle inequality,
\[
\|\Delta\overline w\|_{L^2(\rho_X)}
\le
\frac12\big(\|\Delta w^{(1)}\|_{L^2}+\|\Delta w^{(2)}\|_{L^2}\big)
\le
c_0B_\phi\sqrt{\frac{2D(\lambda)D(x_0;\lambda)\log N}{N}}
\]
on the intersection of the two fold-wise $L^2$ events. A union bound makes their total failure probability $\le 2/N$.

\noindent\textit{Step 4: bound $I_3^{\mathrm{split}}$.}
Since $I_3^{\mathrm{split}}=\frac12\sum_{m=1}^2\langle \Delta w^{(m)},\mu(\theta)\rangle_{\cH}$, $
|I_3^{\mathrm{split}}|
\le
\frac12\sum_{m=1}^2 |\langle \Delta w^{(m)},\mu(\theta)\rangle_{\cH}|$.
For each fixed $m$, Lemma~\ref{lem:I3-clean} applies verbatim to the fold operator $\wh T_K^{(m)}$
with sample size $N'=|\cI_m|=N/2$ and stability event $\sE_m$.
Thus (using $\log(N/2)\le \log N$) we obtain, for all sufficiently large $N$,
\[
|\langle \Delta w^{(m)},\mu(\theta)\rangle_{\cH}|
\le
8\kappa\|\eta(\cdot;\theta)\|_{\cH}\sqrt{\frac{D(x_0;\lambda)\log(N/2)}{N/2}}
\le
8\sqrt{2}\kappa\|\eta(\cdot;\theta)\|_{\cH}\sqrt{\frac{D(x_0;\lambda)\log N}{N}}
\]
with conditional failure probability at most $2(N')^{-1}=4/N$ on $\sE_m$.
Unconditionally, for each $m$, $\PP(\text{the above bound fails})\le \PP(\sE_m^c)+\frac{4}{N}$. By Lemma~\ref{lem:A_bound_1/2_sobolev_split}, $\PP(\sE_m^c)\le \frac{2}{N}$.
Therefore the failure probability for each fold is at most $\frac{6}{N}$, and by a union bound over $m=1,2$,
the bound for both folds holds with probability at least $1-\frac{12}{N}$.
Averaging over the two folds gives (i).

\noindent\textit{Step 5: labeled fluctuation bound for $I_{2,R}^{\mathrm{split}}$.}
Define the labeled centered fluctuation
\[
I_{2,R}^{\mathrm{split}}
:=
\frac1n\sum_{i=1}^n \Delta\overline w(X_i)r(Y_i,X_i;\theta)-\EE[\Delta\overline w(X)r(Y,X;\theta)].
\]
Condition on all unlabeled covariates (so $\Delta\overline w$ is fixed).
Then the summands are i.i.d. mean-zero, bounded by
$|\Delta\overline w(X_i)r(Y_i,X_i;\theta)|\le B_{Y-f}(\theta)\|\Delta\overline w\|_\infty$,
and their conditional variance is at most $B_{Y-f}^2(\theta)\|\Delta\overline w\|_{L^2(\rho_X)}^2$.
Bernstein's inequality in the form
\[
\PP\left(
\left|\frac1n\sum_{i=1}^n Z_i\right|
\ge
\sqrt{\frac{2\sigma^2 t}{n}}+\frac{2Mt}{3n}
\Big|\Delta\overline w
\right)\le 2e^{-t}
\]
with $\sigma^2=B_{Y-f}^2\|\Delta\overline w\|_{L^2}^2$ and $M=B_{Y-f}\|\Delta\overline w\|_\infty$
yields: for any $t\ge 1$, with conditional probability at least $1-2e^{-t}$,
\[
|I_{2,R}^{\mathrm{split}}|
\le
B_{Y-f}(\theta)\|\Delta\overline w\|_{L^2}\sqrt{\frac{2t}{n}}
+
\frac{2}{3}B_{Y-f}(\theta)\|\Delta\overline w\|_\infty\frac{t}{n}.
\]
Take $t:=2\log n$ so $2e^{-t}\le 2n^{-2}$.
On $\sE$ and the $L^2$ events from Step 3, and using Step 2 for $\|\Delta\overline w\|_\infty$, we get
\[
|I_{2,R}^{\mathrm{split}}|
\le
B_{Y-f}(\theta)c_0B_\phi\sqrt{\frac{2D(\lambda)D(x_0;\lambda)\log N}{N}}
\sqrt{\frac{4\log n}{n}}
+
\frac{2}{3}B_{Y-f}(\theta)3B_\phi\sqrt{D(\lambda)D(x_0;\lambda)}\frac{2\log n}{n}.
\]
Using $\sqrt{\frac{2\log N}{N}}\sqrt{\frac{4\log n}{n}} = \sqrt{8\frac{\log n\log N}{nN}}$, which we bound via $\sqrt{ab}\le \frac{a+b}{2}$ with $a=\frac{4\log n}{n}$ and $b=\frac{2\log N}{N}$:
\[
\sqrt{\frac{8\log n\log N}{nN}}
\le
\frac12\left(\frac{4\log n}{n}+\frac{2\log N}{N}\right)
=
\frac{2\log n}{n}+\frac{\log N}{N}.
\]
Thus, on the same event and with failure probability at most $2n^{-2}$,
\begin{align*}
|I_{2,R}^{\mathrm{split}}|
&\le
B_{Y-f}(\theta)B_\phi\sqrt{D(\lambda)D(x_0;\lambda)}
\left[
c_0\left(\frac{2\log n}{n}+\frac{\log N}{N}\right)
+
4\frac{\log n}{n}
\right]
\\&=
B_{Y-f}(\theta)B_\phi\sqrt{D(\lambda)D(x_0;\lambda)}
\left[
(2c_0+4)\frac{\log n}{n}
+c_0\frac{\log N}{N}
\right].
\end{align*}

\noindent\textit{Step 6: unlabeled fluctuation bound for $I_{2,U}^{\mathrm{split}}$.}
Define
\[
I_{2,U}^{\mathrm{split}}
:=
\frac12\sum_{m=1}^2
\left\{
\frac{2}{N}\sum_{u\in\cI_m}\Delta w^{(3-m)}(\wt X_u)u(\wt X_u;\theta)
-
\EE[\Delta w^{(3-m)}(X)u(X;\theta)]
\right\}.
\]
Fix $m$ and condition on the training fold $\{\wt X_u:u\in\cI_{3-m}\}$.
Then $\Delta w^{(3-m)}$ is fixed and $\{\wt X_u:u\in\cI_m\}$ are i.i.d.
As above, $|u(\cdot;\theta)|\le B_f(\theta)$.
Bernstein with $|\cI_m|=N/2$ gives: for any $t\ge 1$, with conditional probability at least $1-2e^{-t}$,
\begin{align*}
|\Delta_m|&=\left|
\frac{2}{N}\sum_{u\in\cI_m}\Delta w^{(3-m)}(\wt X_u)u(\wt X_u;\theta)
-
\EE[\Delta w^{(3-m)}(X)u(X;\theta)]
\right|
\\&\le
B_f(\theta)\|\Delta w^{(3-m)}\|_{L^2}\sqrt{\frac{2t}{|\cI_m|}}
+
\frac{2}{3}B_f(\theta)\|\Delta w^{(3-m)}\|_\infty\frac{t}{|\cI_m|}.
\end{align*}
Take $t:=2\log N$ so $2e^{-t}\le 2N^{-2}$.
On $\sE$ and the fold-wise $L^2$ events from Step 3, and using Step 2 for $\|\Delta w^{(3-m)}\|_\infty$, we obtain
\[
\left|\Delta_m\right|
\le
B_f(\theta)c_0B_\phi\sqrt{\frac{2D(\lambda)D(x_0;\lambda)\log N}{N}}
\sqrt{\frac{4\log N}{N/2}}
+
\frac{2}{3}B_f(\theta)3B_\phi\sqrt{D(\lambda)D(x_0;\lambda)}\frac{2\log N}{N/2}.
\]
Since $\sqrt{\frac{4\log N}{N/2}}=\sqrt{\frac{8\log N}{N}}$ and $\frac{2\log N}{N/2}=\frac{4\log N}{N}$, the variance term becomes $\sqrt{\frac{2\log N}{N}}\sqrt{\frac{8\log N}{N}} = \sqrt{16}\frac{\log N}{N} = 4\frac{\log N}{N}$.
This yields
\begin{align*}
\left|\Delta_m\right|
&\le
4c_0B_f(\theta)B_\phi\sqrt{D(\lambda)D(x_0;\lambda)}\frac{\log N}{N}
+
8B_f(\theta)B_\phi\sqrt{D(\lambda)D(x_0;\lambda)}\frac{\log N}{N}
\\&=
(4c_0+8)B_f(\theta)B_\phi\sqrt{D(\lambda)D(x_0;\lambda)}\frac{\log N}{N}.
\end{align*}
A union bound over $m=1,2$ makes the total failure probability for these two Bernstein events at most $4N^{-2}$.
Finally, since $|I_{2,U}^{\mathrm{split}}|\le \frac12\sum_{m=1}^2|\Delta_m|$, the same bound holds for $|I_{2,U}^{\mathrm{split}}|$.

\noindent\textit{Step 7: conclude.}
On the intersection of:
(i) $\sE$ (failure probability $\le 4/N$ since $\PP(\sE_m^c)\le 2/N$ by Lemma~\ref{lem:A_bound_1/2_sobolev_split}),
(ii) the two fold-wise $L^2$ events in Step 3 (failure probability $\le 2/N$),
(iii) the labeled Bernstein event in Step 5 (failure probability $\le 2n^{-2}$),
(iv) the two unlabeled Bernstein events in Step 6 (failure probability $\le 4N^{-2}$),
we have
\begin{align*}
|I_2^{\mathrm{split}}|
&\le
|I_{2,R}^{\mathrm{split}}|+|I_{2,U}^{\mathrm{split}}|
\\&\le
B_\phi\sqrt{D(\lambda)D(x_0;\lambda)}
\left[
(2c_0+4)B_{Y-f}(\theta)\frac{\log n}{n}
+
\Big(c_0B_{Y-f}(\theta)+(4c_0+8)B_f(\theta)\Big)\frac{\log N}{N}
\right].
\end{align*}
The claimed success probability follows by the union bound:
$
1-\left(\frac{4}{N}+\frac{2}{N}+\frac{2}{n^2}+\frac{4}{N^2}\right)
=
1-\frac{6}{N}-\frac{2}{n^2}-\frac{4}{N^2}$.
This completes the proof of (ii), and (i) was proved in Step 4.
\end{proof}
\begin{proof}[Proof of Lemma~\ref{lem:Hhat-H-J123_split}]
\noindent\textit{Step 1: decomposition into $J_1^{\mathrm{split}}+J_2^{\mathrm{split}}+J_3^{\mathrm{split}}$.}
Define
\begin{align*}
J_\lambda(x_0;\theta)
&:=
\EE\big[w_{x_0,\lambda}(X)\partial_\theta r(Y,X;\theta)\big]
+
\EE\big[w_{x_0,\lambda}(X)\partial_\theta u(X;\theta)\big],\\
J_1^{\mathrm{split}}(x_0;\theta)
&:=
\Big\{\frac1n\sum_{i=1}^n w_{x_0,\lambda}(X_i)\partial_\theta r(Y_i,X_i;\theta)-\EE[w_{x_0,\lambda}(X)\partial_\theta r(Y,X;\theta)]\Big\}\\
&\hphantom{:=}+
\Big\{\frac1N\sum_{u=1}^N w_{x_0,\lambda}(\wt X_u)\partial_\theta u(\wt X_u;\theta)-\EE[w_{x_0,\lambda}(X)\partial_\theta u(X;\theta)]\Big\},\\
J_3^{\mathrm{split}}(x_0;\theta)
&:=
\EE[\Delta\overline w(X)\partial_\theta r(Y,X;\theta)]
+
\EE[\Delta\overline w(X)\partial_\theta u(X;\theta)],\\
J_2^{\mathrm{split}}(x_0;\theta)
&:=
\Big\{\frac1n\sum_{i=1}^n \Delta\overline w(X_i)\partial_\theta r(Y_i,X_i;\theta)\Big\}
+
\Big\{\frac1N\sum_{m=1}^2\sum_{u\in\cI_m}\Delta w^{(3-m)}(\wt X_u)\partial_\theta u(\wt X_u;\theta)\Big\}
-J_3^{\mathrm{split}}(x_0;\theta).
\end{align*}
Then for every $\theta$,
\[
\wh J_\lambda(x_0;\theta)-J_\lambda(x_0;\theta)
=
J_1^{\mathrm{split}}(x_0;\theta)+J_2^{\mathrm{split}}(x_0;\theta)+J_3^{\mathrm{split}}(x_0;\theta).
\]

\noindent\textit{Step 2: uniform bound for $J_1^{\mathrm{split}}$.}
By Assumption~\ref{ass:deriv-bdd-jac}, for all $\theta\in\Theta_0$, $|\partial_\theta r(Y,X;\theta)|\le G_{Y-f}$ and $|\partial_\theta u(X;\theta)|\le G_f$.
Also $\EE[w_{x_0,\lambda}(X)^2]\le D(x_0;\lambda)$ and $\|w_{x_0,\lambda}\|_\infty\le B_\phi\sqrt{D(\lambda)D(x_0;\lambda)}$.
Applying the same scalar Bernstein argument as in Lemma~\ref{lem:I1_clean} to the labeled and unlabeled terms and taking a union bound yields:
for all sufficiently large $N$, with probability at least $1-\frac{8}{N}$,
\begin{equation}\label{eq:J1-split-hp}
\sup_{\theta\in\Theta_0}|J_1^{\mathrm{split}}(x_0;\theta)|
\le
2\sqrt{2\log N}\sqrt{D(x_0;\lambda)}
\left(\frac{G_{Y-f}}{\sqrt n}+\frac{G_f}{\sqrt N}\right).
\end{equation}

\noindent\textit{Step 3: uniform bound for $J_2^{\mathrm{split}}$ on $\sE$.}
On $\sE$, the proof of Lemma~\ref{lem:I23_split_bounds_v2}(ii) applies verbatim after replacing
$r(Y,X;\theta)$ and $u(X;\theta)$ by $\partial_\theta r(Y,X;\theta)$ and $\partial_\theta u(X;\theta)$,
and replacing $B_{Y-f}(\theta),B_f(\theta)$ by $G_{Y-f},G_f$.
Let $c_0>0$ be the absolute constant in Lemma~\ref{lem:delta_f_L2_bound}.
Thus on $\sE$, with probability at least $1-\frac{6}{N}-\frac{2}{n^2}-\frac{4}{N^2}$,
\begin{equation}\label{eq:J2-split-hp}
\sup_{\theta\in\Theta_0}|J_2^{\mathrm{split}}(x_0;\theta)|
\le
B_\phi\sqrt{D(\lambda)D(x_0;\lambda)}
\left[
(2c_0+4)G_{Y-f}\frac{\log n}{n}
+
\Big(c_0G_{Y-f}+(4c_0+8)G_f\Big)\frac{\log N}{N}
\right].
\end{equation}

\noindent\textit{Step 4: uniform bound for $J_3^{\mathrm{split}}$ on $\sE$.}
By definition of $\Delta\overline w$,
\[
J_3^{\mathrm{split}}(x_0;\theta)=\langle \Delta\overline w,\mu^1(\theta)\rangle_{\cH}
=\frac12\sum_{m=1}^2\langle \Delta w^{(m)},\mu^1(\theta)\rangle_{\cH},
\]
where $\mu^1(\theta)=\EE[\partial_\theta\ell(Y;\theta)K_X]$ as in \eqref{eq:H-repr-J}.
For each fixed $m$, on $\sE_m$ the $J_3$ argument in Step 4 of Lemma~\ref{lem:Hhat-H-J123} applies to the fold operator
$\wh T_K^{(m)}$ with sample size $|\cI_m|=N/2$, and uses only Assumptions~\ref{ass:deriv-bdd-jac} and \ref{ass:eigenfunctions_bound}.
In particular, for all sufficiently large $N$, on $\sE_m$ and with probability at least $1-\frac{6}{N}$,
\[
\sup_{\theta\in\Theta_0}\big|\langle \Delta w^{(m)},\mu^1(\theta)\rangle_{\cH}\big|
\le
64B_\phi(G_{Y-f}+G_f)\sqrt{\frac{D(\lambda)D(x_0;\lambda)\log N}{N}}.
\]
By a union bound over $m=1,2$, on $\sE$ and with probability at least $1-\frac{12}{N}$,
\begin{equation}\label{eq:J3-split-hp}
\sup_{\theta\in\Theta_0}|J_3^{\mathrm{split}}(x_0;\theta)|
\le
64B_\phi(G_{Y-f}+G_f)\sqrt{\frac{D(\lambda)D(x_0;\lambda)\log N}{N}}.
\end{equation}

\noindent\textit{Step 5: conclude.}
Intersect the events in \eqref{eq:J1-split-hp}, \eqref{eq:J2-split-hp}, and \eqref{eq:J3-split-hp}.
On $\sE$, this intersection has probability at least
\[
1-\frac{8}{N}-\left(\frac{6}{N}+\frac{2}{n^2}+\frac{4}{N^2}\right)-\frac{12}{N}
=
1-\frac{26}{N}-\frac{2}{n^2}-\frac{4}{N^2}.
\]
On this event,
\[
\sup_{\theta\in\Theta_0}\big|\wh J_\lambda(x_0;\theta)-J_\lambda(x_0;\theta)\big|
\le
\sup_{\theta\in\Theta_0}|J_1^{\mathrm{split}}(x_0;\theta)|
+\sup_{\theta\in\Theta_0}|J_2^{\mathrm{split}}(x_0;\theta)|
+\sup_{\theta\in\Theta_0}|J_3^{\mathrm{split}}(x_0;\theta)|.
\]
Under \eqref{eq:D-regime-J-split}, each upper bound in \eqref{eq:J1-split-hp}, \eqref{eq:J2-split-hp}, and \eqref{eq:J3-split-hp}
tends to $0$ as $N\to\infty$ (using also $n\wedge N\to\infty$).
Hence there exists $N_0$ such that for all $N\ge N_0$,
\[
\sup_{\theta\in\Theta_0}\big|\wh J_\lambda(x_0;\theta)-J_\lambda(x_0;\theta)\big|
\le c_J/2
\qquad\text{on }\sE.
\]
Therefore, using Assumption~\ref{ass:J_lambda_bound}, for all $\theta\in\Theta_0$,
\[
|\wh J_\lambda(x_0;\theta)|
\ge
|J_\lambda(x_0;\theta)|-\big|\wh J_\lambda(x_0;\theta)-J_\lambda(x_0;\theta)\big|
\ge c_J-c_J/2=c_J/2,
\]
which proves \eqref{eq:H-invert-final-split}. The inverse bound follows immediately: $\sup_{\theta\in\Theta_0}\big|\wh J_\lambda(x_0;\theta)^{-1}\big|
\le \frac{2}{c_J}$.
\end{proof}
\subsection{Technical Novelties and Comparison with Classical KRR Theory}
\label{subsec:technical-novelties}
\noindent
While our theoretical framework builds upon the mathematical tools of RKHS, the conditional estimation nature of PPCI introduces fundamental challenges that classical Kernel Ridge Regression (KRR) theory cannot resolve. Below, we highlight three major technical departures from standard KRR analyses, which also elucidate the necessity of our proof techniques.

\textbf{1. Shared-Design Dependence and the $I_2$ Interaction.}

In classical non-parametric KRR, the objective is to estimate a global regression function $f_0(x)$ from responses $Y_i=f_0(X_i)+\epsilon_i$. The KRR prediction at a target point $x_0$ admits the linear representation
\[
\wh f(x_0)=\frac1n\sum_{i=1}^n \wh w^R_{x_0,\lambda}(X_i)Y_i,
\qquad
\wh w^R_{x_0,\lambda}:=(\wh T_K+\lambda I)^{-1}K_{x_0}.
\]
This RKHS localization weight is the linear representer of the KRR prediction functional. It should not be confused with the noiseless KRR fitted function obtained by using $K(X_i,x_0)$ as responses, which solves $(\wh T_K+\lambda I)v=\wh T_KK_{x_0}$ instead of $(\wh T_K+\lambda I)w=K_{x_0}$. Conditionally on the covariates, the weights in classical KRR are deterministic and the observation noise satisfies $\EE[\epsilon_i| X_i]=0$, so the variance term is controlled by $n^{-2}\sum_i\wh w^R_{x_0,\lambda}(X_i)^2\sigma^2$.

In sharp contrast, our non-split PPCI procedure couples the empirical RKHS localization weight with same-data evaluations of $\ell$. The interaction term $I_2=\langle \Delta w,\wh\mu-\mu\rangle_{\cH}$ contains products in which the same covariate contributes both to $\Delta w$ and to the moment fluctuation. Conditioning on the pooled design no longer produces mean-zero noise. To handle this shared-design dependence, we expand $\Delta w$ through the resolvent identity and analyze the resulting weak bilinear form. Its diagonal component is small, as it has only $O(N')$ terms, and its off-diagonal component is degenerate after conditioning on independent coordinates. This replaces the earlier pointwise LOO/KRR route with an argument that is directly aligned with the empirical RKHS localization weight.

\textbf{2. Avoiding the RKHS Penalty for the $I_3$ term.}
Another critical departure arises in bounding the operator approximation error, $I_3 = \langle \wh w_{x_0, \lambda} - w_{x_0, \lambda}, \mu(\theta) \rangle_{\cH}$. In classical KRR, the empirical estimator $\wh f$ targets the population regularized function $f_\lambda = (T_K + \lambda I)^{-1} T_K f_0$. By the resolvent identity, the pointwise estimation error involves bounding the term $\langle (\wh T_K + \lambda I)^{-1}(T_K - \wh T_K)f_\lambda, K_{x_0} \rangle_{\cH}$. Classical proofs routinely rely on the Cauchy-Schwarz inequality in $\cH$ to decouple the operator approximation error from the target function, yielding the bound:
\[
\big| \langle (\wh T_K + \lambda I)^{-1}(T_K - \wh T_K)f_\lambda, K_{x_0} \rangle_{\cH} \big|
\le \| (\wh T_K + \lambda I)^{-1}(T_K - \wh T_K) \|_{\op}
\cdot \| f_\lambda \|_{\cH}\cdot \| K_{x_0} \|_{\cH}.
\]
This strategy is sharp in classical settings because the global target $f_0 \in \cH$ is a fixed underlying truth; consequently, its regularized version satisfies $\|f_\lambda\|_{\cH} \le \|f_0\|_{\cH} = O(1)$.

However, in our framework, the target object is the RKHS localization weight itself: $w_{x_0, \lambda} = (T_K + \lambda I)^{-1}K_{x_0}$. By applying the same resolvent identity, the error term $I_3$ expands as $\langle (\wh T_K + \lambda I)^{-1}(T_K - \wh T_K) w_{x_0, \lambda}, \mu(\theta) \rangle_{\cH}$. Directly applying the classical Cauchy-Schwarz argument here would yield:
\[
|I_3|
\le \| (\wh T_K + \lambda I)^{-1}(T_K - \wh T_K) \|_{\op}
\cdot \| w_{x_0, \lambda} \|_{\cH}\cdot \| \mu(\theta) \|_{\cH}.
\]
As $\lambda \to 0$, $w_{x_0, \lambda}$ approximates a Dirac delta at $x_0$. Since the delta function does not reside in the RKHS, the RKHS norm of our target explodes at the rate of $\|w_{x_0, \lambda}\|_{\cH}^2 \approx D(x_0; \lambda)/\lambda \propto \lambda^{-1 - d/(2m)}$. Isolating this exploding norm via Cauchy-Schwarz inevitably incurs an additional $1/\sqrt{\lambda}$ penalty, resulting in a loose and suboptimal convergence rate.

To achieve a tight bound, we completely bypass the Cauchy-Schwarz inequality in $\cH$. Instead, we project the RKHS inner product back into a scalar empirical process (Step 3 in the proof of Lemma~\ref{lem:I3-clean}): $T_1 = \frac{1}{N'} \sum_{i=1}^{N'} \{ w_{x_0, \lambda}(\bar{X}_i)b(\bar{X}_i) - \EE[w_{x_0,\lambda}(X)b(X)] \}$. This allows us to bound the variance proxy using the $L^2(\rho_X)$ norm rather than the $\cH$ norm. By noting that $\EE[w_{x_0, \lambda}(X)^2] \le D(x_0; \lambda) \propto \lambda^{-d/(2m)}$, we successfully eliminate the extra variance explosion, recovering the minimax-optimal rate. We remark that this strategy of reducing abstract RKHS operations to scalar empirical processes is a recurring technique throughout our analysis, which proves instrumental in establishing tight rates for other error components as well.

\textbf{3. Uniform Control of the Out-of-RKHS Jacobian.}
Finally, classical KRR analyses typically conclude once the regression function itself is bounded. However, in PPCI, statistical inference fundamentally relies on solving the empirical equation $\wh \eta_\lambda(x_0; \theta) = 0$. The asymptotic behavior of the estimator is therefore governed by the localized Jacobian, $J_\lambda(x_0; \theta) = \partial_\theta \eta_\lambda(x_0; \theta)$.

A major technical hurdle is that even if the conditional moment $\eta(\cdot; \theta)$ belongs to $\cH$, the derivative of the estimating function, $\partial_\theta \ell$, frequently falls outside the RKHS $\cH$. Standard KRR convergence theories cannot cover this out-of-RKHS object. To guarantee the asymptotic validity of our confidence intervals, we must independently establish uniform stability bounds over $\Theta_0$ for this empirical Jacobian. We show that provided the regularization parameter $\lambda$ decays sufficiently slowly, the out-of-RKHS empirical Jacobian does not degenerate, a property entirely absent from the global KRR literature.
\end{document}